\pdfoutput=1
\documentclass[11pt,oneside,letterpaper,oldfontcommands]{memoir}
\setsecnumdepth{subsection} 

\setlrmarginsandblock{1in}{1in}{*}
\setulmarginsandblock{1.3in}{1in}{*}
\checkandfixthelayout

\usepackage{times}
\usepackage{nicefrac}
\usepackage{slivkins-setup,slivkins-theorems}

\usepackage{diagbox,bigstrut,makecell} 
\newcommand{\hlinestrut}{\bigstrut[b]\\\hline\bigstrut[t]}

\usepackage[round]{natbib}

\usepackage[linesnumbered,lined,boxed,algochapter]{algorithm2e}
\SetKwInOut{Parameter}{parameter}

\usepackage{algorithmic} 
\usepackage{tikz} 

\usepackage{graphicx}
\graphicspath{ {/} }
\usepackage{float}
\usepackage{wrapfig}
\usepackage{caption}
\usepackage{hyperref,memhfixc}

\renewcommand{\refeq}[1]{Eq.~(\ref{#1})}

\usepackage{relsize} 

\newcommand{\term}[1]{\ensuremath{\mathtt{#1}}\xspace}
\newcommand{\ibar}{\bar{\imath}}
\newcommand{\jbar}{\bar{\jmath}}
\newcommand{\abar}{\bar{a}}
\newcommand*\dif{\mathop{}\!\mathrm{d}} 
\newcommand{\eqAND}{\text{ and }}
\newcommand{\eqWHERE}{\text{,~~where }}
\newcommand{\algTAB}{\text{~~~~}}

\newcommand{\ALG}{\term{ALG}} 
\newcommand{\ADV}{\term{ADV}} 
\newcommand{\OPT}{\term{OPT}}
\newcommand{\OPTFD}{\term{OPT_{FD}}} 
\newcommand{\OPTFA}{\term{OPT_{FA}}} 
\newcommand{\cost}{\term{cost}}
\newcommand{\rew}{\term{rew}}
\newcommand{\reg}{\term{reg}}
\newcommand{\REW}{\term{REW}}
\newcommand{\APP}{\term{APP}}
\newcommand{\Err}{\term{err}} 
\newcommand{\rad}{\term{rad}} 

\newcommand{\conf}{\term{conf}}
\newcommand{\UCB}{\term{UCB}}
\newcommand{\LCB}{\term{LCB}}
\newcommand{\ConfRegion}{\term{ConfRegion}}

\newcommand{\Hedge}{\term{Hedge}}
\newcommand{\ExpThree}{\term{Exp3}}
\newcommand{\ExpFour}{\term{Exp4}}
\newcommand{\UcbOne}{\term{UCB1}}
\newcommand{\FPL}{\term{FTPL}}
\newcommand{\FRL}{\term{FTRL}}
\newcommand{\BPL}{\term{BTPL}}
\newcommand{\LinUCB}{\term{LinUCB}}
\newcommand{\UcbBwK}{\term{UcbBwK}}

\newcommand{\KL}{\term{KL}} 
\newcommand{\RC}{\term{RC}} 
\newcommand{\Rule}{\term{Rule}} 
\newcommand{\High}{\term{High}}
\newcommand{\Low}{\term{Low}}

\newcommand{\BReg}{\term{BR}} 
\newcommand{\proj}{\term{proj}} 
\newcommand{\muR}{\tilde{\mu}} 
\newcommand{\BETA}{\mathtt{Beta}} 
\newcommand{\UU}{\mathbb{U}}         

\newcommand{\algC}{c_{\ALG}} 
\newcommand{\DE}{\mathtt{DE}}
\newcommand{\COV}[1][c]{\mathtt{COV}_{#1}}
\newcommand{\Ball}{\mathbf{B}}
\newcommand{\Index}{\mathtt{index}}
\newcommand{\mainLB}{\ensuremath{\mathtt{MainLB}}\xspace}
\newcommand{\CABmetric}[1][d]{\ensuremath{\left([0,1],\,\ell_2^{1/#1}\right)}\xspace}
\newcommand{\MaxMinCOV}{\mathtt{MaxMinCOV}}        

\newcommand{\Experts}{\mathcal{E}} 
\newcommand{\UV}{$u$-$v$\xspace}
\newcommand{\AlgSB}{\term{SemiBanditHedge}}
\newcommand{\AlgSBfpl}{\term{SemiBanditFTPL}}

\newcommand{\DS}{Decision Service\xspace}

\newcommand{\IPS}{\term{IPS}}

\newcommand{\SimpleNews}{\term{SimpleNews}}

\newcommand{\adfs}{action-specific features\xspace}
\newcommand{\Adfs}{Action-specific features\xspace}
\newcommand{\stationarityTimescale}{stationarity timescale\xspace}
\newcommand{\StationarityTimescale}{Stationarity timescale\xspace}

\newcommand{\BwK}{\term{BwK}}
\newcommand{\LP}{\term{LP}}   
\newcommand{\OPTLP}{\OPT_{\LP}}
\newcommand{\Lag}{\mL}   
\newcommand{\LagrangeBwK}{\term{LagrangeBwK}}

\newcommand{\Null}{\term{null}}
\newcommand{\vM}{\mathbf{M}} 
\newcommand{\vMexp}{\mathbf{M}_{\term{exp}}} 
\newcommand{\vo}{\vec{o}}
\newcommand{\LagGap}{G_\term{LAG}} 
\newcommand{\myArms}{a\not\in\{a^*,\Null\}} 
\newcommand{\cmin}{c_{\min}}

\newcommand{\Drows}{\Delta_{\term{rows}}}   
\newcommand{\Dcols}{\Delta_{\term{cols}}}   
\newcommand{\Dpairs}{\Delta_{\term{pairs}}} 
\newcommand{\MaxMinVal}{v^{\sharp}}  
\newcommand{\unitV}[2]{\mathbf{e}^{#1}_{#2}}  
\newcommand{\unitVrow}[1]{\unitV{\mathtt{row}}{#1}}

\newcommand{\rec}{\term{rec}}   
\newcommand{\samples}[1]{\mathcal{S}_{1,#1}} 
\newcommand{\AvgR}[1]{\term{AVG}_{1,#1}}
\newcommand{\propref}[1]{Property \eqref{#1}}
\newcommand{\HiddenExploration}{\term{HiddenExploration}}
\newcommand{\RepeatedHE}{\term{RepeatedHE}}
\newcommand{\ExploreFn}{a_{\term{trg}}} 

\newcommand{\ExploreE}{\mE_{\term{explore}}}   
\newcommand{\ExploitE}{\mE_{\term{exploit}}}   
\newcommand{\greedy}{\term{GREEDY}} 
\newcommand{\iSig}{\term{sig}}  
\newcommand{\sigSpace}{\ensuremath{\Omega_{\term{sig}}}} 

\newcommand{\sigmin}{\sigma_{\min}}    
\newcommand{\minsupp}{\term{minsupp}}  
\newcommand{\PoI}{\term{PoI}}

\newcommand{\barX}{\overline{X}}

\newcommand{\Hint}[2][Hint]{\xhdrem{#1}: #2} 
\newcommand{\TakeAway}[1]{\xhdrem{Take-away}: #1} 
\newcommand{\Note}[1]{\xhdrem{Note}: #1}

\usepackage{framed}
\newenvironment{BoxedProblem}[2][Problem protocol]
    {\begin{oframed}
        {\noindent {\bf #1:} #2}\newline
        \rule{\textwidth}{0.4pt}\vspace{1mm}}
    {\vspace{-1mm}\end{oframed}}

\newenvironment{ChAbstract}
    {\begin{center}\begin{minipage}{.9\textwidth}}
    {\end{minipage}\end{center}}

\newcommand{\prereqs}[1]{\vspace{1mm}\noindent\emph{Prerequisites:} #1}

\newcommand{\sectionBibNotes}[1][]{\section{Literature review and discussion#1}}
\newcommand{\sectionExercises}{\section{Exercises and hints}}

\title{{\chaptitlefont Introduction to Multi-Armed Bandits}}

\author{{\LARGE Aleksandrs Slivkins}\vspace{2mm}
\\ {\large Microsoft Research NYC}}

\date{}

\begin{document}

\frontmatter
\maketitle

\vspace{-10mm}
{\large
\begin{center}
\begin{tabular}{rl}
First draft:& January 2017\\
Published:& November 2019\\
Latest version:& April 2024
\end{tabular}
\end{center}
}
\vspace{2mm}

\begin{abstract}
Multi-armed bandits a simple but very powerful framework for algorithms that make decisions over time under uncertainty. An enormous body of work has accumulated over the years, covered in several books and surveys. This book provides a more introductory, textbook-like treatment of the subject. Each chapter tackles a particular line of work, providing a self-contained, teachable technical introduction and a brief review of the further developments; many of the chapters conclude with exercises.

The book is structured as follows. The first four chapters are on IID rewards, from the basic model to impossibility results to Bayesian priors to Lipschitz rewards. The next three chapters cover adversarial rewards, from the full-feedback version to adversarial bandits to extensions with linear rewards and combinatorially structured actions. Chapter~\ref{ch:CB} is on contextual bandits, a middle ground between IID and adversarial bandits in which the change in reward distributions is completely explained by observable contexts. The last three chapters cover connections to economics, from learning in repeated games to bandits with supply/budget constraints to exploration in the presence of incentives. The appendix provides sufficient background on concentration and KL-divergence.

The chapters on ``bandits with similarity information", ``bandits with knapsacks" and ``bandits and agents" can also be consumed as standalone surveys on the respective topics.




\vspace{5mm}
\noindent \textbf{Published with Foundations and Trends\circledR\xspace in Machine Learning, November 2019.}

\vspace{5mm}
\noindent This online version is a revision of the ``Foundations and Trends" publication.
It contains numerous edits for presentation and accuracy (based in part on readers' feedback), some new exercises, and updated and expanded literature reviews.
\emph{Further comments, suggestions and bug reports are very welcome!}

\vspace{5mm}
\noindent $\copyright$ 2017-2024: Aleksandrs Slivkins. \\
Author's webpage: \url{https://www.microsoft.com/en-us/research/people/slivkins}.\\
Email: {\tt slivkins at microsoft.com}.

\end{abstract}

\thispagestyle{empty}

\chapter*{\vspace{-10mm}Preface\vspace{-5mm}}
Multi-armed bandits is a rich, multi-disciplinary research area which receives attention from computer science, operations research, economics and statistics. It has been studied since \citep{Thompson-1933}, with a big surge of activity in the past 15-20 years. An enormous body of work has accumulated over time, various subsets of which have been covered in several books
\citep{Berry-book,CesaBL-book,Gittins-book11,Bubeck-survey12}.

This book provides a more textbook-like treatment of the subject, based on the following principles. The literature on multi-armed bandits can be partitioned into a dozen or so lines of work. Each chapter tackles one line of work, providing a self-contained introduction and pointers for further reading. We favor fundamental ideas and elementary proofs over the strongest possible results. We emphasize accessibility of the material: while exposure to machine learning and probability/statistics would certainly help, a standard undergraduate course on algorithms, \eg one based on \citep{KT-book05}, should suffice for background. With the above principles in mind, the choice specific topics and results is based on the author's  subjective understanding of what is important and ``teachable", \ie presentable in a relatively simple manner. Many important results has been deemed too technical or advanced to be presented in detail.

The book is based on a graduate course at University of Maryland, College Park, taught by the author in Fall 2016. Each chapter corresponds to a week of the course.  Five chapters were used in a similar course at Columbia University, co-taught by the author in Fall 2017. Some of the material has been updated since then, to improve presentation and reflect the latest developments.

To keep the book manageable, and also more accessible, we chose not to dwell on the deep connections to online convex optimization. A modern treatment of this fascinating subject can be found, \eg in \citet{Shalev-Shwartz-survey12,Hazan-OCO-book}. Likewise, we do not venture into reinforcement learning, a rapidly developing research area and subject of several textbooks such as \citet{SuttonBartoRL-book98,CsabaRL-book-2010,RLTheoryBook-20}. A course based on this book would be complementary to graduate-level courses on online convex optimization and reinforcement learning. Also, we do not discuss Markovian models of multi-armed bandits; this direction is covered in depth in \citet{Gittins-book11}.


The author encourages colleagues to use this book in their courses. A brief email regarding which chapters have been used, along with any feedback, would be appreciated.

\OMIT{ 
The present draft needs some polishing, and, at places, a more detailed discussion of related work. (However, our goal is to provide pointers for further reading rather than a comprehensive discussion.) The author plans to add more material, in addition to the ten chapters already in the manuscript: an introductory chapter on the scope and motivations, and  a chapter on connections to incentives and mechanism design. In the meantime, the author would be grateful for feedback and is open to suggestions.
} 

\OMIT{Several important topics have been omitted from the course, and therefore from the current draft, due to schedule constraints. Time permitting, the author hopes to include some of these topics in the near future: most notably, a chapter on the MDP formulation of bandits and the Gittins'  algorithm.}

\xhdr{A simultaneous book.}
An excellent recent book on bandits, \citet{LS19bandit-book}, has evolved over several years simultaneously and independently with mine. Their book is longer, provides deeper treatment for some topics (esp. for adversarial and linear bandits), and omits some others (\eg Lipschitz bandits, bandits with knapsacks, and connections to economics). Reflecting the authors' differing tastes and presentation styles, the two books are complementary to one another.

\xhdr{Acknowledgements.}
Most chapters originated as lecture notes from my course at UMD; the initial versions of these lectures were scribed by the students. Presentation of some of the fundamental results is influenced by \citep{Bobby-class07}. I am grateful to Alekh Agarwal, Bobby Kleinberg, Akshay Krishnamurthy,  Yishay Mansour, John Langford, Thodoris Lykouris, Rob Schapire, and Mark Sellke for discussions, comments, and advice. Chapters~\ref{ch:games}, \ref{ch:BwK} have benefited tremendously from numerous conversations with Karthik Abinav Sankararaman. Special thanks go to my PhD advisor Jon Kleinberg and my postdoc mentor Eli Upfal; Jon has shaped my taste in research, and Eli has introduced me to multi-armed bandits back in 2006. Finally, I wish to thank my parents and my family for love, inspiration and support.

\setlength\beforechapskip{-\baselineskip}
\newpage
\begin{small}
\tableofcontents*
\end{small}
\setlength\beforechapskip{+\baselineskip}
\addtocontents{toc}{\protect\vspace{-5mm}}

\mainmatter

\chapter*[Introduction: Scope and Motivation]{Introduction: Scope and Motivation}
\label{ch:intro}
\addcontentsline{toc}{chapter}{Introduction: Scope and Motivation}
Multi-armed bandits is a simple but very powerful framework for algorithms that make decisions over time under uncertainty. Let us outline some of the problems that fall under this framework.

We start with three running examples, concrete albeit very stylized:

\begin{description}
\item[News website] When a new user arrives, a website site picks an article header to show, observes whether the user clicks on this header. The site's goal is maximize the total number of clicks.

\item[Dynamic pricing] A store is selling a digital good, \eg an app or a song. When a new customer arrives, the store chooses a price offered to this customer. The customer buys (or not) and leaves forever. The store's goal is to maximize the total profit.

\item[Investment] Each morning, you choose one stock to invest into, and invest \$1. In the end of the day, you observe the change in value for each stock. The goal is to maximize the total wealth.
\end{description}

\noindent Multi-armed bandits unifies these examples (and many others). In the basic version, an algorithm has $K$ possible actions to choose from, a.k.a. \emph{arms}, and $T$ rounds. In each round, the algorithm chooses an arm and collects a reward for this arm. The reward is drawn independently from some distribution which is fixed (\ie depends only on the chosen arm), but not known to the algorithm. Going back to the running examples:

\begin{center}
\begin{tabular}{l|l|l}
Example         & Action                    & Reward                            \\ \hline
News website       & an article  to display    & $1$ if clicked, $0$ otherwise     \\
Dynamic pricing & a price to offer          & $p$ if sale, $0$ otherwise        \\
Investment      & a stock to invest into    & change in value during the day
\end{tabular}
\end{center}

In the basic model, an algorithm observes the reward for the chosen arm after each round, but not for the other arms that could have been chosen. Therefore, the algorithm typically needs to \emph{explore}: try out different arms to acquire new information. Indeed, if an algorithm always chooses arm $1$, how would it know if arm $2$ is better? Thus, we have a tradeoff between exploration and \emph{exploitation}: making optimal near-term decisions based on the available information. This tradeoff, which arises in numerous application scenarios, is essential in multi-armed bandits. Essentially, the algorithm strives to learn which arms are best (perhaps approximately so), while not spending too much time exploring.

The term ``multi-armed bandits" comes from a stylized gambling scenario in which a gambler faces several slot machines, a.k.a. one-armed bandits, that appear identical, but yield different payoffs.

\subsection*{Multi-dimensional problem space}

Multi-armed bandits is a huge problem space, with many ``dimensions" along which the models can be made more expressive and closer to reality. We discuss some of these modeling dimensions below. Each dimension gave rise to a prominent line of work, discussed later in this book.

\xhdr{Auxiliary feedback.} What feedback is available to the algorithm after each round, other than the reward for the chosen arm? Does the algorithm observe rewards for the other arms? Let's check our examples:

\begin{center}
\begin{tabular}{l|l|l}
Example         & Auxiliary feedback                                & Rewards for any other arms?\\ \hline
News website       & N/A                                               & no (\emph{bandit feedback}).\\
Dynamic pricing & sale $\Rightarrow$ sale at any lower price,       & yes, for some arms, \\
                & no sale $\Rightarrow$ no sale at any higher price & ~~ but not for all (\emph{partial feedback}). \\
Investment      & change in value for all other stocks              & yes, for all arms (\emph{full feedback}).
\end{tabular}
\end{center}

We distinguish three types of feedback: \emph{bandit feedback}, when the algorithm observes the reward for the chosen arm, and no other feedback; \emph{full feedback}, when the algorithm observes the rewards for all arms that could have been chosen; and \emph{partial feedback}, when some information is revealed, in addition to the reward of the chosen arm, but it does not always amount to full feedback.

This book mainly focuses on problems with bandit feedback. We also cover some of the fundamental results on full feedback, which are essential for developing subsequent bandit results. Partial feedback sometimes arises in extensions and special cases, and can be used to improve performance.

\xhdr{Rewards model.} Where do the rewards come from? Several alternatives has been studied:

\begin{itemize}
\item \emph{IID rewards:} the reward for each arm is drawn independently from a fixed distribution that depends on the arm but not on the round $t$.

\item \emph{Adversarial rewards:} rewards can be arbitrary, as if they are chosen by an ``adversary" that tries to fool the algorithm. The adversary may be \emph{oblivious} to the algorithm's choices, or \emph{adaptive} thereto.

\item \emph{Constrained adversary:} rewards are chosen by an adversary that is subject to some constraints, \eg
reward of each arm cannot change much from one round to another, or the reward of each arm can change at most a few times, or the total change in rewards is upper-bounded.

\item \emph{Random-process rewards}: an arm's state, which determines rewards, evolves over time as a random process, \eg a random walk or a Markov chain. The state transition in a particular round may also depend on whether the arm is chosen by the algorithm.
\end{itemize}

\xhdr{Contexts.} In each round, an algorithm may observe some \emph{context} before choosing an action. Such context often comprises the known properties of the current user, and allows for personalized actions.

\begin{center}
\begin{tabular}{l|l}
Example         & Context\\ \hline
News website    & user location and demographics \\
Dynamic pricing & customer's device (\eg cell or laptop), location, demographics\\
Investment      & current state of the economy.
\end{tabular}
\end{center}

\noindent Reward now depends on both the context and the chosen arm. Accordingly, the algorithm's goal is to find the best \emph{policy} which maps contexts to arms.


\xhdr{Bayesian priors.} In the \emph{Bayesian} approach, the problem instance comes from a known distribution, called the \emph{Bayesian prior}. One is typically interested in provable guarantees in expectation over this distribution.

\xhdr{Structured rewards.} Rewards may have a known structure, \eg arms correspond to points in $\R^d$, and in each round the reward is a linear (resp., concave or Lipschitz) function of the chosen arm.

\xhdr{Global constraints.} The algorithm can be subject to global constraints that bind across arms and across rounds. For example, in dynamic pricing there may be a limited inventory of items for sale.

\xhdr{Structured actions.} An algorithm may need to make several decisions at once, \eg a news website may need to pick a slate of articles, and a seller may need to choose prices for the entire slate of offerings.

\subsection*{Application domains}

Multi-armed bandit problems arise in a variety of application domains. The original application has been the design of ``ethical" medical trials, so as to attain useful scientific data while minimizing harm to the patients. Prominent modern applications concern the Web: from tuning the look and feel of a website, to choosing which content to highlight, to optimizing web search results, to placing ads on webpages. Recommender systems can use exploration to improve its recommendations for movies, restaurants, hotels, and so forth. Another cluster of applications pertains to economics: a seller can optimize its prices and offerings; likewise, a frequent buyer such as a procurement agency can optimize its bids; an auctioneer can adjust its auction over time; a crowdsourcing platform can improve the assignment of tasks, workers and prices. In computer systems, one can experiment and learn, rather than rely on a rigid design, so as to optimize datacenters and networking protocols. Finally, one can teach a robot to better perform its tasks.

\vspace{-2mm}

\begin{center}
\begin{tabular}{l|l|l}
Application domain      & Action                                    & Reward  \\ \hline
medical trials          & which drug to prescribe                   & health outcome. \\
web design              & \eg font color or page layout             & \#clicks.\\
content optimization    & which items/articles to emphasize         & \#clicks. \\
web search              & search results for a given query          & \#satisfied users.\\
advertisement           & which ad to display                       & revenue from ads. \\
recommender systems     & \eg which movie to watch                  & $1$ if follows recommendation. \\
sales optimization      & which products to offer at which prices   & revenue. \\
procurement             & which items to buy at which prices        & \#items procured\\
auction/market design   & \eg which reserve price to use            & revenue \\
crowdsourcing           & match tasks and workers, assign prices    & \#completed tasks\\
datacenter design       & \eg which server to route the job to      & job completion time. \\
Internet                & \eg which TCP settings to use?            & connection quality. \\
radio networks          & which radio frequency to use?             & \#successful transmissions. \\
robot control           & a ``strategy" for a given task            & job completion time.
\end{tabular}
\end{center}

\subsection*{\vspace{-2mm}
(Brief) bibliographic notes}

Medical trials has a major motivation for introducing multi-armed bandits and exploration-exploitation tradeoff \citep{Thompson-1933,Gittins-index-79}. Bandit-like designs for medical trials belong to the realm of \emph{adaptive} medical trials
\citep{Chow-adaptive-2008}, which can also include other ``adaptive" features such as early stopping, sample size re-estimation, and changing the dosage.

Applications to the Web trace back to
\citep{yahoo-bandits07,yahoo-bandits-icml07,Langford-nips07}
for ad placement,
\citep{Langford-www10,Langford-wsdm11} for news optimization,
and \citep{RBA-icml08} for web search. A survey of the more recent literature is beyond our scope.
Bandit algorithms tailored to recommendation systems are studied, \eg in
\citep{Bresler-CF-nips14,CF-bandits-sigir16,Bresler-CF-sigmetrics16}.

\OMIT{ 
Exploration in recommendation systems entails a substantial economic aspect: essentially, how to incentivize human users to explore,
see \citet{Slivkins-xroads17} for a brief survey.
} 

Applications to problems in economics comprise many aspects:
optimizing seller's prices, a.k.a.
\emph{dynamic pricing} or \emph{learn-and-earn}, \citep[][a survey]{Boer-survey15};
optimizing seller's product offerings, a.k.a. \emph{dynamic assortment} \citep[\eg][]{Zeevi-assortment-13,Shipra-ec16};
optimizing buyers prices, a.k.a. \emph{dynamic procurement} \citep[\eg][]{DynProcurement-ec12,BwK-focs13};
design of auctions \citep[\eg][]{DynAuctions-survey11,RepeatedAuctions-soda13,Transform-ec10-jacm};
design of information structures
\citep[][starting from]{Kremer-JPE14},
and design of crowdsourcing platforms
\citep[][a survey]{Crowdsourcing-PositionPaper13}.


Applications of bandits to Internet routing and congestion control were started in theory, starting with \citep{Bobby-stoc04,Bobby-infocom05}, and in systems
\citep{Schapira-nsdi15,Shapira-nsdi18,Junchen-sigcomm16,Junchen-nsdi17}. 
Bandit problems directly motivated by radio networks have been studied starting from
\citep{MultiPlayerMAB-Poor08,MultiPlayerMAB-Liu10,MultiPlayerMAB-Anima11}.

\chapter{Stochastic Bandits}
\label{ch:IID}

\begin{ChAbstract}
This chapter covers bandits with IID rewards, the basic model of multi-arm bandits. We present several algorithms, and analyze their  performance in terms of regret. The ideas introduced in this chapter extend far beyond the basic model, and will resurface throughout the book.

\prereqs{Hoeffding inequality (Appendix~\ref{app:concentration}).}
\end{ChAbstract}

\section{Model and examples}
\label{IID:model}


We consider the basic model with IID rewards, called \emph{stochastic bandits}.%
\footnote{Here and elsewhere, \emph{IID} stands for ``independent and identically distributed".}
An algorithm has $K$ possible actions to choose from, a.k.a. \emph{arms}, and there are $T$ rounds, for some known $K$ and $T$. In each round, the algorithm chooses an arm and collects a reward for this arm. The algorithm's goal is to maximize its total reward over the $T$ rounds. We make three essential assumptions:
\begin{itemize}
\item The algorithm observes only the reward for the selected action, and nothing else. In particular, it does not observe rewards for other actions that could have been selected. This is called \emph{bandit feedback}.

\item The reward for each action is IID. For each action $a$, there is a distribution $\mD_a$ over reals, called the \emph{reward distribution}. Every time this action is chosen, the reward is sampled independently from this distribution. The reward distributions are initially unknown to the algorithm.

\item Per-round rewards are bounded; the restriction to the interval $[0, 1]$ is for simplicity.
\end{itemize}

\noindent Thus, an algorithm interacts with the world according to the protocol summarized below.

\begin{BoxedProblem}{Stochastic bandits}
{\bf Parameters:} $K$ arms, $T$ rounds (both known);
reward distribution $\mD_a$ for each arm $a$ (unknown).

\noindent In each round $t \in [T]$:
\begin{OneLiners}
  \item[1.] Algorithm picks some arm $a_t$.
  \item[2.] Reward $r_t\in [0,1]$ is sampled independently from distribution $\mD_a$, $a=a_t$.
  \item[3.] Algorithm collects reward $r_t$, and observes nothing else.
\end{OneLiners}
\end{BoxedProblem}

We are primarily interested in the \emph{mean reward vector} $\mu\in [0,1]^K$, where $\mu(a) = \E[\mD_a]$ is the mean reward of arm $a$. Perhaps the simplest reward distribution is the Bernoulli distribution, when the reward of each arm $a$ can be either 1 or 0 (``success or failure", ``heads or tails"). This reward distribution is fully specified by the mean reward, which in this case is simply the probability of the successful outcome. The problem instance is then fully specified by the time horizon $T$ and the mean reward vector.

Our model is a simple abstraction for an essential feature of reality that is present in many application scenarios. We proceed with three motivating examples:

\begin{enumerate}
    \item \textbf{News}: in a very stylized news application, a user visits a news site, the site presents a news header, and a user either clicks on this header or not. The goal of the website is to maximize the number of clicks. So each possible header is an arm in a bandit problem, and clicks are the rewards. Each user is drawn independently from a fixed distribution over users, so that in each round the click happens independently with a probability that depends only on the chosen header.

    \item \textbf{Ad selection}:
        In website advertising, a user visits a webpage, and a learning algorithm selects one of many possible ads to display. If ad $a$ is displayed, the website observes whether the user clicks on the ad, in which case the advertiser pays some amount $v_a\in [0,1]$. So each ad is an arm, and the paid amount is the reward. The $v_a$ depends only on the displayed ad, but does not change over time. The click probability for a given ad does not change over time, either.

    \item \textbf{Medical Trials:} a patient visits a doctor and the doctor can
        prescribe one of several possible treatments, and observes the treatment effectiveness. Then the next patient arrives, and so forth. For simplicity of this example, the effectiveness of a treatment is quantified as a number in $[0,1]$. Each treatment can be considered as an arm, and the reward is defined as the
        treatment effectiveness. As an idealized assumption, each patient is drawn independently from a fixed distribution over patients, so the effectiveness of a given treatment is IID.

\end{enumerate}

\noindent Note that the reward of a given arm can only take two possible values in the first two examples, but could, in principle, take arbitrary values in the third example.

\begin{remark}\label{IID:rem:conventions}
We use the following conventions in this chapter and throughout much of the book. We will use \emph{arms} and \emph{actions} interchangeably. Arms are denoted with $a$, rounds with $t$. There are $K$ arms and $T$ rounds. The set of all arms is $\mA$. The mean reward of arm $a$ is
    $\mu(a) := \E[\mD_a]$.
The best mean reward is denoted
    $\mu^* := \max_{a\in\mA} \mu(a)$.
The difference
    $\Delta(a) := \mu^*-\mu(a)$
describes how bad arm $a$ is compared to $\mu^*$; we call it the \emph{gap} of arm $a$. An optimal arm is an arm $a$ with $\mu(a)=\mu^*$; note that it is not necessarily unique. We take $a^*$ to denote some optimal arm. $[n]$ denotes the set $\{1,2 \LDOTS n\}$.
\end{remark}

\xhdr{Regret.}
How do we argue whether an algorithm is doing a good job across different problem instances? The problem is, some problem instances inherently allow higher rewards than others. One standard approach is to compare the algorithm's cumulative reward to the \emph{best-arm benchmark} $\mu^* \cdot T$: the expected reward of always playing an optimal arm, which is the best possible total expected reward  for a particular problem instance. Formally, we define the following quantity, called \emph{regret} at round $T$:
\begin{align}\label{IID:eq:pseudo-regret}
\textstyle R(T) = \mu^* \cdot T -  \sum_{t=1}^{T} \mu(a_t).
\end{align}
Indeed, this is how much the algorithm ``regrets" not knowing the best arm in advance.
Note that $a_t$, the arm chosen at round $t$, is a random quantity, as it may depend on randomness in rewards and/or in the algorithm. So, $R(T)$ is also a random variable. We will typically talk about \emph{expected} regret $\E[R(T)]$.

We mainly care about the dependence of $\E[R(T)]$ regret on the time horizon $T$. We also consider the dependence on the number of arms $K$ and the mean rewards $\mu(\cdot)$. We are less interested in the fine-grained dependence on the reward distributions (beyond the mean rewards). We will usually use big-O notation to focus on the asymptotic dependence on the parameters of interests, rather than keep track of the constants.

\begin{remark}[Terminology]
Since our definition of regret sums over all rounds, we sometimes call it \emph{cumulative} regret. When/if we need to highlight the distinction between $R(T)$ and $\E[R(T)]$, we say \emph{realized regret} and \emph{expected regret}; but most of the time we just say ``regret" and the meaning is clear from the context. The quantity $R(T)$ is sometimes called \emph{pseudo-regret} in the literature.
\end{remark}

\section{Simple algorithms: uniform exploration}
\label{IID:sec:uniform}

We start with a simple idea: explore arms uniformly (at the same rate), regardless of what has been observed previously, and pick an empirically best arm for exploitation. A natural incarnation of this idea, known as \emph{Explore-first} algorithm, is to dedicate an initial segment of rounds to exploration, and the remaining rounds to exploitation.

\LinesNumbered
\begin{algorithm}[H]
    \SetAlgoLined
    Exploration phase: try each arm $N$ times\;
    Select the arm $\hat{a}$ with the highest average reward (break ties arbitrarily)\;
    Exploitation phase: play arm $\hat{a}$ in all remaining rounds.
    \caption{Explore-First with parameter $N$.}
    \label{IID:alg1}
\end{algorithm}

The parameter $N$ is fixed in advance; it will be chosen later as function of the time horizon $T$ and the number of arms $K$, so as to minimize regret. Let us analyze regret of this algorithm.

Let the average reward for each action $a$ after exploration phase be denoted $\bar\mu(a)$. We want the average reward to be a good estimate of the true expected rewards, i.e. the following quantity should be small: $|\bar\mu(a) - \mu(a)|$. We bound it using the Hoeffding inequality (Theorem~\ref{app:thm:Hoeffding}):
\begin{align}\label{IID:eq:1}
\Pr \sbr{ |\bar\mu(a) - \mu(a)| \leq \rad } \geq 1 - \nicefrac{2}{T^4}
\eqWHERE
\rad := \sqrt{2 \log(T)\,/\, N}.
\end{align}

\begin{remark}\label{IID:rem-conf}
Thus, $\mu(a)$ lies in the known interval
    $[\bar\mu(a)-\rad,\bar\mu(a)+\rad]$
with high probability. A known interval containing some scalar quantity is called the \emph{confidence interval} for this quantity. Half of this interval's length (in our case, $\rad$) is called the \emph{confidence radius}.%
\footnote{It is called a ``radius" because an interval can be seen as a ``ball" on the real line.}
\end{remark}

We define the \emph{clean event} to be the event that \eqref{IID:eq:1} holds for all
arms simultaneously. We will argue separately the clean event, and the ``bad event" -- the complement of the clean event.

\begin{remark}
With this approach, one does not need to worry about probability in the rest of the analysis. Indeed, the probability has been taken care of by defining the clean event and observing that \eqref{IID:eq:1} holds therein. We do not need to worry about the bad event, either: essentially, because its probability is so tiny. We will use this ``clean event" approach in many other proofs, to help simplify the technical details. This simplicity usually comes at the cost of slightly larger constants in $O()$, compared to more careful arguments which explicitly track low-probability events throughout the proof.
\end{remark}

For simplicity, let us start with the case of $K=2$ arms. Consider the clean event. We will show that if we chose the worse arm, it is not so bad because the expected rewards for the two arms would be close.

Let the best arm be $a^*$, and suppose the algorithm chooses the other arm $a \neq a^*$. This must have been because its average reward was better than that of $a^*$:
    $\bar\mu(a) > \bar\mu(a^*)$.
Since this is a clean event, we have:
\begin{align*}
\mu(a) + \rad \geq \bar\mu(a) > \bar\mu(a^*) \geq \mu(a^*) - \rad
\end{align*}
Re-arranging the terms, it follows that
    $\mu(a^*) - \mu(a) \leq  2\,\rad.$

Thus, each round in the exploitation phase contributes at most
    $2\,\rad$
to regret. Each round in exploration trivially contributes at most $1$. We derive an upper bound on the regret, which consists of two parts: for exploration, when each arm is chosen $N$ times, and then for the remaining $T-2N$ rounds of exploitation:
\begin{align*}
    R(T) \leq N + 2\,\rad\cdot (T-2N)
         < N + 2\,\rad \cdot T.
\end{align*}

Recall that we can select any value for $N$, as long as it is given to the algorithm in advance. So, we can choose $N$ so as to (approximately) minimize the right-hand side. Since the two summands are, resp., monotonically increasing and monotonically decreasing in $N$, we can set $N$ so that they are approximately equal. For
    $N=T^{2/3}\, (\log T)^{1/3}$,
we obtain:
\begin{align*}
    R(T) &\leq O\left( T^{2/3}\; (\log T)^{1/3}\right).
\end{align*}

It remains to analyze the ``bad event". Since regret can be at most $T$ (each round contributes at most $1$), and the bad event happens with a very small probability, regret from this event can be neglected. Formally:
\begin{align}
\E\sbr{R(T)}
    &=  \E\sbr{ R(T) \mid \text{clean event}} \times \Pr\sbr{\text{clean event}} \;+\;
            \E\sbr{R(T) \mid \text{bad event}}\times \Pr\sbr{\text{bad event}} \nonumber\\
        &\leq \E\sbr{R(T) \mid \text{clean event}}
            + T\times O\rbr{T^{-4}} \nonumber \\
        &\leq O\rbr{ (\log T)^{1/3} \times T^{2/3}}. \label{IID:eq:clean-coda}
\end{align}
This completes the proof for $K=2$ arms.

For $K>2$ arms, we apply the union bound for \eqref{IID:eq:1} over the $K$
arms, and then follow the same argument as above. Note that $T\geq K$ without loss of generality, since we need to explore each arm at least once. For the
final regret computation, we will need to take into account the dependence on $K$: specifically, regret accumulated in exploration phase is now upper-bounded by $KN$. Working through the proof, we obtain
    $R(T) \leq NK + 2\,\rad\cdot T$.
As before, we approximately minimize it by approximately minimizing the two summands. Specifically, we plug in
    $N=(T/K)^{2/3} \cdot O(\log T)^{1/3}$.
Completing the proof same way as in \eqref{IID:eq:clean-coda}, we obtain:
\begin{theorem}\label{IID:thm:explore-first}
Explore-first achieves regret
    $\E\sbr{R(T)} \leq T^{2/3} \times O(K \log T)^{1/3}$.
\end{theorem}

\subsection{Improvement: Epsilon-greedy algorithm}

One problem with Explore-first is that its performance in the exploration phase may be very bad if many/most of the arms have a large gap $\Delta(a)$. It is usually better to spread exploration more uniformly  over time. This is done in the \emph{Epsilon-greedy} algorithm:

\LinesNotNumbered
\begin{algorithm}[H]
    \SetAlgoLined
    \For{each round $t=1,2, \ldots   $}{
    Toss a coin with success probability $\eps_t$\;
    \uIf{success}{
        explore: choose an arm uniformly at random
    }
    \uElse{
        exploit: choose the arm with the highest average reward so far
    }
    }
    \caption{Epsilon-Greedy with exploration probabilities
        $(\eps_1,\eps_2,  \ldots)$.}
    \label{IID:alg2}
\end{algorithm}

Choosing the best option in the short term is often called the ``greedy" choice in the computer science literature, hence the name ``Epsilon-greedy". The exploration is uniform over arms, which is similar to the ``round-robin" exploration in the explore-first algorithm. Since exploration is now spread uniformly over time, one can hope to derive meaningful regret bounds even for small $t$. We focus on exploration probability $\eps_t \sim  t^{-1/3}$ (ignoring the dependence on $K$ and $\log t$ for a moment), so that the expected number of exploration rounds up to round $t$ is on the order of $t^{2/3}$, same as in Explore-first with time horizon $T=t$. We derive the same regret bound as in Theorem~\ref{IID:thm:explore-first}, but now it holds for all rounds $t$.

\begin{theorem}\label{IID:thm:eps-greedy}
Epsilon-greedy algorithm with exploration probabilities $\eps_t=t^{-1/3} \cdot (K\log t)^{1/3}$ achieves regret bound
    $\E[R(t)] \leq t^{2/3} \cdot O(K \log t)^{1/3}$
for each round $t$.
\end{theorem}

\noindent The proof relies on a more refined clean event, introduced in the next section; we leave it as Exercise~\ref{IID:ex:eps-greedy}.

\subsection{Non-adaptive exploration}
\label{IID:sec:non-adaptive}

Explore-first and Epsilon-greedy do not adapt their exploration schedule to the history of the observed rewards. We refer to this property as \emph{non-adaptive exploration}, and formalize it as follows:

\begin{definition}\label{IID:def:non-adaptive}
A round $t$ is an \emph{exploration round} if the data $\rbr{a_t,r_t}$ from this round is used by the algorithm in the future rounds. A deterministic  algorithm satisfies \emph{non-adaptive exploration} if the set of all exploration rounds and the choice of arms therein is fixed before round $1$. A randomized algorithm satisfies \emph{non-adaptive exploration} if it does so for each realization of its random seed.
\end{definition}

Next, we obtain much better regret bounds by adapting exploration to the observations. Non-adaptive exploration is indeed the key obstacle here. Making this point formal requires information-theoretic machinery developed in Chapter~\ref{ch:LB}; see Section~\ref{LB:sec:non-adaptive} for the precise statements.


\section{Advanced algorithms: adaptive exploration}

We present two algorithms which achieve much better regret bounds. Both algorithms adapt exploration to the observations so that very under-performing arms are phased out sooner.

Let's start with the case of  $K=2$ arms. One natural idea is as follows:
\begin{align}\label{BIC:eq:idea}
\text{alternate the arms until we are confident which arm is better, and play this arm thereafter}.
\end{align}
However, how exactly do we determine whether and when we are confident? We flesh this out next.

\subsection{Clean event and confidence bounds}
\label{IID:sec:clean}

Fix round $t$ and arm $a$. Let $n_t(a)$ be the number of rounds before $t$ in which this arm is chosen, and let $\bar{\mu}_t(a)$ be the average reward in these rounds. We will use Hoeffding Inequality (Theorem~\ref{app:thm:Hoeffding}) to derive
\begin{align}\label{IID:eq:desired}
\Pr\sbr{ |\bar{\mu}_t(a)-\mu(a)| \le r_t(a)} \ge 1 - \tfrac{2}{T^4}
\eqWHERE
r_t(a) = \sqrt{2\log(T)\,/\, n_t(a)}.
\end{align}


However, \refeq{IID:eq:desired} does not follow immediately. This is because Hoeffding Inequality only applies to a fixed number of independent random variables, whereas here we have $n_t(a)$ random samples from reward distribution $\mD_a$, where $n_t(a)$ is itself is a random variable. Moreover, $n_t(a)$ can depend on the past rewards from arm $a$, so conditional on a particular realization of $n_t(a)$, the samples from $a$ are not necessarily independent! For a simple example, suppose an algorithm chooses arm $a$ in the first two rounds, chooses it again in round $3$ if and only if the reward was $0$ in the first two rounds, and never chooses it again.

So, we need a slightly more careful argument. We present an elementary version of this argument (whereas a more standard version relies on the concept of \emph{martingale}). For each arm $a$, let us imagine there is a \emph{reward tape}: an $1\times T$ table with each cell independently sampled from $\mD_a$, as shown in \reffig{IID:fig:tape}.

\begin{figure}[h]
 \centering
 \includegraphics[width = 1.0\textwidth]{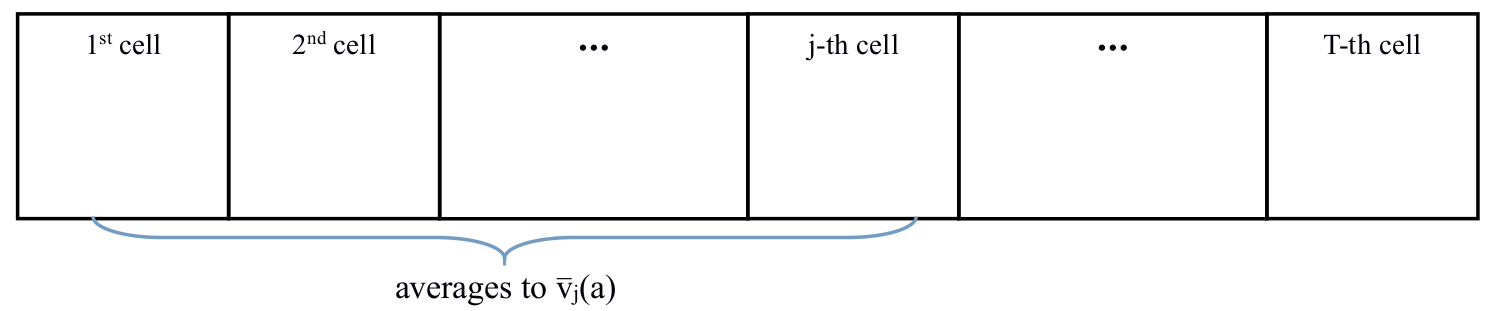}
  \caption{the $j$-th cell contains the reward of the $j$-th time we pull arm $a$, \ie reward of arm $a$ when $n_t(a)=j$}\label{IID:fig:tape}
\end{figure}

\noindent Without loss of generality, the reward tape encodes rewards as follows: the $j$-th time a given arm $a$ is chosen by the algorithm, its reward is taken from the $j$-th cell in this arm's tape. Let $\bar{v}_j(a)$ represent the average reward at arm $a$ from first $j$ times that arm $a$ is chosen. Now one can use Hoeffding Inequality to derive that
\[ \forall j \quad
    \Pr\sbr{ |\bar{v}_j(a) - \mu(a)|\le r_t(a)}\ge 1 - \nicefrac{2}{T^4}.\]
 Taking a union bound, it follows that (assuming $K=\text{\#arms}\leq T$)
\begin{align}\label{IID:eq:clean}
 \Pr\sbr{\mE}\ge 1 - \nicefrac{2}{T^2}
 \eqWHERE
 \mE := \cbr{\forall a\forall t\quad  |\bar{\mu}_t(a)-\mu(a)| \le r_t(a)}.
 \end{align}




\noindent The event $\mE$ in \eqref{IID:eq:clean} will be the \emph{clean event} for the subsequent analysis.

For each arm $a$ at round $t$, we define \emph{upper} and \emph{lower confidence bounds},
\begin{align*}
    \UCB_t(a) &= \bar{\mu}_t(a) + r_t(a), \\
    \LCB_t(a) &= \bar{\mu}_t(a) - r_t(a).
\end{align*}
As per Remark~\ref{IID:rem-conf}, we have \emph{confidence interval}
    $\sbr{\LCB_t(a), \UCB_t(a)}$
and \emph{confidence radius} $r_t(a)$.

\subsection{Successive Elimination algorithm}
\label{IID:sec:succ}

Let's come back to the case of $K=2$ arms, and recall the idea \eqref{BIC:eq:idea}. Now we can naturally implement this idea via the confidence bounds. The full algorithm for two arms is as follows:

\LinesNotNumbered
\begin{algorithm}[H]
    \SetAlgoLined
    Alternate two arms until $\UCB_t(a) < \LCB_t(a')$ after some even round $t$\;
    Abandon arm $a$, and use arm $a'$ forever since.
    \caption{``High-confidence elimination" algorithm for two arms}
    \label{IID:SE-two}
\end{algorithm}

For analysis, assume the clean event. Note that the ``abandoned" arm cannot be the best arm. But how much regret do we accumulate \emph{before} disqualifying one arm?

Let $t$ be the last round when we did \emph{not} invoke the stopping rule, i.e., when the confidence intervals of the two arms still overlap (see \reffig{IID:fig:last round}). Then
    $$\Delta:=|\mu(a) - \mu(a')|\le 2(r_t(a) + r_t(a')).$$

\begin{figure}[h]
 \centering
 \includegraphics[height = 0.3\textheight]{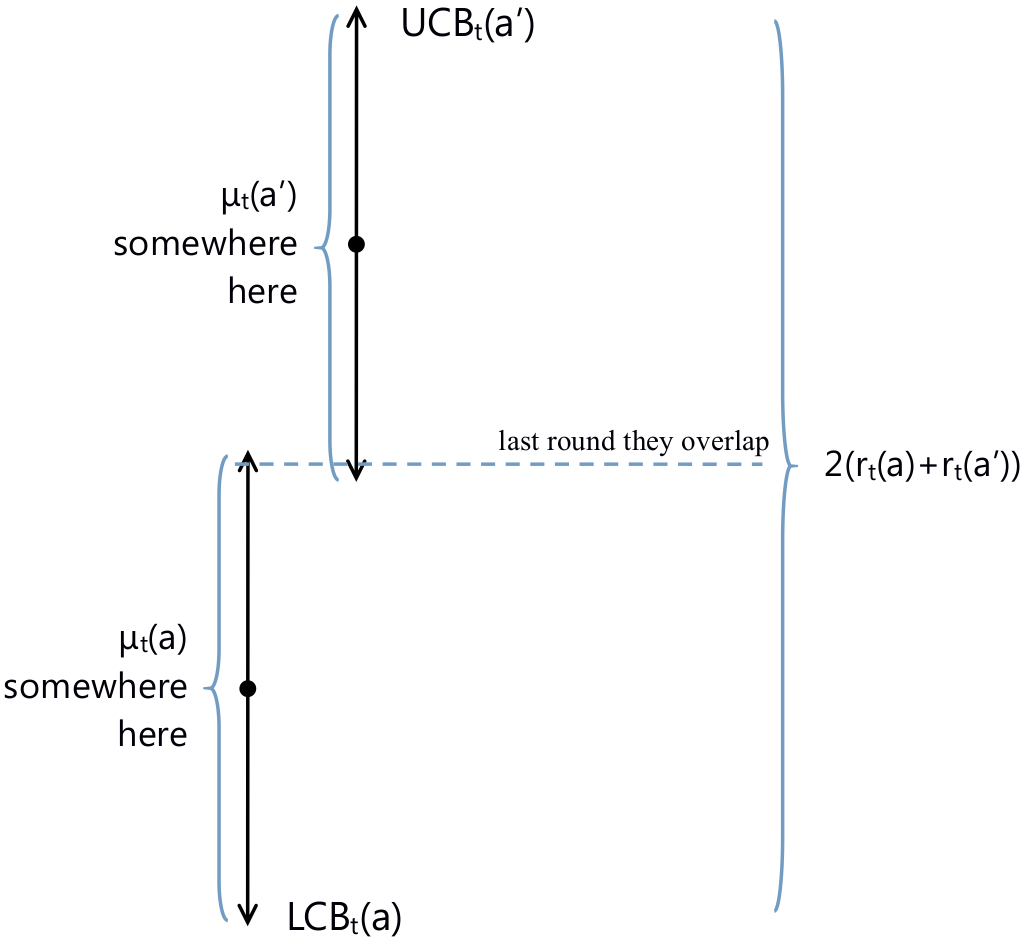}
  \caption{$t$ is the last round that the two confidence intervals still overlap}\label{IID:fig:last round}
\end{figure}

Since the algorithm has been alternating the two arms before time $t$, we have $n_t(a) = \nicefrac{t}{2}$ (up to floor and ceiling), which yields
\[  \Delta
    \leq 2\rbr{r_t(a) + r_t(a')}
    \leq 4 \sqrt{2\log(T) \,/\,\Flr{t/2} }
    = O\rbr{\sqrt{\log(T)\,/\, t}}.\]
Then the total regret accumulated till round $t$ is
\[  R(t)
    \leq \Delta\times t
    \leq O\rbr{t \cdot \sqrt{\tfrac{\log{T}}{t}}}
    = O\rbr{\sqrt{t\log{T}}}.\]
Since we've chosen the best arm from then on, we have
    $R(t)\leq O\rbr{\sqrt{t\log{T}}} $.
To complete the analysis, we need to argue that the ``bad event" $\bar\mE$ contributes a negligible amount to regret, like in \eqref{IID:eq:clean-coda}:
\begin{align*}
\E\sbr{R(t)}
    &= \E\sbr{R(t) \mid \text{clean event}}\times \Pr\sbr{\text{clean event}}
        \;+\; \E\sbr{R(t) \mid \text{bad event}}\times \Pr\sbr{\text{bad event}} \\
        &\leq \E\sbr{ R(t) \mid \text{clean event}}
            + t\times O(T^{-2}) \\
        &\leq O\rbr{\sqrt{t\log T}}.
\end{align*}
We proved the following:

\begin{lemma}
For two arms, Algorithm~\ref{IID:SE-two} achieves regret
    $\E\sbr{R(t)} \leq O\rbr{\sqrt{t\log T}}$
for each round $t\leq T$.
\end{lemma}

\begin{remark}
The $\sqrt{t}$ dependence in this regret bound should be contrasted with the $T^{2/3}$ dependence for Explore-First. This improvement is possible due to adaptive exploration.
\end{remark}

This approach extends to $K>2$ arms as follows: alternate the arms until some arm $a$ is \emph{worse} than some other arm with high probability. When this happens, discard all such arms $a$ and go to the next phase. This algorithm is called \emph{Successive Elimination}.

\LinesNotNumbered \SetAlgoLined
\begin{algorithm}[H]
\begin{algorithmic}
\STATE All arms are initially designated as \emph{active}
\LOOP[new phase]
\STATE play each active arm once
\STATE deactivate all arms $a$ such that, letting $t$ be the current round,
\STATE \algTAB $\UCB_t(a) < \LCB_t(a')$ for some other arm $a'$
\COMMENT{deactivation rule}
\ENDLOOP
\end{algorithmic}
    \caption{Successive Elimination}
    \label{IID:alg3}
\end{algorithm}

To analyze the performance of this algorithm, it suffices to focus on the clean event \eqref{IID:eq:clean}; as in the case of $K=2$ arms, the contribution of the ``bad event" $\bar\mE$ can be neglected.

Let $a^*$ be an optimal arm, and note that it cannot be deactivated. Fix any arm $a$ such that $\mu(a) < \mu(a^*)$. Consider the last round $t\leq T$ when deactivation rule was invoked and arm $a$ remained active. As in the argument for $K=2$ arms, the confidence intervals of $a$ and $a^*$ must overlap at round $t$. Therefore,
\begin{align*}
\Delta(a)
    := \mu(a^*)-\mu(a)
    & \leq 2\rbr{r_t(a^*)+r_t(a)}
    = 4\cdot r_t(a).
\end{align*}
The last equality is because $n_t(a)=n_t(a^*)$, since the algorithm has been alternating active arms and both $a$ and $a^*$ have been active before round $t$. By the choice of $t$, arm $a$ can be played at most once afterwards: $n_T(a) \leq 1+n_t(a)$. Thus, we have the following crucial property:
\begin{align}\label{IID:eq:SE-main-prop}
\Delta(a)\leq O(r_T(a)) =O\rbr{  \sqrt{\left.\log(T) \right/ n_T(a) }}
\quad\text{for each arm $a$ with $\mu(a)<\mu(a^*)$}.
\end{align}
Informally: an arm played many times cannot be too bad. The rest of the analysis only relies on \eqref{IID:eq:SE-main-prop}. In other words, it does not matter which algorithm achieves this property.

The contribution of arm $a$ to regret at round $t$, denoted $R(t;a)$, can be expressed as $\Delta(a)$ for each round this arm is played; by \eqref{IID:eq:SE-main-prop} we can bound this quantity as
\[ R(t;a) = n_t(a)\cdot \Delta(a)
    \le  n_t(a) \cdot O\left(\sqrt{\log(T) \,/\, n_t(a)}\right)
    = O\rbr{ \sqrt{n_t(a)\log{T}}}. \]
\noindent Summing up over all arms, we obtain that
\begin{align}\label{IID:eq:SE-almost-done}
\textstyle
R(t) = \sum_{a\in \mA}R(t;a)
    \le O\rbr{\sqrt{\log T}}\sum_{a\in \mA}\sqrt{n_t(a)}.
\end{align}
Since $f(x)=\sqrt{x}$ is a real concave function, and $\sum_{a \in \mA} n_t(a) = t$, by Jensen's Inequality we have
\[ \frac{1}{K}\sum_{a \in \mA}\sqrt{n_t(a)} \le \sqrt{\frac{1}{K}\sum_{a \in \mA} n_t(a)} = \sqrt{\frac{t}{K}}. \]
Plugging this into \eqref{IID:eq:SE-almost-done}, we see that
$R(t)\le O\rbr{\sqrt{Kt\log{T}}}$. Thus, we have proved:

\begin{theorem}\label{IID:thm:SE-sqrtT}
Successive Elimination algorithm achieves regret
\begin{align}\label{IID:eq:thm:SE-sqrtT}
\E\sbr{R(t)}=O\rbr{ \sqrt{Kt\log T} }
    \quad \text{for all rounds $t\leq T$}.
\end{align}
\end{theorem}

We can also use property \eqref{IID:eq:SE-main-prop} to obtain another regret bound. Rearranging the terms in \eqref{IID:eq:SE-main-prop}, we obtain
    $n_T(a) \leq O\rbr{ \log(T) \,/\, [\Delta(a)]^2}$.
In words, a bad arm cannot be played too often. So,
\begin{align}\label{IID:eq:SE-perArm}
R(T;a) = \Delta(a)\cdot n_T(a)
    \le \Delta(a)\cdot O\rbr{ \frac{\log{T}}{[\Delta(a)]^2} }
    = O\rbr{ \frac{\log{T}}{\Delta(a)} }.
\end{align}
Summing up over all suboptimal arms, we obtain the following theorem.

\begin{theorem}\label{IID:thm:SE-logT}
Successive Elimination algorithm achieves regret
\begin{align}\label{IID:eq:thm:SE-logT}
\E[R(T)]\leq O(\log{T})
    \sbr{ \sum_{\text{arms $a$ with $\mu(a)<\mu(a^*)$}}\frac{1}{\mu(a^*)-\mu(a)} }.
\end{align}
\end{theorem}

This regret bound is logarithmic in $T$, with a constant that can be arbitrarily large depending on a problem instance. In particular, this constant is at most $O(\nicefrac{K}{\Delta})$, where
\begin{align}\label{IID:eq:min-gap}
\Delta &= \min_{\text{suboptimal arms $a$}} \Delta(a)
    &\EqComment{minimal gap}.
\end{align}
The distinction between regret bounds achievable with an absolute constant (as in Theorem~\ref{IID:thm:SE-sqrtT}) and regret bounds achievable with an instance-dependent constant is typical for multi-armed bandit problems. The existence of logarithmic regret bounds is another benefit of adaptive exploration.

\begin{remark}
For a more formal terminology, consider a regret bound of the form $C\cdot f(T)$, where $f(\cdot)$ does not depend on the mean reward vector $\mu$, and the ``constant" $C$ does not depend on $T$. Such regret bound is called \emph{instance-independent} if $C$ does not depend on $\mu$, and \emph{instance-dependent} otherwise.
\end{remark}

\begin{remark}
It is instructive to derive Theorem~\ref{IID:thm:SE-sqrtT} in a different way: starting from the logarithmic regret bound in \eqref{IID:eq:SE-perArm}. Informally, we need to get rid of arbitrarily small gaps $\Delta(a)$ in the denominator. Let us fix some $\eps >0$, then regret consists of two parts:
\begin{OneLiners}
\item all arms $a$ with $\Delta(a)\le\eps$ contribute at most $\eps$ per round, for a total of $\eps T$;
\item each arm $a$ with $\Delta(a)>\eps$ contributes
    $R(T;a)\leq O\rbr{\frac{1}{\eps} \log T}$, under the clean event.
\end{OneLiners}
Combining these two parts and assuming the clean event, we see that
    $$R(T)\le O\left(\eps T + \tfrac{K}{\eps}\log{T}\right).$$
Since this holds for any $\eps>0$, we can choose one that minimizes the right-hand side. Ensuring that $\eps T = \frac{K}{\eps}\log{T}$ yields $\eps = \sqrt{\frac{K}{T}\log{T}}$, and therefore
$R(T) \le O\rbr{\sqrt{KT\log T}}$.
\end{remark}

\subsection{Optimism under uncertainty}
\label{IID:sec:UCB}

Let us consider another approach for adaptive exploration, known as \emph{optimism under uncertainty}: assume each arm is as good as it can possibly be given the observations so far, and choose the best arm based on these optimistic estimates. This intuition leads to the following simple algorithm called \UcbOne:

\LinesNotNumbered
\begin{algorithm}[H]
\begin{algorithmic}
\STATE Try each arm once
\FOR{each round $t = 1 \LDOTS T$}
\STATE pick arm some $a$ which maximizes $\UCB_t(a)$.
\ENDFOR
\end{algorithmic}
    \caption{Algorithm \UcbOne}
    \label{IID:alg4}
\end{algorithm}

\begin{remark}
Let's see why UCB-based selection rule makes sense. An arm $a$ can have a large $\UCB_t(a)$ for two reasons (or combination thereof): because the average reward $\bar{\mu}_t(a)$ is large, in which case this arm is likely to have a high reward, and/or because the confidence radius $r_t(a)$ is large, in which case this arm has not been explored much. Either reason makes this arm worth choosing.  Put differently, the two summands in
    $\UCB_t(a) = \bar{\mu}_t(a) + r_t(a)$
represent, resp., exploitation and exploration, and summing them up is one natural way to resolve exploration-exploitation tradeoff.
\end{remark}

To analyze this algorithm, let us focus on the clean event \eqref{IID:eq:clean}, as before. Recall that $a^*$ is an optimal arm, and $a_t$ is the arm chosen by the  algorithm in round $t$. According to the algorithm,
    $\UCB_t(a_t) \ge \UCB_t(a^*)$.
Under the clean event, $\mu(a_t) + r_t(a_t) \ge \bar{\mu}_t(a_t)$
    and $\UCB_t(a^*)\geq \mu(a^*)$.
Therefore:
\begin{align}\label{IID:eq:UCB-trick}
\mu(a_t) + 2r_t(a_t)
    \geq \bar{\mu}_t(a_t) + r_t(a_t)
    = \UCB_t(a_t) \ge \UCB_t(a^*)
    \geq \mu(a^*).
\end{align}
It follows that
\begin{align}\label{IID:eq:UCB-trick-2}
\Delta(a_t) := \mu(a^*) - \mu(a_t) \le 2r_t(a_t)
    = 2\sqrt{2\log(T) \,/\,n_t(a_t)}.
\end{align}
This cute trick resurfaces in the analyses of several UCB-like algorithms for more general settings.

For each arm $a$ consider the last round $t$ when this arm is chosen by the algorithm. Applying \eqref{IID:eq:UCB-trick-2} to this round gives us property \eqref{IID:eq:SE-main-prop}. The rest of the analysis follows from that property, as in the analysis of Successive Elimination.

\begin{theorem}
Algorithm \UcbOne satisfies regret bounds in \eqref{IID:eq:thm:SE-sqrtT} and \eqref{IID:eq:thm:SE-logT}.
\end{theorem}

\section{Forward look: bandits with initial information}
\label{IID:sec:fwd}

Some information about the mean reward vector $\mu$ may be known to the algorithm beforehand, and may be used to improve performance. This ``initial information" is typically specified via a constraint on $\mu$ or a Bayesian prior thereon. In particular, such models may allow for non-trivial regret bounds that are independent of the number of arms, and hold even for infinitely many arms.

\xhdr{Constrained mean rewards.}
The canonical modeling approach embeds arms into $\R^d$, for some fixed $d\in\N$. Thus, arms correspond to points in $\R^d$, and $\mu$ is a function on (a subset of) $\R^d$ which maps arms to their respective mean rewards. The constraint is that $\mu$ belongs to some family $\mF$ of ``well-behaved" functions. Typical assumptions are:
\begin{OneLiners}
\item \emph{linear functions}: $\mu(a) = w\cdot a$ for some fixed but unknown vector $w\in \R^d$.
\item \emph{concave functions}: the set of arms is a convex subset in $\R^d$, $\mu''(\cdot)$ exists and is negative.

\item \emph{Lipschitz functions}:
    $|\mu(a)-\mu(a')| \leq L\cdot \|a-a'\|_2$
  for all arms $a,a'$ and some fixed constant $L$.
\end{OneLiners}
Such assumptions introduce dependence between arms, so that one can infer something about the mean reward of one arm by observing the rewards of some other arms. In particular, Lipschitzness allows only ``local" inferences: one can learn something about arm $a$ only by observing other arms that are not too far from $a$. In contrast, linearity and concavity allow ``long-range" inferences: one can learn about arm $a$ by observing arms that lie very far from $a$.

One usually proves regret bounds that hold for each $\mu\in\mF$. A typical regret bound allows for infinitely many arms, and only depends on the time horizon $T$ and the dimension $d$. The drawback is that such results are only as good as the \emph{worst case} over $\mF$. This may be overly pessimistic if the ``bad" problem instances occur very rarely, or overly optimistic if $\mu\in\mF$ is itself a very strong assumption.

\xhdr{Bayesian bandits.} Here, $\mu$ is drawn independently from some distribution $\PP$, called the \emph{Bayesian prior}. One is interested in \emph{Bayesian regret}: regret in expectation over $\PP$. This is special case of the Bayesian approach, which is very common in statistics and machine learning: an instance of the model is sampled from a known distribution, and performance is measured in expectation over this distribution.

The prior $\PP$ implicitly defines the family $\mF$ of feasible mean reward vectors (\ie the support of $\PP$), and moreover specifies whether and to which extent some mean reward vectors in $\mF$ are more likely than others. The main drawbacks are that the sampling assumption may be very idealized in practice, and the ``true" prior may not be fully known to the algorithm.

\OMIT{ 
An important special case is \emph{independent priors}: the mean rewards $(\mu(a): a\in\mA)$ are mutually independent. Then the prior $\PP$ can be represented as a product
        $\PP = \prod_{\text{arms $a$}} \PP_a$,
    where $\PP_a$ is the prior for arm $a$ (meaning that the mean reward $\mu(a)$ is drawn from $\PP_a$). Likewise, the support $\mF$ is a product set
        $\mF = \prod_{\text{arms $a$}} \mF_a$,
    where each $\mF_a$ is the set of all possible values for $\mu(a)$. Per-arm priors $\PP_a$ typically considered in the literature include a uniform distribution over a given interval, a Gaussian (truncated or not), and just a discrete distribution over several possible values.
} 

\sectionBibNotes
\label{sec:IID-further}

This chapter introduces several techniques that are broadly useful in multi-armed bandits, beyond the specific setting discussed in this chapter. These are the four algorithmic techniques (Explore-first, Epsilon-greedy, Successive Elimination, and UCB-based arm selection), the `clean event' technique in the analysis, and the ``UCB trick" from \eqref{IID:eq:UCB-trick}. Successive Elimination is from \citet{EvenDar-colt02}, and \UcbOne is from \citet{bandits-ucb1}. Explore-first and Epsilon-greedy algorithms have been known for a long time; it is unclear what the original references are. The original version of \UcbOne has confidence radius
\begin{align}\label{IID:eq:confRad-UCBpaper}
    r_t(a) = \sqrt{\alpha\cdot \ln(t) \,/\, n_t(a)}
\end{align}
with $\alpha=2$; note that $\log T$ is replaced with $\log t$ compared to the exposition in this chapter (see \eqref{IID:eq:desired}).
This version allows for the same regret bounds, at the cost of a somewhat more complicated analysis.

\xhdr{Optimality.}
Regret bounds in \eqref{IID:eq:thm:SE-sqrtT} and \eqref{IID:eq:thm:SE-logT}  are near-optimal, according to the lower bounds which we discuss in Chapter~\ref{ch:LB}. The instance-dependent regret bound in \eqref{IID:eq:thm:SE-sqrtT} is optimal up to $O(\log T)$ factors. \citet{Bubeck-colt09} shave off the $\log T$ factor, obtaining an instance dependent regret bound $O(\sqrt{KT})$.

The logarithmic regret bound in \eqref{IID:eq:thm:SE-logT} is optimal up to constant factors. A line of work strived to optimize the multiplicative constant in $O()$.
\citet{bandits-ucb1,Bubeck-thesis,Garivier-colt11} analyze this constant for \UcbOne, and eventually improve it to $\frac{1}{2 \ln 2}$.
\footnote{More precisely, \citet{Garivier-colt11} derive the constant $\frac{\alpha}{2 \ln 2}$, using confidence radius \eqref{IID:eq:confRad-UCBpaper} with any $\alpha>1$. The original analysis in \citet{bandits-ucb1} obtained constant $\frac{8}{\ln 2}$ using $\alpha=2$.}
This factor is the best possible in view of the lower bound in Section~\ref{LB:sec:per-instance}. Further, \citep{Audibert-TCS09,Honda-colt10,Garivier-colt11,Maillard-colt11-KL} refine the \UcbOne algorithm and obtain regret bounds that are at least as good as those for \UcbOne, and get better for some reward distributions.

\xhdr{High-probability regret.} In order to upper-bound expected regret $\E[R(T)]$, we actually obtained a high-probability upper bound on $R(T)$. This is common for regret bounds obtained via the ``clean event" technique. However, high-probability regret bounds take substantially more work in some of the more advanced bandit scenarios,  \eg for adversarial bandits (see Chapter~\ref{ch:adv}).

\xhdr{Regret for all rounds at once.}
What if the time horizon $T$ is not known in advance? Can we achieve similar regret bounds that hold for all rounds $t$, not just for all $t\leq T$? Recall that in Successive Elimination and \UcbOne, knowing $T$ was needed only to define the confidence radius $r_t(a)$. There are several remedies:

\begin{itemize}
\item If an upper bound on $T$ is known, one can use it instead of $T$ in the algorithm. Since our regret bounds depend on $T$ only logarithmically, rather significant over-estimates can be tolerated.

\item Use \UcbOne with confidence radius
    $r_t(a) = \sqrt{\frac{2\log{t}}{n_t(a)}}$,
as in \citep{bandits-ucb1}. This version does not input $T$, and its regret analysis works for an arbitrary $T$.

\item Any algorithm for known time horizon can (usually) be converted to an algorithm for an arbitrary time horizon using the \emph{doubling trick}. Here, the new algorithm proceeds in phases of exponential duration. Each phase $i=1,2, \ldots$ lasts $2^i$ rounds, and executes a fresh run of the original algorithm. This approach achieves the ``right" theoretical guarantees (see Exercise~\ref{IID:ex:doubling-trick}). However, forgetting everything after each phase is not very practical.
\end{itemize}

\vspace{-1mm}
\xhdr{Instantaneous regret.}
An alternative notion of regret considers each round separately: \emph{instantaneous regret} at round $t$ (also called \emph{simple regret}) is defined as
    $\Delta(a_t) = \mu^* - \mu(a_t)$,
where $a_t$ is the arm chosen in this round. In addition to having low cumulative regret, it may be desirable to spread the regret more ``uniformly" over rounds, so as to avoid spikes in instantaneous regret. Then one would also like an upper bound on instantaneous regret that decreases monotonically over time. See Exercise~\ref{IID:ex:inst-regret}.

\xhdr{Bandits with predictions.}
While the standard goal for bandit algorithms is to maximize cumulative reward, an alternative goal is to output a prediction $a^*_t$ after each round $t$. The algorithm is then graded only on the quality of these predictions. In particular, it does not matter how much reward is accumulated. There are two standard ways to formalize this objective: (i) minimize instantaneous regret $\mu^*-\mu(a^*_t)$, and (ii) maximize the probability of choosing the best arm: $\Pr[a^*_t=a^*]$. The former is called \emph{pure-exploration bandits}, and the latter is called \emph{best-arm identification}. Essentially, good algorithms for cumulative regret, such as Successive Elimination and \UcbOne, are also good for this version (more on this in Exercises~\ref{IID:ex:inst-regret} and~\ref{IID:ex:bandits-with-predictions}). However, improvements are possible in some regimes \citep[\eg][]{Tsitsiklis-bandits-04,EvenDar-icml06,Bubeck-alt09,Audibert-colt10}.
See Exercise~\ref{IID:ex:bandits-with-predictions}.

\xhdr{Available arms.} What if arms are not always available to be selected by the algorithm? The \emph{sleeping bandits} model in \citet{sleeping-colt08} allows arms to ``fall asleep", \ie become unavailable.%
\footnote{Earlier work \citep{Freund-colt97,Blum-sleeping97} studied this problem under full feedback.}
Accordingly, the benchmark is not the best fixed arm, which may be asleep at some rounds, but the best fixed \emph{ranking} of arms: in each round, the benchmark selects the highest-ranked arm that is currently ``awake". In the \emph{mortal bandits} model in \citet{MortalMAB-nips08}, arms become permanently unavailable, according to a (possibly randomized) schedule that is known to the algorithm. Versions of \UcbOne works well for both models.

\xhdr{Partial feedback.} Alternative models of partial feedback has been studied. In \emph{dueling bandits} \citep{Yue-dueling12,Yue-dueling-icml09}, numeric reward is unavailable to the algorithm. Instead, one can choose \emph{two} arms for a ``duel", and find out whose reward in this round is larger. Motivation comes from web search optimization. When actions correspond to slates of search results, a numerical reward would only be a crude approximation of user satisfaction. However, one can measure which of the two slates is better in a much more reliable way using an approach called \emph{interleaving}, see survey \citep{OnlineEvalForIR-survey-2016} for background. Further work on this model includes \citep{Ailon:2014,RUCB2014,ContextualDuelingMAB-colt15}.

Another model, termed \emph{partial monitoring} \citep{Bartok-MathOR14,Antos-TCS13}, posits that the outcome for choosing an arm $a$ is a (reward, feedback) pair, where the feedback can be an arbitrary message. Under the IID assumption, the outcome is chosen independently from some fixed distribution $D_a$ over the possible outcomes. So, the feedback could be \emph{more} than bandit feedback, \eg it can include the rewards of some other arms, or \emph{less} than bandit feedback, \eg it might only reveal whether or not the reward is greater than $0$, while the reward could take arbitrary values in $[0,1]$. A special case is when the structure of the feedback is defined by a graph on vertices, so that the feedback for choosing arm $a$ includes the rewards for all arms adjacent to $a$ \citep{Alon-nips13,Alon-colt15}.

\xhdr{Relaxed benchmark.}
If there are too many arms (and no helpful additional structure), then perhaps competing against the best arm is perhaps too much to ask for. But suppose we relax the benchmark so that it ignores the best $\eps$-fraction of arms, for some fixed $\eps>0$. We are interested in regret relative to the best remaining arm, call it \emph{$\eps$-regret} for brevity. One can think of $\freps$ as an effective number of arms. Accordingly, we’d like a bound on \eps-regret that depends on $\freps$, but not on the actual number of arms $K$. \citet{Bobby-soda06-bandit} achieves $\eps$-regret of the form
		$\freps\cdot T^{2/3}\cdot \polylog(T)$.
(The algorithm is very simple, based on a version on the ``doubling trick" described above, although the analysis is somewhat subtle.) This result allows the $\eps$-fraction to be defined relative to an arbitrary ``base distribution" on arms, and extends to infinitely many arms. It is unclear whether better regret bounds are possible for this setting.

\xhdr{Bandits with initial information.} Bayesian Bandits, Lipschitz bandits and Linear bandits are covered in Chapter~\ref{ch:TS}, Chapter~\ref{ch:Lip} and Chapter~\ref{ch:lin}, respectively. Bandits with concave rewards / convex costs require fairly advanced techniques, and are not covered in this book. This direction has been initiated in \citet{Bobby-nips04,FlaxmanKM-soda05}, with important recent advances \citep{Hazan-nips14,Bubeck-colt15,bubeck2017kernel}. A comprehensive treatment of this subject can be found in \citep{Hazan-OCO-book}.

\sectionExercises

All exercises below
are fairly straightforward given the material in this chapter.

\begin{exercise}[rewards from a small interval]\label{IID:ex:small-interval}
Consider a version of the problem in which all the realized rewards are in the interval
    $\sbr{ \nicefrac12, \nicefrac12+\eps}$
for some $\eps\in(0,\nicefrac12)$. Define versions of \UcbOne and Successive Elimination attain  improved regret bounds (both logarithmic and root-T) that depend on the $\eps$.

\Hint{Use a version of Hoeffding Inequality with ranges.}
\end{exercise}

\begin{exercise}[Epsilon-greedy]\label{IID:ex:eps-greedy}
Prove Theorem~\ref{IID:thm:eps-greedy}: derive the $O(t^{2/3})\cdot (K\log t)^{1/3}$ regret bound for the epsilon-greedy algorithm exploration probabilities
    $\eps_t=t^{-1/3}\cdot (K\log t)^{1/3}$.

\Hint{Fix round $t$ and analyze $\E\sbr{\Delta(a_t)}$ for this round separately. Set up the ``clean event" for rounds $1 \LDOTS t$ much like in Section~\ref{IID:sec:clean} (treating $t$ as the time horizon), but also include the number of exploration rounds up to time $t$.}
\end{exercise}

\begin{exercise}[instantaneous regret]\label{IID:ex:inst-regret}
Recall that instantaneous regret at round $t$ is
    $\Delta(a_t) = \mu^* - \mu(a_t)$.

\begin{itemize}
\item[(a)] Prove that Successive Elimination achieves ``instance-independent" regret bound of the form
\begin{align}\label{hw1:eq:IR-SE}
\E\sbr{\Delta(a_t)} \leq \frac{\polylog(T)}{\sqrt{t/K}}
\quad \text{for each round $t\in[T]$}.
\end{align}

\OMIT{ 
\item[(b)] Let us argue that \UcbOne does not achieve the regret bound in \eqref{hw1:eq:IR-SE}. More precisely, let us consider a version of \UcbOne with $\UCB_t(a) = \bar{\mu}_t(a)+2\cdot r_t(a)$. (It is easy to see that our analysis carries over to this version.) Focus on two arms, and prove that this algorithm cannot achieve a regret bound of the form
\begin{align}\label{hw1:eq:IR-\UcbOne}
\E[\Delta(a_t)] \leq \frac{\polylog(T)}{t^\gamma}, \;\gamma>0
\quad \text{for each round $t\in[T]$}.
\end{align}

    \Hint{Fix mean rewards and focus on the clean event. If \eqref{hw1:eq:IR-\UcbOne} holds, then the bad arm cannot be played after some time $T_0$. Consider the last time the bad arm is played, call it $t_0\leq T_0$. Derive a lower bound on the UCB of the best arm at $t_0$ (stronger lower bound than the one proved in class). Consider what this lower bound implies for the UCB of the bad arm at time $t_0$. Observe that eventually, after some number of plays of the best arm, the bad arm will be chosen again, assuming a large enough time horizon $T$. Derive a contradiction with \eqref{hw1:eq:IR-\UcbOne}.}

    \TakeAway{for ``bandits with predictions", the simple solution of predicting the last-played arm to be the best arm does not always work, even for a good algorithm such as \UcbOne.}
} 

\item[(b)] Derive an ``instance-independent" upper bound on instantaneous regret of Explore-first.

\end{itemize}
\end{exercise}

\begin{exercise}[pure exploration]\label{IID:ex:bandits-with-predictions}
Recall that in ``pure-exploration bandits", after $T$ rounds the algorithm outputs a prediction: a guess $y_T$ for the best arm. We focus on the instantaneous regret $\Delta(y_T)$ for the prediction.

\begin{itemize}
\item[(a)] Take any bandit algorithm with an instance-independent regret bound
    $E[R(T)]\leq f(T)$,
and construct an algorithm for ``pure-exploration bandits" such that
    $\E[\Delta(y_T)]\leq f(T)/T$.

\Note{Surprisingly, taking $y_T = a_t$ does not seem to work in general -- definitely not immediately.
Taking $y_T$ to be the arm with a maximal empirical reward does not seem to work, either.
But there is a simple solution ...}

\TakeAway{We can easily obtain
    $\E[\Delta(y_T)] = O(\sqrt{K\log(T)/T}$
from standard algorithms such as \UcbOne and Successive Elimination. However, as parts (bc) show, one can do much better!}

\item[(b)] Consider Successive Elimination with $y_T=a_T$. Prove that (with a slightly modified definition of the confidence radius) this algorithm can achieve
\begin{align*}
    \E\sbr{\Delta(y_T)} \leq  T^{-\gamma}
        \quad\text{if $T>T_{\mu,\gamma}$},
\end{align*}
where $T_{\mu,\gamma}$ depends only on the mean rewards $\mu(a): a\in \mA$ and the $\gamma$. This holds for an arbitrarily large constant $\gamma$, with only a multiplicative-constant increase in regret.

\Hint{Put the $\gamma$ inside the confidence radius, so as to make the ``failure probability" sufficiently low.}

\item[(c)] Prove that alternating the arms (and predicting the best one) achieves, for any fixed $\gamma<1$:
\begin{align*}
    \E\sbr{\Delta(y_T)} \leq  e^{-\Omega(T^\gamma)}
        \quad\text{if $T>T_{\mu,\gamma}$},
\end{align*}
where $T_{\mu,\gamma}$ depends only on the mean rewards $\mu(a): a\in\mA$ and the $\gamma$.

\Hint{Consider Hoeffding Inequality with an arbitrary constant $\alpha$ in the confidence radius. Pick $\alpha$ as a function of the time horizon $T$ so that the failure probability is as small as needed.}

\end{itemize}
\end{exercise}

\begin{exercise}[Doubling trick]\label{IID:ex:doubling-trick}
Take any bandit algorithm $\mA$ for fixed time horizon $T$. Convert it to an algorithm $\mA_\infty$ which runs forever, in phases $i=1,2,3,\, ... $ of $2^i$ rounds each. In each phase $i$ algorithm $\mA$ is restarted and run with time horizon $2^i$.

\begin{itemize}
\item[(a)] State and prove a theorem which converts an instance-independent upper bound on regret for $\mA$ into similar bound for $\mA_\infty$ (so that this theorem applies to both \UcbOne and Explore-first).

\item[(b)] Do the same for $\log(T)$ instance-dependent upper bounds on regret. (Then regret increases by a $\log(T)$ factor.)
\end{itemize}
\end{exercise}

\chapter{Lower Bounds}
\label{ch:LB}
\begin{ChAbstract}
This chapter is about what bandit algorithms \emph{cannot} do. We present several fundamental results which imply that the regret rates in Chapter~\ref{ch:IID} are essentially the best possible.

\prereqs{Chapter~\ref{ch:IID} (minimally: the model and the theorem statements).}

\end{ChAbstract}

We revisit the setting of stochastic bandits from Chapter~\ref{ch:IID} from a different perspective: we ask what bandit algorithms \emph{cannot} do.  We are interested in lower bounds on regret which apply to all bandit algorithms at once. Rather than analyze a particular bandit algorithm, we show that any bandit algorithm cannot achieve a better regret rate. We prove that any algorithm suffers regret $\Omega(\sqrt{KT})$ on some problem instance. Then we use the same technique to derive stronger lower bounds for non-adaptive exploration. Finally, we formulate and discuss the instance-dependent $\Omega(\log T)$ lower bound (albeit without a proof). These fundamental lower bounds elucidate what are the best possible \emph{upper} bounds that one can hope to achieve.

The $\Omega(\sqrt{KT})$ lower bound is stated as follows:

\begin{theorem}\label{LB:thm:LB-root-t-basic}
Fix time horizon $T$ and the number of arms $K$. For any bandit algorithm, there exists a problem instance such that
    $\E[R(T)]\geq \Omega(\sqrt{KT})$.
\end{theorem}

This lower bound is ``worst-case", leaving open the possibility that a particular bandit algorithm has low regret for many/most other problem instances. To prove such a lower bound, one needs to construct a family $\mF$ of problem instances that can ``fool" any algorithm. Then there are two standard ways to proceed:
\begin{itemize}
\item[(i)] prove that any algorithm has high regret on some instance in $\mF$,
\item[(ii)] define a distribution over problem instances in $\mF$, and prove that any algorithm has high regret in expectation over this distribution.
\end{itemize}

\begin{remark}
Note that (ii) implies (i), is because if regret is high in expectation over problem instances, then there exists at least one problem instance with high regret. Conversely, (i) implies (ii) if $|\mF|$ is a constant: indeed, if we have high regret $H$ for some problem instance in $\mF$, then in expectation over a uniform distribution over $\mF$ regret is least $H/|\mF|$. However, this argument breaks if $|\mF|$ is large. Yet, a stronger version of (i) which says that regret is high for a \emph{constant fraction} of the instances in $\mF$ implies (ii), with uniform distribution over the instances, regardless of how large $|\mF|$ is.
\end{remark}

On a very high level, our proof proceeds as follows.
We consider 0-1 rewards and the following family of problem instances, with parameter $\eps>0$ to be adjusted in the analysis:
\begin{align}\label{LB:eq:instances}
\mI_j =
  \begin{cases}
    \mu_i = \nicefrac{1}{2} + \nicefrac{\eps}{2} & \text{ for arm } i = j \\
    \mu_i = \nicefrac{1}{2} & \text{ for each arm } i \neq j.
  \end{cases}
\end{align}
for each $j=1,2 \LDOTS K$, where $K$ is the number of arms. Recall from the previous chapter that sampling each arm $\tilde{O}(1/\epsilon^2)$ times suffices for our upper bounds on regret.%
\footnote{Indeed, \refeq{IID:eq:SE-main-prop} in Chapter~\ref{ch:IID} asserts that Successive Elimination would sample each suboptimal arm at most $\tilde{O}(1/\epsilon^2)$ times, and the remainder of the analysis applies to any algorithm which satisfies \eqref{IID:eq:SE-main-prop}.}
We will prove that sampling each arm  $\Omega(1/\epsilon^2)$ times is \emph{necessary} to determine whether this arm is good or bad. This leads to  regret $\Omega(K/\eps)$. We complete the proof by plugging in $\eps=\Theta(\sqrt{K/T})$.

However, the technical details are quite subtle. We present them in several relatively gentle steps.


\section{Background on KL-divergence}
\label{LB:sec:KL-divergence}

The proof relies on \emph{KL-divergence}, an important tool from information theory. We provide a brief introduction to KL-divergence for finite sample spaces, which suffices for our purposes. This material is usually covered in introductory courses on information theory.

Throughout, consider a finite sample space $\Omega$, and let $p, q$ be two probability distributions on $\Omega$. Then, the Kullback-Leibler divergence or \emph{KL-divergence} is defined as:
\begin{align*}
\KL(p, q) = \sum_{x \in \Omega} p(x) \ln \frac{p(x)}{q(x)}
= \mathbb{E}_p \left[ \ln \frac{p(x)}{q(x)}  \right].
\end{align*}

\noindent This is a notion of distance between two distributions, with the properties that it is non-negative, 0 iff $p=q$, and small if the distributions $p$ and $q$ are close to one another. However, KL-divergence is not symmetric and does not satisfy the triangle inequality.

\begin{remark}
KL-divergence is a mathematical construct with amazingly useful properties, see Theorem~\ref{LB:thm:KL-props}. The precise definition is not essential to us: any other construct with the same properties would work just as well. The deep reasons as to why KL-divergence should be defined in this way are beyond our scope.

That said, here's see some intuition behind this definition. Suppose data points $x_1 \LDOTS x_n\in \Omega$ are drawn independently from some fixed, but unknown distribution $p^*$. Further, suppose we know that this distribution is either $p$ or $q$, and we wish to use the data to estimate which one is more likely. One standard way to quantify whether distribution $p$ is more likely than $q$ is the \emph{log-likelihood ratio},
\[ \Lambda_n := \sum_{i=1}^n \ln \frac{p(x_i)}{q(x_i)}. \]
KL-divergence is the expectation of this quantity when $p^*=p$,  and also the limit as $n\to\infty$:
\[ \lim_{n\to\infty} \Lambda_n  = \E[\Lambda_n] = \KL(p,q) \quad
        \text{if $p^*=p$}. \]
\end{remark}


We present fundamental properties of KL-divergence which we use in this chapter. Throughout, let $\RC_\eps$, $\eps\geq 0$, denote a biased random coin with bias $\nicefrac{\eps}{2}$, \ie a distribution over $\{0,1\}$ with expectation $(1+\eps)/2$.

\begin{theorem}\label{LB:thm:KL-props}
KL-divergence satisfies the following properties:
\begin{itemize}
\item[(a)] {\bf Gibbs' Inequality}: $\KL(p, q) \geq 0$ for any two distributions $p,q$, with equality if and only if $p = q$.

\item[(b)] {\bf Chain rule for product distributions}: Let the sample space be a product $\Omega = \Omega_1 \times \Omega_1 \times \dots \times \Omega_n$. Let $p$ and $q$ be two distributions on $\Omega$ such that
        $p = p_1 \times p_2 \times \dots \times p_n$
    and
        $q = q_1 \times q_2 \times \dots \times q_n$,
    where $p_j,q_j$ are distributions on $\Omega_j$, for each $j\in[n]$. Then $\KL(p, q) = \sum_{j = 1}^n \KL(p_j, q_j)$.

\item[(c)] {\bf Pinsker's inequality}: for any event $A \subset \Omega$
we have $2 \left( p(A) - q(A) \right)^2 \le \KL(p, q)$.

\item[(d)] {\bf Random coins}:
        $\KL(\RC_\eps, \RC_0) \le 2\eps^2$,
and
        $\KL(\RC_0, \RC_\eps) \le \eps^2$
for all $\eps\in(0,\tfrac12)$.

\end{itemize}
\end{theorem}

The proofs of these properties are fairly simple (and teachable). We include them in Appendix~\ref{app:KL} for completeness, and so as to ``demystify" the KL technique.

A typical usage is as follows. Consider the setting from part (b) with $n$ samples from two random coins: $p_j = \RC_\eps$ is a biased random coin, and $q_j = \RC_0$ is a fair random coin, for each $j\in[n]$. Suppose we are interested in some event $A\subset \Omega$, and we wish to prove that $p(A)$ is not too far from $q(A)$ when $\eps$ is small enough. Then:
\begin{align*}
2(p(A) - q(A))^2
    &\leq \KL(p, q)                 &\EqComment{by Pinsker's inequality} \\
    &= \sum_{j = 1}^n \KL(p_j, q_j) &\EqComment{by Chain Rule} \\
    &\leq n\cdot \KL(\RC_\eps,\RC_0)&\EqComment{by definition of $p_j,q_j$}  \\
    &\leq 2n\eps^2.                  &\EqComment{by part (d)}
\end{align*}
It follows that
    $|p(A) - q(A)| \leq \eps\,\sqrt{n}$.
In particular,
    $|p(A) - q(A)| < \tfrac12$ whenever $\eps<\tfrac{1}{2\sqrt{n}}$.
Thus:

\begin{lemma}\label{LB:lm:KL-example}
Consider sample space $\Omega=\{0,1\}^n$ and two distributions on $\Omega$, $p=\RC_\eps^n$ and $q=\RC_0^n$, for some $\eps>0$.
Then
$|p(A) - q(A)| \leq \eps\,\sqrt{n}$
for any event $A\subset \Omega$.
\end{lemma}

\begin{remark}
The asymmetry in the definition of KL-divergence does not matter in the argument above: we could have written $\KL(q,p)$ instead of $\KL(p,q)$. Likewise, it does not matter throughout this chapter.
\end{remark}

\section{A simple example: flipping one coin}
\label{LB:sec:example}

We start with a simple application of the KL-divergence technique, which is also interesting as a standalone result. Consider a biased random coin: a distribution on $\{0,1\}$) with an unknown mean $\mu\in [0,1]$. Assume that $\mu\in\{\mu_1,\mu_2\}$ for two known values $\mu_1 > \mu_2$. The coin is flipped $T$ times. The goal is to identify if $\mu = \mu_1$ or $\mu = \mu_2$ with low probability of error.

Let us make our goal a little more precise. Define $\Omega := \{0,1\}^T$ to be the sample space for the outcomes of $T$ coin tosses. Let us say that we need a decision rule
    \[ \Rule: \Omega \rightarrow \{\High, \Low\} \]
which satisfies the following two properties:
\begin{align}
\Pr\sbr{\Rule(\text{observations})= \High \mid \mu = \mu_1}
    & \geq 0.99, \label{LB:coin:p1} \\
\Pr\sbr{\Rule(\text{observations})= \Low~~ \mid \mu = \mu_2}
    &\geq 0.99. \label{LB:coin:p2}
\end{align}
\noindent How large should $T$ be for for such a decision rule to exist? We know that $T\sim (\mu_1-\mu_2)^{-2}$ is sufficient. What we prove is that it is also necessary. We focus on the special case when both $\mu_1$ and $\mu_2$ are close to $\frac12$.

\begin{lemma}
Let $\mu_1 = \frac{1+\eps}{2}$ and $\mu_2 = \frac{1}{2}$. Fix a decision rule which satisfies \eqref{LB:coin:p1} and \eqref{LB:coin:p2}. Then $T>\tfrac{1}{4\,\eps^2}$.
\end{lemma}

\begin{proof}
Let $A_0\subset \Omega$ be the event this rule returns $\High$. Then
\begin{align}\label{LB:coin:two}
\Pr[A_0\mid \mu=\mu_1] - \Pr[A_0\mid \mu=\mu_2]  \geq 0.98.
\end{align}

Let
    $P_i(A) = \Pr[A \mid \mu=\mu_i]$,
for each event $A\subset \Omega$ and each $i\in\{1,2\}$. Then
    $P_i = P_{i,1} \times \ldots \times P_{i,T}$,
where $P_{i, t}$ is the distribution of the $t^{th}$ coin toss if $\mu=\mu_i$. Thus, the basic KL-divergence argument summarized in Lemma~\ref{LB:lm:KL-example} applies to distributions $P_1$ and $P_2$. It follows that
    $|P_1(A) - P_2(A)| \leq \eps\,\sqrt{T}$.
Plugging in $A=A_0$ and $T\leq \tfrac{1}{4\eps^2}$, we obtain
    $|P_1(A_0) - P_2(A_0)| <\tfrac12$, contradicting \eqref{LB:coin:two}.
\end{proof}

Remarkably, the proof applies to all decision rules at once!

\section{Flipping several coins: ``best-arm identification"}
\label{LB:sec:bandists-with-predictions}

Let us extend the previous example to multiple coins. We consider a bandit problem with $K$ arms, where each arm is a biased random coin with unknown mean. More formally, the reward of each arm is drawn independently from a fixed but unknown Bernoulli distribution. After $T$ rounds, the algorithm outputs an arm $y_T$: a prediction for which arm is optimal (has the highest mean reward). We call this version ``best-arm identification". We are only be concerned with the quality of prediction, rather than regret.

As a matter of notation, the set of arms is $[K]$, $\mu(a)$ is the mean reward of arm $a$, and a problem instance is specified as a tuple
    $\mI = (\mu(a) : a \in [K] )$.

For concreteness, let us say that a good algorithm for ``best-arm identification"
should satisfy
        \begin{equation}
                \label{LB:bandit:estimate}
                        \Pr\sbr{\text{prediction $y_T$ is correct } \mid \mI} \geq 0.99
        \end{equation}
for each problem instance $\mI$. We will use the family \eqref{LB:eq:instances} of problem instances, with parameter $\eps>0$, to argue that one needs
    $T \geq \Omega \left( \frac{K}{\eps^2} \right)$
for any algorithm to ``work", \ie satisfy property \eqref{LB:bandit:estimate}, on all instances in this family. This result is of independent interest, regardless of the regret bound that we've set out to prove.

In fact, we prove a stronger statement which will also be the crux in the proof of the regret bound.

\begin{lemma}\label{LB:lm:MAB-with-predictions}
Consider a ``best-arm identification" problem with
    $T \leq \frac{cK}{\eps^2}$,
for a small enough absolute constant $c>0$. Fix any deterministic algorithm for this problem. Then there exists at least $\Cel{K/3}$ arms $a$ such that, for problem instances $\mI_a$ defined in \eqref{LB:eq:instances}, we have
    \begin{align}\label{LB:eq:lm:MAB-with-predictions}
            \Pr\sbr{y_T = a \mid \mI_a} < \nicefrac34.
    \end{align}
\end{lemma}

\noindent The proof for $K=2$ arms is simpler, so we present it first. The general case is deferred to Section~\ref{LB:sec:MAB-with-predictions-proof}.


\begin{proof}[Proof ($K=2$ arms)]
Let us set up the sample space which we will use in the proof. Let
    \[ \rbr{r_t(a):\; a\in[K], t\in[T]} \]
be a tuple of mutually independent Bernoulli random variables such that
    $\E\sbr{r_t(a)} = \mu(a)$.
We refer to this tuple as the \emph{rewards table}, where we interpret $r_t(a)$ as the reward received by the algorithm for the $t$-th time it chooses arm $a$. The sample space is
    $\Omega = \{0,1\}^{K \times T}$,
where each outcome $\omega\in\Omega$ corresponds to a particular realization of the rewards table. Any event about the algorithm can be interpreted as a subset of $\Omega$, \ie the set of outcomes $\omega\in\Omega$ for which this event happens; we assume this interpretation throughout without further notice. Each problem instance $\mI_j$ defines distribution $P_j$ on $\Omega$:
\[
    P_j(A) = \Pr\sbr{ A \mid \text{instance $\mI_j$}} \quad \text{for each $A\subset \Omega$}.
\]
Let $P_j^{a,t}$ be the distribution of $r_t(a)$ under instance $\mI_j$, so that
        $P_j = \prod_{a\in[K],\; t\in [T]} P_j^{a,t}$.

We need to prove that \eqref{LB:eq:lm:MAB-with-predictions} holds for at least one of the arms. For the sake of contradiction, assume it fails for both arms.
Let
    $A = \{y_T  = 1\}\subset \Omega$
be the event that the algorithm predicts arm $1$. Then
        $P_1(A)\geq \nicefrac{3}{4}$ and $P_2(A)\leq \nicefrac{1}{4}$,
so their difference is $P_1(A)-P_2(A)\geq \nicefrac12$.

To arrive at a contradiction, we use a similar KL-divergence argument as before:
\begin{align}
2\rbr{P_1(A) - P_2(A)}^2
    & \leq \KL(P_1, P_2)
        &\EqComment{by Pinsker's inequality} \nonumber \\
    &=\sum_{a = 1}^{K} \sum_{t=1}^T \KL(P_1^{a,t}, P_2^{a,t})
        &\EqComment{by Chain Rule} \nonumber\\
    &\leq 2T \cdot 2\eps^2
        &\EqComment{by Theorem~\ref{LB:thm:KL-props}(d)}.
         \label{LB:eq:pf:predictions:2arms:KL}
\end{align}
The last inequality holds because for each arm $a$ and each round $t$, one of the distributions $P_1^{a,t}$ and $P_2^{a,t}$ is a fair coin $\RC_0$, and another is a biased coin $\RC_\eps$. Simplifying \eqref{LB:eq:pf:predictions:2arms:KL},
\[ P_1(A) - P_2(A) \leq \eps\sqrt{2\,T} < \tfrac{1}{2}
    \quad \text{whenever $T \leq (\tfrac{1}{4\eps})^2$}. \qedhere
\]
\end{proof}

\begin{corollary}\label{LB:cor:MAB-with-predictions}
Assume $T$ is as in Lemma~\ref{LB:lm:MAB-with-predictions}. Fix any algorithm for ``best-arm identification". Choose an arm $a$ uniformly at random, and run the algorithm on instance $\mI_a$. Then
             $\Pr[y_T \neq a] \geq \tfrac{1}{12} $,
where the probability is over the choice of arm $a$ and the randomness in rewards and the algorithm.
\end{corollary}
\begin{proof}
Lemma~\ref{LB:lm:MAB-with-predictions} immediately implies this corollary for deterministic algorithms. The general case follows because any randomized algorithm can be expressed as a distribution over deterministic algorithms.
\end{proof}


Finally, we use Corollary~\ref{LB:cor:MAB-with-predictions} to finish our proof of the $\sqrt{KT}$ lower bound on regret.

\begin{theorem}\label{LB:thm:LB-root-t}
Fix time horizon $T$ and the number of arms $K$. Fix a bandit algorithm. Choose an arm $a$ uniformly at random, and run the algorithm on problem instance $\mI_a$. Then
\begin{align}\label{LB:eq:thm:LB-root-t}
\E[R(T)]\geq \Omega(\sqrt{KT}),
\end{align}
where the expectation is over the choice of arm $a$ and the randomness in rewards and the algorithm.
\end{theorem}

\begin{proof}
Fix the parameter $\eps>0$ in \eqref{LB:eq:instances}, to be adjusted later, and assume that
    $T \leq \frac{cK}{\eps^2}$,
where $c$ is the constant from Lemma~\ref{LB:lm:MAB-with-predictions}.

Fix round $t$. Let us interpret the algorithm as a ``best-arm identification" algorithm, where the prediction is simply $a_t$, the arm chosen in this round. We can apply Corollary~\ref{LB:cor:MAB-with-predictions}, treating  $t$ as the time horizon, to deduce that
    $\Pr[a_t \neq a] \geq \tfrac{1}{12}$.
In words, the algorithm chooses a non-optimal arm with probability at least $\tfrac{1}{12}$. Recall that for each problem instances $\mI_a$, the ``gap"
$\Delta(a_t) := \mu^*-\mu(a_t)$ is $\eps/2$ whenever a non-optimal arm is chosen. Therefore,
\[ \E[\Delta(a_t)] = \Pr[a_t\neq a]\cdot \tfrac{\eps}{2} \geq \eps/24. \]
Summing up over all rounds,
    $ \E[R(T)] = \sum_{t=1}^T \E[\Delta(a_t)] \geq \eps T/24 $.
We obtain \eqref{LB:eq:thm:LB-root-t} with $\eps = \sqrt{cK/T}$.
\end{proof}

\section{Proof of Lemma~\ref{LB:lm:MAB-with-predictions} for the general case}
\label{LB:sec:MAB-with-predictions-proof}

Reusing the proof for $K=2$ arms only works for time horizon $T\leq c/\eps^2$, which yields the lower bound of $\Omega(\sqrt{T})$. Increasing $T$ by the factor of $K$ requires a more delicate version of the KL-divergence argument, which improves the right-hand side of \eqref{LB:eq:pf:predictions:2arms:KL} to $O(T\eps^2/K)$.

For the sake of the analysis, we will consider an additional problem instance
\[\mI_0 = \{
    \mu_i = \tfrac{1}{2} \text{ for all arms $i$ }\},
\]
which we call the ``base instance". Let $\E_0[\cdot]$ be the expectation given this problem instance. Also, let $T_a$ be the total number of times arm $a$ is played.

We consider the algorithm's performance on problem instance $\mI_0$, and focus on arms $j$ that are ``neglected" by the algorithm, in the sense that the algorithm does not choose arm $j$ very often \emph{and} is not likely to pick $j$ for the guess $y_T$. Formally, we observe the following:
\begin{align}
\text{There are $\geq \tfrac{2K}{3}$ arms $j$ such that }
    & {\textstyle \E_0}(T_j) \leq \tfrac{3T}{K},
          \label{LB:generalk:expect} \\
\text{There are $\geq \tfrac{2K}{3}$ arms $j$ such that }
    & P_0(y_T=j) \leq \tfrac{3}{K}. \label{LB:generalk:prob}
\end{align}
(To prove \eqref{LB:generalk:expect}, assume for contradiction that we have more than
    $\frac{K}{3}$ arms with $\E_0(T_j) > \frac{3T}{K}$.
Then the expected total number of times these arms are played is strictly greater than $T$, which is a contradiction. \eqref{LB:generalk:prob} is proved similarly.) By Markov inequality,
\[
    \textstyle{\E_0}(T_j) \leq \frac{3T}{K}
    \text{ implies that }
    \Pr[T_j \leq \frac{24T}{K}] \geq \nicefrac{7}{8}.
\]
Since the sets of arms in \eqref{LB:generalk:expect} and \eqref{LB:generalk:prob} must overlap on least $\frac{K}{3}$ arms, we conclude:
\begin{equation}\label{LB:generalk:comb}
\text{There are at least $\nicefrac{K}{3}$ arms $j$ such that }
            \Pr\sbr{T_j \leq \tfrac{24T}{K}} \geq \nicefrac{7}{8}
            \text{ and }
            P_0(y_T=j) \leq \nicefrac{3}{K}.
\end{equation}

We will now refine our definition of the sample space. For each arm $a$, define the $t$-round sample space $\Omega_a^t=\{0,1\}^t$, where each outcome corresponds to a particular realization of the tuple
        $(r_s(a):\; s\in[t])$.
(Recall that we interpret $r_t(a)$ as the reward received by the algorithm for the $t$-th time it chooses arm $a$.) Then the ``full" sample space we considered before can be expressed as $\Omega = \prod_{a\in[K]} \Omega_a^T$.

Fix an arm $j$ satisfying the two properties in \eqref{LB:generalk:comb}. We will prove that
\begin{align}\label{LB:eq:lm:MAB-with-predictions-pf-crux}
P_j\sbr{Y_T = j} \leq \nicefrac12.
\end{align}
Since there are at least $K/3$ such arms, this suffices to imply the Lemma.

We consider a ``reduced" sample  space in which arm $j$ is played only
    $m= \min\rbr{T,\,24T/K}$ times:
    \begin{align}
    \Omega^* = \Omega_j^m\times \prod_{\text{arms $a\neq j$}} \Omega_a^T.
    \end{align}
For each problem instance $\mI_\ell$, we define distribution $P^*_\ell$ on $\Omega^*$ as follows:
\[
    P^*_\ell(A) = \Pr[A \mid \mI_\ell] \quad \text{for each $A\subset \Omega^*$}.
\]
In other words, distribution $P^*_\ell$ is a restriction of $P_\ell$ to the reduced sample space $\Omega^*$.

We apply the KL-divergence argument to distributions $P^*_0$ and $P^*_j$.  For each event $A\subset \Omega^*$:
\begin{align*}
2\rbr{P_0^*(A) - P_j^*(A)}^2
    & \leq \KL(P_0^*, P_j^*)
        &\EqComment{by Pinsker's inequality} \\
    &= \sum_{\text{arms $a$}}\; \sum_{t=1}^T \KL(P_0^{a,t}, P_j^{a,t})
        &\EqComment{by Chain Rule} \\
    &= \sum_{\text{arms $a \neq j$}}\; \sum_{t=1}^T \KL(P_0^{a,t}, P_j^{a,t}) + \sum_{t=1}^{m} \KL(P_0^{j,t}, P_j^{j,t})\\
        &\leq 0 + m\cdot 2\eps^2
            &\EqComment{by Theorem~\ref{LB:thm:KL-props}(d)}.
\end{align*}
The last inequality holds because each arm $a\neq j$ has identical reward distributions under problem instances $\mI_0$ and $\mI_j$ (namely the fair coin $\RC_0$), and for arm $j$ we only need to sum up over $m$ samples rather than $T$.

Therefore, assuming $T \leq \frac{cK}{\eps^2}$ with small enough constant $c$, we can conclude that
\begin{align}\label{LB:eq:pf:predictions:Karms:KL}
    |P_0^*(A) - P_j^*(A)| \leq \eps \sqrt{m} < \tfrac{1}{8}
      \quad \text{for all events $A\subset \Omega^*$}.
\end{align}

To apply \eqref{LB:eq:pf:predictions:Karms:KL}, we need to make sure that
    $A\subset \Omega^*$,
\ie that whether this event holds is completely determined by the first $m$ samples of arm $j$ (and all samples of other arms). In particular, we cannot take $A=\{y_T=j\}$, the event that we are interested in, because this event may depend on more than $m$ samples of arm $j$. Instead, we apply \eqref{LB:eq:pf:predictions:Karms:KL} twice: to events
\begin{align}\label{LB:eq:pf:predictions:Karms:events}
A=\{ y_T=j \text{ and } T_j\leq m \}
    \text{ and }
A' = \{ T_j > m \}.
\end{align}
Recall that we interpret these events as subsets of the sample space
    $\Omega = \{0,1\}^{K \times T}$,
\ie as a subset of possible realizations of the rewards table.
Note that $A,A'\subset \Omega^*$; indeed, $A'\subset \Omega^*$ because whether the algorithm samples arm $j$ more than $m$ times is completely determined by the first $m$ samples of this arm (and all samples of the other arms). We are ready for the final computation:
\begin{align*}
P_j(A)
    &\leq \tfrac18 + P_0(A)
        & \EqComment{by \eqref{LB:eq:pf:predictions:Karms:KL}} \\
    &\leq \tfrac18 + P_0(y_T=j) \\
    &\leq \tfrac14
        & \EqComment{by our choice of arm $j$}. \\
P_j(A')
    &\leq \tfrac18 + P_0(A')
        & \EqComment{by \eqref{LB:eq:pf:predictions:Karms:KL}} \\
    &\leq \tfrac14
        & \EqComment{by our choice of arm $j$}. \\
P_j(Y_T = j)
    &\leq P_j^*(Y_T=j~and~T_j\leq m) + P_j^*(T_j > m) \\
    &= P_j(A)+P_j(A') \leq \nicefrac12.
\end{align*}
This completes the proof of \refeq{LB:eq:lm:MAB-with-predictions-pf-crux}, and hence that of the Lemma.


\section{Lower bounds for non-adaptive exploration}
\label{LB:sec:non-adaptive}

The same information-theoretic technique implies much stronger lower bounds for non-adaptive exploration, as per Definition~\ref{IID:def:non-adaptive}.
First, the $T^{2/3}$ upper bounds from Section~\ref{IID:sec:uniform} are essentially the best possible.

\begin{theorem}\label{LB:thm:non-adaptive}
Consider any algorithm which satisfies non-adaptive exploration. Fix time horizon $T$ and the number of arms $K<T$. Then there exists a problem instance such that
    $\E[R(T)]\geq \Omega\rbr{T^{2/3}\cdot K^{1/3}}$.
\end{theorem}

Second, we rule out logarithmic upper bounds such as \eqref{IID:eq:thm:SE-logT}. The statement is more nuanced, requiring the algorithm to be at least somewhat reasonable in the worst case.%
\footnote{The worst-case assumption is necessary. For example, if we only focus on problem instances with minimum gap at least $\Delta$, then Explore-first with $N = O(\Delta^{-2} \log T)$ rounds of exploration yields logarithmic regret.}


\begin{theorem}\label{LB:thm:non-adaptive-refined}
In the setup of Theorem~\ref{LB:thm:non-adaptive}, suppose
    $\E[R(T)]\leq C\cdot T^\gamma$
for all problem instances, for some numbers $\gamma\in[\nicefrac23,1)$ and $C>0$. Then for any problem instance a random permutation of arms yields
\[ \E\sbr{R(T)}\textstyle
    \geq \Omega\rbr{ C^{-2}\cdot T^\lambda\cdot \sum_a \Delta(a)},
\quad\text{where $\lambda=2(1-\gamma)$.}\]
\end{theorem}

\noindent In particular, if an algorithm achieves regret
    $\E[R(T)] \leq \tildeO\rbr{T^{2/3}\cdot K^{1/3}}$
over all problem instances, like Explore-first and Epsilon-greedy, this algorithm incurs a similar regret for \emph{every} problem instance, if the arms therein are randomly permuted:
    $\E[R(T)] \geq \tilde{\Omega}\rbr{\Delta\cdot T^{2/3}\cdot K^{1/3}}$,
where $\Delta$ is the minimal gap. This follows by taking $C = \tildeO(K^{1/3})$ in the theorem. 

The KL-divergence technique is ``imported" via Corollary~\ref{LB:cor:MAB-with-predictions}. Theorem~\ref{LB:thm:non-adaptive-refined} is fairly straightforward given this corollary, and Theorem~\ref{LB:thm:non-adaptive} follows by taking $C = K^{1/3}$; see Exercise~\ref{LB:ex:non-adaptive} and hints therein.

\section{Instance-dependent lower bounds (without proofs)}
\label{LB:sec:per-instance}

The other fundamental lower bound asserts $\Omega(\log T)$ regret with an instance-dependent constant and, unlike the $\sqrt{KT}$ lower bound, applies to every problem instance. This lower bound complements the $\log(T)$ \emph{upper} bound that we proved for algorithms UCB1 and Successive Elimination. We formulate and explain this lower bound below, without presenting a proof. The formulation is quite subtle, so we present it in stages.

Let us focus on 0-1 rewards. For a particular problem instance, we are interested in how $\E[R(t)]$ grows with $t$. We start with a simpler and weaker version:

\begin{theorem}
No algorithm can achieve regret
    $\E[R(t)] = o(c_\mI\;\log t)$
for all problem instances $\mI$, where the ``constant" $c_\mI$ can depend on the problem instance but not on the time $t$.
\end{theorem}

This version guarantees at least one problem instance on which a given algorithm has ``high" regret. We would like to have a stronger lower bound which guarantees ``high" regret for each problem instance. However, such lower bound is impossible because of a trivial counterexample: an algorithm which always plays arm $1$, as dumb as it is, nevertheless has $0$ regret on any problem instance for which arm $1$ is optimal. To rule out such counterexamples, we require the algorithm to perform reasonably well (but not necessarily optimally) across all problem instances.

\begin{theorem}\label{LB:thm:LB-log-t}
Fix $K$, the number of arms. Consider an algorithm such that
\begin{align}\label{LB:eq:thm:LB-log-t:assn}
\E[R(t)]\leq O(C_{\mI,\alpha}\; t^\alpha)
\quad\text{for each problem instance $\mI$ and each $\alpha>0$}.
\end{align}
Here the ``constant" $C_{\mI,\alpha}$ can depend on the problem instance $\mI$ and the $\alpha$, but not on time $t$.

Fix an arbitrary problem instance $\mI$. For this problem instance:
\begin{align}\label{LB:eq:thm:LB-log-t:LB}
\text{There exists time $t_0$ such that for any $t\geq t_0$}\quad
\E[R(t)]\geq C_\mI\ln(t),
 \end{align}
for some constant $C_\mI$ that depends on the problem instance, but not on time $t$.
 \end{theorem}

\begin{remark}
For example, Assumption \eqref{LB:eq:thm:LB-log-t:assn} is satisfied for any algorithm with $\E[R(t)]\leq (\log t)^{1000}$.
\end{remark}

Let us refine Theorem~\ref{LB:thm:LB-log-t} and specify how the instance-dependent constant $C_\mI$ in \eqref{LB:eq:thm:LB-log-t:LB} can be chosen. In what follows, $\Delta(a)=\mu^*-\mu(a)$ be the ``gap" of arm $a$.


\begin{theorem}\label{LB:thm:LB-log-refined}
For each problem instance $\mI$ and any algorithm that satisfies \eqref{LB:eq:thm:LB-log-t:assn},
\begin{OneLiners}
\item[(a)]the bound \eqref{LB:eq:thm:LB-log-t:LB} holds with
    \[ C_\mI = \sum_{a:\;\Delta(a)>0}\; \frac{\mu^*(1-\mu^*)}{\Delta(a)}. \]
\item[(b)] for each $\eps>0$, the bound \eqref{LB:eq:thm:LB-log-t:LB} holds with $C_\mI = C^0_\mI-\eps$, where
    \[ C^0_\mI =\sum_{a:\;\Delta(a)>0}\;\frac{\Delta(a)}{\KL(\mu(a),\,\mu^*)}-\eps.\]
\end{OneLiners}
\end{theorem}

\begin{remark}
The lower bound from part (a) is similar to the upper bound achieved by UCB1 and Successive Elimination:
    $R(T) \leq \sum_{a:\;\Delta(a)>0}\; \frac{O(\log T)}{\Delta(a)} $.
In particular, we see that the upper bound is optimal up to a constant factor when $\mu^*$ is bounded away from 0 and $1$, \eg when
    $\mu^*\in \sbr{\nicefrac14,\nicefrac34}$.
\end{remark}

\begin{remark}
Part (b) is a stronger (\ie larger) lower bound which implies the more familiar form in (a). Several algorithms in the literature are known to come arbitrarily close to this lower bound. In particular, a version of Thompson Sampling (another standard algorithm discussed in Chapter~\ref{ch:TS}) achieves regret
    \[ R(t) \leq (1+\delta)\,C^0_\mI\; \ln(t) + C'_\mI/\delta^2 , \quad \forall\delta>0,\]
where $C^0_\mI$ is from part (b) and $C'_\mI$ is some other instance-dependent constant.%
\end{remark}

\sectionBibNotes

The $\Omega(\sqrt{KT})$ lower bound on regret is from \cite{bandits-exp3}. KL-divergence and its properties is ``textbook material" from information theory, \eg see \cite{CoverThomas}. The outline and much of the technical details in the present exposition are based on the lecture notes from \citet{Bobby-class07}. That said, we present a substantially simpler proof, in which we replace the general ``chain rule" for KL-divergence with the special case of independent distributions (Theorem~\ref{LB:thm:KL-props}(b) in Section~\ref{LB:sec:KL-divergence}). This special case is much easier to formulate and apply, especially for those not deeply familiar with information theory. The proof of Lemma~\ref{LB:lm:MAB-with-predictions} for general $K$ is modified accordingly. In particular, we define the ``reduced" sample space $\Omega^*$ with only a small number of samples from the ``bad" arm $j$, and apply the KL-divergence argument to carefully defined events in \eqref{LB:eq:pf:predictions:Karms:events}, rather than a seemingly more natural event $A=\{y_T=j\}$.

Lower bounds for non-adaptive exploration have been folklore in the community. The first published version traces back to \citet{MechMAB-ec09}, to the best of our knowledge. They define a version of non-adaptive exploration and derive similar lower bounds as ours, but for a slightly different technical setting.

The logarithmic lower bound from Section~\ref{LB:sec:per-instance} is due to \cite{Lai-Robbins-85}. Its proof is also based on the KL-divergence technique. Apart from the original paper, it can can also be found in \citep{Bubeck-survey12}. Our exposition is more explicit on ``unwrapping" what this lower bound means.

While these two lower bounds essentially resolve the basic version of multi-armed bandits, they do not suffice for many other versions. Indeed, some bandit problems posit auxiliary constraints on the problem instances, such as Lipschitzness or linearity (see Section~\ref{IID:sec:fwd}), and the lower-bounding constructions need to respect these constraints. Typically such lower bounds do not depend on the number of actions (which may be very large or even infinite). In some other bandit problems the constraints are on the algorithm, \eg a limited inventory. Then much stronger lower bounds may be possible.


Therefore, a number of problem-specific lower bounds have been proved over the years. A representative, but likely incomplete, list is below:

\begin{itemize}
\item for dynamic pricing \citep{KleinbergL03,DynPricing-ec12} and Lipschitz bandits \citep{Bobby-nips04,contextualMAB-colt11,LipschitzMAB-JACM,SmoothedRegret-colt19};
     see Chapter~\ref{ch:Lip} for definitions and algorithms, and Section~\ref{Lip:sec:CAB-LB} for a (simple) lower bound and its proof.

\item for linear bandits \citep[\eg][]{DaniHK-nips07,DaniHK-colt08,Paat-mor10,Shamir-colt15}
and combinatorial (semi-)bandits \citep[\eg][]{Audibert-MOR14,Kveton-aistats15};
see Chapter~\ref{ch:lin} for definitions and algorithms.

\item for pay-per-click ad auctions \citep{MechMAB-ec09,DevanurK09}. Ad auctions are parameterized by click probabilities of ads, which are a priori unknown but can be learned over time by a bandit algorithm. The said algorithm is constrained to be compatible with advertisers' incentives.

\item for dynamic pricing with limited supply \citep{BZ09,DynPricing-ec12} and bandits with resource constraints \citep{BwK-focs13,AdvBwK-focs19,Karthik-BwK-2020}; see Chapter~\ref{ch:BwK} for definitions and algorithms.

\item for best-arm identification \citep[\eg][]{Kaufmann-jmlr2016,Carpentier-colt16}.

\end{itemize}

Some lower bounds in the literature are derived from first principles, like in this chapter, \eg the lower bounds in \citet{KleinbergL03,Bobby-nips04,MechMAB-ec09,BwK-focs13}. Some other lower bounds are derived by reduction to more basic ones
\citep[\eg the lower bounds in][and the one in Section~\ref{Lip:sec:CAB-LB}]{DynPricing-ec12,LipschitzMAB-JACM}.
The latter approach focuses on constructing the problem instances and side-steps the lengthy KL-divergence arguments.


\sectionExercises


\begin{exercise}[non-adaptive exploration]
\label{LB:ex:non-adaptive}
Prove Theorems~\ref{LB:thm:non-adaptive} and~\ref{LB:thm:non-adaptive-refined}. Specifically, consider an algorithm which satisfies non-adaptive exploration. Let $N$ be the number of exploration rounds. Then:
\begin{OneLiners}
\item[(a)] there is a problem instance such that
    $\E[R(T)] \geq \Omega(\;T\cdot \sqrt{K/\E[N]}\;)$.

\item[(b)] for each problem instance, randomly permuting the arms yields
        $\E[R(T)] \geq \E[N]\cdot\tfrac{1}{K}\sum_a \Delta(a)$.

\item[(c)] use parts (a,b) to derive Theorems~\ref{LB:thm:non-adaptive} and~\ref{LB:thm:non-adaptive-refined}.
\end{OneLiners}

\Hint{For part (a), start with a deterministic algorithm. Consider each round separately and invoke Corollary~\ref{LB:cor:MAB-with-predictions}. For a randomized algorithm, focus on event $\{ N\leq 2\,\E[N] \}$; its probability is at least $\nicefrac12$ by Markov inequality. Part (b) can be proved from first principles.

To derive Theorem~\ref{LB:thm:non-adaptive-refined}, use part (b) to lower-bound $\E[N]$, then apply part (a). Set $C = K^{1/3}$  in Theorem~\ref{LB:thm:non-adaptive-refined} to derive Theorems~\ref{LB:thm:non-adaptive}.}
\end{exercise}

%
%

\chapter{Bayesian Bandits and Thompson Sampling}
\label{ch:TS}
\begin{ChAbstract}
We introduce a Bayesian version of stochastic bandits, and discuss Thompson Sampling, an important algorithm for this version, known to perform well both in theory and in practice.
The exposition is self-contained, introducing concepts from Bayesian statistics as needed.

\prereqs{Chapter~\ref{ch:IID}.}
\end{ChAbstract}


The Bayesian bandit problem adds the \emph{Bayesian assumption} to stochastic bandits: the problem instance $\mI$ is drawn initially from some known distribution $\PP$. The time horizon $T$ and the number of arms $K$ are fixed. Then an instance of stochastic bandits is specified by the mean reward vector $\mu\in [0,1]^K$ and the reward distributions $(\mD_a:\, a\in [K])$. The distribution $\PP$ is called the \emph{prior distribution}, or the \emph{Bayesian prior}. The goal is to optimize \emph{Bayesian regret}: expected regret for a particular problem instance $\mI$, as defined in \eqref{IID:eq:pseudo-regret}, in expectation over the problem instances:
\begin{align}\label{TS:eq:BR-def}
 \BReg(T)
    := \E_{\mI\sim \PP}\sbr{ \E\sbr{ R(T) \mid  \mI }}
    = \textstyle \E_{\mI\sim \PP}\sbr{\mu^* \cdot T -  \sum_{t\in [T]} \mu(a_t)}.
\end{align}

Bayesian bandits follow a well-known approach from \emph{Bayesian statistics}: posit that the unknown quantity is sampled from a known distribution, and optimize in expectation over this distribution. Note that a ``worst-case" regret bound (an upper bound on $\E[R(T)]$ which holds for all problem instances) implies the same upper bound on Bayesian regret.

\xhdr{Simplifications.}
We make several assumptions to simplify presentation. First, the realized rewards come from a \emph{single-parameter family} of distributions. There is a family of real-valued distributions $(\mD_\nu$, $\nu\in[0,1])$, fixed and known to the algorithm, such that each distribution $\mD_\nu$ has expectation $\nu$. Typical examples are Bernoulli rewards and unit-variance Gaussians. The reward of each arm $a$ is drawn from distribution $\mD_{\mu(a)}$, where $\mu(a)\in [0,1]$ is the mean reward. We will keep the single-parameter family fixed and implicit in our notation. Then the problem instance is completely specified by the \emph{mean reward vector} $\mu\in [0,1]^K$, and the prior $\PP$ is simply a distribution over $[0,1]^K$ that $\mu$ is drawn from.

Second, unless specified otherwise, the realized rewards can only take finitely many different values, and the prior $\PP$ has a finite support, denoted $\mF$. Then we can focus on concepts and arguments essential to Thompson Sampling, rather than worry about the intricacies of integrals and probability densities. However, the definitions and lemmas stated below carry over to arbitrary priors and arbitrary reward distributions.

Third, the best arm $a^*$ is unique for each mean reward vector in the support of $\PP$. This is just for simplicity: this assumption can be easily removed at the cost of slightly more cumbersome notation.

\newpage
\section{Bayesian update in Bayesian bandits}

An essential operation in Bayesian statistics is \emph{Bayesian update}: updating the prior distribution given the new data. Let us discuss how this operation plays out for Bayesian bandits.

\subsection{Terminology and notation}

Fix round $t$. Algorithm's data from the first $t$ rounds is a sequence of action-reward pairs, called \emph{$t$-history}:
    \[ H_t = \rbr{(a_1,r_1) \LDOTS (a_t,r_t) } \in (\mA\times \R)^t.\]
It is a random variable which depends on the mean reward vector $\mu$, the algorithm, and the reward distributions (and the randomness in all three). A fixed sequence
\begin{align}\label{TS:eq:realized-history}
     H =  ((a'_1,r'_1) \LDOTS (a'_t,r'_t) ) \in (\mA\times \R)^t
\end{align}
is called a \emph{feasible $t$-history} if it satisfies
    $\Pr[H_t=H]>0$
for some bandit algorithm; call such algorithm \emph{$H$-consistent}. One such algorithm, called the \emph{$H$-induced algorithm}, deterministically chooses arm $a'_s$ in each round $s\in [t]$. Let $\mH_t$ be the set of all feasible $t$-histories; it is finite, because each reward can only take finitely many values. In particular,
    $ \mH_t = (\mA\times \{0,1\})^t $
for Bernoulli rewards and a prior $\PP$ such that
    $\Pr[\mu(a)\in (0,1)]=1$ for all arms $a$.

In what follows, fix a feasible $t$-history $H$.
We are interested in the conditional probability
\begin{align}\label{TS:eq:posterior-defn}
\PP_H(\mM) := \Pr\sbr{ \mu\in\mM \mid H_t = H },
    \qquad\forall \mM \subset [0,1]^K.
\end{align}
This expression is well-defined for the $H$-induced algorithm, and more generally for any $H$-consistent bandit algorithm. We interpret  $\PP_H$ as a distribution over $[0,1]^K$.

Reflecting the standard terminology in Bayesian statistics, $\PP_H$ is called the \emph{(Bayesian) posterior distribution} after round $t$. The process of deriving $\PP_H$ is called \emph{Bayesian update} of $\PP$ given $H$.

\subsection{Posterior does not depend on the algorithm}

A fundamental fact about Bayesian bandits is that distribution $\PP_H$ does not depend on which $H$-consistent bandit algorithm has collected the history. Thus, w.l.o.g. it is the $H$-induced algorithm.

\begin{lemma}\label{TS:lm:posterior}
Distribution $\PP_H$ is the same for all $H$-consistent bandit algorithms.
\end{lemma}

The proof takes a careful argument; in particular, it is essential that the algorithm's action probabilities are determined by the history, and reward distribution is determined by the chosen action.

\begin{proof}
It suffices to prove the lemma for a singleton set
    $\mM = \{\muR\}$,
for any given vector
    $\muR\in [0,1]^K$.
Thus, we are interested in the conditional probability of $\{ \mu=\muR \}$. Recall that the reward distribution with mean reward $\muR(a)$ places probability
    $\mD_{\muR(a)}(r)$
on every given value $r\in \R$.

Let us use induction on $t$. The base case is $t=0$. To make it well-defined, let us define the $0$-history as $H_0 = \emptyset $, so that $H=\emptyset$ to be the only feasible $0$-history. Then, all algorithms are $\emptyset$-consistent, and the conditional probability
    $\Pr[\mu = \muR \mid H_0=H]$
is simply the prior probability $\PP(\muR)$.

The main argument is the induction step. Consider round $t\geq 1$. Write $H$ as concatenation of some feasible $(t-1)$-history $H'$ and an action-reward pair $(a,r)$. Fix an $H$-consistent bandit algorithm, and let
    \[ \pi(a) = \Pr\sbr{a_t=a \mid H_{t-1} = H'} \]
be the probability that this algorithm assigns to each arm $a$ in round $t$ given the history $H'$. Note that this probability does not depend on the mean reward vector $\mu$.

\begin{align*}
 \frac{\Pr[\mu = \muR \eqAND H_t=H]}{\Pr[H_{t-1}=H']}
    &= \Pr\sbr{ \mu = \muR \eqAND (a_t, r_t) = (a,r) \mid H_{t-1}=H' }\\
    &= \PP_{H'}(\muR) \cdot \Pr[ (a_t, r_t) = (a,r) \mid \mu = \muR \eqAND H_{t-1}=H' ] \\
    &= \PP_{H'}(\muR) \\
       &\quad\Pr\sbr{ r_t=r  \mid a_t=a \eqAND \mu = \muR \eqAND H_{t-1}=H' } \\
        &\quad\Pr\sbr{ a_t=a \mid \mu = \muR \eqAND H_{t-1}=H' } \\
     &= \PP_{H'}(\muR) \cdot \mD_{\muR(a)}(r) \cdot \pi(a).
\end{align*}
Therefore,
\begin{align*}
\Pr[H_t = H]
    =  \pi(a) \cdot \Pr[H_{t-1}=H'] \;
        \sum_{\muR \in \mF}
         \PP_{H'}(\muR) \cdot \mD_{\muR(a)}(r).
\end{align*}
It follows that
\begin{align*}
\PP_H(\muR)
    = \frac{\Pr[\mu = \muR \eqAND H_t=H]}{\Pr[H_t=H]}
    = \frac{\PP_{H'}(\muR) \cdot \mD_{\muR(a)}(r)}
        {\sum_{\muR\in \mF}\PP_{H'}(\muR) \cdot \mD_{\muR(a)}(r)}.
\end{align*}
By the induction hypothesis, the posterior distribution
    $\PP_{H'}$
does not depend on the algorithm. So, the expression above does not depend on the algorithm, either.
\end{proof}

It follows that $\PP_H$ stays the same if the rounds are permuted:

\begin{corollary}\label{TS:cor:posterior}
$\PP_H = \PP_{H'}$ whenever
    $ H' = \left( \left( a'_{\sigma(t)},\, r'_{\sigma(t)}\right) :\; t\in[T] \right)$
for some permutation $\sigma$ of $[t]$.
\end{corollary}

\begin{remark}
Lemma~\ref{TS:lm:posterior} should not be taken for granted. Indeed, there are two very natural extensions of Bayesian update for which this lemma does \emph{not} hold. First, suppose we condition on an arbitrary observable event. That is, fix a set $\mH$ of feasible $t$-histories. For any algorithm with $\Pr[H_t \in \mH]>0$, consider  the posterior distribution given the event $\{H_t\in \mH\}$:
\begin{align}\label{TS:eq:update-example-1}
    \Pr[\mu \in \mM \mid H_t\in \mH],
        \quad \forall \mM\subset [0,1]^K.
\end{align}
This distribution may depend on the bandit algorithm. For a simple example, consider a problem instance with Bernoulli rewards, three arms $\mA  =\{a,a',a''\}$, and a single round. Say $\mH$ consists of two feasible $1$-histories, $H = (a,1)$ and $H' = (a',1)$. Two algorithms, $\ALG$ and $\ALG'$, deterministically choose arms $a$ and $a'$, respectively. Then the distribution \eqref{TS:eq:update-example-1} equals $\PP_H$ under $\ALG$, and $\PP_{H'}$ under $\ALG'$.

Second, suppose we condition on a \emph{subset} of rounds. The algorithm's history for a subset $S \subset [T]$ of rounds, called \emph{$S$-history}, is an ordered tuple
\begin{align}\label{TS:eq:HS}
    H_S = ((a_t,r_t):\, t\in S) \in (\mA\times \R)^{|S|}.
\end{align}
For any feasible $|S|$-history $H$, the posterior distribution  given the event $\{H_S =  H\}$, denoted $\PP_{H,S}$, is
\begin{align}\label{TS:eq:update-example-2}
    \PP_{H,S}(\mM) :=\Pr[\mu \in \mM \mid H_S =  H],
        \quad \forall \mM\subset [0,1]^K.
\end{align}
However, this distribution may depend on the bandit algorithm, too. Consider a problem instance with Bernoulli rewards, two arms $\mA = \{a,a'\}$, and two rounds. Let $S = \{2\}$ (\ie we only condition on what happens in the second round), and $H = (a,1)$. Consider two algorithms, $\ALG$ and $\ALG'$, which choose different arms in the first round (say, $a$ for $\ALG$ and $a'$ for $\ALG'$), and choose arm $a$ in the second round if and only if they receive a reward of $1$ in the first round.  Then the distribution \eqref{TS:eq:update-example-2} additionally conditions on $H_1 = (a,1)$
under $\ALG$, and on on $H_1 = (a',1)$ under $\ALG'$.
\end{remark}

\subsection{Posterior as a new prior}

The posterior $\PP_H$ can be used as a prior for a subsequent Bayesian update. Consider a feasible $(t+t')$-history, for some $t'$. Represent it as a concatenation of $H$ and another feasible $t'$-history $H'$, where $H$ comes first. Denote such concatenation as $H\oplus H'$. Thus, we have two events:
\[ \{H_t = H\} \eqAND \{H_S = H'\} \eqWHERE S = [t+t'] \setminus [t]. \]
(Here $H_S$ follows the notation from \eqref{TS:eq:HS}.) We could perform the Bayesian update in two steps: (i) condition on $H$ and derive the posterior $\PP_H$, and (ii) condition on $H'$ using $\PP_H$ as the new prior. In our notation, the resulting posterior can be written compactly as
    $(\PP_H)_{H'}$.
We prove that the ``two-step" Bayesian update described above is equivalent to the "one-step`` update given $H\oplus H'$. In a formula,
    $\PP_{H\oplus H'} = (\PP_H)_{H'}$.

\begin{lemma}\label{TS:lm:oplus}
Let $H'$ be a feasible $t'$-history. Then
    $\PP_{H\oplus H'} = (\PP_H)_{H'}$.
More explicitly:
\begin{align}\label{TS:eq:lm-oplus}
\PP_{H\oplus H'}(\mM)
    = \Pr_{\mu \sim \PP_H} [\mu\in \mM \mid H_{t'} = H'],
    \quad\forall \mM \subset [0,1]^K.
\end{align}
\end{lemma}

One take-away is that $\PP_H$ encompasses all pertinent information from $H$, as far as mean rewards are concerned. In other words, once $\PP_H$ is computed, one can forget about $\PP$ and $H$ going forward.

The proof is a little subtle: it relies on the $(H\oplus H')$-induced algorithm for the main argument, and carefully applies Lemma~\ref{TS:lm:posterior} to extend to arbitrary bandit algorithms.

\begin{proof}
It suffices to prove the lemma for a singleton set
    $\mM = \{\muR\}$,
for any given vector
    $\muR\in \mF$.

Let $\ALG$ be the $(H\oplus H')$-induced algorithm. Let $H^{\ALG}_t$ denote its $t$-history, and let $H^{\ALG}_S$ denote its $S$-history, where
    $S = [t+t'] \setminus [t]$.
    We will prove that
\begin{align}\label{TS:eq:lm-oplus-pf-1}
\PP_{H\oplus H'}(\mu = \muR)
    = \Pr_{\mu \sim \PP_H} [\mu = \muR \mid H^{\ALG}_S = H'].
\end{align}

    We are interested in two events:
    \[ \mE_t = \{H^{\ALG}_t = H \}  \eqAND \mE_S = \{H^{\ALG}_S = H' \}.\]
Write $\QQ$ for $\Pr_{\mu \sim \PP_H}$ for brevity. We prove that
    $\QQ[\cdot]= \Pr[\,\cdot \mid \mE_t]$
for some events of interest. Formally,
\begin{align*}
\QQ[\mu = \muR]
    &= \PP[\mu = \muR \mid \mE_t]
    &\EqComment{by definition of $\PP_H$}\\
 \QQ[\mE_S  \mid \mu = \muR]
    &= \Pr[\mE_S  \mid \mu = \muR, \mE_t]
    &\EqComment{by definition of $\ALG$}\\
\QQ[\mE_S \eqAND \mu = \muR ]
    &= \QQ[\mu = \muR] \cdot \QQ[\mE_S  \mid \mu = \muR] \\
    &= \PP[\mu = \muR \mid \mE_t] \cdot \Pr[\mE_S  \mid \mu = \muR, \mE_t] \\
    &= \Pr[\mE_S \eqAND \mu = \muR \mid \mE_t ].
\end{align*}
Summing up over all $\muR\in \mF$, we obtain:
\begin{align*}
\textstyle \QQ[\mE_S]
    = \sum_{\muR\in\mF}\, \QQ[\mE_S \eqAND \mu = \muR ]
    = \sum_{\muR\in\mF}\, \Pr[\mE_S \eqAND \mu = \muR \mid \mE_t ]
    = \Pr[\mE_S \mid \mE_t].
\end{align*}
\noindent Now, the right-hand side of \eqref{TS:eq:lm-oplus-pf-1} is
\begin{align*}
\QQ[\mu = \muR \mid \mE_S]
    &= \frac{\QQ[\mu = \muR \eqAND \mE_S]}{\QQ[\mE_S]}
    = \frac{\Pr[\mE_S \eqAND \mu = \muR \mid \mE_t ]}
            {\Pr[\mE_S \mid \mE_t ]}
    = \frac{\Pr[\mE_t \eqAND \mE_S \eqAND \mu = \muR  ]}
            {\Pr[ \mE_t \eqAND \mE_S ]} \\
    &= \Pr[\mu = \muR \mid \mE_t \eqAND \mE_S].
\end{align*}
The latter equals $\PP_{H\oplus H'}(\mu = \muR)$, proving \eqref{TS:eq:lm-oplus-pf-1}.

It remains to switch from \ALG to an arbitrary bandit algorithm. We apply Lemma~\ref{TS:lm:posterior} twice, to both sides of \eqref{TS:eq:lm-oplus-pf-1}. The first application is simply that $\PP_{H\oplus H'}(\mu = \muR)$ does not depend on the bandit algorithm. The second application is for prior distribution $\PP_H$ and feasible $t'$-history $H'$. Let $\ALG'$ be the $H'$-induced algorithm $\ALG'$, and let $H^{\ALG'}_{t'}$ be its $t'$-history. Then
\begin{align}\label{TS:eq:lm-oplus-pf-2}
    \Pr_{\mu \sim \PP_H} [\mu = \muR \mid H^{\ALG}_S = H']
    =  \Pr_{\mu \sim \PP_H} [\mu = \muR \mid H^{\ALG'}_{t'} = H']
    = \Pr_{\mu \sim \PP_H} [\mu = \muR \mid H_{t'} = H'].
\end{align}
The second equality is by Lemma~\ref{TS:lm:posterior}. We defined $\ALG'$ to switch from the $S$-history to the $t'$-history. Thus, we've proved that the right-hand side of \eqref{TS:eq:lm-oplus-pf-2} equals
    $\PP_{H\oplus H'}(\mu = \muR)$
for an arbitrary bandit algorithm.
\end{proof}

\subsection{Independent priors}

Bayesian update simplifies for for independent priors: essentially, each arm can be updated separately. More formally, the prior $\PP$ is called \emph{independent} if $(\mu(a):a\in \mA)$ are mutually independent random variables.

Fix some feasible $t$-history $H$, as per
\eqref{TS:eq:realized-history}. Let
    $S_a = \{s\in [t]:\, a'_s=a\}$
be the subset of rounds when a  given arm $a$ is chosen, according to  $H$.
The portion of $H$ that concerns arm $a$ is defined as an ordered tuple
\[ \proj(H;a) = ( (a'_s,r'_s):\,s\in S_a ).\]
We think of $\proj(H;a)$ as a projection of $H$ onto arm $a$, and call it \emph{projected history} for arm $a$. Note that it is itself a feasible $|S_a|$-history. Define the posterior distribution $\PP^a_H$ for arm $a$:
\begin{align}\label{TS:eq:posterior-a}
 \PP^a_H(\mM_a)
    &:=  \PP_{\proj(H;a)}(\mu(a) \in \mM_a),
    \quad\forall \mM_a\subset  [0,1].
\end{align}
$\PP^a_H$ does not depend on the bandit algorithm, by Lemma~\ref{TS:lm:posterior}. Further, for the $H$-induced algorithm
\[
\PP^a_H(\mM_a) = \Pr[\mu(a) \in \mM_a \mid \proj(H_t;a) =\proj(H;a)].
\]

Now we are ready for a formal statement:

\begin{lemma}\label{TS:lm:independent}
Assume the prior $\PP$ is independent. Fix a subset $\mM_a \subset [0,1]$ for each arm $a\in \mA$. Then
\begin{align*}
\PP_H(\cap_{a\in\mA}\, \mM_a)
    &= \prod_{a\in\mA}\, \PP^a_H(\mM_a).
\end{align*}
\end{lemma}

\begin{proof}
The only subtlety is that we focus the $H$-induced bandit algorithm. Then the pairs
    $(\mu(a),\proj(H_t;a))$, $a\in\mA$
are mutually independent random variables. We are interested in these two events, for each arm $a$:
\begin{align*}
    \mE_a &= \{ \mu(a)\in \mM_a\} \\
    \mE_a^H &= \{ \proj(H_t;a) =\proj(H;a) \}.
\end{align*}
Letting
    $\mM = \cap_{a\in\mA}\, \mM_a$, we have:
\begin{align*}
\Pr[ H_t=H \eqAND \mu\in \mM ]
    &= \Pr\left[ \bigcap_{a\in\mA} (\mE_a \cap \mE_a^H) \right]
    = \prod_{a\in\mA} \Pr\left[ \mE_a \cap \mE_a^H \right].
\end{align*}
Likewise,
    $\Pr[ H_t=H ]= \prod_{a\in\mA} \Pr[ \mE_a^H ]$.
Putting this together,
\begin{align*}
\PP_H(\mM)
    &=\frac{\Pr[ H_t=H \eqAND \mu\in \mM ]}{\Pr[ H_t=H]}
    = \prod_{a\in\mA} \frac{\Pr[\mE_a \cap \mE_a^H]}{\Pr[\mE_a^H]} \\
    &= \prod_{a\in\mA} \Pr[\mu\in\mM \mid  \mE_a^H]
    = \prod_{a\in\mA} \PP_H^a(\mM_a). \qedhere
\end{align*}
\end{proof}

\OMIT{ 
Let $\mH_t$ denote the support of $H_t$, \ie set of all possible round-$t$ histories.

The prior $\PP$ induces a distribution over the sample space
    $ \Omega = \mF\times \mH_T$,
where $\mF$ is the set of all feasible mean reward vectors. This distribution is called the \emph{extended prior}. To keep the notation simple, we will also denote it with $\PP$. As per our assumptions, $\Omega$ is finite unless specified otherwise.  Throughout this chapter, we will work in the probability space $(\Omega,\PP)$.

\xhdr{Posterior distributions.}
Given a particular realization $H$ of the round-$t$ history $H_t$, one can define a conditional distribution $\PP_t$ over the mean reward vectors:
\[ \PP_t(\mu_0) := \PP[\mu = \mu_0 \mid H_t = H],
    \quad\forall \mu_0\in \mF,\, H\in\mH_t .\]
This distribution is called the \emph{posterior distribution} at time $t$.

Say we have a quantity $X$ determined by the mean reward vector $\mu$, such as the best arm $a^*$. We view $X$ as a random variable $\QQ$ whose distribution is induced by the prior $\PP$:
\[  \QQ(x) = \PP[X=x] \qquad\text{for all possible values $x$}.\]
$\QQ$ is called the \emph{prior distribution} for $X$. Likewise, we have the conditional distribution $\QQ_t$ induced by the posterior $\PP_t$, called \emph{posterior distribution} for $X$ at time $t$:
\[  \QQ_t(x) = \PP_t[X=x] \qquad\text{for all possible values $x$}.\]

\xhdr{Bayesian update.} The act of deriving the posterior $\PP_t$ given the prior $\PP$ and the realized history $H_t=H$ is called \emph{Bayesian update}. One crucial property of Bayesian update is that no matter which algorithm has collected a given realized history, the posterior stays the same.
} 

\section{Algorithm specification and implementation}

Consider a simple algorithm for Bayesian bandits, called \emph{Thompson Sampling}. For each round $t$ and arm $a$, the algorithm computes the posterior probability that $a$ is the best arm, and samples $a$ with this probability.

\vspace{2mm}
\LinesNotNumbered
\begin{algorithm}[H]
\SetAlgoLined
\For{each round $t=1,2, \ldots $}{
Observe $H_{t-1}=H$, for some feasible $(t-1)$-history $H$\;
Draw arm $a_t$ independently from distribution $p_t(\cdot\, |H)$, where
        \[ p_t(a \mid H) := \Pr[  a^*=a \mid H_{t-1}=H]  \quad \text{for each arm $a$}.\]
}
\caption{Thompson Sampling.}
\label{TS:alg:TS}
\end{algorithm}

\begin{remark}
The probabilities $p_t(\cdot \mid H)$ are determined by $H$, by Lemma~\ref{TS:lm:posterior}.
\end{remark}

\OMIT{ 
Consider the round-$t$ posterior $\PP_{H_t}$. Think of the posterior $\PP_H$ as a mapping from feasible $t$-histories $H$ to distributions over $\mu$. Then, $\PP_{H_t}$ is the distribution that a particular realization of $H_t$ maps to. For brevity, denote
    $\PP_t  = \PP_{H_t}$.
}

Thompson Sampling admits an alternative characterization:

\LinesNotNumbered
\begin{algorithm}[H]
\SetAlgoLined
\For{each round $t=1,2, \ldots $}{
Observe $H_{t-1}=H$, for some feasible $(t-1)$-history $H$\;
Sample mean reward vector $\mu_t$ from the posterior distribution
$\PP_H$\;
Choose the best arm $\tilde{a}_t$ according to $\mu_t$.
}
\caption{Thompson Sampling: alternative characterization.}
\label{TS:alg:TS-alt}
\end{algorithm}

\vspace{2mm}

It is easy to see that this characterization is in fact equivalent to the original algorithm.

\begin{lemma}
For each round $t$, arms $a_t$ and $\tilde{a}_t$ are identically distributed given $H_t$.
\end{lemma}

\begin{proof}
Fix a feasible $t$-history $H$.  For each arm $a$ we have:
\begin{align*}
\Pr[ \tilde{a}_t = a \mid H_{t-1} = H ]
    &= \PP_H(a^*=a)
        & \EqComment{by definition of $\tilde{a}_t$} \\
    & = p_t(a \mid H)
        & \EqComment{by definition of $p_t$}. \qquad\qquad\qedhere
\end{align*}
\end{proof}

The algorithm further simplifies when we have independent priors. By Lemma~\ref{TS:lm:independent}, it suffices to consider the posterior distribution $\PP^a_H$ for each arm $a$ separately, as per
\eqref{TS:eq:posterior-a}.

\LinesNotNumbered
\begin{algorithm}[H]
\SetAlgoLined
\For{each round $t=1,2, \ldots   $}{
Observe $H_{t-1}=H$, for some feasible $(t-1)$-history $H$\;
For each arm $a$, sample mean reward $\mu_t(a)$ independently from distribution $\PP_H^a$\;
Choose an arm with largest $\mu_t(a)$.
}
\caption{Thompson Sampling for independent priors.}
\label{TS:alg:TS-independent}
\end{algorithm}


\subsection{Computational aspects}

While Thompson Sampling is mathematically well-defined, it may be computationally inefficient. Indeed, let us consider a brute-force computation for the round-$t$ posterior $\PP_H$:
\begin{align}\label{TS:eq:brute-force}
\PP_H(\muR)
    = \frac{\Pr[ \mu =\muR \eqAND H_t=H ]}{\PP(H_t=H)}
    = \frac{\PP(\muR)\cdot \Pr[H_t=H \mid \mu =\muR] }
            {\sum_{\muR\in \mF}\PP(\muR)\cdot\Pr[H_t=H \mid \mu=\muR] },
            \quad\forall \muR \in \mF.
\end{align}
Since
    $\Pr[H_t=H \mid \mu=\muR]$
can be computed in time $O(t)$,%
\footnote{Here and elsewhere, we count addition and multiplication as unit-time operations.}
the probability $\PP(H_t=H)$ can be computed in time $O(t\cdot |\mF|)$. Then the posterior probabilities $\PP_H(\cdot)$ and the sampling probabilities
    $p_t(\cdot\mid H)$
can be computed by a scan through all $\muR\in \mF$. Thus, a brute-force implementation of Algorithm~\ref{TS:alg:TS} or Algorithm~\ref{TS:alg:TS-alt} takes at least
    $O(t\cdot |\mF|)$
running time in each round $t$, which may be prohibitively large.

A somewhat faster computation can be achieved via a \emph{sequential} Bayesian update. After each round $t$, we treat the posterior $\PP_H$ as a new prior. We perform a Bayesian update given the new data point
    $(a_t,r_t) = (a,r)$,
to compute the new posterior
    $\PP_{H\oplus (a,r)}$.
This approach is sound by Lemma~\ref{TS:lm:oplus}. The benefit in terms of the running time is that in each round the update is on the history of length $1$. In particular, with similar brute-force approach as in \eqref{TS:eq:brute-force}, the per-round running time improves to $O(|\mF|)$.


\begin{remark}
While we'd like to have both low regret and a computationally efficient implementation, either one of the two may be interesting: a slow algorithm can serve as a proof-of-concept that a given regret bound can be achieved, and a fast algorithm without provable regret bounds can still perform well in practice.
\end{remark}

With independent priors, one can do the sequential Bayesian update for each arm $a$ separately. More formally, fix round $t$, and suppose in this round $(a_t,r_t) = (a,r)$. One only needs to update the posterior for $\mu(a)$. Letting $H = H_t$ be the realized $t$-history, treat the current posterior $\PP_H^a$ as a new prior, and perform a Bayesian update to compute the new posterior
    $\PP_{H'}^a$, where $H' = H\oplus (a,r)$.
Then:
\begin{align}\label{TS:eq:brute-force-independent}
\PP_{H'}^a(x)
    = \Pr_{\mu(a)\sim \PP_H^a}[ \mu(a)=x \mid (a_t,r_t) = (a,r)]
    = \frac{\PP_H^a(x) \cdot \mD_{x}(r) }
            { \sum_{x\in \mF_a}\; \PP_H^a(x) \cdot \mD_{x}(r)},
            \quad\forall x \in \mF_a,
\end{align}
where $\mF_a$ is the support of $\mu(a)$. Thus, the new posterior $\PP_{H'}^a$ can be computed in time $O(|\mF_a|)$. This is an exponential speed-up compared $|\mF|$ (in a typical case when
    $|\mF| \approx \prod_a |\mF_a|$).

\xhdr{Special cases.}
Some special cases admit much faster computation of the posterior $\PP^a_H$ and much faster sampling therefrom. Here are two well-known special cases when this happens. For both cases, we relax the problem setting so that the mean rewards can take arbitrarily real values.

To simplify our notation, posit that there is only one arm $a$. Let $\PP$ be the prior on its mean reward $\mu(a)$. Let $H$ be a feasible $t$-history $H$, and let $\REW_H$ denote the total reward in $H$.

\begin{description}
\item[Beta-Bernoulli]
Assume Bernoulli rewards. By Corollary~\ref{TS:cor:posterior}, the posterior $\PP_H$ is determined by the prior $\PP$, the number of samples $t$, and the total reward $\REW_H$. Suppose the prior is the uniform distribution on the $[0,1]$ interval, denoted $\UU$. Then the posterior $\UU_H$ is traditionally called \emph{Beta distribution} with parameters
    $\alpha = 1+\REW_H$ and $\beta= 1+t$,
and denoted  $\BETA(\alpha,\beta)$. For consistency, $\BETA(1,1) = \UU$: if $t=0$, the posterior $\UU_H$ given the empty history $H$ is simply the prior $\UU$.

A \emph{Beta-Bernoulli conjugate pair} is a combination of Bernoulli rewards and a prior
    $\PP = \BETA(\alpha_0,\beta_0)$
for some parameters $\alpha_0,\beta_0\in \N$. The posterior $\PP_H$ is simply
    $\BETA(\alpha_0+\REW_H,\beta_0+t)$.
This is because
    $\PP = \UU_{H_0}$
for an appropriately chosen feasible history $H_0$, and
    $\PP_H  = \UU_{H_0\oplus H}$
by Lemma~\ref{TS:lm:oplus}.

(Corollary~\ref{TS:cor:posterior} and Lemma~\ref{TS:lm:oplus} extend to priors $\PP$ with infinite support, including Beta distributions.)

\item[Gaussians]
A \emph{Gaussian conjugate pair} is a combination of a Gaussian reward distribution and a Gaussian prior $\PP$. Letting $\mu,\mu_0$ be their resp. means, $\sigma,\sigma_0$ be their resp. standard deviations, the posterior $\PP_H$ is also a Gaussian whose mean and standard deviation are determined (via simple formulas) by the parameters $\mu,\mu_0,\sigma,\sigma_0$ and the summary statistics $\REW_H,t$ of $H$.
\end{description}

Beta distributions and Gaussians are well-understood. In particular, very fast algorithms exist to sample from either family of distributions.




\section{Bayesian regret analysis}
\label{TS:sec:Bayesian-regret}

Let us analyze Bayesian regret of Thompson Sampling, by connecting it to the upper and lower  confidence bounds studied in Chapter~\ref{ch:IID}. We prove:

\begin{theorem}\label{TS:thm:TS-UB}
Bayesian Regret of Thompson Sampling is
    $\BReg(T) = O(\sqrt{KT\log(T)})$.
\end{theorem}

Let us recap some of the definitions from Chapter~\ref{ch:IID}: for each arm $a$ and round $t$,
\begin{align}
        r_t(a) &= \sqrt{2\cdot \log(T)\,/\, n_t(a)}
        &\EqComment{confidence radius} \nonumber\\
        \UCB_t(a) &= \bar{\mu}_t(a) + r_t(a)
        &\EqComment{upper confidence bound}
        \label{TS:eq:conf}\\
        \LCB_t(a) &= \bar{\mu}_t(a) - r_t(a)
        &\EqComment{lower confidence bound}. \nonumber
\end{align}
 Here, $n_t(a)$ is the number of times arm $a$ has been played so far, and $\bar{\mu}_t(a)$ is the average reward from this arm. As we've seen before,
    $\mu(a) \in \sbr{\LCB_t(a),\> \UCB_t(a)}$
with high probability.

The key lemma in the proof of Theorem~\ref{TS:thm:TS-UB} holds for a more general notion of the confidence bounds, whereby they can be arbitrary functions of the arm $a$ and the $t$-history $H_t$: respectively,
        $U(a, \> H_t)$ and $L(a, \> H_t)$.
There are two properties we want these functions to have, for some $\gamma>0$ to be specified later:%
\footnote{As a matter of notation, $x^-$ is the negative portion of the number $x$, \ie $x^- =0$ if $x\geq 0$, and $x^- =|x|$ otherwise.}
\begin{align}
\E\sbr{ \sbr{U(a, \, H_t) - \mu(a)}^- }
    &\leq \tfrac{\gamma}{TK}
&\quad \text{for all arms $a$ and rounds $t$},
    \label{TS:eq:prop1} \\
\E\sbr{ \sbr{\mu(a) - L(a,\, H_t)}^- }
    &\leq \tfrac{\gamma}{TK}
    &\quad \text{for all arms $a$ and rounds $t$}.
\label{TS:eq:prop2}
\end{align}

\noindent The first property says that the upper confidence bound $U$ does not exceed the mean reward by too much \emph{in expectation}, and the second property makes a similar statement about $L$. As usual, $K$ denotes the number of arms. The confidence radius can be defined as
    $r(a, \> H_t) =  \frac{U(a, \> H_t) - L(a, \> H_t)} {2}$.

\begin{lemma}\label{TS:lm:TS-UB}
Assume we have lower and upper bound functions that satisfy
properties $\eqref{TS:eq:prop1}$ and $\eqref{TS:eq:prop2}$, for some parameter $\gamma>0$. Then Bayesian Regret of Thompson Sampling can be bounded as follows:
\begin{align*}
        \BReg(T) \leq \textstyle  2\gamma + 2\sum_{t=1}^{T} \E\sbr{ r(a_t, \> H_t)}.
        \end{align*}
\end{lemma}

\begin{proof}
Fix round $t$. Interpreted as random variables, the chosen arm $a_t$ and the best arm $a^*$ are identically distributed given $t$-history $H_t$: for each feasible $t$-history $H$,
\begin{align*}
\Pr[ a_t = a \mid H_t = H] = \Pr[a^*=a \mid H_t = H]
\quad \text{for each arm $a$}.
\end{align*}
It follows that
        \begin{equation}
        \label{TS:equal}
        \E[ \> U(a^*, \> H) \>  \mid  \> H_t=H \> ] = \E[ \> U(a_t, \> H) \>  \mid  \> H_t =H\> ].
        \end{equation}
Then Bayesian Regret suffered in round $t$ is
\begin{align*}
        \BReg_t
        &:= \E\sbr{\mu(a^*) - \mu(a_t)} \\
        &= \E_{H\sim H_t} \sbr{ \E\sbr{\mu(a^*) - \mu(a_t)  \mid  H_t = H}}   \\
        &= \E_{H\sim H_t} \sbr{ \E\sbr{ U(a_t, H) -\mu(a_t) + \mu(a^*) -U(a^*, H)  \mid H_t=H}}
        &\EqComment{by \refeq{TS:equal}} \\
    &= \underbrace{\E\sbr{U(a_t, H_t) -\mu(a_t)}}_{\text{Summand 1}} +
        \underbrace{\E\sbr{\mu(a^*) -U(a^*, H_t)}}_{\text{Summand 2}}.
\end{align*}

We will use properties \eqref{TS:eq:prop1} and \eqref{TS:eq:prop2} to bound both summands. Note that we cannot \emph{immediately} use these properties because they assume a fixed arm $a$, whereas both $a_t$ and $a^*$ are random variables.

\begin{align*}
\E\sbr{\mu(a^*) - U(a^*, H_t)}
    &&\EqComment{Summand 1} \\
    & \leq \E\sbr{(\mu(a^*) - U(a^*, H_t))^+} \\
        &\leq  \E\sbr{ \sum_{\text{arms $a$}} \sbr{\mu(a) - U(a, H_t)}^+} \\
        &= \textstyle  \sum_{\text{arms $a$}} \E\sbr{\rbr{U(a,H_t) - \mu(a)}^-}  \\
        & \leq K \cdot \frac{\gamma}{KT} =\frac{\gamma}{T}
        &\EqComment{by property \eqref{TS:eq:prop1}}. \\
\E\sbr{U(a_t, H_t) -\mu(a_t)}
    &&\EqComment{Summand 2} \\
    & = \E\sbr{2r(a_t, H_t) + L(a_t, H_t) - \mu(a_t)}
        &\EqComment{by definition of $r_t(\cdot)$} \\
        &= \E\sbr{2r(a_t, H_t)} + \E\sbr{L(a_t, H_t) - \mu(a_t)} \\
\E [L(a_t, H_t) - \mu(a_t)]
        & \leq \E\sbr{ \rbr{L(a_t, H_t) - \mu(a_t)}^+} \\
        &\leq \E\sbr{ \sum_{\text{arms $a$}} \rbr{L(a, H_t) - \mu(a)}^+} \\
        &= \underset{\text{arms $a$}}{\sum} \E\sbr{\rbr{\mu(a) - L(a,H_t)}^-} \\
        & \leq K \cdot \frac{\gamma}{KT} = \frac{\gamma}{T}
        &\EqComment{by property \eqref{TS:eq:prop2}}.
        \end{align*}

\noindent Thus,
    $\BReg_t(T) \leq 2\tfrac{\gamma}{T} + 2\,\E[r(a_t, H_t)]$.
The theorem follows by summing up over all rounds $t$.
\end{proof}


\begin{remark}
Thompson Sampling does not need to know what $U$ and $L$ are!
\end{remark}

\begin{remark}
Lemma~\ref{TS:lm:TS-UB} does not rely on any specific structure of the prior. Moreover, it can be used to upper-bound Bayesian regret of Thompson Sampling for a particular \emph{class} of priors whenever one has ``nice" confidence bounds $U$ and $L$ for this class.
\end{remark}

\begin{proof}[Proof of Theorem~\ref{TS:thm:TS-UB}]
Let us use the confidence bounds and the confidence radius from \eqref{TS:eq:conf}. Note that they satisfy properties $\eqref{TS:eq:prop1}$ and $\eqref{TS:eq:prop2}$ with $\gamma =2$. By Lemma~\ref{TS:lm:TS-UB},
        \begin{align*}
        \BReg(T) \leq O\rbr{\sqrt{\log T}} \sum_{t=1}^T
    \E\left[\; \frac{1}{\sqrt{n_t(a_t)}} \;\right].
        \end{align*}
Moreover,
        \begin{align*}
        \sum_{t=1}^T \sqrt{\frac{1}{n_t(a_t)}}
&= \sum_{\text{\text{arms $a$}}} \quad
    \sum_{\text{rounds $t$:\; $a_t = a$}}\;\;  \frac{1}{\sqrt{n_t(a)}} \\
         & = \sum_\text{arms $a$}\;\; \sum_{j=1}^{n_{T+1}(a)} \frac{1}{\sqrt{j}}
           = \underset{\text{arms $a$}}{\sum} O(\sqrt{n(a)}).
        \end{align*}
It follows that
\begin{align*}
        \BReg(T) & \leq O\rbr{\sqrt{\log T}} \underset{\text{arms $a$}}{\sum} \sqrt{n(a)}
        \leq O\rbr{\sqrt{\log T}} \sqrt{K\underset{\text{arms $a$}}{\sum}n(a)} =O\rbr{\sqrt{KT\log T}},
        \end{align*}
where the intermediate step is by the arithmetic vs. quadratic mean inequality.
\end{proof}

\section{Thompson Sampling with no prior (and no proofs)}
\label{TS:sec:TS-priorIndependent}

Thompson Sampling can also be used for the original problem of stochastic bandits, \ie the problem without a built-in prior $\PP$. Then $\PP$ is a ``fake prior": just a parameter to the algorithm, rather than a feature of reality. Prior work considered two such ``fake priors":
\begin{OneLiners}
\item[(i)] independent uniform priors and 0-1 rewards.%

\item[(ii)] independent standard Gaussian priors and unit-variance Gaussian rewards.
\end{OneLiners}

\begin{remark}
Under both approaches, the prior specifies the ``shape" of the reward distributions: resp., Bernoulli and unit-variance Gaussian. However, this prior is just a parameter in the algorithm, and does \emph{not} imply an assumption on the actual reward distributions. In particular, whenever some reward $r_t\not\in \cbr{0,1}$ is received under approach (i), one flips a random coin with expectation $r_t$, and pass the outcome of this coin flip to Thompson Sampling (so that the input is consistent with the prior). Approach (ii) treats the realized rewards \emph{as if} they are generated by a unit-variance Gaussian.
\end{remark}

\noindent Let us state the regret bounds for both approaches.

\begin{theorem}\label{TS:thm:TS-priorFree-rootT}
Thompson Sampling, with approach (i) or (ii), achieves expected regret
 \[ \E[R(T)] \leq O(\sqrt{KT\log{T}}).\]
\end{theorem}

\begin{theorem}\label{TS:thm:TS-priorFree-logT}
For each problem instance, Thompson Sampling with approach (i) achieves, for all $\eps>0$,
\begin{align*}
\E[R(T)] \leq (1+ \eps)\,C\,\log(T) + \frac{f(\mu)}{\eps^2},
\quad\text{where}\quad
C=\sum_{\text{arms $a$}:\, \Delta(a)<0}\; \frac{\mu(a^*) - \mu(a)}{\KL(\mu(a), \mu^*)}.
\end{align*}
Here $f(\mu)$ depends on the mean reward vector $\mu$, but not on $\eps$ or $T$.
\end{theorem}

The $C$ term is the optimal constant in the regret bounds: it matches the constant in Theorem~\ref{LB:thm:LB-log-refined}. This is a partial explanation for why Thompson Sampling is so good in practice. However, this is not a \emph{full} explanation because the term $f(\mu)$ can be quite big, as far as one can prove.

\sectionBibNotes

Thompson Sampling is the first bandit algorithm in the literature \citep{Thompson-1933}. While it has been well-known for a long time, strong provable guarantees did not appear until recently.  A detailed survey for various variants and developments can be found in \citet{TS-survey-FTML18}.

The material in Section~\ref{TS:sec:Bayesian-regret} is from \citet{Russo-MathOR-14}.%
\footnote{Lemma~\ref{TS:lm:TS-UB} follows from \citet{Russo-MathOR-14}, making their technique more transparent.} They refine this approach to obtain improved upper bounds for some specific classes of priors, including priors over linear and ``generalized linear" mean reward vectors, and priors given by a Gaussian Process. \citet{bubeck2013prior} obtain $O(\sqrt{KT})$ regret for arbitrary priors, shaving off the $\log(T)$ factor from Theorem~\ref{TS:thm:TS-UB}. \citet{Russo-JMLR-16} obtain regret bounds that scale with the entropy of the optimal-action distribution induced by the prior.

The prior-independent results in Section~\ref{TS:sec:TS-priorIndependent} are from \citep{Shipra-colt12,Shipra-aistats13,Shipra-aistats13-JACM} and \citet{Kaufmann-alt12}. The first ``prior-independent" regret bound for Thompson Sampling, a weaker version of Theorem~\ref{TS:thm:TS-priorFree-logT}, has appeared in \citet{Shipra-colt12}.
Theorem~\ref{TS:thm:TS-priorFree-rootT} is from \citet{Shipra-aistats13,Shipra-aistats13-JACM}.%
\footnote{For standard-Gaussian priors, \citet{Shipra-aistats13,Shipra-aistats13-JACM} achieve a slightly stronger version, $O(\sqrt{KT\log K})$.}
They also prove a matching lower bound for the Bayesian regret of Thompson Sampling for standard-Gaussian priors.
Theorem~\ref{TS:thm:TS-priorFree-logT} is from \citep{Kaufmann-alt12,Shipra-aistats13,Shipra-aistats13-JACM}.%
\footnote{\citet{Kaufmann-alt12} prove a slightly weaker version in which $\ln(T)$ is replaced with $\ln(T)+\ln \ln(T)$.}

\citet{Bubeck-colt15,TorTS-nips19,TSfirstorder} extend Thompson Sampling to \emph{adversarial bandits} (which are discussed in Chapter~\ref{ch:adv}). In particular, variants of the algorithm have been used in some of the recent state-of-art results on bandits \citep{Bubeck-colt15,TorTS-nips19},
building on the analysis technique from \citet{Russo-JMLR-16}.

\chapter{Bandits with Similarity Information}
\label{ch:Lip}
\begin{ChAbstract}
We consider stochastic bandit problems in which an algorithm has auxiliary information on similarity between arms. We focus on \emph{Lipschitz bandits}, where the similarity information is summarized via a Lipschitz constraint on the expected rewards. Unlike the basic model in Chapter~\ref{ch:IID}, Lipschitz bandits remain tractable even if the number of arms is very large or infinite.

\prereqs{Chapter~\ref{ch:IID}; Chapter~\ref{ch:LB} (results and construction only).}
\end{ChAbstract}

A bandit algorithm may have auxiliary information on similarity between arms, so that ``similar" arms have similar expected rewards. For example, arms can correspond to ``items" (\eg documents) with feature vectors, and similarity between arms can be expressed as (some version of) distance between the feature vectors. Another example is dynamic pricing and similar problems, where arms correspond to offered prices for buying, selling or hiring; then similarity between arms is simply the difference between prices. In our third example, arms are configurations of a complex system, such as a server or an ad auction.

We capture these examples via an abstract model called \emph{Lipschitz bandits}.%
\footnote{We discuss some non-Lipschitz models of similarity in the literature review. In particular, the basic versions of dynamic pricing and similar problems naturally satisfy a 1-sided version of Lipschitzness, which suffices for our purposes.}
 We consider stochastic bandits, as defined in Chapter~\ref{ch:IID}. Specifically, the reward for each arm $x$ is an independent sample from some fixed but unknown distribution whose expectation is denoted $\mu(x)$ and realizations lie in $[0,1]$. This  basic model is endowed with auxiliary structure which expresses similarity. In the paradigmatic special case, called \emph{continuum-armed bandits}, arms correspond to points in the interval $X = [0,1]$, and expected rewards obey a Lipschitz condition:
\begin{align}\label{Lip:eq:Lipschitz}
|\mu(x) - \mu(y) | \leq L\cdot |x - y| \quad
\text{for any two arms $x,y\in X$},
\end{align}
where $L$ is a constant known to the algorithm.%
\footnote{A function $\mu:X\to \R$ which satisfies \eqref{Lip:eq:Lipschitz} is called \emph{Lipschitz-continuous} on $X$, with \emph{Lipschitz constant} $L$.}
In the general case, arms lie in an arbitrary metric space $(X,\mD)$ which is known to the algorithm, and the right-hand side in \eqref{Lip:eq:Lipschitz} is replaced with $\mD(x,y)$.

The technical material is organized as follows. We start with fundamental results on continuum-armed bandits. Then we present the general case of Lipschitz bandits and recover the same results; along the way, we present sufficient background on metric spaces. We proceed to develop a more advanced algorithm which takes advantage of ``nice" problem instances. Many auxiliary results are deferred to exercises, particularly those on lower bounds, metric dimensions, and dynamic pricing.

\section{Continuum-armed bandits}
\label{Lip:sec:CAB}

To recap, continuum-armed bandits (\emph{CAB}) is a version of stochastic bandits where the set of arms is the interval $X = [0,1]$ and mean rewards satisfy \eqref{Lip:eq:Lipschitz} with some known Lipschitz constant $L$. Note that we have infinitely many arms, and in fact, \emph{continuously} many. While bandit problems with a very large, let alone infinite, number of arms are hopeless in general, CAB is tractable due to the Lipschitz condition.

A problem instance is specified by reward distributions, time horizon $T$, and Lipschitz constant $L$. As usual, we are mainly interested in the mean rewards $\mu(\cdot)$ as far as reward distributions are concerned, while the exact shape thereof is mostly unimportant.%
\footnote{However, recall that some reward distributions allow for smaller confidence radii, \eg see
Exercise~\ref{IID:ex:small-interval}.}

\subsection{Simple solution: fixed discretization}
\label{Lip:sec:CAB-UB}

A simple but powerful technique, called \emph{fixed discretization}, works as follows. We pick a fixed, finite set of arms $S \subset X$, called \emph{discretization} of $X$, and use this set as an approximation for $X$. Then we focus only on arms in $S$, and run an off-the-shelf algorithm \ALG for stochastic bandits, such as \UcbOne or Successive Elimination, that only considers these arms. Adding more points to $S$ makes it a better approximation of $X$, but also increases regret of $\ALG$ on $S$. Thus, $S$ should be chosen so as to optimize this tradeoff.

The best arm in $S$ is denoted $\mu^*(S) = \sup_{x\in S} \mu(x)$. In each round, algorithm \ALG can only hope to approach expected reward $\mu^*(S)$, and additionally suffers \emph{discretization error}
\begin{align}
    \DE(S) = \mu^*(X)-\mu^*(S).
\end{align}
More formally, we can represent the algorithm's expected regret as
\begin{align*}
\E[R(T)]
    &:= T\cdot \mu^*(X)- W(\ALG) \\
    &= \rbr{ T\cdot \mu^*(S) - W(\ALG) } +
      T\cdot \rbr{ \mu^*(X)-\mu^*(S) } \\
    &= \E[R_S(T)] + T\cdot \DE(S),
\end{align*}
where $W(\ALG)$ is the expected total reward of the algorithm, and $R_S(T)$ is the regret relative to $\mu^*(S)$.

Let us assume that \ALG attains the near-optimal regret rate from Chapter~\ref{ch:IID}:
\begin{align}\label{Lip:eq:opt-regret}
\E[R_S(T)]\leq \algC\cdot\sqrt{|S|\,T\log T}
\quad\text{for any subset of arms $S\subset X$},
\end{align}
where $\algC$ is an absolute constant specific to \ALG. Then
\begin{align}\label{Lip:eq:discretization-generic}
\E[R(T)] \leq \algC\cdot \sqrt{|S|\, T \log T} + T\cdot \DE(S).
\end{align}
This is a concrete expression for the tradeoff between the size of $S$ and its discretization error.

\emph{Uniform discretization} divides the interval $[0, 1]$ into intervals of fixed length $\eps$, called \emph{discretization step}, so that $S$ consists of all integer multiples of $\eps$. It is easy to see that $\DE(S)\leq L\eps$. Indeed, if $x^*$ is a best arm on $X$, and $y$ is the closest arm to $x^*$ that lies in $S$, it follows that $|x^*-y|\leq \eps$, and therefore $\mu(x^*)-\mu(y)\leq L\eps$. To optimize regret, we approximately equalize the two summands in \eqref{Lip:eq:discretization-generic}.

\begin{theorem}\label{Lip:thm:CAB-uniform}
Consider continuum-armed bandits with Lipschitz constant $L$ and time horizon $T$.
Uniform discretization with algorithm \ALG satisfying \eqref{Lip:eq:opt-regret} and discretization step
    $\eps=(TL^2/\log T)^{-1/3}$
attains
    \[ \E[R(T)] \leq L^{1/3}\cdot T^{2/3} \cdot (1+\algC)(\log T)^{1/3}. \]
\end{theorem}

The main take-away here is the
    $\tildeO\rbr{L^{1/3}\cdot T^{2/3}}$
regret rate. The explicit constant and logarithmic dependence are less important.

\subsection{Lower Bound}
\label{Lip:sec:CAB-LB}

Uniform discretization is in fact optimal in the worst case: we have an
    $\Omega\rbr{L^{1/3}\cdot T^{2/3}}$
lower bound on regret. We prove this lower bound via a relatively simple reduction from the main lower bound, the $\Omega(\sqrt{KT})$ lower bound from Chapter~\ref{ch:LB}, henceforth called \mainLB.

The new lower bound involves problem instances with 0-1 rewards and the following structure. There is a unique best arm $x^*$ with $\mu(x^*) = \nicefrac12 + \eps$, where $\eps>0$ is a parameter to be adjusted later in the analysis. All arms  $x$ have mean reward $\mu(x) = \nicefrac12$ except those near $x^*$. The Lipschitz condition requires a smooth transition between $x^*$ and the faraway arms; hence, we will have a ``bump" around $x^*$.

\begin{figure}[h]
\begin{center}
\includegraphics[height=4cm]{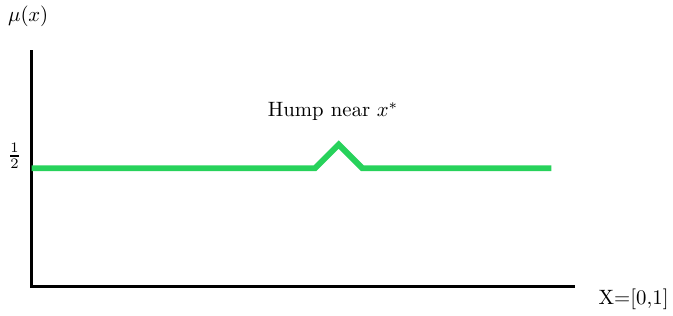}
\end{center}
\end{figure}

Formally, we define a problem instance $\mI(x^*,\eps)$ by
\begin{align}\label{Lip:eq:bump-function}
\mu(x) = \begin{cases}
                     \nicefrac12,& \text{all arms $x$ with } |x - x^*|\geq \eps/L\\
                     \nicefrac12 + \eps - L\cdot |x-x^*|,&\text{otherwise}.
     \end{cases}
\end{align}
It is easy to see that any such problem instance satisfies the Lipschitz Condition \eqref{Lip:eq:Lipschitz}. We will refer to $\mu(\cdot)$ as the \emph{bump function}. We are ready to state the lower bound:

\begin{theorem}
\label{Lip:thm:LB-CAB}
Let \ALG be any algorithm for continuum-armed bandits with time horizon $T$ and Lipschitz constant $L$. There exists a problem instance $\mI=\mI(x^*,\eps)$, for some $x^*\in[0,1]$ and $\eps>0$, such that
\begin{align}\label{Lip:eq:CAB-LB}
 \E\sbr{ R(T)\mid\mI} \geq \Omega \rbr{L^{1/3}\cdot T^{2/3}}.
\end{align}
\end{theorem}

For simplicity of exposition, assume the Lipschitz constant is $L=1$; arbitrary $L$ is treated similarly.

Fix $K\in \N$ and partition arms into $K$ disjoint intervals of length $\tfrac{1}{K}$. Use bump functions with $\eps=\tfrac{1}{2K}$ such that each interval either contains a bump or is completely flat. More formally, we use problem instances $\mI(x^*,\eps)$ indexed by $a^*\in[K] := \{1 \LDOTS K\}$, where
the best arm is $x^* = (2\eps-1)\cdot a^*+\eps$.

The intuition for the proof is as follows. We have these $K$ intervals defined above. Whenever an algorithm chooses an arm $x$ in one such interval, choosing the \emph{center} of this interval is best: either this interval contains a bump and the center is the best arm, or all arms in this interval have the same mean reward $\nicefrac12$. But if we restrict to arms that are centers of the intervals, we have a family of problem instances of $K$-armed bandits, where all arms have mean reward $\nicefrac12$ except one with mean reward $\nicefrac12+\eps$. This is precisely the family of instances from \mainLB. Therefore, one can apply the lower bound from \mainLB, and tune the parameters to obtain \eqref{Lip:eq:CAB-LB}. To turn this intuition into a proof, the main obstacle is to prove that choosing the center of an interval is really the best option. While this is a trivial statement for the immediate round, we need to argue carefullt that choosing an arm elsewhere would not be advantageous later on.

Let us recap \mainLB in a way that is convenient for this proof. Recall that \mainLB considered problem instances with 0-1 rewards such that all arms $a$ have mean reward $\nicefrac12$, except the best arm $a^*$ whose mean reward is $\nicefrac12+\eps$. Each instance is parameterized by best arm $a^*$ and $\eps>0$, and denoted $\mJ(a^*,\eps)$.

\begin{theorem}[\mainLB]
\label{Lip:thm:mainLB}
Consider stochastic bandits, with $K$ arms and time horizon $T$ (for any $K,T$). Let \ALG be any algorithm for this problem. Pick any positive $\eps\leq \sqrt{cK/T}$, where $c$ is an absolute constant from Lemma~\ref{LB:lm:MAB-with-predictions}. Then there exists a problem instance $\mJ= \mJ(a^*,\eps)$, $a^*\in[K]$, such that
  \[\E \left[ R(T)\mid\mJ \right] \geq \Omega (\eps T).\]
\end{theorem}

To prove Theorem~\ref{Lip:thm:LB-CAB}, we show how to reduce the problem instances from \mainLB to CAB in a way that we can apply \mainLB and derive the claimed lower bound \eqref{Lip:eq:CAB-LB}. On a high level, our plan is as follows: (i) take any problem instance $\mJ$ from \mainLB and ``embed" it into CAB, and (ii) show that any algorithm for CAB will, in fact, need to solve $\mJ$, (iii) tune the parameters to derive \eqref{Lip:eq:CAB-LB}.

\begin{proof}[Proof of Theorem~\ref{Lip:thm:LB-CAB} ($L=1$)]
We use the problem instances $\mI(x^*, \eps)$ as described above. More precisely, we fix $K\in \N$, to be specified later, and let
    $\eps = \tfrac{1}{2K}$.
We index the instances by $a^*\in[K]$, so that
\[ x^* = f(a^*), \text{ where }
    f(a^*) := (2\eps-1)\cdot a^*+\eps .\]

We use problem instances $\mJ(a^*,\eps)$ from \mainLB, with $K$ arms and the same time horizon $T$. The set of arms in these instances is denoted $[K]$. Each instance $\mJ = \mJ(a^*,\eps)$ corresponds to an instance $\mI=\mI(x^*,\eps)$ of CAB with $x^*=f(a^*)$. In particular, each arm $a\in [K]$ in $\mJ$ corresponds to an arm $x=f(a)$ in $\mI$. We view $\mJ$ as a discrete version of $\mI$. In particular, we have $\mu_\mJ(a)=\mu(f(a))$, where $\mu(\cdot)$ is the reward function for $\mI$, and $\mu_\mJ(\cdot)$ is the reward function for $\mJ$.

Consider an execution of \ALG on problem instance $\mI$ of CAB, and use it to construct an algorithm $\ALG'$ which solves an instance $\mJ$ of $K$-armed bandits. Each round in algorithm $\ALG'$ proceeds as follows. \ALG is called and returns some arm $x\in[0,1]$. This arm falls into the interval
    $\left[ f(a) - \eps,\; f(a) + \eps \right)$
for some $a\in [K]$. Then algorithm $\ALG'$ chooses arm $a$. When $\ALG'$ receives reward $r$, it uses $r$ and $x$ to compute reward $r_x\in\{0,1\}$ such that
    $\E[r_x\mid x]=\mu(x)$,
and feed it back to \ALG. We summarize this in a table:
\begin{center}
\begin{tabular}{p{8cm}|p{8cm}}
  \hline
  \ALG for CAB instance $\mI$ & $\ALG'$ for $K$-armed bandits instance $\mJ$\\\hline
  chooses arm $x\in [0,1]$
  & \\
    & chooses arm $a\in [K]$ with
    $x\in \left[ f(a) - \eps,\; f(a) + \eps \right)$\\
  & receives reward $r\in\{0,1\}$ with mean $\mu_\mJ(a)$ \\
  receives reward $r_x\in\{0,1\}$ with mean $\mu(x)$ & \\\hline
\end{tabular}
\end{center}

To complete the specification of $\ALG'$, it remains to define reward $r_x$ so that
    $\E[r_x\mid x]=\mu(x)$,
and $r_x$ can be computed using information available to $\ALG'$ in a given round. In particular, the computation of $r_x$ cannot use $\mu_\mJ(a)$ or $\mu(x)$, since they are not know to $\ALG'$. We define $r_x$ as follows:
\begin{align}
r_x = \begin{cases}
    r & \text{with probability $p_x \in [0,1]$} \\
    X & \text{otherwise},
\end{cases}
\end{align}
where $X$ is an independent toss of a Bernoulli random variable with expectation $\nicefrac12$, and probability $p_x$ is to be specified later. Then
\begin{align*}
\E[r_x \mid x]
    &= p_x\cdot \mu_\mJ(a)+ (1-p_x)\cdot \nicefrac12 \\
    &= \nicefrac12 + \left( \mu_\mJ(a)-\nicefrac12 \right)\cdot p_x \\
    &= \begin{cases}
            \nicefrac12            & \text{if $x\neq x^*$} \\
            \nicefrac12+\eps p_x   & \text{if $x=x^*$}
        \end{cases} \\
    &=\mu(x)
\end{align*}
if we set $p_x=1 - \left| x-f(a)\right|/\eps$ so as to match the definition of $\mI(x^*,\eps)$ in \refeq{Lip:eq:bump-function}.

For each round $t$, let $x_t$ and $a_t$ be the arms chosen by \ALG and $\ALG'$, resp. Since $\mu_\mJ(a_t)\geq \nicefrac12$, we have
    \[ \mu(x_t) \leq \mu_\mJ(a_t). \]
It follows that the total expected reward of \ALG on instance $\mI$ cannot exceed that of $\ALG'$ on instance $\mJ$. Since the best arms in both problem instances have the same expected rewards $\nicefrac12+\eps$, it follows that
    \[ \E[R(T) \mid \mI] \geq \E[R'(T) \mid \mJ],\]
where $R(T)$ and $R'(T)$ denote regret of \ALG and $\ALG'$, respectively.

Recall that algorithm $\ALG'$ can solve any $K$-armed bandits instance of the form $\mJ=\mJ(a^*,\eps)$. Let us apply Theorem~\ref{Lip:thm:mainLB} to derive a lower bound on the regret of $\ALG'$. Specifically, let us fix $K=(T/4c)^{1/3}$, so as to ensure that $\eps\leq \sqrt{cK/T}$, as required in Theorem~\ref{Lip:thm:mainLB}. Then for some instance $\mJ=\mJ(a^*,\eps)$,
\[ \E[R'(T)\mid \mJ] \geq \Omega(\sqrt{\eps T}) = \Omega(T^{2/3})\]
Thus, taking the corresponding instance $\mI$ of CAB, we conclude that
    $\E[R(T)\mid \mI] \geq \Omega(T^{2/3})$.
\end{proof}

\section{Lipschitz bandits}

A useful interpretation of continuum-armed bandits is that arms lie in a known metric space $(X,\mD)$, where $X=[0,1]$ is a ground set and
    $\mD(x,y) = L\cdot |x-y|$
is the metric. In this section we extend this problem and the uniform discretization approach to arbitrary metric spaces.

The general problem, called \emph{Lipschitz bandits Problem}, is a stochastic bandit problem such that the expected rewards $\mu(\cdot)$ satisfy the Lipschitz condition relative to some known metric $\mD$ on the set $X$ of arms:
\begin{align}\label{Lip:eq:Lip-general}
    |\mu(x) - \mu(y)| \leq \mD(x,y) \quad \text{for any two arms $x,y$}.
\end{align}
The metric space $(X,\mD)$ can be arbitrary, as far as this problem formulation is concerned (see Section~\ref{Lip:sec:bg}). It represents an abstract notion of known similarity between arms. Note that w.l.o.g. $\mD(x,y)\leq 1$. The set of arms $X$ can be finite or infinite, this distinction is irrelevant to the essence of the problem. While some of the subsequent definitions are more intuitive for infinite $X$, we state them so that they are meaningful for finite $X$, too. A problem instance is specified by the metric space $(X,\mD)$, reward distributions, and the time horizon $T$; for reward distributions, we are mainly interested in mean rewards $\mu(\cdot)$.

That $\mD$ is a metric is without loss of generality, in the following sense. Suppose the algorithm is given constraints
    $|\mu(x) - \mu(y)| \leq \mD_0(x,y)$
for all arms $x\neq y$ and some numbers $\mD_0(x,y)\in (0,1]$. Then one could define $\mD$ as the shortest-paths closure of $\mD_0$, \ie
    $\mD(x,y) = \min \sum_i \mD_0(x_i,x_{i+1})$,
where the $\min$ is over all finite sequences of arms
    $\sigma = (x_1 \LDOTS x_{n_\sigma})$
which start with $x$ and end with $y$.

\subsection{Brief background on metric spaces}
\label{Lip:sec:bg}

A metric space is a pair $(X, \mD)$, where $X$ is a set (called the \emph{ground set}) and
$\mD$ is a \emph{metric} on $X$, \ie a function $\mD: X\times X\to \R$
which satisfies the following axioms:
\begin{align*}
\mD(x, y)&\geq 0 &\EqComment{non-negativity}, \\
\mD(x,y) &= 0 \Leftrightarrow  x = y  &\EqComment{identity of indiscernibles}, \\
\mD(x,y) &= \mD(y, x) &\EqComment{symmetry}, \\
\mD(x, z) &\leq \mD(x,y) + \mD(y, z)   &\EqComment{triangle inequality}.
\end{align*}
\noindent Intuitively, the metric represents ``distance" between the elements of $X$. For $Y\subset X$, the pair $(Y,\mD)$ is also a metric space, where, by a slight abuse of notation, $\mD$ denotes the same metric restricted $Y$. The \emph{diameter} of $Y$ is the maximal distance between any two points in $Y$, \ie
    $\sup_{x,y\in Y} \mD(x,y)$.
A metric space $(X,\mD)$ is called finite (resp., infinite) if so is $X$.

Some notable examples:
\begin{itemize}
\item $X=[0,1]^d$, $d\in\N$ and the metric is the $p$-norm, $p\geq 1$:
    \[ \textstyle \ell_p(x,y) = \|x-y \|_p := \rbr{ \sum_{i=1}^d (x_i-y_i)^p }^{1/p}. \]
    Most commonly are $\ell_1$-metric (a.k.a. Manhattan metric) and $\ell_2$-metric (a.k.a. Euclidean distance).

\item $X = [0,1]$ and the metric is
    $\mD(x,y) = |x-y|^{1/d}$, $d\geq 1$.
In a more succinct notation, this metric space is denoted
    \CABmetric.

 \item $X$ is the set of nodes of a graph, and $\mD$ is the \emph{shortest-path metric}: $\mD(x,y)$ is the length of a shortest path between nodes $x$ and $y$.

 \item $X$ is the set of leaves in a rooted tree, where each internal node $u$ is assigned a weight $w(u)>0$. The \emph{tree distance} between two leaves $x$ and $y$ is $w(u)$, where $u$ is the least common ancestor of $x$ and $y$. (Put differently, $u$ is the root of the smallest subtree containing both $x$ and $y$.) In particular, \emph{exponential tree distance} assigns weight $c^h$ to each node at depth $h$, for some constant $c\in(0,1)$.
\end{itemize}

\subsection{Uniform discretization}
\label{Lip:sec:LipMAB-fixed}

We generalize uniform discretization to arbitrary metric spaces as follows. We take an algorithm \ALG satisfying \eqref{Lip:eq:opt-regret}, as in Section~\ref{Lip:sec:CAB-UB}, and we run it on a fixed subset of arms $S\subset X$.

\begin{definition}\label{Lip:def:mesh}
A subset $S\subset X$ is called an \emph{\eps-mesh}, $\eps>0$, if every point $x\in X$ is within distance $\eps$ from $S$, in the sense that $\mD(x, y) \leq \eps$ for some $y\in S$.
\end{definition}

\noindent It is easy to see that the discretization error of an \eps-mesh is $\DE(S)\leq \eps$. The developments in Section~\ref{Lip:sec:CAB-UB} carry over word-by-word, up until \refeq{Lip:eq:discretization-generic}. We summarize these developments as follows.

\begin{theorem}\label{Lip:thm:unif-S}
Consider Lipschitz bandits with time horizon $T$. Optimizing over the choice of an $\eps$-mesh, uniform discretization with algorithm \ALG satisfying \eqref{Lip:eq:opt-regret} attains regret
\begin{align}\label{Lip:eq:thm:unif-S}
\E[R(T)]
    \leq \inf_{\eps>0,\; \text{\eps-mesh $S$}}\;
        \eps T + \algC\cdot \sqrt{|S|\, T\log T}.
\end{align}
\end{theorem}

\begin{remark}\label{Lip:rem:net}
How do we \emph{construct} a good $\eps$-mesh? For many examples the construction is trivial, \eg a uniform discretization in continuum-armed bandits or in the next example. For an arbitrary metric space and a given $\eps>0$, a standard construction starts with an empty set $S$, adds any arm that is within distance $>\eps$ from all arms in $S$, and stops when such arm does not exist. Now, consider different values of $\eps$ in exponentially decreasing order:
    $\eps = 2^{-i}$, $i=1,2,3,\,\ldots$.
For each such $\eps$, compute an $\eps$-mesh $S_\eps$ as specified above, and stop at the largest $\eps$ such that
    $\eps T \geq \algC\cdot \sqrt{|S_\eps|\,T\log T}$.
This $\eps$-mesh optimizes the right-hand side in \eqref{Lip:eq:thm:unif-S} up to a constant factor, see Exercise~\ref{Lip:ex:net}
\end{remark}

\begin{example}\label{Lip:ex:LipMAB-hypercube}
To characterize the regret of uniform discretization more compactly, suppose the metric space is $X = [0,1]^d$, $d\in\N$ under the $\ell_p$ metric, $p\geq 1$. Consider a subset $S\subset X$ that consists of all points whose coordinates are multiples of a given $\eps>0$. Then $|S|\leq \Cel{1/\eps}^d$ points, and its discretization error is $\DE(S)\leq c_{p,d}\cdot \eps $, where $c_{p,d}$ is a constant that depends only on $p$ and $d$; \eg $c_{p,d}=d$ for $p=1$. Plugging this into \eqref{Lip:eq:discretization-generic} and taking
    $\eps=(T\,/\,\log T)^{-1/(d+2)}$,
we obtain
\begin{align*}
\E[R(T)]
    \leq (1+\algC)\cdot T^{(d+1)/(d+2)}\; (c\log T)^{1/(d+2)}.
\end{align*}
\end{example}

The $\tildeO(T^{(d+1)/(d+2)})$ regret rate in Example~\ref{Lip:ex:LipMAB-hypercube} extends to arbitrary metric spaces via the notion of \emph{covering dimension}. In essence, it is the smallest $d$ such that each $\eps>0$ admits an $\eps$-mesh of size $O(\eps^{-d})$.

\begin{definition}\label{Lip:def:cover}
Fix a metric space $(X,\mD)$. An \emph{\eps-covering}, $\eps>0$ is a collection of subsets $X_i\subset X$ such that the diameter of each subset is at most $\eps$ and $X = \bigcup X_i$. The smallest number of subsets in an $\eps$-covering is called the \emph{covering number} and denoted $N_{\eps}(X)$. The \emph{covering dimension}, with multiplier $c>0$, is
\begin{align}\label{Lip:eq:def:cover}
\COV(X) = \inf_{d\geq 0}
    \cbr{ N_{\eps}(X)\leq c\cdot \eps^{-d}\quad \forall \eps> 0 }.
\end{align}
\end{definition}

\begin{remark}
Any $\eps$-covering gives rise to an $\eps$-mesh: indeed, just pick one arbitrary representative from each subset in the covering. Thus, there exists an $\eps$-mesh $S$ with $|S| = N_\eps(X)$. We define the covering numbers in terms of $\eps$-coverings (rather than more directly in terms of \eps-meshes) in order to ensure that $N_\eps(Y)\leq N_\eps(X)$ for any subset $Y\subset X$, and consequently $\COV(Y)\leq \COV(X)$. In words, covering numbers characterize complexity of a metric space, and a subset $Y\subset X$ cannot be more complex than $X$.
\end{remark}

In particular, the covering dimension in Example~\ref{Lip:ex:LipMAB-hypercube} is $d$, with a large enough multiplier $c$. Likewise, the covering dimension of \CABmetric is $d$, for any $d\geq 1$; note that it can be fractional.

\begin{remark}
Covering dimension is often stated without the multiplier:
    $ \COV[](X) = \inf_{c>0} \COV(X)$.
This version is meaningful (and sometimes more convenient) for infinite  metric spaces. However, for finite metric spaces it is trivially $0$, so we need to fix the multiplier for a meaningful definition. Furthermore, our definition allows for more precise regret bounds, both for finite and infinite metric spaces.
\end{remark}

To de-mystify the infimum in \refeq{Lip:eq:def:cover}, let us state a corollary:
if $d = \COV(X)$ then
\begin{align}\label{Lip:eq:cover-cor}
N_\eps(X) \leq c'\cdot \eps^{-d}
\qquad \forall c'>c,\,\eps>0.
\end{align}

Thus, we take an $\eps$-mesh $S$ of size
    $|S|=N_\eps(X)\leq O(c/\eps^d)$.
Taking
     $\eps=(T\,/\,\log T)^{-1/(d+2)}$,
like in Example~\ref{Lip:ex:LipMAB-hypercube}, and plugging this into into \refeq{Lip:eq:discretization-generic}, we obtain the following theorem:

\begin{theorem}\label{Lip:thm:LipMAB-unif}
Consider Lipschitz bandits on a metric space $(X,\mD)$, with time horizon $T$. Let \ALG be any algorithm satisfying \eqref{Lip:eq:opt-regret}. Fix any $c>0$, and let $d=\COV(X)$ be the covering dimension. Then there exists a subset $S\subset X$ such that running $\ALG$ on the set of arms $S$ yields regret
\begin{align*}
\E[R(T)]
    \leq (1+\algC)\cdot T^{(d+1)/(d+2)}\cdot (c\,\log T)^{1/(d+2)}.
\end{align*}
Specifically, $S$ can be any $\eps$-mesh of size
    $|S|\leq O(c/\eps^d)$,
where
    $\eps=(T\,/\,\log T)^{-1/(d+2)}$;
such $S$ exists.
\end{theorem}

The upper bounds in Theorem~\ref{Lip:thm:unif-S} and Theorem~\ref{Lip:thm:LipMAB-unif} are  the best possible in the worst case, up to $O(\log T)$ factors. The lower bounds follow from the same proof technique as Theorem~\ref{Lip:thm:LB-CAB}, see the exercises in Section~\ref{Lip:sec:ex-LB} for precise formulations and hints. A representative example is as follows:

\begin{theorem}
\label{Lip:thm:LB-CovDim}
Consider Lipschitz bandits on metric space
    \CABmetric, for any $d\geq 1$,
with time horizon $T$. For any algorithm, there exists a problem instance  exists a problem instance $\mI$ such that
\begin{align}\label{Lip:eq:LB-CovDim}
 \E\sbr{R(T)\mid\mI} \geq \Omega\rbr{T^{(d+1)/(d+2)}}.
\end{align}
\end{theorem}

\section{Adaptive discretization: the Zooming Algorithm}
\label{Lip:sec:zoom}

Despite the existence of a matching lower bound, the fixed discretization approach is wasteful. Observe that the discretization error of $S$ is at most the minimal distance between $S$ and the best arm $x^*$:
    \[ \DE(S)\leq \mD(S,x^*) := \min_{x\in S} \mD(x,x^*).\]
So it should help to decrease $|S|$ while keeping $\mD(S,x^*)$ constant. Thinking of arms in $S$ as ``probes" that the algorithm places in the metric space. If we know that $x^*$ lies in a particular ``region" of the metric space, then we do not need to place probes in other regions. Unfortunately, we do not know in advance where $x^*$ is, so we cannot optimize $S$ this way if $S$ needs to be chosen in advance.

\OMIT{\begin{enumerate}
\item Thinking of arms in $S$ as ``probes" that the algorithm places in the metric space, these ``probes" are placed uniformly in the metric space, and in the regret analysis there is a penalty associated with each probe. Whereas it is better to have less probes, as long as the best arm is ``covered''.
\item Similarity information is discarded once we select the mesh $S$.
\end{enumerate}
} 

However, an algorithm could approximately learn the mean rewards over time, and adjust the placement of the probes accordingly, making sure that one has more probes in more ``promising" regions of the metric space. This approach is called \emph{adaptive discretization}. Below we describe one implementation of this approach, called the \emph{zooming algorithm}. On a very high level, the idea is that we place more probes in regions that could produce better rewards, as far as the algorithm knowns, and less probes in regions which are known to yield only low rewards with high confidence. What can we hope to prove for this algorithm, given the existence of a matching lower bound for fixed discretization? The goal here is to attain the same worst-case regret as in Theorem~\ref{Lip:thm:LipMAB-unif}, but do ``better" on ``nice" problem instances. Quantifying what we mean by ``better" and ``nice" here is an important aspect of the overall research challenge.

The zooming algorithm will bring together three techniques: the ``UCB technique" from algorithm \UcbOne, the new technique of ``adaptive discretization", and the ``clean event" technique in the analysis.

\OMIT{
In what follows, we make two assumptions to simplify presentation: there is a known time horizon $T$, and the realized rewards can take only finitely many possible values. Both assumptions can be relaxed with a little more work.}

\subsection{Algorithm}

The algorithm maintains a set $S\subset X$ of ``active arms". In each round, some arms may be ``activated" according to the ``activation rule'', and one active arm is selected to according to the ``selection rule''. Once an arm is activated, it cannot be ``deactivated".  This is the whole algorithm: we just need to specify the activation rule and the selection rule.

\xhdr{Confidence radius/ball.}
Fix round $t$ and an arm $x$ that is active at this round.
Let $n_t(x)$ is the number of rounds before round $t$ when this arm is chosen, and let $\mu_t(x)$ be the average reward in these rounds. The confidence radius of arm $x$ at time $t$ is defined as
    \[r_t(x) = \sqrt{\frac{2\log T}{n_t(x) + 1}}.\]
Recall that this is essentially the smallest number so as to guarantee with high probability that
    \[|\mu(x) - \mu_t(x)|\leq r_t(x).\]

\noindent The \emph{confidence ball} of arm $x$ is a closed ball in the metric space with center $x$ and radius $r_t(x)$.
\[  \Ball_t(x) = \{ y\in X:\; \mD(x,y) \leq r_t(x) \}. \]

\xhdr{Activation rule.}
We start with some intuition. We will estimate the mean reward $\mu(x)$ of a given active arm $x$ using only the samples from this arm. While samples from ``nearby" arms could potentially improve the estimates, this choice simplifies the algorithm and the analysis, and does not appear to worsen the regret bounds. Suppose arm $y$ is not active in round $t$, and lies very close to some active arm $x$, in the sense that
    $\mD(x,y) \ll r_t(x) $.
Then the algorithm does not have enough samples of $x$ to distinguish $x$ and $y$. Thus, instead of choosing arm $y$ the algorithm might as well choose arm $x$. We conclude there is no real need to activate $y$ yet. Going further with this intuition, there is no real need to activate any arm that is covered by the confidence ball of any active arm. We would like to maintain the following invariant:
\textEqNum{Lip:eq:zoom-inv}
{In each round, all arms are covered by confidence balls of the active arms.}

As the algorithm plays some arms over time, their confidence radii and the confidence balls get smaller, and some arm $y$ may become uncovered. Then we simply activate it! Since immediately after activation the confidence ball of $y$ includes the entire metric space, we see that the invariant is preserved.

Thus, the activation rule is very simple:
\textEq{If some arm $y$ becomes uncovered by confidence balls of the active arms, activate $y$.}

With this activation rule, the zooming algorithm has the following ``self-adjusting property". The algorithm ``zooms in" on a given region $R$ of the metric space (\ie activates many arms in $R$) if and only if the arms in $R$ are played often. The latter happens (under any reasonable selection rule) if and only if the arms in $R$ have high mean rewards.

\xhdr{Selection rule.}
We extend the technique from algorithm \UcbOne. If arm $x$ is active at time $t$, we define
\begin{align}\label{Lip:eq:index-defn}
    \Index_t(x) = \bar{\mu_t} (x) + 2 r_t(x)
\end{align}
The selection rule is very simple:
\textEq{Play an active arm with the largest index (break ties arbitrarily).}

Recall that algorithm \UcbOne chooses an arm with largest upper confidence bound (UCB) on the mean reward, defined as
    $\UCB_t(x) = \mu_t(x) + r_t(x)$.
So $\Index_t(x)$ is very similar, and shares the intuition behind \UcbOne: if $\Index_t(x)$ is large, then either $\mu_t(x)$ is large, and so $x$ is likely to be a good arm, or $r_t(x)$ is large, so arm $x$ has not been played very often, and should probably be explored more. And the `+' in \eqref{Lip:eq:index-defn} is a way to trade off exploration and exploitation. What is new here, compared to \UcbOne, is that $\Index_t(x)$ is a UCB not only on the mean reward of $x$, but also on the mean reward of any arm in the confidence ball of $x$.

To summarize, the algorithm is as follows:

\LinesNotNumbered \SetAlgoLined
\begin{algorithm}[H]
    {\bf Initialize:} set of active arms $S \leftarrow \emptyset$.\\
    \For{each round $t=1,2, \ldots   $}{
    \tcp{activation rule}
    \uIf{some arm $y$ is not covered by the confidence balls of active arms}{
        pick any such arm $y$ and ``activate" it: $S \leftarrow S \cap \{y\}$.
    }
    \tcp{selection rule}
    Play an active arm $x$ with the largest $\Index_t(x)$.
    }
    \caption{Zooming algorithm for adaptive discretization.}
    \label{Lip:alg:zooming}
\end{algorithm}

\subsection{Analysis: clean event}

We define a ``clean event" $\mE$ much like we did in Chapter~\ref{ch:IID}, and prove that it holds with high probability. The proof is more delicate than in Chapter~\ref{ch:IID}, essentially because we cannot immediately take the Union Bound over all of $X$. The rest of the analysis would simply assume that this event holds.

We consider a $K\times T$ table of realized rewards, with $T$ columns and a row for each arm $x$. The $j$-th column for arm $x$ is the reward for the $j$-th time this arm is chosen by the algorithm. We assume, without loss of generality, that this entire table is chosen before round 1: each cell for a given arm is an independent draw from the reward distribution of this arm. The clean event is defined as a property of this reward table. For each arm $x$, the clean event is
    \[\mE_x = \left\{\;
        |\mu_t(x) - \mu(x) | \leq r_t(x)\quad
        \text{for all rounds $t\in [T+1]$} \;\right\}.\]
Here $[T]:= \{1,2 \LDOTS T\}$. For convenience, we define $\mu_t(x)=0$ if arm $x$ has not yet been played by the algorithm, so that in this case the clean event holds trivially. We are interested in the event
    $\mE =  \cap_{x\in X} \mE_x$.

To simplify the proof of the next claim, we assume that realized rewards take values on a finite set.

\begin{claim}
Assume that realized rewards take values on a finite set. Then
$\Pr[\mE] \geq 1-\frac{1}{T^2}$.
\end{claim}

\begin{proof}
By Hoeffding Inequality,
    $\Pr[\mE_x] \geq 1 - \frac{1}{T^4}$
for each arm $x\in X$. However, one cannot immediately apply the Union Bound here because there may be too many arms.

Fix an instance of Lipschitz bandits. Let $X_0$ be the set of all arms that can possibly be activated by the algorithm on this problem instance. Note that $X_0$ is finite; this is because the algorithm is deterministic, the time horizon $T$ is fixed, and, as we assumed upfront, realized rewards can take only finitely many values. (This is the only place where we use this assumption.)

Let $N$ be the total number of arms activated by the algorithm.
Define arms
    $y_j \in X_0$, $j\in[T]$, as follows
\[
y_j = \left\{\begin{array}{ll}
                j\text{-th arm activated,}&\text{\quad if }j\leq N\\
                y_N,&\text{\quad otherwise}.
              \end{array}\right.
\]
Here $N$ and $y_j$'s are random variables, the randomness coming from the reward realizations. Note that $\{y_1 \LDOTS y_T \}$ is precisely the set of arms activated in a given execution of the algorithm. Since the clean event holds trivially for all arms that are not activated, the clean event can be rewritten as
    $\mE = \cap_{j=1}^T \mE_{y_j}.$
In what follows, we prove that the clean event $\mE_{y_j}$ happens with high probability for each $j\in [T]$.

Fix an arm $x\in X_0$ and fix $j\in [T]$. Whether the event $\{y_j = x\}$ holds is determined by the rewards of other arms. (Indeed, by the time arm $x$ is selected by the algorithm, it is already determined whether $x$ is the $j$-th arm activated!) Whereas whether the clean event $\mE_x$ holds is determined by the rewards of arm $x$ alone. It follows that the events $\{y_j = x\}$ and $\mE_x$ are independent. Therefore, if $\Pr[y_j=x]>0$ then
\[\Pr[\mE_{y_j} \mid y_j = x] = \Pr[\mE_x \mid y_j=x] = \Pr[\mE_x] \geq 1 - \tfrac{1}{T^4}\]
Now we can sum over all $x\in X_0$:
\[\Pr[\mE_{y_j}] =
  \sum_{x\in X_0}\Pr[y_j = x]\cdot \Pr[\mE_{y_j} \mid x = y_j] \geq 1 - \tfrac{1}{T^4}\]

\noindent To complete the proof, we apply the Union bound over all $j\in[T]$:
\[\Pr[\mE_{y_j}, j \in [T] ] \geq 1 - \tfrac{1}{T^3} . \qedhere\]
\end{proof}

We assume the clean event $\mE$ from here on.

\subsection{Analysis: bad arms}

Let us analyze the ``bad arms": arms with low mean rewards. We establish two crucial properties: that active bad arms must be far apart in the metric space (Corollary~\ref{Lip:cor:LipMAB-zooming-farApart}), and that each ``bad" arm cannot be played too often (Corollary~\ref{Lip:cor:LipMAB-zooming-notOften}).
As usual, let $\mu^* = \sup_{x\in X} \mu(x)$ be the best reward, and let  $\Delta(x)=\mu^*-\mu(x)$ denote the ``gap" of arm $x$. Let $n(x)=n_{T+1}(x)$ be the total number of samples from arm $x$.

The following lemma encapsulates a crucial argument which connects  the best arm and the arm played in a given round. In particular, we use  the main trick from the analysis of \UcbOne and the Lipschitz property.

\begin{lemma}\label{Lip:lm:LipMAB-zooming-crux}
$\Delta(x)\leq 3\, r_t(x)$
    for each arm $x$ and each round $t$.
\end{lemma}

\begin{proof}
Suppose arm $x$ is played in this round. By the covering invariant, the best arm $x^*$ was covered by the confidence ball of some active arm $y$, \ie
    $x^* \in \Ball_t(y)$.
It follows that
\[
\Index(x)
  \geq \Index(y)
  = \underbrace{\mu_t(y) + r_t(y)}_{\geq \mu(y)} + r_t(y)
  \geq \mu(x^*)
  = \mu^* \]
The last inequality holds because of the Lipschitz condition. On the other hand:
\[
\Index(x)
  = \underbrace{\mu_t(x)}_{\leq \mu(x) + r_t(x)} + 2 \cdot r_t(x)
  \leq \mu(x) + 3\cdot r_t(x)
\]
Putting these two equations together:
    $\Delta(x):= \mu^* - \mu(x)\leq 3\cdot r_t(x)$.

Now suppose arm $x$ is not played in round $t$. If it has never been played before round $t$, then $r_t(x)>1$ and the lemma follows trivially. Else, letting $s$ be the last time when $x$ has been played before round $t$, we see that
    $r_t(x)=r_s(x)\geq \Delta(x)/3$.
\end{proof}

\begin{corollary}\label{Lip:cor:LipMAB-zooming-farApart}
For any two active arms $x,y$, we have
$    \mD(x,y) > \tfrac13\; \min\left( \Delta(x),\; \Delta(y)\right)$.
\end{corollary}

\begin{proof}
W.l.o.g. assume that $x$ has been activated before $y$. Let $s$ be the time when $y$ has been activated. Then $\mD(x,y)>r_s(x)$ by the activation rule. And
    $r_s(x)\geq \Delta(x)/3$
by Lemma~\ref{Lip:lm:LipMAB-zooming-crux}.
\end{proof}

\begin{corollary}\label{Lip:cor:LipMAB-zooming-notOften}
For each arm $x$, we have
$n(x) \leq O(\log T)\; \Delta^{-2}(x)$.
\end{corollary}

\begin{proof}
Use Lemma~\ref{Lip:lm:LipMAB-zooming-crux} for $t=T$, and plug in the definition of the confidence radius.
\end{proof}

\subsection{Analysis: covering numbers and regret}

For $r>0$, consider the set of arms whose gap is between $r$ and $2r$:
    \[X_r = \{x\in X:  r\leq \Delta(x) < 2r\}.\]

Fix $i\in\N$ and let $Y_i = X_r$, where $r=2^{-i}$. By Corollary~\ref{Lip:cor:LipMAB-zooming-farApart}, for any two arms $x,y\in Y_i$,  we have $\mD(x,y) > r/3$. If we cover $Y_i$ with subsets of diameter $r/3$, then arms $x$ and $y$ cannot lie in the same subset. Since one can cover $Y_i$ with $N_{r/3}(Y_i)$ such subsets, it follows that
    $|Y_i|\leq N_{r/3}(Y_i)$.

Using Corollary~\ref{Lip:cor:LipMAB-zooming-notOften}, we have:
\begin{align*}
R_i(T) := \sum_{x\in Y_i}\Delta(x) \cdot n_t(x)
  \leq \frac{O(\log T)}{\Delta(x)} \cdot N_{r/3}(Y_i)
  \leq \frac{O(\log T)}{r} \cdot N_{r/3}(Y_i).
\end{align*}
Pick $\delta>0$, and consider arms with $\Delta(\cdot)\leq \delta$ separately from those with $\Delta(\cdot)> \delta$. Note that the total regret from the former cannot exceed $\delta$ per round. Therefore:
\begin{align}
R(T)
    &\leq \delta T + \sum_{i:\; r = 2^{-i} > \delta} R_i(T)  \nonumber \\
    &\leq \delta T + \sum_{i:\; r = 2^{-i} > \delta}
        \frac{\Theta(\log T)}{r}\; N_{r/3}(Y_i)  \label{Lip:eq:LipMAB-zooming-CovNum} \\
    &\leq\delta T + O(c\cdot \log T)\cdot (\tfrac{1}{\delta})^{d+1}
    \label{Lip:eq:LipMAB-zooming-delta}
\end{align}
where $c$ is a constant and $d$ is some number such that
\[N_{r/3}(X_r) \leq c \cdot r^{-d} \quad \forall r>0.\]
The smallest such $d$ is called the \emph{zooming dimension}:

\begin{definition}
For an instance of Lipschitz MAB, the \emph{zooming dimension} with multiplier $c>0$ is
\[\inf_{d\geq 0}
    \left\{ N_{r/3}(X)\leq c\cdot r^{-d}\quad \forall r> 0 \right\}.\]
\end{definition}

Choosing
    $\delta=(\frac{\log T}{T})^{1/(d+2)}$
in \eqref{Lip:eq:LipMAB-zooming-delta}, we obtain
$R(T) \leq  O\rbr{ T^{(d+1)/(d+2)}\; (c\log T)^{1/(d+2)}}$.
Note that we make this chose in the analysis only; the algorithm does not depend on the $\delta$.

\begin{theorem}\label{Lip:thm:LipMAB-zooming}
Consider Lipschitz bandits with time horizon $T$. Assume that realized rewards take values on a finite set. For any given problem instance and any $c>0$, the zooming algorithm attains regret
\begin{align*}
\E[R(T)]
    \leq  O\rbr{ T^{(d+1)/(d+2)}\; (c\log T)^{1/(d+2)}},
\end{align*}
where $d$ is the zooming dimension with multiplier $c$.
\end{theorem}

While the covering dimension is a property of the metric space, the zooming dimension is a property of the problem instance: it depends not only on the metric space, but on the mean rewards. In general, the zooming dimension is at most as large as the covering dimension, but may be much smaller (see Exercises~\ref{Lip:ex:CovZoomDim} and~\ref{Lip:ex:target-set}). This is because in the definition of the covering dimension one needs to cover all of $X$, whereas in the definition of the zooming dimension one only needs to cover set $X_r$

While the regret bound in Theorem~\ref{Lip:thm:LipMAB-zooming} is appealingly simple, a more precise regret bound is given in \eqref{Lip:eq:LipMAB-zooming-CovNum}. Since the algorithm does not depend on $\delta$, this bound holds for all $\delta>0$.

\newpage
\sectionBibNotes

Continuum-armed bandits have been introduced in \citet{Agrawal-bandits-95}, and further studied in \citep{Bobby-nips04,AuerOS/07}. Uniform discretization has been introduced in \citet{KleinbergL03} for dynamic pricing, and in \citet{Bobby-nips04} for continuum-armed bandits; \citet{LipschitzMAB-stoc08} observed that it easily extends to Lipschitz bandits. While doomed to $\tildeO(T^{2/3})$ regret in the worst case, uniform discretization yields $\tildeO(\sqrt{T})$ regret under strong concavity
\citet{KleinbergL03,AuerOS/07}.

Lipschitz bandits have been introduced in \cite{LipschitzMAB-stoc08} and in a near-simultaneous and independent paper \citep{xbandits-nips08}. The zooming algorithm is from \cite{LipschitzMAB-stoc08}, see \citet{LipschitzMAB-JACM} for a definitive journal version. \citet{xbandits-nips08} present a similar but technically different algorithm, with similar regret bounds.

Lower bounds for uniform discretization trace back to \citet{Bobby-nips04}, who introduces the proof technique in Section~\ref{Lip:sec:CAB-LB} and obtains Theorem~\ref{Lip:thm:LB-CovDim}.
The other lower bounds follow easily from the same technique. The explicit dependence on $L$ in Theorem~\ref{Lip:thm:LB-CAB} first appeared in \citet{Bubeck-alt11}, and the formulation in Exercise~\ref{Lip:ex:LB-COV}(b) is from \citet{xbandits-nips08}. Our presentation in Section~\ref{Lip:sec:CAB-LB} roughly follows that in \citet{contextualMAB-colt11} and \citep{xbandits-nips08}.

A line of work that pre-dated and inspired Lipschitz bandits posits that the algorithm is given a ``taxonomy" on arms: a tree whose leaves are arms, where arms in the same subtree being ``similar" to one another \citep{Kocsis-ecml06,yahoo-bandits07,Munos-uai07}. Numerical similarity information is not revealed. While these papers report successful empirical performance of their algorithms on some examples, they do not lead to non-trivial regret bounds. Essentially, regret scales as the number of arms in the worst case, whereas in Lipschitz bandits regret is bounded in terms of covering numbers.

Covering dimension is closely related to several other ``dimensions", such as Haussdorff dimension, capacity dimension, box-counting dimension, and Minkowski-Bouligand Dimension, that characterize the covering properties of a metric space in fractal geometry \citep[\eg see][]{Schroeder-1991}. Covering numbers and covering dimension have been widely used in machine learning to characterize the complexity of the hypothesis space in classification problems; however, we are not aware of a clear technical connection between this usage and ours. Similar but stronger notions of ``dimension" of a metric space have been studied in theoretical computer science: the ball-growth dimension \citep[\eg][]{Karger-stoc02,Abr05-SPAA,Slivkins-podc07} and the doubling dimension \citep[\eg][]{Gup03,Tal04,Slivkins-focs04}.%
\footnote{A metric has ball-growth dimension $d$ if doubling the radius of a ball increases the number of points by at most $O(2^d)$. A metric has doubling dimension $d$ if any ball can be covered with at most $O(2^d)$ balls of half the radius.} These notions allow for (more) efficient algorithms in many different problems: space-efficient distance representations such as metric embeddings, distance labels, and sparse spanners; network primitives such as routing schemes and distributed hash tables; approximation algorithms for optimization problems such as traveling salesman, $k$-median, and facility location.

\subsection{Further results on Lipschitz bandits}

\xhdr{Zooming algorithm.}
The analysis of the zooming algorithm, both in \cite{LipschitzMAB-stoc08,LipschitzMAB-JACM} and in this chapter, goes through without some of the assumptions. First, there is no need to assume that the metric satisfies triangle inequality (although this assumption is useful for the intuition). Second, Lipschitz condition \eqref{Lip:eq:Lip-general} only needs to hold for pairs $(x,y)$ such that $x$ is the best arm.%
  \footnote{A similar algorithm in \citet{xbandits-nips08} only requires the Lipschitz condition when both arms are near the best arm.}
Third, no need to restrict realized rewards to finitely many possible values (but one needs a slightly more careful analysis of the clean event). Fourth, no need for a fixed time horizon: The zooming algorithm can achieve the same regret bound for all rounds at once, by an easy application of the ``doubling trick" from Section~\ref{sec:IID-further}.

The zooming algorithm attains improved regret bounds for several special cases \citep{LipschitzMAB-stoc08}. First, if the maximal payoff is near $1$. Second,
when $\mu(x) = 1-f(\mD(x,S))$, where $S$ is a ``target set'' that is not revealed to the algorithm. Third, if the realized reward from playing each arm $x$ is $\mu(x)$ plus an independent noise, for several noise distributions; in particular, if rewards are deterministic.

The zooming algorithm achieves near-optimal regret bounds, in a very strong sense \citep{contextualMAB-colt11}. The ``raw" upper bound in \eqref{Lip:eq:LipMAB-zooming-CovNum} is optimal up to logarithmic factors, for any algorithm, any metric space, and any given value of this upper bound. Consequently, the upper bound in Theorem~\ref{Lip:thm:LipMAB-zooming} is optimal, up to logarithmic factors, for any algorithm and any given value $d$ of the zooming dimension that does not exceed the covering dimension. This holds for various metric spaces, \eg \CABmetric and
    $\left([0,1]^d,\, \ell_2\right)$.

The zooming algorithm, with similar upper and lower bounds, can be extended to the ``contextual" version of Lipschitz bandits \citep{contextualMAB-colt11}, see Sections~\ref{CB:sec:Lip} and~\ref{CB:sec:further} for details and discussion.

\begin{table}[t]
\begin{center}
\begin{tabular}{r|l|l}
\diagbox[width=9em]{$c_\mI$}{$f_\mI(t)$}
    & worst-case
    & instance-dependent
\hlinestrut
worst-case
    & \eg $f(t) = \tildeO(t^{2/3})$ for CAB
    & \eg zooming dimension
\hlinestrut
instance-dependent
    & \makecell[l]{\eg $f(t) = \log(t)$ for $K<\infty$ arms;
        \\see Table~\ref{Lip:tab:PMO-results} for $K=\infty$ arms}
    & \makecell[c]{---}
\end{tabular}
\end{center}
\vspace{-3mm}
\caption{Worst-case vs. instance-dependent regret rates of the form \eqref{Lip:eq:lit-generic}.}
\label{Lip:tab:types}
\end{table}

\xhdr{Worst-case vs. instance-dependent regret bounds.}
This distinction is more subtle than in stochastic bandits. Generically, we have regret bounds of the form
\begin{align}\label{Lip:eq:lit-generic}
\E\sbr{R(t)\mid \mI} \leq c_\mI\cdot f_\mI(t)+o(f_\mI(t))
    \quad\text{for each problem instance $\mI$ and all rounds $t$}.
\end{align}
Here $f_\mI(\cdot)$ defines the asymptotic shape of the regret bound, and $c_\mI$ is the \emph{leading constant} that cannot depend on time $t$. Both $f_\mI$ and $c_\mI$ may depend on the problem instance $\mI$. Depending on a particular result, either of them could be worst-case, \ie the same for all problem instances on a given metric space. The subtlety is that the instance-dependent vs. worst-case distinction can be applied to $f_\mI$ and $c_\mI$ separately, see Table~\ref{Lip:tab:types}. Indeed, both are worst-case in this chapter's results on uniform discretization. In Theorem~\ref{Lip:thm:LipMAB-zooming} for adaptive discretization,  $f_\mI(t)$ is instance-dependent, driven by the zooming dimension, whereas $c_\mI$ is an absolute constant. In contrast, the $\log(t)$ regret bound for stochastic bandits with finitely many arms features a worst-case $f_\mI$ and instance-dependent constant $c_\mI$. (We are not aware of any results in which both $c_\mI$ and $f_\mI$ are instance-dependent.) Recall that we have essentially matching upper and lower bounds for uniform dicretization and for the dependence on the zooming dimension, the top two quadrants in Table~\ref{Lip:tab:types}. The bottom-left quadrant is quite intricate for infinitely many arms, as we discuss next.

\xhdr{Per-metric optimal regret rates.} We are interested in regret rates \eqref{Lip:eq:lit-generic}, where $c_\mI$ is instance-dependent, but $f_\mI=f$ is the same for all problem instances on a given metric space;  we abbreviate this as $O_\mI(f(t))$. The upper/lower bounds for stochastic bandits imply that $O_\mI(\log t)$ regret is feasible and optimal for finitely many arms, regardless of the metric space. Further, the
    $\Omega(T^{1-1/(d+2)})$
lower bound for \CABmetric, $d\geq 1$ in Theorem~\ref{Lip:thm:LB-CovDim} holds even if we allow an instance-dependent constant \citep{Bobby-nips04}. The construction in this result is similar to that in Section~\ref{Lip:sec:CAB-LB}, but more complex. It needs to ``work" for all times $t$, so it contains bump functions \eqref{Lip:eq:bump-function} for all scales $\eps$ simultaneously. This lower bound extends to metric spaces with covering dimension $d$, under a homogeneity condition \citep{LipschitzMAB-stoc08,LipschitzMAB-JACM}.

However, better regret rates may be possible for arbitrary infinite metric spaces. The most intuitive example involves a ``fat point" $x\in X$ such that cutting out any open neighborhood of $x$ reduces the covering dimension by at least $\eps>0$. Then one can obtain a regret rate ``as if" the covering dimension were $d-\eps$. \cite{LipschitzMAB-stoc08,LipschitzMAB-JACM} handle this phenomenon in full generality. For an arbitrary infinite metric space, they define a ``refinement" of the covering dimension, denoted $\MaxMinCOV(X)$, and prove matching upper and lower regret bounds ``as if" the covering dimension were equal to this refinement.
$\MaxMinCOV(X)$ is always upper-bounded by $\COV(X)$, and could be as low as $0$ depending on the metric space.
Their algorithm is a version of the zooming algorithm with quotas on the number of active arms in some regions of the metric space. Further, \cite{DichotomyMAB-soda10,LipschitzMAB-JACM} prove that the transition from $O_\mI(\log t)$ to $O_\mI(\sqrt{t})$ regret is sharp and corresponds to the distinction between countable and uncountable set of arms. The full characterization of optimal regret rates is summarized in Table~\ref{Lip:tab:PMO-results}. This work makes deep connections between bandit algorithms, metric topology, and transfinite ordinal numbers.

\begin{table}[t]
\begin{center}
\begin{tabular}{l|l|l}
If the metric completion of $(X,\mD)$ is ...  & then regret can be ...    & but not ... \hlinestrut
finite
                          & $O(\log t)$           & $o(\log t)$ \\
compact and countable    & $\omega(\log t)$      & $O(\log t)$ \\
compact and uncountable  & & \\
~~~~~~~~~~~~
$\MaxMinCOV=0$
    & $\tilde{O}\left( t^\gamma \right)$, $\gamma >\nicefrac12$
    & $o(\sqrt{t})$  \\
~~~~~~~~~~~~
$\MaxMinCOV=d\in (0,\infty)$
    & $\tilde{O}\left( t^\gamma \right)$, $\gamma >\tfrac{d+1}{d+2}$
    & $o\left( t^\gamma \right)$, $\gamma <\tfrac{d+1}{d+2}$ \\
~~~~~~~~~~~~
$\MaxMinCOV=\infty$   & $o(t)$  & $O\left( t^\gamma \right)$, $\gamma<1$ \\
non-compact           & $O(t)$  & $o(t)$
\end{tabular}
\end{center}
\vspace{-3mm}
\caption{Per-metric optimal regret bounds for Lipschitz MAB}
\label{Lip:tab:PMO-results}
\end{table}

\cite{DichotomyMAB-soda10,LipschitzMAB-JACM} derive a similar characterization for a version of Lipschitz bandits with full feedback, \ie when the algorithm receives feedback for all arms. $O_\mI(\sqrt{t})$ regret is feasible for any metric space of finite covering dimension. One needs an exponentially weaker version of covering dimension to induce regret bounds of the form $\tildeO_\mI(t^\gamma)$, for some $\gamma\in(\nicefrac12,1)$.

\xhdr{Per-instance optimality.}
What is the best regret bound for a given instance of Lipschitz MAB? \citet{Combes-colt14} ask this question in the style of Lai-Robbins lower bound for stochastic bandits (Theorem~\ref{LB:thm:LB-log-t}). Specifically, they consider finite metric spaces, assume that the algorithm satisfies \eqref{LB:eq:thm:LB-log-t:assn}, the precondition in Theorem~\ref{LB:thm:LB-log-t}, and focus on optimizing the leading constant $c_\mI$ in \refeq{Lip:eq:lit-generic} with $f(t) = \log(t)$. For an arbitrary problem instance with finitely many arms, they derive a lower bound on $c_\mI$, and provide an algorithm comes arbitrarily close to this lower bound. However, this approach may increase the $o(\log T)$ term in \refeq{Lip:eq:lit-generic}, possibly hurting the worst-case performance.

\xhdr{Beyond IID rewards.}
Uniform discretization easily extends for adversarial rewards \citep{Bobby-nips04}, and matches the regret bound in Theorem~\ref{Lip:thm:LipMAB-unif} (see Exercise~\ref{adv:ex:Lip}).
Adaptive discretization extends to adversarial rewards, too, albeit with much additional work:
\cite{AdvZooming-colt21} connect it to techniques from adversarial bandits, and generalize Theorem~\ref{Lip:thm:LipMAB-zooming} for a suitable version of zooming dimension.

Some other variants with non-IID rewards have been studied. \cite{Munos-ecml10} consider a full-feedback problem with adversarial rewards and Lipschitz condition in the Euclidean space $(\R^d, \ell_2)$, achieving a surprisingly strong regret bound of $O_d(\sqrt{T})$. \citet{Azar-icml14} consider a version in which the IID condition is replaced by more sophisticated ergodicity and mixing assumptions, and essentially recover the performance of the zooming algorithm. \cite{contextualMAB-colt11} extends the zooming algorithm to a version of Lipschitz bandits where expected rewards do not change too fast over time. Specifically, $\mu(x,t)$,  the expected reward of a given arm $x$ at time $t$, is Lipschitz relative to a known metric on pairs $(x,t)$. Here round $t$ is interpreted as a context in Lipschitz contextual bandits; see also Section~\ref{CB:sec:Lip} and the literature discussion in Section~\ref{CB:sec:further}.

\OMIT{\cite{contextualMAB-colt11} considers contextual bandits with Lipschitz condition on expected payoffs, and provides a ``meta-algorithm'' which uses an off-the-shelf bandit algorithm such as \ExpThree \citep{bandits-exp3} as a subroutine and adaptively refines the space of contexts. }

\subsection{Partial similarity information}
\label{Lip:sec:lit-partial}

Numerical similarity information required for the Lipschitz bandits may be difficult to obtain in practice. A canonical example is the ``taxonomy bandits" problem mentioned above, where an algorithm is given a taxonomy (a tree) on arms but not a metric which admits the Lipschitz condition \eqref{Lip:eq:Lip-general}. One goal here is to obtain regret bounds that are (almost) as good as if the metric were known.

\citet{ImplicitMAB-nips11} considers the metric implicitly defined by an instance of taxonomy bandits:  the distance between any two arms is the ``width" of their least common subtree $S$, where the width of $S$ is defined as
    $W(S) := \max_{x,y\in S} |\mu(x)-\mu(y)|$.
(Note that $W(S)$ is not known to the algorithm.)  This is the best possible metric, \ie a metric with smallest distances, that admits the Lipschitz condition \eqref{Lip:eq:Lip-general}. \citet{ImplicitMAB-nips11} puts forward an extension of the zooming algorithm which partially reconstructs the implicit metric, and almost matches the regret bounds of the zooming algorithm for this metric. In doing so, it needs to deal with \emph{another} exploration-exploitation tradeoff: between learning more about the widths and exploiting this knowledge to run the zooming algorithm. The idea is to have ``active subtrees" $S$ rather than ``active arms", maintain a lower confidence bound (LCB) on $W(S)$, and use it instead of the true width. The LCB can be obtained any two sub-subtrees $S_1$, $S_2$ of $S$. Indeed, if one chooses arms from $S$ at random according to some fixed distribution, then
    $W(S) \geq |\mu(S_1)-\mu(S_2)|$,
where $\mu(S_i)$ is the expected reward when sampling from $S_i$, and with enough samples the empirical average reward from $S_i$ is close to its expectation. The regret bound depends on the ``quality parameter": essentially, how deeply does one need to look in each subtree $S$ in order to find sub-subtrees $S_1, S_2$ that give a sufficiently good lower bound on $W(S)$. However, the algorithm does not need to know this parameter upfront. \cite{Bull-bandits14} considers a somewhat more general setting where multiple taxonomies on arms are available, and some of them may work better for this problem than others. He carefully traces out the conditions under which one can achieve $\tilde{O}(\sqrt{T})$ regret.

A similar issue arises when arms correspond to points in $[0,1]$ but no Lipschitz condition is given. This setting can be reduced to ``taxonomy bandits" by positing a particular taxonomy on arms, \eg the root corresponds to $[0,1]$, its children are $[0,\nicefrac12)$ and $[\nicefrac12,1]$, and so forth splitting each interval into halves.

\cite{RepeatedPA-ec14} consider a related problem in the context of crowdsourcing markets. Here the algorithm is an employer who offers a quality-contingent contract to each arriving worker, and adjusts the contract over time. On an abstract level, this is a bandit problem in which arms are contracts: essentially, vectors of prices. However, there is no Lipschitz-like assumption. \cite{RepeatedPA-ec14} treat this problem as a version of ``taxonomy bandits", and design a version of the zooming algorithm. They estimate the implicit metric in a problem-specific way, taking advantage of the structure provided by the employer-worker interactions, and avoid the dependence on the ``quality parameter" from \citet{ImplicitMAB-nips11}.

Another line of work studies the ``pure exploration" version of ``taxonomy bandits", where the goal is to output a ``predicted best arm" with small instantaneous regret
\citep{Munos-nips11,Valko-icml13,Grill-nips15},
see \citet{Munos-trends14} for a survey. The main result essentially recovers the regret bounds for the zooming algorithm as if a suitable distance function were given upfront. The algorithm posits a parameterized family of distance functions, guesses the parameters, and runs a zooming-like algorithm for each guess.

\cite{Bubeck-alt11} study a version of continuum-armed bandits with strategy set $[0,1]^d$ and Lipschitz constant $L$ that is not revealed to the algorithm, and match the regret rate in Theorem~\ref{Lip:thm:LB-CAB}. This result is powered by an assumption that $\mu(\cdot)$ is twice differentiable, and a bound on the second derivative is known to the algorithm. \cite{Minsker-colt13} considers the same strategy set, under metric $\|x-y\|_\infty^\beta$, where the ``smoothness parameter" $\beta\in(0,1]$ is not known. His algorithm achieves near-optimal instantaneous regret as if the $\beta$ were known, under some structural assumptions.

\subsection{Generic non-Lipschitz models for bandits with similarity}
\label{Lip:sec:lit-alt}

One drawback of Lipschitz bandits as a model is that the distance $\mD(x,y)$ only gives a ``worst-case" notion of similarity between arms $x$ and $y$. In particular, the distances may need to be very large in order to accommodate a few outliers, which would make $\mD$ less informative elsewhere.%
\footnote{This concern is partially addressed by relaxing the Lipschitz condition in the analysis of the zooming algorithm.}
With this criticism in mind, \cite{GPbandits-icml10,GPbandits-nips11,GPbandits-icml12} define a probabilistic model, called \emph{Gaussian Processes Bandits},  where the expected payoff function is distributed according to a suitable Gaussian Process on $X$, thus ensuring a notion of ``probabilistic smoothness'' with respect to $X$.

\citet{SmoothedRegret-colt19} side-step Lipschitz assumptions by relaxing the benchmark. They define a new benchmark, which replaces each arm $a$ with a low-variance distribution around this arm, called the \emph{smoothed arm}, and compares the algorithm's reward to that of the best smoothed arm; call it the \emph{smoothed benchmark}. For example, if the set of arms is $[0,1]$, then the smoothed arm can be defined as the uniform distribution on the interval $[a-\eps,a+\eps]$, for some fixed $\eps>0$. Thus, very sharp peaks in the mean rewards -- which are impossible to handle via the standard best-arm benchmark without Lipschitz assumptions -- are now smoothed over an interval. In the most general version, the smoothing distribution may be arbitrary, and arms may lie in an arbitrary ``ambient space" rather than $[0,1]$ interval. 
(The ``ambient space" is supposed to be natural given the application domain; formally it is described by a metric on arms and a measure on balls in that metric.)
Both fixed and adaptive discretization carries over to the smoothed benchmark, without any Lipschitz assumptions \citep{SmoothedRegret-colt19}. These results usefully extend to contextual bandits with policy sets, as defined in Section~\ref{CB:sec:policy-class} \citep{SmoothedRegret-colt19,SmoothedRegret-nips20}.

\cite{functionalMAB-colt11} and \cite{Combes-nips17} consider stochastic bandits with, essentially, an arbitrary known family $\mF$ of mean reward functions, as per Section~\ref{IID:sec:fwd}. \cite{functionalMAB-colt11} posit a finite $|\mF|$ and obtain a favourable regret bound when a certain complexity measure of $\mF$ is small, and any two functions in $\mF$ are sufficiently well-separated. However, their results do not subsume any prior work on Lipschitz bandits. \cite{Combes-nips17} extend the per-instance optimality approach from \citet{Combes-colt14}, discussed above, from Lipschitz bandits to an arbitrary $\mF$, under mild assumptions.

\emph{Unimodal bandits} assume that mean rewards are \emph{unimodal}: \eg when the set of arms is $X=[0,1]$, there is a single best arm $x^*$, and mean rewards $\mu(x)$ increase for all arms $x<x^*$ and decrease for all arms $x>x^*$. For this setting, one can obtain $\tildeO(\sqrt{T})$ regret under some additional assumptions on $\mu(\cdot)$: smoothness \citep{UnimodalMAB-Cope09}, Lipschitzness \citep{UnimodalMAB-icml11}, or continuity \citep{UnimodalMAB-arxiv14}. One can also consider a more general version of unimodality relative to a known partial order on arms \citep{UnimodalMAB-icml11,UnimodalMAB-icml14}.

\subsection{Dynamic pricing and bidding}
\label{Lip:sec:lit-DP}

A notable class of bandit problems has arms that correspond to monetary amounts, \eg offered prices for selling (\emph{dynamic pricing}) or bying (\emph{dynamic procurement}), offered wages for hiring, or bids in an auction (\emph{dynamic bidding}). Most studied is the basic case when arms are real numbers, \eg prices rather than price vectors. All these problems satisfy a version of monotonicity, \eg decreasing the price cannot result in fewer sales. This property suffices for both uniform and adaptive discretization without any additional Lipschitz assumptions. We work this out for dynamic pricing, see the exercises in Section~\ref{Lip:sec:ex-DP}.

Dynamic pricing as a bandit problem has been introduced in \cite{KleinbergL03}, building on the earlier work \citep{Blum03}. \cite{KleinbergL03} introduce uniform discretization, and observe that it attains $\tildeO(T^{2/3})$ regret for stochastic rewards, just like in Section~\ref{Lip:sec:CAB-UB}, and likewise for adversarial rewards, with an appropriate algorithm for adversarial bandits. Moreover, uniform discretization (with a different step $\eps$) achieves $\tildeO(\sqrt{T})$ regret, if expected reward $\mu(x)$ is strongly concave as a function of price $x$; this condition, known as \emph{regular demands}, is standard in theoretical economics. \cite{KleinbergL03} prove a matching $\Omega(T^{2/3})$ lower bound in the worst case, even if one allows an instance-dependent constant. The construction contains a version of bump functions \eqref{Lip:eq:bump-function} for all scales $\eps$ simultaneously, and predates a similar lower bound for continuum-armed bandits from \citet{Bobby-nips04}. That the zooming algorithm works without additional assumptions is straightforward from the original analysis in \citep{LipschitzMAB-stoc08,LipschitzMAB-JACM}, but has not been observed until \cite{AdvZooming-colt21}.

$\tildeO(\sqrt{T})$ regret can be achieved in some auction-related problems when additional feedback is available to the algorithm. \citet{Weed-colt16} achieve $\tildeO(\sqrt{T})$ regret for dynamic bidding in first-price auctions, when the algorithm observes full feedback (\ie the minimal winning bid) whenever it wins the auction. \citet{Chara-ec18} simplifies the algorithm in this result, and extends it to a more general auction model. Both results extend to adversarial outcomes. \citet{RepeatedAuctions-soda13} achieve $\tildeO(\sqrt{T})$ regret in a ``dual" problem, where the algorithm optimizes the auction rather than the bids. Specifically, the algorithm adjusts the \emph{reserve price} (the lowest acceptable bid) in a second-price auction. The algorithm receives more-than-bandit feedback: indeed, if a sale happens for a particular reserve price, then exactly the same sale would have happened for any smaller reserve price.

Dynamic pricing and related problems become more difficult when arms are price vectors, \eg when multiple products are for sale, the algorithm can adjust the price for each separately, and customer response for one product depends on the prices for others. In particular, one needs  structural assumptions such as Lipschitzness, concavity or linearity. One exception is quality-contingent contract design \citep{RepeatedPA-ec14}, as discussed in Section~\ref{Lip:sec:lit-partial}, where a version of zooming algorithm works without additional assumptions.

Dynamic pricing and bidding is often studied in more complex environments with supply/budget constraints and/or forward-looking strategic behavior. We discuss these issues in Chapters~\ref{ch:games} and~\ref{ch:BwK}. In particular, the literature on dynamic pricing with limited supply is reviewed in Section~\ref{BwK:sec:further-DP}.

\sectionExercises

\subsection{Construction of $\eps$-meshes}
\label{Lip:sec:ex-mesh}

Let us argue that the construction in Remark~\ref{Lip:rem:net} attains the infimum in \eqref{Lip:eq:thm:unif-S} up to a constant factor.

Let
    $\eps(S) = \inf\cbr{ \eps>0:\; \text{$S$ is an $\eps$-mesh} }$,
for $S\subset X$. Restate the right-hand side in \eqref{Lip:eq:thm:unif-S} as
\begin{align}\label{Lip:eq:ex:net-inf}
 \inf_{S\subset X} f(S)
    \eqWHERE
    f(S) = \algC\cdot \sqrt{|S|\,T\log T} + T\cdot \eps(S).
\end{align}
Recall that $\E[R(T)] \leq f(S)$ if we run algorithm $\ALG$ on set $S$ of arms.

Recall the construction in Remark~\ref{Lip:rem:net}. Let $\mN_i$ be the $\eps$-mesh computed therein for a given $\eps = 2^{-i}$. Suppose the construction stops at some $i=i^*$. Thus, this construction gives regret
    $\E[R(T)] \leq f(N_{i^*})$.

\begin{exercise}\label{Lip:ex:net}
Compare $f(N_{i^*})$ against \eqref{Lip:eq:ex:net-inf}. Specifically, prove that
\begin{align}\label{Lip:eq:ex:net-result}
f(\mN_{i^*}) \leq 16\,\inf_{S\subset X} f(S).
\end{align}
The key notion in the proof is \emph{$\eps$-net}, $\eps>0$: it is an \eps-mesh where any two points are at distance $>\eps$.

\begin{itemize}
\item[(a)] Observe that each $\eps$-mesh computed in Remark~\ref{Lip:rem:net} is in fact an $\eps$-net.

\item[(b)] Prove that for any $\eps$-mesh $S$ and any $\eps'$-net $\mN$, $\eps'\geq 2\eps$, it holds that $|\mN|\leq |S|$.\\
    Use (a) to conclude that
        $\min_{i\in\N} f(\mN_i) \leq 4\,f(S) $.

\item[(c)] Prove that $f(\mN_{i^*}) \leq 4\,\min_{i\in\N} f(\mN_i)$.
Use (b) to conclude that \eqref{Lip:eq:ex:net-result} holds.
\end{itemize}
\end{exercise}

\subsection{Lower bounds for uniform discretization}
\label{Lip:sec:ex-LB}

\begin{exercise}\label{Lip:ex:LB}
Extend the construction and analysis in Section~\ref{Lip:sec:CAB-LB}:
\begin{itemize}
\item[(a)] ... from Lipschitz constant $L=1$ to an arbitrary $L$, \ie prove Theorem~\ref{Lip:thm:LB-CAB}.

\item[(b)] ... from continuum-armed bandits to \CABmetric, \ie prove Theorem~\ref{Lip:thm:LB-CovDim}.

\item[(c)] ... to an arbitrary metric space $(X,\mD)$ and an arbitrary $\eps$-net $\mN$ therein (see Exercise~\ref{Lip:ex:net}):\\
prove that for any algorithm there is a problem instance on $(X,\mD)$ such that
\begin{align}
    \E[R(T)] \geq \Omega\rbr{\min\rbr{\eps T,\,|\mN|/\eps}}
        \quad\text{for any time horizon $T$}.
\end{align}

\end{itemize}
\end{exercise}

\begin{exercise}\label{Lip:ex:LB-unif}
Prove that the optimal uniform discretization from \refeq{Lip:eq:thm:unif-S} is optimal up to $O(\log T)$ factors. Specifically, using the notation from \refeq{Lip:eq:ex:net-inf}, prove the following: for any metric space $(X,\mD)$, any algorithm and any time horizon $T$ there is a problem instance on $(X,\mD)$ such that
    \[ \E[R(T)] \geq \textstyle \widetilde{\Omega}(\; \inf_{S\subset X} f(S) \;).\]

\end{exercise}

\begin{exercise}[Lower bounds via covering dimension]\label{Lip:ex:LB-COV}
Consider Lipschitz bandits in a metric space $(X,\mD)$. Fix  $d <\COV(X)$, for some fixed absolute constant $c>0$.  Fix an arbitrary algorithm \ALG. We are interested proving that this algorithm suffers lower bounds of the form
\begin{align}\label{Lip:eq:ex:LB-COV}
    \E[R(T)] \geq \Omega\rbr{T^{(d+1)/(d+2)}}.
\end{align}
The subtlety is, for which $T$ can this be achieved?
\begin{itemize}

\item[(a)] Assume that the covering property is nearly tight at a particular scale $\eps>0$, namely:
\begin{align}\label{Lip:eq:ex:LB-COV-assn}
 N_{\eps}(X)\geq c'\cdot \eps^{-d}
 \quad \text{for some absolute constant $c'$.}
\end{align}
Prove that \eqref{Lip:eq:ex:LB-COV} holds for some problem instance and $T=c'\cdot \eps^{-d-2}$.

\item[(b)] Assume \eqref{Lip:eq:ex:LB-COV-assn} holds for all $\eps>0$. Prove that for each $T$ there is a problem instance with \eqref{Lip:eq:ex:LB-COV}.

\item[(c)] Prove that \eqref{Lip:eq:ex:LB-COV} holds for \emph{some} time horizon $T$ and some problem instance.

\item[(d)] Assume that $d<\COV[](X) := \inf_{c>0} \COV(X)$. Prove that there are \emph{infinitely many} time horizons $T$ such that \eqref{Lip:eq:ex:LB-COV} holds for some problem instance $\mI_T$. In fact, there is a distribution over these problem instances (each endowed with an infinite time horizon) such that \eqref{Lip:eq:ex:LB-COV} holds for \emph{infinitely many} $T$.


\end{itemize}

\Hint[Hints]{Part (a) follows from Exercise~\ref{Lip:ex:LB}(c), the rest follows from part (a). For part (c), observe that \eqref{Lip:eq:ex:LB-COV-assn} holds for some $\eps$. For part (d), observe that \eqref{Lip:eq:ex:LB-COV-assn} holds for arbitrarily small $\eps>0$, \eg with $c'=1$, then apply part (a) to each such $\eps$. To construct the distribution over problem instances $\mI_T$, choose a sufficiently sparse sequence $(T_i:\,i\in\N)$, and include each instance $\mI_{T_i}$ with probability $\sim 1/\log(T_i)$, say. To remove the $\log T$ factor from the lower bound, carry out the argument above for some $d'\in (d,\COV(X))$. }

\end{exercise}

\subsection{Examples and extensions}

\begin{exercise}[Covering dimension and zooming dimension]\label{Lip:ex:CovZoomDim}
~~
\begin{OneLiners}
\item[(a)] Prove that the covering dimension of
    $\left([0,1]^d, \ell_2\right)$, $d\in \N$ and
        \CABmetric, $d\geq 1$ is $d$.
\item[(b)] Prove that the zooming dimension cannot exceed the covering dimension. More precisely: if $d=\COV(X)$, then the zooming dimension with multiplier $3^d\cdot c$ is at most $d$.
\item[(c)] Construct an example in which the zooming dimension is much smaller than $\COV(X)$.
\end{OneLiners}
\end{exercise}

\begin{exercise}[Lipschitz bandits with a target set]\label{Lip:ex:target-set}
Consider Lipschitz bandits on metric space $(X,\mD)$ with $\mD(\cdot,\cdot)\leq \nicefrac12$.  Fix the best reward $\mu^*\in[\nicefrac34,1]$ and a subset $S\subset X$ and assume that
\[\mu(x) = \mu^* - \mD(x,S) \quad \forall x\in X, \qquad
    \text{where } \mD(x,S) := \inf_{y\in S} \mD(x,y).\]
In words, the mean reward is determined by the distance to some ``target set" $S$.
\begin{OneLiners}
\item[(a)] Prove that
    $\mu^* - \mu(x) \leq \mD(x,S)$ for all arms $x$,
and that this condition suffices for the analysis of the zooming algorithm, instead of the full Lipschitz condition \eqref{Lip:eq:Lip-general}.
\item[(b)] Assume the metric space is $([0,1]^d,\ell_2)$, for some $d\in \N$.  Prove that the zooming dimension of the problem instance (with a suitably chosen multiplier) is at most the covering dimension of $S$.
\end{OneLiners}

\TakeAway In part (b), the zooming algorithm achieves regret $\tilde{O}(T^{(b+1)/(b+2)})$, where $b$ is the covering dimension of the target set $S$. Note that $b$ could be much smaller than $d$, the covering dimension of the entire metric space. In particular, one achieves regret $\tilde{O}(\sqrt{T})$ if $S$ is finite.
\end{exercise}

\subsection{Dynamic pricing}
\label{Lip:sec:ex-DP}

Let us apply the machinery from this chapter to dynamic pricing. This problem naturally satisfies a monotonicity condition: essentially, you cannot sell less if you decrease the price. Interestingly, this condition suffices for our purposes, without any additional Lipschitz assumptions.

\begin{BoxedProblem}{Dynamic pricing}
\noindent In each round $t \in [T]$:
\begin{OneLiners}
  \item[1.] Algorithm picks some price $p_t\in[0,1]$ and offers one item for sale at this price.
  \item[2.] A new customer arrives with private value $v_t\in[0,1]$, not visible to the algorithm.
  \item[3.] The customer buys the item if and only if $v_t\geq p_t$.\\
   The algorithm's reward is $p_t$ if there is a sale, and $0$ otherwise.
\end{OneLiners}
\end{BoxedProblem}

\noindent Thus, arms are prices in $X=[0,1]$.  We focus on \emph{stochastic} dynamic pricing: each private value $v_t$ is sampled independently from some fixed distribution which is not known to the algorithm.

\begin{exercise}[monotonicity]\label{Lip:ex:DP-mono}
Observe that the sale probability
    $\Pr[\text{sale at price $p$}]$
is monotonically non-increasing in $p$. Use this monotonicity property to derive one-sided Lipschitzness:
\begin{align}\label{Lip:eq:ex:DP-unif}
    \mu(p) -\mu(p') \leq p-p' \quad \text{for any two prices $p>p'$}.
\end{align}
\end{exercise}

\begin{exercise}[uniform discretization]\label{Lip:ex:DP-unif}
Use \eqref{Lip:eq:ex:DP-unif} to recover the regret bound for uniform discretization: with any algorithm \ALG satisfying \eqref{Lip:eq:opt-regret} and discretization step
    $\eps=(T/\log T)^{-1/3}$
one obtains
    \[ \E[R(T)] \leq T^{2/3} \cdot (1+\algC)(\log T)^{1/3}. \]
\end{exercise}

\begin{exercise}[adaptive discretization]\label{Lip:ex:DP-adaptive}
Modify the zooming algorithm as follows. The confidence ball of arm $x$ is redefined as the interval $[x,\,x+r_t(x)]$. In the activation rule, when some arm is not covered by the confidence balls of active arms, pick the smallest (infimum) such arm and activate it. Prove that this modified algorithm achieves the regret bound in Theorem~\ref{Lip:thm:LipMAB-zooming}.

\Hint{The invariant~\eqref{Lip:eq:zoom-inv} still holds. The one-sided Lipschitzness \eqref{Lip:eq:ex:DP-unif} suffices for the proof of Lemma~\ref{Lip:lm:LipMAB-zooming-crux}.}
\end{exercise}

\chapter{Full Feedback and Adversarial Costs}
\label{ch:FF}
\begin{ChAbstract}
For this one chapter, we shift our focus from bandit feedback to full feedback.  As the IID assumption makes the problem ``too easy", we introduce and study the other extreme, when rewards/costs are chosen by an adversary. We define and analyze two classic algorithms, \emph{weighted majority} and \emph{multiplicative-weights update}, a.k.a. \Hedge.
\end{ChAbstract}

Full feedback is defined as follows: in the end of each round, the algorithm observes the outcome not only for the chosen arm, but for all other arms as well. To be in line with the literature on such problems, we express the outcomes as \emph{costs} rather than \emph{rewards}. As IID costs are quite ``easy" with full feedback, we consider the other extreme: costs can arbitrarily change over time, as if they are selected by an adversary.

The protocol for full feedback and adversarial costs is as follows:

\begin{BoxedProblem}{Bandits with full feedback and adversarial costs}
{\bf Parameters:} $K$ arms, $T$ rounds (both known).

\noindent In each round $t \in [T]$:
\begin{OneLiners}
  \item[1.] Adversary chooses costs $c_t(a)\geq 0$ for each arm $a\in[K]$.
  \item[2.] Algorithm picks arm $a_t\in[K]$.
  \item[3.] Algorithm incurs cost $c_t(a_t)$ for the chosen arm.
  \item[4.] The costs of all arms, $c_t(a):\; a\in[K]$, are revealed.
\end{OneLiners}
\end{BoxedProblem}

\begin{remark}
While some results rely on bounded costs, \eg $c_t(a)\leq 1$, we do not assume this by default.
\end{remark}

One real-life scenario with full feedback is investments on a stock market. For
a simple (and very stylized) example, recall one from the Introduction. Suppose each morning choose one stock and invest \$1 into it. At the end of the day, we observe not only the price of the chosen stock, but prices of all stocks. Based on this feedback, we determine which stock to invest for the next day.

A paradigmatic special case of bandits with full feedback is sequential prediction with experts advice. Suppose we need to predict labels for observations, and we are assisted with a committee of experts. In each round, a new observation arrives, and each expert predicts a correct label for it. We listen to the experts, and pick an answer to respond. We then observe the correct answer and costs/penalties of all other answers. Such a process can be described by the following protocol: \newpage

\begin{BoxedProblem}{Sequential prediction with expert advice}
{\bf Parameters:} $K$ experts, $T$ rounds, $L$ labels, observation set $\mX$ (all known).

\noindent For each round $t \in [T]$:
\begin{OneLiners}
  \item[1.] Adversary chooses observation $x_t\in\mX$ and correct label $z^*_t\in [L]$. \newline
  Observation $x_t$ is revealed, label $z^*_t$ is not.
  \item[2.] The $K$ experts predict labels $z_{1, t}, \dots, z_{K, t}\in [L]$.
  \item[3.] Algorithm picks an expert $e = e_t\in [K]$.
  \item[4.] Correct label $z^*_t$ is revealed.
  \item[5.] Algorithm incurs cost $c_t = c\rbr{ z_{e,t},\,z^*_t}$,
  for some known \emph{cost function} $c: [L]\times [L]\to [0,\infty)$.
\end{OneLiners}
\end{BoxedProblem}


\noindent The basic case is \emph{binary costs}:
    $c(z,z^*) = \ind{z\neq z^*}$,
\ie the cost is $0$ if the answer is correct, and $1$ otherwise. Then the total cost is simply the number of mistakes.

The goal is to do approximately as well as the best expert. Surprisingly, this can be done without any domain knowledge, as explained in the rest of this chapter.

\begin{remark}
Because of this special case, the general case of bandits with full feedback is usually called \emph{online learning with experts}, and defined in terms of \emph{costs} (as penalties for incorrect predictions) rather than \emph{rewards}. We will talk about arms, actions and experts interchangeably throughout this chapter.
\end{remark}

\begin{remark}[IID costs]\label{FF:rem:IID}
Consider the special case when the adversary chooses the cost $c_t(a)\in [0,1]$ of each arm $a$ from some fixed distribution $\mD_a$, same for all rounds $t$. With full feedback, this special case is ``easy": indeed, there is no need to explore, since costs of all arms are revealed after each round. With a naive strategy such as playing arm with the lowest average cost, one can achieve regret
    $O \left( \sqrt{T \log\left( KT \right)} \right)$.
Further, there is a nearly matching lower regret bound
     $\Omega(\sqrt{T}+\log K)$.
The proofs of these results are left as exercise. The upper bound can be proved by a simple application of clean event/confidence radius technique that we've been using since Chapter~\ref{ch:IID}. The $\sqrt{T}$ lower bound follows from the same argument as the bandit lower bound for two arms in Chapter~\ref{ch:LB}, as this argument does not rely on bandit feedback. The  $\Omega(\log K)$ lower bound holds for a simple special case, see
Theorem~\ref{FF:thm:binary-prediction-LB}.
\end{remark}

\section{Setup: adversaries and regret}
\label{FF:sec:adv}

Let us elaborate on the types of adversaries one could consider, and the appropriate notions of regret. A crucial distinction is whether the cost functions $c_t(\cdot)$ depend on the algorithm's choices. An adversary is called \emph{oblivious} if they don't, and \emph{adaptive} if they do (\ie oblivious / adaptive to the algorithm). Like before, the ``best arm" is an arm $a$ with a lowest total cost, denoted
    $\cost(a) = \sum_{t = 1}^T c_t(a)$,
and regret is the difference in total cost compared to this arm.
However, defining this precisely is a little subtle, especially when the adversary is randomized. We explain all this in detail below.

\xhdr{Deterministic oblivious adversary.}
Thus, the costs $c_t(\cdot)$, $t\in[T]$ are deterministic and do not depend on the algorithm's choices. Without loss of generality, the entire ``cost table"
    $\rbr{c_t(a):\; a\in[K],\, t\in [T]}$
is chosen before round $1$. The best arm is naturally defined as
    $\argmin_{a\in[K]} \cost(a)$, and regret is defined as
    \begin{align}\label{FF:eq:experts-regret-hindsight}
       R(T) = \cost(\ALG) - \min_{a\in[K]} \cost (a),
    \end{align}
where $\cost(\ALG)$ denotes the total cost incurred by the algorithm. One drawback of such adversary is that it does not model IID costs, even though IID rewards are a simple special case of adversarial rewards.

\xhdr{Randomized oblivious adversary.}
The costs $c_t(\cdot)$, $t\in[T]$ do not depend on the algorithm's choices, as before, but can be randomized. Equivalently, the adversary fixes a distribution $\mD$ over the cost tables before round $1$, and then draws a cost table from this distribution. Then IID costs are indeed a simple special case. Since $\cost(a)$ is now a random variable, there are two natural (and different) ways to define the ``best arm":
\begin{itemize}
\item $\argmin_a \cost(a)$: this is the best arm \emph{in hindsight}, \ie after all costs have been observed. It is a natural notion if we start from the deterministic oblivious adversary.
\item $\argmin_a \E[\cost(a)]$: this is be best arm \emph{in foresight}, \ie an arm you'd pick if you only know the distribution $\mD$. This is a natural notion if we start from IID costs.
\end{itemize}
Accordingly, there are two natural versions of regret: with respect to the best-in-hindsight arm, as in \eqref{FF:eq:experts-regret-hindsight}, and with respect to the best-in-foresight arm,
\begin{align}\label{FF:eq:experts-regret-foresight}
    R(T) = \cost(\ALG) - \min_{a\in[K]} \E[\cost (a)].
\end{align}
For IID costs, this notion coincides with the definition of regret from Chapter~\ref{ch:IID}.

\begin{remark}
The notion in \eqref{FF:eq:experts-regret-hindsight} is usually referred to simply as \emph{regret} in the literature, whereas \eqref{FF:eq:experts-regret-foresight} is often called \emph{pseudo-regret}. We will use this terminology when we need to distinguish between the two versions.
\end{remark}

\begin{remark}
Pseudo-regret cannot exceed regret, because the best-in-foresight arm is a weaker benchmark. Some positive results for pseudo-regret carry over to regret, and some don't. For IID rewards/costs, the $\sqrt{T}$ upper regret bounds from Chapter~\ref{ch:IID} extend to regret, whereas the $\log(T)$ upper regret bounds do not extend in full generality, see Exercise~\ref{FF:ex:IID} for details.
\end{remark}

\xhdr{Adaptive adversary} can change the costs depending on the algorithm's past choices. Formally, in each round $t$, the costs $c_t(\cdot)$ may depend on arms $a_1 \LDOTS a_{t-1}$, but not on $a_t$ or the realization of algorithm's internal randomness. This models scenarios when algorithm's actions may alter the environment that the algorithm operates in. For example:
    \begin{itemize}
    \item an algorithm that adjusts the layout of a website may cause users to permanently change their behavior, \eg they may gradually get used to a new design, and get dissatisfied with the old one.

    \item a bandit algorithm that selects news articles for a website may attract some users and repel some others, and/or cause the users to alter their reading preferences.

    \item if a dynamic pricing algorithm offers a discount on a new product, it may cause many people to buy this product and (eventually) grow to like it and spread the good word. Then more people would be willing to buy this product at full price.

    \item if a bandit algorithm adjusts the parameters of a repeated auction (\eg a reserve price), auction participants may adjust their behavior over time, as they become more familiar with the algorithm.
    \end{itemize}


In game-theoretic applications, adaptive adversary can be used to model a game between an algorithm and a self-interested agent that responds to algorithm's moves and strives to optimize its own utility. In particular, the agent may strive to hurt the algorithm if the game is zero-sum. We will touch upon game-theoretic applications in Chapter~\ref{ch:games}.

An adaptive adversary is assumed to be randomized by default. In particular, this is because the adversary can adapt the costs to the algorithm's choice of arms in the past, and the algorithm is usually randomized. Thus, the distinction between regret and pseudo-regret comes into play again.

Crucially, which arm is best may depend on the algorithm's actions. For example, if an algorithm always chooses arm 1 then arm 2 is consistently much better, whereas \emph{if the algorithm always played arm 2}, then arm 1 may have been better. One can side-step these issues by considering the \emph{best-observed arm}: the best-in-hindsight arm according to the costs actually observed by the algorithm.

Regret guarantees relative to the best-observed arm are not always satisfactory, due to many problematic examples such as the one above. However, such guarantees are worth studying for several reasons. First, they  \emph{are} meaningful in some scenarios, \eg when algorithm's actions do not substantially affect the total cost of the best arm.
Second, such guarantees may be used as a tool to prove results on oblivious adversaries (\eg see next chapter). Third, such guarantees are essential in several important applications to game theory, when a bandit algorithm controls a player in a repeated game (see Chapter~\ref{ch:games}). Finally, such guarantees often follow from the analysis of oblivious adversary with very little extra work.

\xhdr{Throughout this chapter,} we consider an adaptive adversary, unless specified otherwise. We are interested in regret relative to the best-observed arm. For ease of comprehension, one can also interpret the same material as working towards regret guarantees against a randomized-oblivious adversary.

Let us (re-)state some notation: the best arm and its cost are
\[\textstyle  a^* \in \argmin_{a\in[K]} \cost(a)
\quad\text{and}\quad
  \cost^* = \min_{a\in[K]} \cost(a),
\]
where
    $\cost(a) = \sum_{t = 1}^T c_t(a)$
is the total cost of arm $a$. Note that $a^*$ and $\cost^*$ may depend on randomness in rewards, and (for adaptive adversary) also on algorithm's actions.

As always, $K$ is the number of actions, and $T$ is the time horizon.

\section{Initial results: binary prediction with experts advice}

We consider \emph{binary prediction with experts advice}, a special case where experts' predictions $z_{i, t}$ can only take two possible values. For example: is this image a face or not? Is it going to rain today or not?

Let us assume that there exists a \emph{perfect expert} who never makes a mistake. Consider a simple algorithm that disregards all experts who made a mistake in the past, and follows the majority of the remaining experts:
\textEq{In each round $t$, pick the action chosen by the majority of the experts who did not err in the past.}
We call this the \emph{majority vote algorithm}. We obtain a strong guarantee for this algorithm:

\begin{theorem}\label{FF:thm:perfect-expert}
Consider binary prediction with experts advice. Assuming a perfect expert, the majority vote algorithm makes at most $\log_2 K$ mistakes,
where $K$ is the number of experts.
\end{theorem}

\begin{proof}
Let $S_t$ be the set of experts who make no mistakes up to round $t$, and let $W_t = |S_t|$. Note that $W_1=K$, and $W_t \geq 1$ for all rounds $t$, because the perfect expert is always in $S_t$.
If the algorithm makes a mistake at round $t$, then
    $W_{t + 1} \leq W_t/2$
because the majority of experts in $S_t$ is wrong and thus excluded from $S_{t+1}$. It follows that
the algorithm cannot make more than $\log_2 K$ mistakes.
\end{proof}

\begin{remark}\label{FF:rem:technique}
This simple proof introduces a general technique that will be essential in the subsequent proofs:

\begin{itemize}
\item Define a quantity $W_t$ which measures the total remaining amount of ``credibility" of the the experts. Make sure that by definition, $W_1$ is upper-bounded, and $W_t$ does not increase over time. Derive a lower bound on $W_T$ from the existence of a ``good expert".
\item Connect $W_t$ with the behavior of the algorithm: prove that $W_t$ decreases by a constant factor whenever the algorithm makes mistake / incurs a cost.

\end{itemize}
\end{remark}

The guarantee in Theorem~\ref{FF:thm:perfect-expert} is in fact optimal (the proof is left as Exercise~\ref{FF:ex:perfect-expert}):

\begin{theorem}\label{FF:thm:binary-prediction-LB}
Consider binary prediction with experts advice. For any algorithm, any $T$ and any $K$, there is a problem instance with a perfect expert such that the algorithm makes at least $\Omega(\min(T,\log K))$ mistakes.
\end{theorem}

Let us turn to the more realistic case where there is no perfect expert among the committee. The majority vote algorithm breaks as soon as all experts make at least one mistake (which typically happens quite soon).

Recall that the majority vote algorithm fully trusts each expert until his first mistake, and completely ignores him afterwards. When all experts may make a mistake, we need a more granular notion of trust. We assign a \emph{confidence weight} $w_a\geq 0$ to each expert $a$: the higher the weight, the larger the confidence. We update the weights over time, decreasing the weight of a given expert whenever he makes a mistake. More specifically, in this case we multiply the weight by a factor $1-\eps$, for some fixed parameter $\eps>0$. We treat each round as a weighted vote among the experts, and we choose a prediction with a largest total weight. This algorithm is called \textit{Weighted Majority Algorithm} (WMA).

\LinesNotNumbered \SetAlgoLined
\begin{algorithm}
\Parameter{$\eps \in [0, 1]$}
\BlankLine
Initialize the weights $w_i=1$ for all experts.\\
For each round $t$:\\
{\Indp
Make predictions using weighted majority vote based on $w$.\\
For each expert $i$:\\
{\Indp
If the $i$-th expert's prediction is correct, $w_i$ stays the same.\\
Otherwise, $w_i \leftarrow w_i (1 - \eps)$.\\
}
}
\caption{Weighted Majority Algorithm}
\end{algorithm}

To analyze the algorithm, we first introduce some notation. Let $w_t(a)$ be the weight of expert $a$ before round $t$, and let
    $W_t = \sum_{a=1}^K w_t(a)$
be the total weight before round $t$.  Let $S_t$ be the set of experts that made incorrect prediction at round $t$. We will use the following fact about logarithms:
\begin{align}\label{FF:lm:log1m}
\ln(1 - x) < -x
\qquad \forall x\in(0, 1).
\end{align}

From the algorithm, $W_1 = K$ and $W_{T+1} > w_t(a^*) = (1 - \eps)^{\cost^*}$.  Therefore, we have
\begin{equation}\label{FF:eq:wt-w1}
\frac{W_{T+1}}{W_1} > \frac{(1 - \eps)^{\cost^*}}{K}.
\end{equation}
Since the weights are non-increasing, we must have
\begin{equation} \label{FF:eq:mono}
W_{t+1} \leq W_t
\end{equation}
If the algorithm makes mistake at round $t$, then
\begin{align*}
W_{t+1} &= \sum_{a\in [K]} w_{t+1}(a) \\
&= \sum_{a \in S_t} (1-\eps) w_t(a) + \sum_{a \not\in S_t} w_t(a)\\
&= W_t - \eps \sum_{a \in S_t} w_t(a).
\end{align*}
Since we are using weighted majority vote, the incorrect prediction must have the majority vote:
\begin{equation*}
\sum_{a \in S_t} w_t(a) \geq \tfrac{1}{2} W_t.
\end{equation*}
Therefore, if the algorithm makes mistake at round $t$, we have
\begin{equation*}
W_{t+1} \leq (1 - \tfrac{\eps}{2}) W_t.
\end{equation*}
Combining with (\ref{FF:eq:wt-w1}) and (\ref{FF:eq:mono}), we get
\begin{equation*}
\frac{(1-\eps)^{\cost^*}}{K} < \frac{W_{T+1}}{W_1} = \prod_{t=1}^{T} \frac{W_{t+1}}{W_t} \leq (1 - \tfrac{\eps}{2})^M,
\end{equation*}
where $M$ is the number of mistakes. Taking logarithm of both sides, we get
\begin{equation*}
\cost^* \cdot \ln(1-\eps) - \ln{K} < M \cdot \ln(1-\tfrac{\eps}{2}) < M \cdot (-\tfrac{\eps}{2}),
\end{equation*}
where the last inequality follows from \refeq{FF:lm:log1m}.  Rearranging the terms, we get
\begin{equation*}
M < \cost^* \cdot \tfrac{2}{\eps} \ln(\tfrac{1}{1-\eps}) + \tfrac{2}{\eps} \ln{K}
< \tfrac{2}{1 - \eps} \cdot \cost^* + \tfrac{2}{\eps} \cdot \ln{K},
\end{equation*}
where the last step follows from \eqref{FF:lm:log1m} with $x = \frac{\eps}{1 - \eps}$.  To summarize, we have proved:

\begin{theorem}\label{FF:thm:WMA}
The number of mistakes made by WMA with parameter $\eps\in(0,1)$ is at most
\begin{equation*}
\frac{2}{1 - \eps} \cdot \cost^* + \frac{2}{\eps} \cdot \ln{K}.
\end{equation*}
\end{theorem}

\begin{remark}
This bound is very meaningful if $\cost^*$ is small, but it does not imply sublinear regret guarantees when $\cost^* = \Omega(T)$. Interestingly, it recovers the $O(\ln K)$ number of mistakes in the special case with a perfect expert, \ie when $\cost^*=0$.
\end{remark}

\section{Hedge Algorithm}
\label{FF:sec:hedge}

We improve over the previous section in two ways: we solve the general case, online learning with experts, and obtain regret bounds that are $o(T)$ and, in fact, optimal. We start with an easy observation that deterministic algorithms are not sufficient for this goal: essentially, any deterministic algorithm can be easily ``fooled'' (even) by a deterministic, oblivious adversary.

\begin{theorem}\label{FF:thm:deterministic-LB}
Consider online learning with $K$ experts and binary costs. Any deterministic algorithm has total cost $T$ for some deterministic, oblivious adversary, even if $\cost^*\leq T/K$.
\end{theorem}

\noindent The easy proof is left as Exercise~\ref{FF:ex:deterministic-LB}. Essentially, the adversary knows exactly what the algorithm is going to do in the next round, and can rig the costs to hurt the algorithm.

\begin{remark}
Note that the special case of \emph{binary} prediction with experts advice is much easier for deterministic algorithms. Indeed, it allows for an ``approximation ratio" arbitrarily close to $2$, as in Theorem~\ref{FF:thm:WMA}, whereas in the general case the ``approximation ratio" cannot be better than $K$.
\end{remark}

We define a randomized algorithm for online learning with experts, called \Hedge. This algorithm maintains a weight $w_t(a)$ for each arm $a$, with the same update rule as in WMA (generalized beyond 0-1 costs in a fairly natural way). We need to use a different rule to select an arm, because (i) we need this rule to be randomized, in light of Theorem~\ref{FF:thm:deterministic-LB}, and (ii) the weighted majority rule is not even well-defined in the general case. We use another selection rule, which is also very natural: at each round, choose an arm with probability proportional to the weights.  The complete specification is shown in Algorithm~\ref{FF:alg:Hedge}:

\LinesNotNumbered \SetAlgoLined
\begin{algorithm}[H]
\Parameter{$\eps \in (0, \nicefrac12)$}
\BlankLine
Initialize the weights as $w_1(a)=1$ for each arm $a$.\\
For each round $t$:\\
{\Indp
Let $p_t(a) = \frac{w_t(a)}{\sum_{a'=1}^K w_t(a')}$.\\
Sample an arm $a_t$ from distribution $p_t(\cdot)$.\\
Observe cost $c_t(a)$ for each arm $a$.\\
For each arm $a$, update its weight $w_{t+1}(a) = w_t(a)\cdot(1 - \eps)^{c_t(a)}$.\\
}
\caption{\Hedge algorithm for online learning with experts}
\label{FF:alg:Hedge}
\end{algorithm}

Below we analyze \Hedge, and prove $O\rbr{\sqrt{T \log K}}$ bound on expected regret. We use the same analysis to derive several important extensions, used in the subsequent sections on adversarial bandits.

\begin{remark}
The $O\rbr{\sqrt{T \log K}}$ regret bound is the best possible for regret. This can be seen on a simple example in which all costs are IID with mean $\nicefrac12$, see Exercise~\ref{FF:ex:IID}(b). Recall that we also have a $\Omega(\sqrt{T})$ bound for pseudo-regret, due to the lower-bound analysis for two arms in Chapter~\ref{ch:LB}.
\end{remark}

As in the previous section, we use the technique outlined in Remark~\ref{FF:rem:technique}, with
    $W_t = \sum_{a =1}^K w_t(a)$
being the total weight of all arms at round $t$.
Throughout, $\eps\in(0,\nicefrac12)$ denotes the parameter in the algorithm.

The analysis is not very long but intricate; some find it beautiful.
We break it in several distinct steps, for ease of comprehension.

\subsection*{Step 1: easy observations}

The weight of each arm after the last round is
\begin{equation*}
w_{T+1}(a) = w_1(a) \prod_{t=1}^T (1 - \eps)^{c_t(a)} = (1 - \eps)^{\cost(a)}.
\end{equation*}
Hence, the total weight of last round satisfies
\begin{equation}\label{FF:eq:hedge-wt}
W_{T+1} > w_{T+1}(a^*) = (1 - \eps)^{\cost^*}.
\end{equation}
From the algorithm, we know that the total initial weight is
$W_1 = K$.

\subsection*{Step 2: multiplicative decrease in $W_t$}

We use polynomial upper bounds for $(1-\eps)^x$, $x>0$. We use two variants, stated in a unified form:
\begin{align}\label{FF:eq:exp-UB}
    (1 - \eps)^x \leq 1 - \alpha x + \beta x^2
    \quad \text{for all $x \in [0,u]$},
\end{align}
where the parameters $\alpha$, $\beta$ and $u$ may depend on $\eps$ but not on $x$. The two variants are:
\begin{OneLiners}
\item a first-order (\ie linear)  upper bound with
    $(\alpha,\beta,u)=(\eps,0,1)$;
\item a second-order (\ie quadratic) upper bound with
    $\alpha=\ln\rbr{\frac{1}{1 - \eps}}$,
$\beta=\alpha^2$ and $u = \infty$.
\end{OneLiners}

We apply \refeq{FF:eq:exp-UB} to $x=c_t(a)$, for each round $t$ and each arm $a$. We continue the analysis for both variants of \refeq{FF:eq:exp-UB} simultaneously: thus, we fix $\alpha$ and $\beta$ and assume that $c_t(\cdot)\leq u$. Then
\begin{align}
\frac{W_{t+1}}{W_t} &= \sum_{a\in[K]} (1 - \eps)^{c_t(a)} \cdot \frac{w_t(a)}{W_t}\notag\\
&< \sum_{a\in[K]} (1 - \alpha\, c_t(a) + \beta\, c_t(a)^2) \cdot p_t(a)\notag\\
&= \sum_{a\in[K]} p_t(a) - \alpha \sum_{a\in[K]} p_t(a)\, c_t(a) + \beta \sum_{a\in[K]} p_t(a)\, c_t(a)^2\notag\\
&= 1 - \alpha F_t + \beta G_t, \label{FF:eq:intermediate}
\end{align}
where
\begin{align}
F_t &= \sum_{a\in[K]} p_t(a) \cdot c_t(a)
    = \E\sbr{ c_t(a_t) \mid \vec{w}_t} \nonumber \\
G_t &=
    \sum_{a\in[K]} p_t(a) \cdot c_t^2(a)
    = \E\sbr{ c_t^2(a_t) \mid \vec{w}_t}. \label{FF:eq:def-G}
\end{align}
Here $\vec{w}_t = \left( w_t(a): a \in [K] \right)$ is the vector of weights at round $t$. Notice that the total expected cost of the algorithm is
    $\E[\cost(\ALG)] = \sum_t \E[F_t]$.

\subsection*{A naive attempt}

Using the (\ref{FF:eq:intermediate}), we can obtain:
\begin{equation*}
\frac{(1 - \eps)^{\cost^*}}{K}
    \leq \frac{W_{T+1}}{W_1}
    = \prod_{t=1}^{T} \frac{W_{t+1}}{W_t}
    < \prod_{t=1}^T (1 - \alpha F_t + \beta G_t).
\end{equation*}
However, it is unclear how to connect the right-hand side to $\sum_t F_t$ so as to argue about $\cost(\ALG)$.

\subsection*{Step 3: the telescoping product}

Taking a logarithm on both sides of \refeq{FF:eq:intermediate}) and using \refeq{FF:lm:log1m}), we obtain
\begin{equation*}
\ln \frac{W_{t+1}}{W_t} < \ln(1 - \alpha F_t + \beta G_t) < -\alpha F_t + \beta G_t.
\end{equation*}
Inverting the signs and summing over $t$ on both sides, we obtain
\begin{align}
\sum_{t\in [T]} (\alpha F_t - \beta G_t) &< - \sum_{t\in [T]} \ln \frac{W_{t + 1}}{W_t}\nonumber \\
&= -\ln \prod_{t\in [T]} \frac{W_{t+1}}{W_t} \nonumber \\
&= -\ln \frac{W_{T+1}}{W_1}\nonumber  \\
&= \ln W_1 - \ln W_{T+1}\nonumber  \\
&< \ln K - \ln(1 - \eps) \cdot \text{cost}^*,
\label{FF:eq:raw}
\end{align}
where we used \eqref{FF:eq:hedge-wt} in the last step.  Taking expectation on both sides, we obtain:
\begin{align}\label{FF:eq:main}
\alpha\cdot \E[\cost(\ALG)]
    < \beta \sum_{t\in [T]} \E[G_t] + \ln{K} - \ln(1-\eps)\cdot \E[\cost^*].
\end{align}

We use this equation in two different ways, depending on the variant of \refeq{FF:eq:exp-UB}. With $\alpha=\eps$ and $\beta=0$ and $c_t(\cdot)\leq 1$,
we obtain:
\begin{align*}
\E[\cost(\ALG)]
    < \tfrac{\ln{K}}{\eps} +
        \underbrace{\tfrac{1}{\eps}\ln(\tfrac{1}{1 - \eps})}_{\text{$\leq 1+\eps$ if $\eps\in(0,\nicefrac12)$}}
            \E[\cost^*]. \\
\E[\cost(\ALG) - \cost^*] < \frac{\ln{K}}{\eps} + \eps\; \E[\cost^*].
\end{align*}
This yields the main regret bound for \Hedge:
\begin{theorem}\label{FF:thm:hedge-main}
Assume all costs are at most $1$. Consider an adaptive adversary such that
    $\cost^*\leq uT$ for some known number $u$; trivially, $u=1$. Then \Hedge with parameter
    $\eps=\sqrt{\tfrac{\ln K}{uT}}$ satisfies
\[ \E[\cost(\ALG)-\cost^*] < 2\cdot \sqrt{uT \ln{K}}.\]
\end{theorem}

\begin{remark}
We also obtain a meaningful performance guarantee which holds with probability $1$, rather than in expectation. Using the same parameters,
    $\alpha=\eps=\sqrt{\tfrac{\ln K}{T}}$
and $\beta=0$, and assuming all costs are at most $1$, \refeq{FF:eq:raw} implies that
\begin{align}\label{FF:eq:Hedge-prob1}
\textstyle
\sum_{t\in [T]} p_t\cdot c_t -\cost^* < 2 \cdot \sqrt{T \ln{K}}.
\end{align}
\end{remark}

\subsection*{Step 4: unbounded costs}

Next, we consider the case where the costs can be unbounded, but we have an upper bound on $\E[G_t]$. We use \refeq{FF:eq:main} with $\alpha=\ln(\frac{1}{1 - \eps})$ and $\beta=\alpha^2$ to obtain:
\begin{align*}
\alpha \E[\cost(\ALG)] < \alpha^2 \sum_{t\in [T]} \E[G_t] + \ln{K} + \alpha \E[\cost^*].
\end{align*}
Dividing both sides by $\alpha$ and moving terms around, we get
\begin{equation*}
\E[\cost(\ALG)-\cost^*] < \frac{\ln{K}}{\alpha} + \alpha \sum_{t\in [T]} \E[G_t] < \frac{\ln{K}}{\eps} + 3\eps \sum_{t\in [T]} \E[G_t],
\end{equation*}
where the last step uses the fact that $\eps < \alpha < 3\eps$ for $\eps\in(0,\nicefrac12)$.  Thus:

\begin{theorem}\label{FF:thm:unbounded-costs}
Assume $\sum_{t\in [T]}\; \E[G_t] \leq uT$ for some known number $u$, where
    $\E[G_t] = \E\sbr{ c_t^2(a_t)}$
as per \eqref{FF:eq:def-G}.
Then \Hedge with parameter
    $\eps = \sqrt{\frac{\ln K}{3uT}}$
achieves regret
\[ \E[\cost(\ALG)-\cost^*] < 2\sqrt{3}\cdot \sqrt{uT \ln{K}}.\]
In particular, if $c_t(\cdot)\leq c$, for some known $c$, then one can take $u=c^2$.
\end{theorem}


In the next chapter we use this lemma to analyze a bandit algorithm.

\subsection*{Step 5: unbounded costs with small expectation and variance}

Consider a randomized, oblivious adversary such that the costs are independent across rounds. Instead of bounding the actual costs $c_t(a)$, let us instead bound their expectation and variance:
\begin{align}\label{FF:eq:Hedge-Gt-UB}
\E\sbr{c_t(a)} \leq \mu
\text{ and }
\text{Var}\sbr{c_t(a)} \leq \sigma^2
\text{ for all rounds $t$ and all arms $a$}.
\end{align}
Then for each round $t$ we have:
\begin{align*}
\E\sbr{c_t(a)^2}
    &= \text{Var}\sbr{c_t(a)} + \rbr{\E[c_t(a)]}^2 \leq \sigma^2 + \mu^2.\\
\E[G_t]
    &= \sum_{a\in[K]} p_t(a)\, \E\sbr{c_t(a)^2}
        \leq \sum_{a\in[K]} p_t(a)\,(\mu^2 + \sigma^2) = \mu^2 + \sigma^2.
\end{align*}
Thus, Theorem~\ref{FF:thm:unbounded-costs} with $u=\mu^2+\sigma^2$ implies the following:

\begin{corollary}\label{FF:cor:unbounded-costs}
Consider online learning with experts, with a randomized, oblivious adversary. Assume the costs are independent across rounds, and satisfy \eqref{FF:eq:Hedge-Gt-UB} for some $\mu$ and $\sigma$ known to the algorithm. Then \Hedge with parameter
    $\eps = \sqrt{\ln{K} / (3T(\mu^2 + \sigma^2))}$
has regret
\[ \E[\cost(\ALG)-\cost^*] < 2\sqrt{3}\cdot \sqrt{T(\mu^2 + \sigma^2) \ln{K}}.\]
\end{corollary}

\sectionBibNotes

The weighted majority algorithm is from \citet{LittWarm94}, and \Hedge algorithm is from \citep{LittWarm94,experts-jacm97,FS97}.%
\footnote{\citet{FS97} handles the full generality of online learning with experts. \citet{LittWarm94} and \cite{experts-jacm97} focus on binary prediction with experts advice, with slightly stronger guarantees.}
 Sequential prediction with experts advise has a long history in economics and statistics, see \citet[Chapter 2]{CesaBL-book} for deeper discussion and bibliographic notes.

The material in this chapter is presented in many courses and books on online learning. This chapter mostly follows the lecture plan from \citep[][Week 1]{Bobby-class07}, but presents the analysis of \Hedge a little differently, so as to make it immediately applicable to the analysis of \ExpThree and \ExpFour in the next chapter.

The problem of online learning with experts, and \Hedge algorithm for this problem, are foundational in a very strong sense. \Hedge serves as a subroutine in adversarial bandits (Chapter~\ref{ch:adv}) and bandits with global constraints (Chapter~\ref{ch:BwK}). The multiplicative-weights technique is essential in several extensions of adversarial bandits, \eg the ones in \citet{bandits-exp3} and \citet{exp4p}. \Hedge also powers many applications ``outside" of multi-armed bandits or online machine learning, particularly in the design of primal-dual algorithms (see \citet{AroraHK} for a survey) and via the ``learning in games" framework (see Chapter~\ref{ch:games} and bibliographic remarks therein). In many of these applications, \Hedge can be replaced by any algorithm for online learning with experts, as long as it satisfies an appropriate regret bound.

While the regret bound for \Hedge comes with a fairly small constant, it is still a multiplicative constant away from the best known lower bound. Sometimes one can obtain upper and lower bounds on regret that match exactly. This has been accomplished for $K=2$ experts
\citep{Cover-1965,Gravin-soda16} and for $K\leq 3$ experts with a geometric time horizon%
\footnote{That is, the game stops in each round independently with probability $\delta$.} 
\citep{Gravin-soda16}, along with an explicit specification of the optimal algorithm.
Focusing on the regret of \Hedge, \citet{Gravin-icalp17} derived an improved the \emph{lower} bound for each $K$, exactly matching a known upper bound from \citet{experts-jacm97}.

Various stronger notions of regret have been studied; we detail them in the next chapter (Section~\ref{adv:sec-further}).

Online learning with $K$ experts can be interpreted as one with action set $\Delta_K$, the set of all distributions over the $K$ experts. This is a special case of \emph{online linear optimization} (OLO) with a convex action set (see Chapter~\ref{ch:lin}). In fact, \Hedge can be interpreted as a special case of two broad families of algorithms for OLO, \emph{follow the regularized leader} and \emph{online mirror descent}; for background, see \eg \citet{Hazan-OCO-book} and \citet{McMahan-survey-jmlr17}. This OLO-based perspective has been essential in much of the recent progress.



\sectionExercises

\begin{exercise}\label{FF:ex:perfect-expert}
Consider binary prediction with expert advice, with a perfect expert. Prove Theorem~\ref{FF:thm:binary-prediction-LB}: prove that any algorithm makes at least $\Omega(\min(T,\log K))$ mistakes in the worst case.

    \TakeAway{The majority vote algorithm is worst-case-optimal for instances with a perfect expert.}

    \Hint{For simplicity, let $K=2^d$ and $T\geq d$, for some integer $d$. Construct a distribution over problem instances such that each algorithm makes $\Omega(d)$ mistakes in expectation. Recall that each expert $e$ corresponds to a binary sequence $e\in \{0,1\}^T$, where $e_t$ is the prediction for round $t$. Put experts in 1-1 correspondence with all possible binary sequences for the first $d$ rounds. Pick the ``perfect expert" u.a.r. among the experts.}

\end{exercise}

\begin{exercise}[IID costs and regret]\label{FF:ex:IID}
Assume IID costs, as in Remark~\ref{FF:rem:IID}.
\begin{itemize}
\item[(a)] Prove that
    $  \min_a \E[\cost(a)]
    \leq \E[\min_a \cost(a)] + O(\sqrt{T\log(KT)})$.

\TakeAway{All $\sqrt{T}$-regret bounds from Chapter~\ref{ch:IID} carry over to regret.}

\Hint{Define the ``clean event" as follows: the event in Hoeffding inequality holds for the cost sequence of each arm.}

\item[(b)] Construct a problem instance with a deterministic adversary  for which any algorithm suffers regret
    \begin{align*}
        \E[\cost(\ALG) - \min_{a\in[K]} \cost (a)]
            \geq \Omega(\sqrt{T\, \log K}).
    \end{align*}

    \Hint{Assume all arms have 0-1 costs with mean $\nicefrac12$.
Use the following fact about random walks:
        \begin{align} \E[\min_a \cost(a)] \leq \tfrac{T}{2}-\Omega(\sqrt{T\,\log K}). \end{align}}

    \Note{This example does not carry over to pseudo-regret. Since each arm has expected reward of $\nicefrac12$ in each round, any algorithm trivially achieves $0$ pseudo-regret for this problem instance.}

    \TakeAway{The $O(\sqrt{T\log K})$ regret bound for \Hedge is the best possible for regret. Further, $\log(T)$ upper regret bounds from Chapter~\ref{ch:IID} do not carry over to regret in full generality.}

\item[(c)] Prove that algorithms \UCB and Successive Elimination achieve logarithmic regret bound \eqref{IID:eq:thm:SE-logT} even for regret, assuming that the best-in-foresight arm $a^*$ is unique.

    \Hint{Under the ``clean event",
        $\cost(a)<T\cdot \mu(a)+O(\sqrt{T\log T})$
    for each arm $a\neq a^*$, where $\mu(a)$ is the mean cost. It follows that $a^*$ is also the best-in-hindsight arm, unless
        $\mu(a)-\mu(a^*)<O(\sqrt{T\log T})$ for some arm $a\neq a^*$ (in which case the claimed regret bound holds trivially). }

\end{itemize}
\end{exercise}

\begin{exercise}\label{FF:ex:deterministic-LB}
Prove Theorem~\ref{FF:thm:deterministic-LB}: prove that any deterministic algorithm for the online learning problem with $K$ experts and 0-1 costs can suffer total cost $T$ for some deterministic-oblivious adversary, even if $\cost^*\leq T/K$.

\TakeAway{With a deterministic algorithm, cannot even recover the guarantee from Theorem~\ref{FF:thm:WMA} for the general case of online learning with experts, let alone have $o(T)$ regret.}

\Hint{Fix the algorithm. Construct the problem instance by induction on round $t$, so that the chosen arm has cost $1$ and all other arms have cost $0$.}
\end{exercise}

\chapter{Adversarial Bandits}
\label{ch:adv}
\begin{ChAbstract}
This chapter is concerned with \emph{adversarial bandits}: multi-armed bandits with adversarially chosen costs. In fact, we solve a more general formulation that explicitly includes expert advice. Our algorithm is based on a reduction to the full-feedback problem studied in Chapter~\ref{ch:FF}.

\prereqs{Chapter~\ref{ch:FF}.}
\end{ChAbstract}

For ease of exposition, we focus on deterministic, oblivious adversary: that is, the costs for all arms and all rounds are chosen in advance. We are interested in regret with respect to the best-in-hindsight arm, as defined in \refeq{FF:eq:experts-regret-hindsight}. We assume bounded per-round costs: $c_t(a)\leq 1$ for all rounds $t$ and all arms $a$.

We achieve regret bound
    $\E[R(T)] \leq O\left( \sqrt{KT\log K} \right)$.
Curiously, this upper regret bound not only matches our result for IID bandits (Theorem~\ref{IID:thm:SE-sqrtT}), but in fact improves it a little bit, replacing the $\log T$ with $\log K$. This regret bound is essentially optimal, due to the  $\Omega(\sqrt{KT})$ lower bound from Chapter~\ref{ch:LB}.

\section*{Recap from Chapter~\ref{ch:FF}}

Let us recap the material on the full-feedback problem, reframing it slightly for this chapter. Recall that in the full-feedback problem, the cost of each arm is revealed after every round. A common interpretation is that each action corresponds to an ``expert" that gives advice or makes predictions, and in each round the algorithm needs to choose which expert to follow. Hence, this problem is also known as the \emph{online learning with experts}. We considered a particular algorithm for this problem, called \Hedge. In each round $t$, it computes a distribution $p_t$ over experts, and samples an expert from this distribution. We obtained the following regret bound (Theorem~\ref{FF:thm:unbounded-costs}):

\begin{theorem}\label{adv:thm:L8-hedge}
Consider online learning with $N$ experts and $T$ rounds. Consider an adaptive adversary and regret $R(T)$ relative to the best-observed expert. Suppose in any run of \Hedge, with any parameter $\eps>0$, it holds that
    $\sum_{t\in [T]}\; \E[G_t] \leq uT$
for some known $u>0$,
where
    $G_t = \sum_{\text{experts $e$}} p_t(e)\; c_t^2(e)$.
Then
  \[ \E[R(T)]  \leq 2\sqrt{3}\cdot \sqrt{ u T \log N }
  \qquad\text{provided that}\quad \eps = \eps_u := \sqrt{\tfrac{\ln N}{3uT}}. \]
\end{theorem}

We will distinguish between ``experts" in the full-feedback problem and ``actions" in the bandit problem. Therefore, we will consistently use ``experts" for the former and ``actions/arms" for the latter.

\newpage
\section{Reduction from bandit feedback to full feedback}

Our algorithm for adversarial bandits is a reduction to the full-feedback setting. The reduction proceeds as follows. For each arm, we create an expert which always recommends this arm. We use \Hedge with this set of experts. In each round $t$, we use the expert $e_t$ chosen by \Hedge to pick an arm $a_t$, and define ``fake costs" $\widehat{c}_t(\cdot)$ on all experts in order to provide \Hedge with valid inputs. This generic reduction is given below:

\begin{algorithm}[H]
{\bf Given}: set $\Experts$ of experts, parameter $\eps\in(0,\tfrac12)$ for \Hedge.

In each round $t$,
\begin{enumerate}
\item Call \Hedge, receive the probability distribution $p_t$ over $\Experts$.
\item Draw an expert $e_t$ independently from $p_t$.
\item \emph{Selection rule}: use $e_t$ to pick arm $a_t$ (TBD).
\item Observe the cost $c_t(a_t)$ of the chosen arm.
\item Define ``fake costs'' $\widehat{c}_t(e)$ for all experts $x \in \Experts$ (TBD).
\item Return the ``fake costs" to \Hedge.
\end{enumerate}
\caption{Reduction from bandit feedback to full feedback}
\label{adv:alg:L8-reduction}
\end{algorithm}

\noindent We will specify \emph{how} to select arm $a_t$ using expert $e_t$, and \emph{how} to define fake costs. The  former provides for sufficient exploration, and the latter ensures that fake costs are unbiased estimates of the true costs.

\section{Adversarial bandits with expert advice}

The reduction defined above suggests a more general problem: what if experts can predict different arms in different rounds? This problem, called \emph{bandits with expert advice}, is one that we will actually solve. We do it for three reasons: because it is a very interesting generalization, because we can solve it with very little extra work, and because separating experts from actions makes the solution clearer. Formally, the problem is defined as follows:

\begin{BoxedProblem}{Adversarial bandits with expert advice}
Given: $K$ arms, set $\mE$ of $N$ experts, $T$ rounds (all known).\\
In each round $t \in [T]$:
\begin{OneLiners}
\item[1.] adversary picks cost $c_t(a)\geq 0$ for each arm $a\in [K]$,
\item[2.] each expert $e\in \mE$ recommends an arm $a_{t,e}$ (observed by the algorithm),
\item[3.] algorithm picks arm $a_t\in [K]$ and receives the corresponding cost $c_t(a_t)$.
\end{OneLiners}
\end{BoxedProblem}
The total cost of each expert is defined as
     $\cost(e) = \sum_{t\in [T]} c_t(a_{t,e})$.
The goal is to minimize regret relative to the best \emph{expert}, rather than the best action:
     \[ R(T) = \cost(\text{\ALG}) - \min_{e\in\mE}{\cost(e)}. \]

\noindent We focus on a deterministic, oblivious adversary: all costs $c_t(\cdot)$ are selected in advance. Further, we assume that the recommendations $a_{t,e}$ are also chosen in advance, \ie the experts cannot learn over time. 

We solve this problem via the reduction in Algorithm~\ref{adv:alg:L8-reduction}, and achieve
\[      \E[R(T)] \leq O\rbr{\sqrt{KT \log N}}. \]
\noindent Note the logarithmic dependence on $N$: this regret bound allows to handle \emph{lots} of experts.

This regret bound is essentially the best possible. Specifically, there is a nearly  matching lower bound on regret that holds for any given triple of parameters $K,T,N$:
\begin{align}\label{adv:eq:exp4-LB}
    \E[R(T)] \geq \min\rbr{T,\; \Omega\rbr{\sqrt{KT \log(N)/\log(K)} }}.
\end{align}
This lower bound can be proved by an ingenious (yet simple) reduction to the basic $\Omega(\sqrt{KT})$ lower regret bound for bandits, see Exercise~\ref{adv:ex:exp4-LB}.



\section{Preliminary analysis: unbiased estimates}

We have two notions of ``cost" on experts. For each expert $e$ at round $t$, we have the true cost $c_t(e) = c_t(a_{t,e})$ determined by the predicted arm $a_{t,e}$, and the \emph{fake cost} $\widehat{c}_t(e)$ that is computed inside the algorithm and then fed to \Hedge. Thus, our regret bounds for \Hedge refer to the \emph{fake regret} defined relative to the fake costs:
\[ \widehat{R}_{\Hedge}(T) = \widehat{\cost}(\Hedge) - \min_{e\in\Experts} \widehat{\cost}(e),\]
where
    $\widehat{\cost}(\Hedge)$ and $\widehat{\cost}(e)$
are the total fake costs for \Hedge and expert $e$, respectively.

We want the fake costs to be unbiased estimates of the true costs. This is because we will need to convert a bound on the fake regret $\widehat{R}_{\Hedge}(T)$ into a statement about the true costs accumulated by \Hedge. Formally, we ensure that
\begin{align}\label{adv:eq:exp4-unbiased}
 \E\sbr{\widehat{c}_t(e) \mid \vec{p}_t} = c_t(e) \quad
    \text{for all experts $e$},
\end{align}
where $\vec{p}_t = (p_t(e):\; \text{all experts $e$})$. We use this as follows:

\begin{claim}\label{adv:cl:Exp4-regrets}
Assuming \refeq{adv:eq:exp4-unbiased}, it holds that
    $\E[R_{\Hedge}(T)] \leq \E[\widehat{R}_{\Hedge}(T)]$.
\end{claim}

\begin{proof}
First, we connect true costs of \Hedge with the corresponding fake costs.
\begin{align*}
\E\sbr{\widehat{c}_t(e_t) \mid \vec{p}_t}
    &= \textstyle \sum_{e\in\Experts}
        \Pr\sbr{e_t=e \mid \vec{p}_t }\; \E\sbr{\widehat{c}_t(e) \mid \vec{p}_t}
        \\
    &= \textstyle \sum_{e\in\Experts} p_t(e)\; c_t(e)
        &\EqComment{use definition of $p_t(e)$ and \refeq{adv:eq:exp4-unbiased}} \\
    &= \E\sbr{c_t(e_t)\mid \vec{p}_t}.
\end{align*}
Taking expectation of both sides,
    $\E[\widehat{c}_t(e_t)] = \E[c_t(e_t)] $.
Summing over all rounds, it follows that
\begin{align*}
    \E\sbr{\widehat{\cost}(\Hedge)} = \E\sbr{\cost(\Hedge)}.
\end{align*}

To complete the proof, we deal with the benchmark:
\[ \E\sbr{\min_{e\in\Experts} \widehat{\cost}(e)}
    \leq \min_{e\in\Experts} \E\sbr{\widehat{\cost}(e)}
    = \min_{e\in\Experts} \E\sbr{\cost(e)}
    = \min_{e\in\Experts} \cost(e).\]
The first equality holds by \eqref{adv:eq:exp4-unbiased}, and the second equality holds because true costs $c_t(e)$ are deterministic.
\end{proof}

\begin{remark}
This proof used the ``full power" of assumption \eqref{adv:eq:exp4-unbiased}. A weaker assumption
    $\E\sbr{\widehat{c}_t(e)} = \E[c_t(e)] $
would not have sufficed to argue about true vs. fake costs of \Hedge.
\end{remark}

\section{Algorithm \ExpFour and crude analysis}
\label{sec:BwK-Exp4}

To complete the specification of Algorithm~\ref{adv:alg:L8-reduction}, we need to define fake costs $\widehat{c}_t(\cdot)$ and specify how to choose an arm $a_t$. For fake costs, we will use a standard trick in statistics called \emph{Inverse Propensity Score} (IPS). Whichever way arm $a_t$ is chosen in each round $t$ given the probability distribution $\vec{p}_t$ over experts, this defines distribution $q_t$ over arms:
    \[ q_t(a) := \Pr\sbr{ a_t= a \mid \vec{p}_t} \quad \text{for each arm $a$}. \]
Using these probabilities, we define the fake costs on each arm as follows:
\[ \widehat{c}_t(a)= \left\{
        \begin{array}{ll}
            c_t(a_t) / q_t(a_t) & \quad a_t = a, \\
            0 & \quad \text{otherwise}.
        \end{array}
    \right.\]
The fake cost on each expert $e$ is defined as the fake cost of the arm chosen by this expert:
    $\widehat{c}_t(e) = \widehat{c}_t(a_{t,e})$.

\begin{remark}
Algorithm~\ref{adv:alg:L8-reduction} can use fake costs as defined above as long as it can compute probability $q_t(a_t)$.
\end{remark}

\begin{claim}\label{adv:cl:unbiased-holds}
\refeq{adv:eq:exp4-unbiased} holds if $q_t(a_{t,e})>0$ for each expert $e$.
\end{claim}

\begin{proof}
Let us argue about each arm $a$ separately. If $q_t(a)>0$ then
\[ \E\sbr{\widehat{c}_t(a) \mid \vec{p}_t}
    = \Pr\sbr{ a_t = a \mid \vec{p}_t}
        \cdot \frac{c_t(a_t)}{q_t(a)}
        + \Pr\sbr{a_t \neq a \mid \vec{p}_t} \cdot 0
    = c_t(a).\]
For a given expert $e$ plug in arm $a=a_{t,e}$, its choice in round $t$.
\end{proof}

So, if an arm $a$ is selected by some expert in a given round $t$, the selection rule needs to choose this arm with non-zero probability, regardless of which expert is actually chosen by \Hedge and what is this expert's recommendation. Further, if probability $q_t(a)$ is sufficiently large, then one can upper-bound fake costs and apply Theorem~\ref{adv:thm:L8-hedge}. On the other hand, we would like to follow the chosen expert $e_t$ most of the time, so as to ensure low costs. A simple and natural way to achieve both objectives is to follow $e_t$ with probability $1-\gamma$, for some small $\gamma>0$, and with the remaining probability choose an arm uniformly at random. This completes the specification of our algorithm, which is known as \ExpFour. We recap it in Algorithm~\ref{adv:alg:exp4}.

\LinesNotNumbered \SetAlgoLined
\begin{algorithm}[t]
{\bf Given}: set $\Experts$ of experts, parameter $\eps\in(0,\tfrac12)$ for \Hedge, exploration parameter $\gamma\in[0,\tfrac12)$.

In each round $t$,
\begin{enumerate}
\item Call \Hedge, receive the probability distribution $p_t$ over $\Experts$.
\item Draw an expert $e_t$ independently from $p_t$.
\item \emph{Selection rule}: with probability $1-\gamma$ follow expert $e_t$; else pick an arm $a_t$ uniformly at random.
\item Observe the cost $c_t(a_t)$ of the chosen arm.
\item Define fake costs for all experts $e$:
    \[ \widehat{c}_t(e)= \left\{
        \begin{array}{ll}
            \frac{c_t(a_t)}{\Pr[ a_t = a_{t,e} \mid \vec{p}_t]}
                & \quad a_t = a_{t,e}, \\
            0 & \quad \text{otherwise}.
        \end{array}
    \right. \]
\item Return the ``fake costs" $\widehat{c}(\cdot)$ to \Hedge.
\end{enumerate}
\caption{Algorithm \ExpFour for adversarial bandits with experts advice}
\label{adv:alg:exp4}
\end{algorithm}

Note that $q_t(a) \geq \gamma/K > 0$ for each arm $a$. According to Claim~\ref{adv:cl:unbiased-holds} and Claim~\ref{adv:cl:Exp4-regrets}, the expected true regret of \Hedge is upper-bounded by its expected fake regret:
    $\E[R_{\Hedge}(T)] \leq \E[\widehat{R}_{\Hedge}(T)]$.

\begin{remark}
Fake costs $\widehat{c}_t(\cdot)$ depend on the probability distribution $\hat{p}_t$ chosen by \Hedge. This distribution depends on the actions selected by \ExpFour in the past, and these actions in turn depend on the experts chosen by \Hedge in the past. To summarize, fake costs depend on the experts chosen by \Hedge in the past. So, fake costs do not form an oblivious adversary, as far as \Hedge is concerned. Thus, we need regret guarantees for \Hedge against an adaptive adversary, even though the true costs are chosen by an oblivious adversary.
\end{remark}

In each round $t$, our algorithm accumulates cost at most $1$ from the low-probability exploration, and cost $c_t(e_t)$ from the chosen expert $e_t$. So the expected cost in this round is
    $\E[c_t(a_t)] \leq \gamma + \E[c_t(e_t)]$.
Summing over all rounds, we obtain:
\begin{align}
 \E\sbr{\cost(\ExpFour)}
    &\leq \E\sbr{\cost(\Hedge)}+\gamma T. \nonumber \\
 \E\sbr{R_\ExpFour(T)}
    &\leq \E\sbr{R_{\Hedge}(T)} + \gamma T
    \leq \E\sbr{\widehat{R}_{\Hedge}(T)} + \gamma T.
    \label{adv:eq:exp4-reduction-regret}
\end{align}
\refeq{adv:eq:exp4-reduction-regret} quantifies the sense in which the regret bound for \ExpFour reduces to the regret bound for \Hedge.

We can immediately derive a crude regret bound via Theorem~\ref{adv:thm:L8-hedge}. Observing
    $\widehat{c}_t(a)\leq 1/q_t(a)\leq K/\gamma$,
we can take $u=(K/\gamma)^2$ in the theorem, and conclude that
\[ \E\sbr{R_\ExpFour(T)}
    \leq  O\rbr{ \nicefrac{K}{\gamma}\cdot K^{1/2}\;(\log N)^{1/4} + \gamma T }. \]
To (approximately) minimize expected regret, choose  $\gamma$ so as to equalize the two summands.

\begin{theorem}\label{adv:thm:exp4-crude}
Consider adversarial bandits with expert advice, with a deterministic-oblivious adversary. Algorithm \ExpFour with parameters
    $\gamma = T^{-1/4}\; K^{1/2}\;(\log N)^{1/4}$
and
    $\eps = \eps_u$, $u=K/\gamma$,
achieves regret
    \[ \E\sbr{R(T)} = O\rbr{ T^{3/4} \;K^{1/2}\;(\log N)^{1/4}}. \]
\end{theorem}

\begin{remark} We did not use any property of \Hedge other than the regret bound in Theorem~\ref{adv:thm:L8-hedge}. Therefore, \Hedge can be replaced with any other full-feedback algorithm with the same regret bound.
\end{remark}
\section{Improved analysis of \ExpFour}

We obtain a better regret bound by analyzing the quantity
\[ \textstyle  \widehat{G}_t := \sum_{e\in\Experts} p_t(e) \;\widehat{c}_t^2(e).\]
We prove that
    $\E[G_t]\leq \tfrac{K}{1-\gamma}$,
and use the regret bound for \Hedge, Theorem~\ref{adv:thm:L8-hedge}, with $u=\tfrac{K}{1-\gamma}$. In contrast, the crude analysis presented above used Theorem~\ref{adv:thm:L8-hedge} with $u=(\nicefrac{K}{\gamma})^2$.

\begin{remark}
This analysis extends to $\gamma=0$. In other words, the uniform exploration step in the algorithm is not necessary. While we previously used $\gamma>0$ to guarantee that $q_t(a_{t,e})>0$ for each expert $e$, the same conclusion also follows from the fact that \Hedge chooses each expert with a non-zero probability. \end{remark}

\newpage
\begin{lemma}
Fix parameter $\gamma\in[0,\tfrac12)$ and round $t$. Then
    $\E[\widehat{G}_t]\leq \tfrac{K}{1-\gamma}$.
\end{lemma}

\begin{proof}
For each arm $a$, let $\Experts_a = \{ e \in \Experts : a_{t,e} = a \}$ be the set of all experts that recommended this arm.  Let
$$p_t(a) := \sum\limits_{e \in \Experts_a} p_t(e)$$
be the probability that the expert chosen by \Hedge recommends arm $a$. Then
\[ q_t(a) = p_t(a) (1- \gamma) + \frac{\gamma}{K} \geq (1-\gamma)\; p_t(a).\]
For each expert $e$, letting $a=a_{t,e}$ be the recommended arm, we have:
\begin{align}\label{adv:eq:fake-UB}
\widehat{c}_t(e) = \widehat{c}_t(a) \leq \frac{c_t(a)}{q_t(a)}
    \leq \frac{1}{q_t(a)} \leq \frac{1}{(1-\gamma)\;p_t(a)}.
\end{align}
Each realization of $\widehat{G}_t$ satisfies:
\begin{align*}
\widehat{G}_t
    &:=\sum_{e\in\Experts} p_t(e)\; \widehat{c}_t^2(e) \\
    &= \sum\limits_a \sum\limits_{e \in \Experts_a} p_t(e)\cdot \widehat{c}_t(e)\cdot \widehat{c}_t(e)
        &\EqComment{re-write as a sum over arms} \\
    &\leq \sum\limits_a \sum\limits_{e \in \Experts_a} \frac{p_t(e)}{(1-\gamma)\; p_t(a)}\; \widehat{c}_t(a)
        &\EqComment{replace one $\widehat{c}_t(a)$ with an upper bound \eqref{adv:eq:fake-UB}} \\
    &= \frac{1}{1-\gamma} \;\sum\limits_a \frac{\widehat{c}_t(a)}{p_t(a)} \sum\limits_{e \in \Experts_a} p_t(e)
        &\EqComment{move ``constant terms" out of the inner sum} \\
    &=\frac{1}{1-\gamma}\;\sum\limits_a \widehat{c}_t(a)
        &\EqComment{the inner sum is just $p_t(a)$}
\end{align*}
To complete the proof, take expectations over both sides and recall that
    $\E[ \widehat{c}_t(a) ] = c_t(a)\leq 1$.
\end{proof}

Let us complete the analysis, being slightly careful with the multiplicative constant in the regret bound:
\begin{align}
    \E\sbr{\widehat{R}_{\Hedge}(T)}
        &\leq 2 \sqrt{3/(1-\gamma)} \cdot \sqrt{TK \log N} \nonumber \\
    \E\sbr{R_{\ExpFour}(T)}
        &\leq 2 \sqrt{3/(1-\gamma)}\cdot \sqrt{TK \log N} +\gamma T
        &\EqComment{by \refeq{adv:eq:exp4-reduction-regret}} \nonumber \\
        &\leq 2 \sqrt{3}\cdot \sqrt{TK \log N} +2\gamma T
        &\EqComment{since $\sqrt{1/(1-\gamma)}\leq 1+\gamma$}
        \label{adv:eq:Exp4-regret-gamma}
\end{align}
(To derive \eqref{adv:eq:Exp4-regret-gamma}, we assumed w.l.o.g. that
    $2 \sqrt{3}\cdot \sqrt{TK \log N} \leq T$.)
This holds for any $\gamma>0$. Therefore:

\begin{theorem}\label{adv:thm:exp4-gamma}
Consider adversarial bandits with expert advice, with a deterministic-oblivious adversary. Algorithm \ExpFour with parameters $\gamma\in[0,\tfrac{1}{2T})$
and
    $\eps = \eps_U$, $U=\tfrac{K}{1-\gamma}$,
achieves regret
 \[ \E[R(T)] \leq 2 \sqrt{3}\cdot \sqrt{TK \log N} +1. \]
\end{theorem}

\OMIT{ 
\begin{remark}
Even though uniform exploration in the selection rule is crucial for our analysis, we can remove it from the algorithm! More precisely, we can pick a sufficiently small exploration parameter $\gamma>0$, define fake costs exactly as they would be defined in \ExpFour, but modify the selection rule so that we always follow the chosen expert $e_t$. The modified algorithm attains the same regret bound, see Exercise~\ref{adv:ex:Exp4}.
\end{remark}
} 

\OMIT{ 

Note that we can choose parameter $\gamma$ to be an arbitrary strictly positive number. So if we make $\gamma$ really small, the algorithm's behavior will essentially coincide with that for $\gamma=0$. Thus, we do not need the uniform exploration step after all! More formally, let $\ExpFour(\gamma)$ denote the algorithm with parameter $\gamma$. Then:

\begin{claim}
$\ExpFour(\gamma)$ is the same algorithm as $\ExpFour(0)$, with probability $1-\gamma T$.
\end{claim}
\begin{proof}
The two algorithms are different only if there exists a round in which $\ExpFour(\gamma)$ chooses uniform exploration. Such round exists with probability at most $\gamma T$.
\end{proof}

Therefore, the difference in expected regret between $\ExpFour(0)$ and $\ExpFour(\gamma)$ is at most
    $(\gamma T)\cdot T$.
Thus, using Theorem~\ref{adv:thm:exp4-gamma} with $\gamma\in(0, 1/2\,T^2)$, we obtain a regret bound for $\gamma=0$:

\begin{theorem}\label{adv:thm:exp4-gamma-zero}
Algorithm \ExpFour with parameter $\gamma=0$ achieves regret
$$\E[R(T)] \leq 2 \sqrt{3}\cdot \sqrt{TK \log N} +1.$$
\end{theorem}
} 

\sectionBibNotes
\label{adv:sec-further}

\ExpFour stands for {\bf exp}loration, {\bf exp}loitation, {\bf exp}onentiation, and {\bf exp}erts.  The specialization to adversarial bandits (without expert advice, \ie with experts that correspond to arms) is called \ExpThree, for the same reason. Both algorithms were introduced (and named) in the seminal paper \citep{bandits-exp3}, along with several extensions. Their analysis is presented in various books and courses on online learning \cite[e.g.,][]{CesaBL-book,Bubeck-survey12}. Our presentation was most influenced by \citep[][Week 8]{Bobby-class07}, but the reduction to \Hedge is made more explicit.

Apart from the stated regret bound, \ExpFour can be usefully applied in several extensions: contextual bandits (Chapter~\ref{ch:CB}), shifting regret, and dynamic regret for slowly changing costs (both: see below).

The lower bound \eqref{adv:eq:exp4-LB} for adversarial bandits with expert advice is due to \citet{RegressorElim-aistats12}. We used a slightly simplified construction from \citet{Seldin-ewrl16} in the hint for  Exercise~\ref{adv:ex:exp4-LB}.

\xhdr{Running time.}
The running time of \ExpFour is $O(N+K)$, so the algorithm becomes very slow when $N$, the number of experts, is very large. Good regret \emph{and} good running time can be obtained for some important special cases with a large $N$. One approach is to replace \Hedge with a different algorithm for online learning with experts which satisfies one or both regret bounds in Theorem~\ref{adv:thm:L8-hedge}. We follow this approach in the next chapter. On the other hand, the running time for \ExpThree is very nice because in each round, we only need to do a small amount of computation to update the weights.

\subsection{Refinements for the ``standard" notion of regret}

Much research has been done on various refined guarantees for adversarial bandits, using the notion of regret defined in this chapter. The most immediate ones are as follows:
\begin{itemize}
\item an algorithm that obtains a similar regret bound for adaptive adversaries (against the best-observed arm), and high probability. This is Algorithm \term{EXP3.P.1} in the original paper \citet{bandits-exp3}.

\item an algorithm with $O(\sqrt{KT})$ regret, shaving off the $\sqrt{\log K}$ factor and matching the lower bound up to constant factors \citep{Bubeck-colt09}.

\item While we have only considered a finite number of experts, similar results can be obtained for \emph{infinite} classes of experts with some special structure. In particular, borrowing the tools from statistical learning theory, it is possible to handle classes of experts with a small VC-dimension.
\end{itemize}

\emph{Data-dependent} regret bounds provide improvements if the realized costs are, in some sense, ``nice", even though the extent of this ``niceness" is not revealed to the algorithm. Such regret bounds come in many flavors, discussed below.

\begin{itemize}

\item \emph{Small benchmark}: the total expected cost/reward of the best arm, denoted $B$. The $\tildeO(\sqrt{KT})$ regret rate can be improved to $\tildeO(\sqrt{KB})$, without knowing $B$ in advance. The reward-maximizing version, when small $B$ means that the best arm is that good, has been solved in the original paper \citep{bandits-exp3}. In the cost-minimizing version, small $B$ has an opposite meaning: the best arm \emph{is} quite good. This version, a.k.a. \emph{small-loss} regret bounds, has been much more challenging \citep{Allenberg-alt06,Rakhlin-colt13,Neu-colt15,Lykouris-nips16,Lykouris-colt18}.

\item \emph{Small change in costs:}
one can obtain regret bounds that are near-optimal in the worst case, and improve when the realized cost of each arm $a$ does not change too much. Small change in costs can be quantified in terms of the total variation
        $ \sum_{t,a} \rbr{ c_t(a)-\cost(a)/T }^2 $
    \citep{Hazan-soda09}, or in terms of path-lengths $\sum_t |c_t(a)-c_{t-1}(a)|$
\citep{Haipeng-colt18,BubeckHaipeng-pathlegth-2019}.
However, these guarantees do not do much when only the change in \emph{expected} costs is small, \eg for IID costs.

\item \emph{Best of both worlds:}
algorithms that work well for both adversarial and stochastic bandits. \citet{BestofBoth-colt12} achieve an algorithm that is near-optimal in the worst case (like \ExpThree), and achieves logarithmic regret (like \UcbOne) if the costs are actually IID Further work refines and optimizes the regret bounds, achieves them with a more practical algorithm, and improves them for the adversarial case if the adversary is constrained
    \citep{BestofBoth-icml14,Auer-colt16,Seldin-colt17,Haipeng-colt18,Haipeng-BoB019,Seldin-aistats19}.

\item \emph{Small change in expected costs.}
A particularly clean model unifies the themes of ``small change" and ``best of both worlds" discussed above. It focuses on the total change in \emph{expected} costs, denoted $C$ and interpreted as
an \emph{adversarial corruption} of an otherwise stochastic problem instance. The goal here is regret bounds that are logarithmic when $C=0$ (\ie for IID costs) and degrade gracefully as $C$ increases, even if $C$ is not known to the algorithm. Initiated in \citet{Thodoris-stoc18}, this direction continued in \citep{Gupta-colt19,Zimmert-jmlr21} and a number of follow-ups.

\end{itemize}

\subsection{Stronger notions of regret}

Let us consider several benchmarks that are \emph{stronger} than the best-in-hindsight arm in \eqref{FF:eq:experts-regret-hindsight}.

\xhdr{Shifting regret.} One can compete with ``policies" that can change the arm from one round to another, but not too often. More formally, an \emph{$S$-shifting policy} is sequence of arms $\pi=(a_t: t\in [T])$ with at most $S$ ``shifts": rounds $t$ such that $a_t\neq a_{t+1}$. \emph{$S$-shifting regret} is defined as the algorithm's total cost minus the total cost of the best $S$-shifting policy:
        \[ R_S(T) = \cost(\ALG) - \min_{\text{$S$-shifting policies $\pi$}} \cost(\pi). \]

Consider this as a bandit problem with expert advice, where each $S$-shifting policy is an expert. The number of experts $N\leq (KT)^S$; while it may be a large number, $\log(N)$ is not too bad! Using \ExpFour and plugging $N\leq (KT)^S$ into Theorem~\ref{adv:thm:exp4-gamma}, we obtain
         $\E[R_S(T)] = \tildeO(\sqrt{KST})$.

While \ExpFour is computationally inefficient, \citet{bandits-exp3} tackle shifting regret using a modification of \ExpThree algorithm, with essentially the same running time as \ExpThree and a more involved analysis. They obtain
    $\E[R_S(T)] = \tildeO(\sqrt{SKT})$
if $S$ is known, and
    $\E[R_S(T)] = \tildeO(S\sqrt{KT})$
if $S$ is not known.

For a fixed $S$, any algorithm suffers regret
    $\Omega(\sqrt{SKT})$
in the worst case \citep{Garivier-alt11}.

\xhdr{Dynamic regret.}
The strongest possible benchmark is the best \emph{current} arm:
    $c^*_t = \min_a c_t(a)$.
We are interested in \emph{dynamic regret}, defined as
\[ R^*(T) = \textstyle \min(\ALG) - \sum_{t\in [T]}\; c^*_t. \]
This benchmark is \emph{too hard} in the worst case, without additional assumptions. In what follows, we consider a randomized oblivious adversary, and make assumptions on the rate of change in cost distributions.

Assume the adversary changes the cost distribution at most $S$ times. One can re-use results for shifting regret, so as to obtain expected dynamic regret
    $\tildeO(\sqrt{SKT})$
when $S$ is known, and
    $\tildeO(S\sqrt{KT})$
when $S$ is not known. The former regret bound is optimal \citep{Garivier-alt11}. The regret bound for unknown $S$ can be improved to
    $\E[R^*(T)]= \tildeO(\sqrt{SKT})$,
matching the optimal regret rate for known $S$, using more advanced algorithms
\citep{Auer-switchingMAB-19,Haipeng-dynamicRegret-19}.

Suppose expected costs change \emph{slowly}, by at most $\eps$ in each round. Then one can obtain bounds on dynamic regret of the form
    $ \E[R^*(T)] \leq C_{\eps,K} \cdot T$,
where $C_{\eps,K}\ll 1$ is a ``constant" determined by $\eps$ and $K$.  The intuition is that the algorithm pays a constant per-round ``price" for keeping up with the changing costs, and the goal is to minimize this price as a function of $K$ and $\eps$. Stronger regret bound (\ie one with smaller $C_{\eps,K}$) is possible if expected costs evolve as a random walk.%
\footnote{Formally, the expected cost of each arm evolves as an independent random walk on $[0,1]$ interval with reflecting boundaries.}
One way to address these scenarios is to use an algorithm with $S$-shifting regret, for an appropriately chosen value of $S$, and restart it after a fixed number of rounds, \eg see Exercise~\ref{adv:ex:slow}. \citet{DynamicMAB-colt08,contextualMAB-colt11} provide algorithms with better bounds on dynamic regret, and obtain matching lower bounds. Their approach is an extension of the \UcbOne algorithm, see Algorithm~\ref{IID:alg4} in Chapter~\ref{ch:IID}. In a nutshell, they add a term $\phi_t$ to the definition of $\UCB_t(\cdot)$, where $\phi_t$ is a known high-confidence upper bound on each arm's change in expected rewards in $t$ steps, and restart this algorithm every $n$ steps, for a suitably chosen fixed $n$. E.g., $\phi_t = \eps t$ in general, and $\phi_t = O(\sqrt{t \log T})$ for the random walk scenario. In principle, $\phi_t$ can be arbitrary, possibly depending on an arm and on the time of the latest restart.


A more flexible model considers the \emph{total variation} in cost distributions, $V = \sum_{t\in [T-1]} V_t$, where $V_t$ is the amount of change in a given round $t$ (defined as a total-variation distance between the cost distribution for rounds $t$ and $t+1$). While $V$ can be as large as $KT$ in the worst case, the idea is to benefit when $V$ is small. Compared to the previous two models, an arbitrary amount of per-round change is allowed, and a larger amount of change is ``weighed" more heavily. The optimal regret rate is $\tildeO(V^{1/3} T^{2/3})$ when $V$ is known; it can be achieved, \eg by \ExpFour algorithm with restarts. The same regret rate can be achieved even without knowing $V$, using a more advanced algorithm \citep{Haipeng-dynamicRegret-19}. This result also obtains the best possible regret rate when each arm's expected costs change by at most $\eps$ per round, without knowing the $\eps$.%
\footnote{However, this approach does not appear to obtain better regret bounds more complicated variants of the ``slow change" setting, such as when each arm's expected reward follows a random walk.}

\OMIT{\citet{contextualMAB-colt11} also handles a more general version, see
\eqref{CB:eq:further-slow} on page \pageref{CB:eq:further-slow} and the discussion around it.}

\xhdr{Swap regret.} Let us consider a strong benchmark that depends not only on the costs but on the algorithm itself. Informally, what if we consider the sequence of actions chosen by the algorithm, and swap each occurrence of each action $a$ with $\pi(a)$, according to some fixed \emph{swapping policy} $\pi: [K] \mapsto [K]$. The algorithm competes with the best swapping policy. More formally, we define \emph{swap regret} as
\begin{align}\label{FF:eq:swap}
R_{\term{swap}}(T)
    = \cost(\ALG) -
      \min_{\text{swapping policies $\pi\in \mF$}} \quad
        \sum_{t\in [T]} c_t(\pi(a_t)),
\end{align}
where $\mF$ is the class of all swapping policies.  The standard definition of regret corresponds to a version of \eqref{FF:eq:swap}, where $\mF$ is the set of all ``constant" swapping functions (\ie those that map all arms to the same one).

The best known regret bound for swap regret is
    $\tildeO(K \sqrt{T})$
\citep{Stoltz-thesis}. Swapping regret has been introduced \citep{BlumMansour-jmlr07}, who also provided an explicit transformation that takes an algorithm with a ``standard" regret bound, and transforms it into an algorithm with a bound on swap regret. Plugging in \ExpThree algorithm, this approach results in $\tildeO(K \sqrt{KT})$ regret rate. For the full-feedback version, the same approach achieves
    $\tildeO(\sqrt{KT})$
regret rate, with a matching lower bound. All algorithmic results are for an oblivious adversary; the lower bound is for adaptive adversary.

Earlier literature in theoretical economics, starting from \citep{HartMasCollel-econometrica00}, designed algorithms for a weaker notion called \emph{internal regret}
\citep{FosterVohra-GEB97,FosterVohra-Bionetrika98,FosterVohra-GEB99,Cesa-BianchiL-ml03}.
The latter is defined as a version of \eqref{FF:eq:swap} where $\mF$ consists of swapping policies that change only one arm. The main motivation was its connection to equilibria in repeated games, more on this in Section~\ref{games:sec:further}. Note that the swap regret is at most $K$ times internal regret.

\xhdr{Counterfactual regret.}
The notion of ``best-observed arm" is not entirely satisfying for adaptive adversaries, as discussed in Section~\ref{FF:sec:adv}. Instead, one can consider a \emph{counterfactual} notion of regret, which asks what would have happened if the algorithm actually played this arm in every round. The benchmark is the best fixed arm, but in this counterfactual sense.  Sublinear regret is impossible against unrestricted adversaries. However, one can obtain non-trivial results against memory-restricted adversaries: ones that can use only $m$ most recent rounds. In particular, one can obtain
    $\tildeO(m K^{1/3} T^{2/3})$
counterfactual regret for all $m<T^{2/3}$, without knowing $m$. This can be achieved using a simple ``batching" trick: use some bandit algorithm \ALG so that one round in the execution of \ALG corresponds to a batch of $\tau$ consecutive rounds in the original problems.  This result is due to \citet{PolicyRegret-icml12}; a similar regret bound for the full-feedback case has appeared in an earlier paper \citep{Merhav-IT02}.

\xhdr{Regret relative to the best \emph{algorithm}.}
Given a (small) family $\mF$ of algorithms, can one design a ``meta-algorithm" which, on each problem instance, does almost as well as the best algorithm in $\mF$? This is an extremely difficult benchmark, even if each algorithm in $\mF$ has a very small set $S$ of possible internal states. Even with $|S|\leq 3$ (and only $3$ algorithms and $3$ actions), no ``meta-algorithm" can achieve expected regret better than $O(T \log^{-3/2} T)$ relative to the best algorithm in $\mF$ (henceforth, \emph{$\mF$-regret}). Surprisingly, there is an intricate algorithm which obtains a similar upper bound on $\mF$-regret,
    $O(\sqrt{K|S|}\cdot T\cdot \log^{-\Omega(1)} T)$.
Both results are from
    \citet{ChasingGhosts-sicomp17}.

\citet{Corral-colt17} bypass these limitations and design a ``meta-algorithm" with much more favorable bounds on $\mF$-regret. This comes at a cost of substantial assumptions on the algorithms' structure and a rather unwieldy fine-print in the regret bounds. Nevertheless, their regret bounds have been productively applied, \eg in \cite{SmoothedRegret-colt19}.

\sectionExercises

\OMIT{ 
\begin{exercise}\label{adv:ex:Exp4}
Extend Theorem~\ref{adv:thm:exp4-gamma} to a version of $\ExpFour$ without uniform exploration.
\TakeAway{We can remove the uniform exploration step from \ExpFour, even though it is crucial for our analysis!}
\Hint{Consider \ExpFour with a sufficiently small parameter $\gamma>0$. Modify the selection rule so that the algorithm always follows the chosen expert $e_t$ (but keep the fake costs exactly the same as in \ExpFour). What is the probability that the modified algorithm completely coincides with \ExpFour? Use \refeq{adv:eq:Exp4-regret-gamma} to derive the regret bound.}
\end{exercise}
} 

\begin{exercise}[lower bound]\label{adv:ex:exp4-LB}
Consider adversarial bandits with experts advice. Prove the lower bound in \eqref{adv:eq:exp4-LB} for any given $(K,N,T)$. More precisely: construct a randomized problem instance for which any algorithm satisfies \eqref{adv:eq:exp4-LB}.

\Hint{Split the time interval $1..T$ into $M=\tfrac{\ln N}{\ln K}$ non-overlapping sub-intervals of duration $T/M$ . For each sub-interval, construct the randomized problem instance from Chapter~\ref{ch:LB} (independently across the sub-intervals). Each expert recommends the same arm within any given sub-interval; the set of experts includes all experts of this form.}
\end{exercise}

\begin{exercise}[fixed discretization]\label{adv:ex:Lip}
Let us extend the fixed discretization approach from Chapter~\ref{ch:Lip} to adversarial bandits.
Consider adversarial bandits with the set of arms $\mA=[0,1]$.  Fix $\eps>0$ and let $S_\eps$ be the \eps-uniform mesh over $\mA$, \ie the set of all points in $[0,1]$ that are integer multiples of $\eps$. For a subset $S\subset \mA$, the optimal total cost is
    $\cost^*(S) := \min_{a\in S} \cost(a)$,
and the discretization error is defined as
    $ \DE(S_\eps)=\rbr{\cost^*(S)-\cost^*(\mA)}/T$.

\begin{itemize}
\item[(a)] Prove that $\DE(S_\eps)\leq L\eps$, assuming Lipschitz property:
\begin{align}
    |c_t(a)-c_t(a')| \leq L\cdot |a-a'| \quad\text{for all arms $a,a'\in\mA$ and all rounds $t$}.
\end{align}

\item[(b)] Consider a version of dynamic pricing (see Section~\ref{Lip:sec:ex-DP}), where the values $v_1 \LDOTS v_T$ are chosen by a deterministic, oblivious adversary. For compatibility, state the problem in terms of costs rather than rewards: in each round $t$, the cost is $-p_t$ if there is a sale, $0$ otherwise. Prove that $\DE(S_\eps)\leq \eps$.

\Note{It is a special case of adversarial bandits with some extra structure which allows us to bound discretization error \emph{without assuming Lipschitzness}.}

\item[(c)] Assume that $\DE(S_\eps)\leq\eps$ for all $\eps>0$. Obtain an algorithm with regret
    $\E[R(T)] \leq O(T^{2/3} \log T)$.

\Hint{Use algorithm $\ExpThree$ with arms $S\subset \mA$, for a well-chosen subset $S$.}
\end{itemize}
\end{exercise}

\begin{exercise}[slowly changing costs]\label{adv:ex:slow}
Consider a randomized oblivious adversary such that the expected cost of each arm changes by at most $\eps$ from one round to another, for some fixed and known $\eps>0$. Use algorithm $\ExpFour$ to obtain dynamic regret
    \begin{align} \E[R^*(T)] \leq O(T)\cdot (\eps\cdot K \log K)^{1/3}. \end{align}

\Hint{Recall the application of $\ExpFour$ to $n$-shifting regret, denote it $\ExpFour(n)$. Let $\OPT_n = \min \cost(\pi) $, where the $\min$ is over all $n$-shifting policies $\pi$, be the benchmark in $n$-shifting regret. Analyze the ``discretization error": the difference between $\OPT_n$ and
    $\OPT^*= \sum_{t=1}^T \min_a c_t(a)$,
the benchmark in dynamic regret. Namely: prove that
    $\OPT_n-\OPT^* \leq O(\eps T^2/n)$.
Derive an upper bound on dynamic regret that is in terms of $n$. Optimize the choice of $n$.}
\end{exercise}

\chapter{Linear Costs and Semi-Bandits}
\label{ch:lin}

\begin{ChAbstract}
This chapter provides a joint introduction to several related lines of work: online routing, combinatorial (semi-)bandits, linear bandits, and online linear optimization. We study bandit problems with linear costs: actions are represented by vectors in $\R^d$, and their costs are linear in this representation. This problem is challenging even under full feedback, let alone bandit feedback; we also consider an intermediate regime called \emph{semi-bandit feedback}. We start with an important special case called \emph{online routing}, and its generalization, combinatorial semi-bandits. We solve both using a version of the bandits-to-\Hedge reduction from Chapters~\ref{ch:adv}. However, this solution is slow. To remedy this, we focus on the full-feedback problem, a.k.a. \emph{online linear optimization}. We present a fundamental algorithm for this problem, called \emph{Follow The Perturbed Leader}, which plugs nicely into the bandits-to-experts reduction and makes it computationally efficient.

\prereqs{Chapters~\ref{ch:FF}-\ref{ch:adv}.}
\end{ChAbstract}

%

We consider \emph{linear costs} throughout this chapter. As in the last two chapters, there are $K$ actions and a fixed time horizon $T$, and each action $a\in [K]$ yields cost $c_t(a)\geq 0$ at each round $t \in [T]$. Actions are represented by low-dimensional real vectors; for simplicity, we assume that all actions lie within a unit hypercube:
    $a \in [0, 1]^d$.
Action costs are linear in $a$, namely:
    $ c_t(a) = a \cdot v_t$
for some weight vector $v_t\in \R^d$ which is the same for all actions, but depends on the current time step.

\section*{Recap: bandits-to-experts reduction}

We build on the bandits-to-experts reduction from Chapter~\ref{ch:adv}. We will use it in a more abstract version, spelled out in Algorithm~\ref{lin:alg:reduction-L9}, with arbitrary ``fake costs" and an arbitrary full-feedback algorithm. We posit that experts correspond to arms, \ie for each arm there is an expert that always recommends this arm.

\LinesNotNumbered \SetAlgoLined
\begin{algorithm}
{\bf Given:} an algorithm \ALG for online learning with experts, and parameter $\gamma\in (0,\tfrac12)$.\\
{\bf Problem:} adversarial bandits with $K$ arms and $T$ rounds; ``experts" correspond to arms.\\
In each round $t\in [T]$:
\begin{enumerate}
    \item call \ALG, receive an expert $x_t$ chosen for this round,
    \\where $x_t$ is an independent draw from some distribution $p_t$ over the experts.
    \item with probability $1-\gamma$ follow expert $x_t$;\\
        else, chose arm via a version of ``random exploration" (TBD)
    \item observe cost $c_t$ for the chosen arm, and perhaps some extra
        feedback (TBD)
    \item define ``fake costs'' $\widehat{c}_t(x)$ for each expert $x$ (TBD), and
        return them to \ALG.
\end{enumerate}
\caption{Reduction from bandit feedback to full feedback.}
\label{lin:alg:reduction-L9}
\end{algorithm}

We assume that ``fake costs" are bounded from above and satisfy \eqref{adv:eq:exp4-unbiased}, and that \ALG satisfies a regret bound against an adversary with bounded costs. For notation, an adversary is called \emph{$u$-bounded} if $c_t(\cdot)\leq u$. While some steps in the algorithm are unspecified, the analysis from Chapter~\ref{ch:adv} carries over word-by-word \emph{no matter how these missing steps are filled in}, and implies the following theorem.


%

\begin{theorem} \label{lin:thm:reduction-template}
Consider Algorithm~\ref{lin:alg:reduction-L9} with algorithm \ALG that achieves regret bound $\E[R(T)] \leq f(T,K,u)$ against adaptive, $u$-bounded adversary, for any given $u>0$ that is known to the algorithm.

Consider adversarial bandits with a deterministic, oblivious adversary. Assume ``fake costs" satisfy
    \[ \E\sbr{\widehat{c}_t(x)\mid p_t } = c_t(x)
    \text{ and } \widehat{c}_t(x) \leq u/\gamma
    \qquad\text{for all experts $x$ and all rounds $t$},
    \]
where $u$ is some number that is known to the algorithm. Then Algorithm~\ref{lin:alg:reduction-L9} achieves regret
    \begin{align*}
        \E[R(T)] \leq f(T,K,u/\gamma) + \gamma T.
    \end{align*}
\end{theorem}

\begin{corollary}\label{lin:cor:reduction-template}
If \ALG is \Hedge in the theorem above, one can take
    $f(T,K,u) = O(u \cdot\sqrt{T \ln{K}})$
by Theorem~\ref{FF:thm:unbounded-costs}. Then, setting
    $\gamma = T^{-1/4}\; \sqrt{u\cdot \log K}$,
Algorithm~\ref{lin:alg:reduction-L9} achieves regret
    \[ \E[R(T)] \leq O\rbr{T^{3/4} \;\sqrt{u\cdot \log K}}.\]
\end{corollary}

We instantiate this algorithm, \ie specify the missing pieces, to obtain a solution for some special cases of linear bandits that we define below.

\section{Online routing problem}
\label{lin:sec:OnlineRouting}

Let us consider an important special case of linear bandits called the \emph{online routing problem}, a.k.a. \emph{online shortest paths}. We are given a graph $G$ with $d$ edges, a source node $u$, and a destination node $v$. The graph can either be directed or undirected. We have costs on edges that we interpret as delays in routing, or lengths in a shortest-path problem. The cost of a path is the sum over all edges in this path. The costs can change over time. In each round, an algorithm chooses among ``actions" that correspond to \UV paths in the graph. Informally, the algorithm's goal in each round is to find the ``best route" from $u$ to $v$: an \UV path with minimal cost (\ie minimal travel time). More formally, the problem is as follows:

\begin{BoxedProblem}{Online routing problem}
Given: graph $G$, source node $u$, destination node $v$.

\vspace{1mm}

\noindent For each round $t \in [T]$:
\begin{enumerate}
    \item Adversary chooses costs $c_t(e)\in [0,1]$ for all edges $e$.
    \item Algorithm chooses \UV-path $a_t \subset \mathtt{Edges}(G)$.
    \item Algorithm incurs cost
        $c_t(a_t) = \sum_{e \in a_t} a_e \cdot c_t(e)$ and receives feedback.
\end{enumerate}
\end{BoxedProblem}

To cast this problem as a special case of ``linear bandits", note that each path can be specified by a subset of edges, which in turn can be specified by a $d$-dimensional binary vector $a\in \{0,1\}^d$. Here edges of the graph are numbered from $1$ to $d$, and for each edge $e$ the corresponding entry $a_e$ equals 1 if and only if this edge is included in the path. Let $v_t = (c_t(e): \text{edges $e\in G$})$ be the vector of edge costs at round $t$.
Then the cost of a path can be represented as a linear product
    $c_t(a) = a\cdot v_t = \sum_{e\in [d]} a_e\; c_t(e) $.

There are three versions of the problem, depending on which feedback is received:
\begin{OneLiners}
\item \emph{Bandit feedback:} only $c_t(a_t)$ is observed;
\item \emph{Semi-bandit feedback:} costs $c_t(e)$ for all edges $e\in a_t$ are observed;
\item \emph{Full feedback:} costs $c_t(e)$ for all edges $e$ are observed.
\end{OneLiners}
Semi-bandit feedback is an intermediate ``feedback regime" which we will focus on.

\xhdr{Full feedback.}
The full-feedback version can be solved with \Hedge algorithm. Applying Theorem~\ref{FF:thm:unbounded-costs} with a trivial upper bound $c_t(\cdot)\leq d$ on action costs, and the trivial upper bound $K\leq 2^d$ on the number of paths, we obtain regret
    $\E[R(T)] \leq O(d\sqrt{dT})$.

\begin{corollary}\label{lin:thm:FF}
Consider online routing with full feedback. Algorithm \Hedge with
with parameter
    $\eps=1/ \sqrt{dT}$
achieves regret
    $\E[R(T)] \leq O(d\sqrt{dT})$.
\end{corollary}

This regret bound is optimal up $O(\sqrt{d})$ factor \citep{Koolen-colt10}. (The root-$T$ dependence on $T$ is optimal, as per Chapter~\ref{ch:LB}.)
An important drawback of using \Hedge for this problem is the running time, which is exponential in $d$. We return to this issue in Section~\ref{lin:sec:FPL}.

\xhdr{Semi-bandit feedback.}
We use the bandit-to-experts reduction (Algorithm~\ref{lin:alg:reduction-L9}) with \Hedge algorithm, as a concrete and simple application of this machinery to linear bandits. We assume that the costs are selected by a deterministic  oblivious adversary, and we do not worry about the running time.


As a preliminary attempt, we can use \ExpThree algorithm for this problem. However, expected regret would be proportional to square root of the number of actions, which in this case may be exponential in $d$.

Instead, we seek a regret bound of the form:
\begin{align*}
    \E[R(T)] \leq \poly(d) \cdot T ^ {\beta}, \quad \text{where $\beta < 1$}.
\end{align*}
To this end, we use Algorithm~\ref{lin:alg:reduction-L9} with \Hedge. The ``extra information" in the reduction is the semi-bandit feedback. Recall that we also need to specify the ``random exploration" and the ``fake costs".

For the ``random exploration step", instead of selecting an action uniformly at random (as we did in \ExpFour), we select an edge $e$ uniformly at random, and pick the corresponding path $a^{(e)}$ as the chosen action. We assume that each edge $e$ belongs to some \UV path $a^{(e)}$; this is without loss of generality, because otherwise we can just remove this edge from the graph.

We define fake costs for each edge $e$ separately; the fake cost of a path is simply the sum of fake costs over its edges. Let $\Lambda_{t,e}$ be the event that in round $t$, the algorithm chooses ``random exploration", \emph{and} in random exploration, it chooses edge $e$. Note that $\Pr[\Lambda_{t,e}]=\gamma/d$. The fake cost on edge $e$ is
\begin{align}\label{lin:eq:semi-bandits-fake-costs}
    \widehat{c}_t(e) =
    \begin{cases}
        \frac{c_t(e)}{\gamma/d} & \text{if event $\Lambda_{t,e}$ happens}\\
        0 & \text{otherwise}
    \end{cases}
\end{align}
 This completes the specification of an algorithm for the online routing problem with semi-bandit feedback; we will refer to this algorithm as \AlgSB.

As in the previous lecture, we prove that fake costs provide
unbiased estimates for true costs:
\begin{align*}
    \E[\widehat{c}_t(e) \mid p_t] = c_t(e)
        \quad \text{for each round $t$ and each edge $e$}.
\end{align*}

Since the fake cost for each edge is at most $d/\gamma$, it follows that $c_t(a)\leq d^2/\gamma$ for each action $a$. Thus, we can immediately use Corollary~\ref{lin:cor:reduction-template} with $u=d^2$. For the number of actions, let us use an upper bound $K\leq 2^d$. Then $u \log K \leq d^3$, and so:

\begin{theorem}
Consider the online routing problem with semi-bandit feedback. Assume deterministic oblivious adversary.
Algorithm \AlgSB achieved regret
    $\E[R(T)] \leq O(d^{3/2}\; T^{3/4})$.
\end{theorem}

\begin{remark}\label{lin:rem:oblivious-suffices}
Fake cost $\widehat{c}_t(e)$ is determined by the corresponding true cost $c_t(e)$ and event $\Lambda_{t,e}$ which does not depend on algorithm's actions. Therefore, fake costs are chosen by a (randomized) oblivious adversary. In particular, in order to apply Theorem~\ref{lin:thm:reduction-template} with a different algorithm $\ALG$ for online learning with experts, it suffices to have an upper bound on regret against an oblivious adversary.
\end{remark}

\section{Combinatorial semi-bandits}
\label{lin:sec:semi}

The online routing problem with semi-bandit feedback is a special case of \emph{combinatorial
semi-bandits}, where edges are replaced with $d$ ``atoms", and \UV paths are replaced with feasible subsets of atoms. The family of feasible subsets can be arbitrary (but it is known to the algorithm).

\begin{BoxedProblem}{Combinatorial semi-bandits}
Given: set $S$ of atoms, and a family $\mF$ of feasible actions (subsets of $S$).

\vspace{1mm}

\noindent For each round $t \in [T]$:
\begin{enumerate}
    \item Adversary chooses costs $c_t(e)\in [0,1]$ for all atoms $e$,
    \item Algorithm chooses a feasible action $a_t \in \mF$,
    \item Algorithm incurs cost
        $c_t(a_t) = \sum_{e \in a_t} a_e \cdot c_t(e)$ \\
        and observes costs $c_t(e)$ for all atoms $e\in a_t$.
\end{enumerate}
\end{BoxedProblem}

The algorithm and analysis from the previous section does not rely on any special properties of \UV paths. Thus, they carry over word-by-word to combinatorial semi-bandits, replacing edges with atoms, and \UV paths with feasible subsets. We obtain the following theorem:

\begin{theorem}
Consider combinatorial semi-bandits with deterministic oblivious adversary.
Algorithm \AlgSB achieved regret
    $\E[R(T)] \leq O(d^{3/2}\; T^{3/4})$.
\end{theorem}

Let us list a few other notable special cases of combinatorial semi-bandits:
\begin{itemize}
    \item \emph{News Articles:} a news site needs to select a subset of articles to display to each user. The
        user can either click on an article or ignore it. Here, rounds correspond to users, atoms are the news articles, the reward is 1 if it is clicked and $0$ otherwise, and feasible subsets can encode various constraints on selecting the articles.
    \item \emph{Ads:} a website needs select a subset of ads to display to each user. For each displayed ad, we observe whether the user clicked on it, in which case the website receives some payment. The payment may depend on both the ad and on the user. Mathematically, the problem is very similar to the news articles: rounds correspond to users, atoms are the ads, and feasible subsets can encode constraints on which ads can or cannot be shown together. The difference is that the payments are no longer 0-1.
    \item \emph{A slate of news articles:} Similar to the news articles problem, but the ordering of the articles on the webpage matters. This the news site needs to select a \emph{slate} (an ordered list) of articles. To represent this problem as an instance of combinatorial semi-bandits, define each ``atom" to mean ``this news article is chosen for that slot". A subset of atoms is feasible if it defines a valid slate: \ie there is exactly one news article assigned to each slot.
    \OMIT{
    \item \emph{Network broadcast:} In each round, we want to transmit a packet from a source to multiple targets in the network. The network is represented as an undirected graph, and in each round, the algorithm needs to choose some Steiner tree: a subtree that connects the source and all targets. Atoms correspond to edges in the graph, and the feasible subsets correspond to (feasible) Steiner trees. \authorcomment{But why costs make sense?}}
    \end{itemize}

Thus, combinatorial semi-bandits is a general setting which captures several motivating examples, and allows for a unified solution. Such results are valuable even if each of the motivating examples is only a very idealized version of reality, \ie it captures  some features of reality but ignores some others.

\xhdr{Low regret \emph{and} running time.}
Recall that \AlgSB is slow: its running time per round is exponential in $d$, as it relies on \Hedge with this many experts. We would like it to be \emph{polynomial} in $d$.

One should not hope to accomplish this in the full generality of combinatorial bandits. Indeed, even if the costs on all atoms were known, choosing the best feasible action (a feasible subset of minimal cost) is a well-known problem of \emph{combinatorial optimization}, which is NP-hard. However, combinatorial optimization allows for polynomial-time solutions in many interesting special cases. For example, in the online routing problem discussed above the corresponding combinatorial optimization problem is a well-known shortest-path problem. Thus, a natural approach
is to assume that we have access to an \emph{optimization oracle}: an algorithm which finds the best feasible action given the costs on all atoms, and express the running time of our algorithm in terms of the number of oracle calls.

In Section~\ref{lin:sec:FPL}, we use this oracle to construct a new algorithm for combinatorial bandits with full feedback, called \emph{Follow The Perturbed Leader} (\FPL). In each round, this algorithm inputs only the costs on the atoms, and makes only one oracle call. We derive a regret bound
\begin{align}\label{lin:eq:regret-FPL-O}
    \E[R(T)] \leq O(u \sqrt{dT})
\end{align}
against an oblivious, $u$-bounded adversary such that the atom costs are at most $\nicefrac{u}{d}$. (Recall that a regret bound against an oblivious adversary suffices for our purposes, as per Remark~\ref{lin:rem:oblivious-suffices}.)

We use algorithm \AlgSB as before, but replace \Hedge with \FPL; call the new algorithm \AlgSBfpl. The analysis from Section~\ref{lin:sec:OnlineRouting} carries over to \AlgSBfpl. We take $u=d^2/\gamma$ as a known upper bound on the fake costs of actions, and note that the fake costs of atoms are at most $\nicefrac{u}{d}$. Thus, we can apply Theorem~\ref{lin:thm:reduction-template} for \FPL with fake costs, and obtain regret
\begin{align*}
\E[R(T)] \leq O(u \sqrt{dT}) + \gamma T.
\end{align*}
\noindent Optimizing the choice of parameter $\gamma$, we immediately obtain the following theorem:

\begin{theorem}\label{lin:thm:semi-bandits-FPL}
Consider combinatorial semi-bandits with deterministic oblivious adversary. Then algorithm \AlgSBfpl with appropriately chosen parameter $\gamma$ achieved regret
    \[ \E[R(T)] \leq O\left( d^{5/4}\; T^{3/4} \right).\]
\end{theorem}

\begin{remark}
In terms of the running time, it is essential that the fake costs on atoms can be computed \emph{fast}: this is because the normalizing probability in \eqref{lin:eq:semi-bandits-fake-costs} is known in advance.

Alternatively, we could have defined fake costs on atoms $e$ as
\begin{align*}
    \widehat{c}_t(e) =
    \begin{cases}
        c_t(e)/\Pr[e\in a_t \mid p_t ] & \text{if $e\in a_t$}\\
        0 & \text{otherwise}.
    \end{cases}
\end{align*}
This definition leads to essentially the same regret bound (and, in fact, is somewhat better in practice). However, computing the probability
    $\Pr[e\in a_t \mid p_t ]$
in a brute-force way requires iterating over all actions, which leads to running times exponential in $d$, similar to \Hedge.
\end{remark}

\begin{remark}
Solving a version with bandit feedback requires more work. The main challenge is to estimate fake costs for all atoms in the chosen action, whereas we only observe the total cost for the action. One solution is to construct a suitable  \emph{basis}: a subset of feasible actions, called \emph{base actions}, such that each action can be represented as a linear combination thereof. Then a version of Algorithm~\ref{lin:alg:reduction-L9}, where in the ``random exploration" step is uniform among the base actions, gives us fake costs for the base actions. The fake cost on each atom is the corresponding linear combination over the base actions. This approach works as long as the linear coefficients are small, and ensuring this property takes some work. This approach is worked out in \citet{Bobby-stoc04}, resulting in regret
    $\E[R(T)] \leq \tildeO(d^{10/3}\cdot T^{2/3})$.
\end{remark}

\section{Online Linear Optimization: Follow The Perturbed Leader}
\label{lin:sec:FPL}

Let us turn our attention to \emph{online linear optimization}, \ie bandits with full-feedback and linear costs. We do not restrict ourselves to combinatorial actions, and instead allow an arbitrary subset $\mA\subset [0,1]^d$ of feasible actions. This subset is fixed over time and known to the algorithm. Recall that in each round $t$, the adversary chooses a ``hidden vector" $v_t\in \R^d$, so that the cost for each action $a\in\mA$ is
    $c_t(a) = a\cdot v_t$.
We posit an upper bound on the costs: we assume that $v_t$ satisfies
    $v_t\in [0,U/d]^d$,
for some known parameter $U$, so that $c_t(a)\leq U$ for each action $a$.

\begin{BoxedProblem}{Online linear optimization}
\noindent For each round $t \in [T]$:
\begin{enumerate}
    \item Adversary chooses hidden vector  $v_t\in \R^d$.
    \item Algorithm chooses action $a=a_t\in\mA \subset [0,1]^d$,
    \item Algorithm incurs cost
        $c_t(a) = v_t\cdot a$
    and observes $v_t$.
\end{enumerate}
\end{BoxedProblem}

We design an algorithm, called Follow The Perturbed Leader (\FPL), that is computationally efficient and satisfies regret bound \eqref{lin:eq:regret-FPL-O}. In particular, this suffices to complete the proof of Theorem~\ref{lin:thm:semi-bandits-FPL}.

We assume than the algorithm has access to an \emph{optimization oracle}: a subroutine which computes the best action for a given cost vector. Formally, we represent this oracle as a function $M$ from cost vectors to feasible actions such that $M(v) \in \argmin_{a\in \mA} a\cdot v $ (ties can be broken arbitrarily). As explained earlier, while in general the oracle is solving an NP-hard problem, polynomial-time algorithms exist for important special cases such as shortest paths. The implementation of the oracle is domain-specific, and is irrelevant to our analysis. We prove the following theorem:

\begin{theorem}\label{lin:thm:FPL}
Assume that $v_t\in [0,U/d]^d$ for some known parameter $U$.
Algorithm \FPL achieves regret
\(\E[R(T)] \leq 2U \cdot \sqrt{dT}\).
The running time in each round is polynomial in $d$ plus one call to the oracle.
\end{theorem}

\begin{remark}
The set of feasible actions $\mA$ can be infinite, as long as a suitable oracle is provided. For example, if $\mA$ is defined by a finite number of linear constraints, the oracle can be implemented via linear programming. Whereas \Hedge is not even well-defined for infinitely many actions.
\end{remark}

We use shorthand
    $v_{i:j} = \sum_{t=i}^j v_t \in \R^d$
to denote the total cost vector between rounds $i$ and $j$.

\newpage

\xhdr{Follow The Leader.} Consider a simple, exploitation-only algorithm called \emph{Follow The Leader}:
    \[ a_{t+1} = M(v_{1:t}).\]
Equivalently, we play an arm with the lowest average cost, based on the observations so far.

While this approach works fine for IID costs, it breaks for adversarial costs. The problem is synchronization: an oblivious adversary can force the algorithm to behave in a particular way, and synchronize its costs with algorithm's actions in a way that harms the algorithm. In fact, this can be done to any deterministic online learning algorithm, as per Theorem~\ref{FF:thm:deterministic-LB}. For concreteness, consider the following example:
\begin{equation*} \begin{split}
  \mA &= \{ (1,0), \; (0,1) \} \\
  v_1 &= (\tfrac{1}{3}, \tfrac{2}{3}) \\
  v_t &= \begin{cases}
        (1,0) &\text{if $t$ is even,} \\
        (0,1) &\text{if $t$ is odd.}
        \end{cases} \\
\end{split} \end{equation*}
Then the total cost vector is
\[  v_{1:t} = \begin{cases}
        (i + \frac{1}{3}, i - \frac{1}{3}) & \text{if $t=2i$,} \\
         (i + \frac{1}{3}, i + \frac{2}{3}) & \text{if $t=2i+1$.}
        \end{cases} \]
Therefore, Follow The Leader picks action $a_{t+1}=(0,1)$ if $t$ is even, and $a_{t+1}=(1,0)$ if $t$ is odd. In both cases, we see that $c_{t+1}(a_{t+1}) =1$. So the total cost for the algorithm is $T$, whereas any fixed action achieves total cost at most $1+T/2$, so regret is, essentially, $T/2$.

\xhdr{Fix: perturb the history!} Let us use randomization to side-step the synchronization issue discussed above. We perturb the history before handing it to the oracle. Namely, we pretend there was a $0$-th round, with cost vector $v_0\in \R^d$ sampled from some distribution $\mD$. We then give the oracle the ``perturbed history", as expressed by the total cost vector $v_{0:t-1}$, namely
    $a_t = M(v_{0:t-1})$.
This modified algorithm is known as \emph{Follow The Perturbed Leader} (\FPL).

\LinesNotNumbered
\begin{algorithm}[H]
    \SetAlgoLined
    Sample $v_0\in \R^d$ from distribution $\mD$\;
    \For{each round $t=1,2, \ldots   $}{
         Choose arm $a_t = M(v_{0:t-1})$,
            where $v_{i:j} = \sum_{t=i}^j v_t \in \R^d$.
    }
    \caption{Follow The Perturbed Leader (\FPL).}
    \label{lin:alg:FPL}
\end{algorithm}

Several choices for distribution $\mD$ lead to meaningful analyses. For ease of exposition, we posit that each each coordinate of $v_0$ is sampled independently and uniformly from the interval
    $[-\frac{1}{\eps}, \frac{1}{\eps}]$.
The parameter $\eps$ can be tuned according to $T$, $U$, and $d$; in the end, we use
    $\eps=\frac{\sqrt{d}}{U\sqrt{T}}$.

\subsection*{Analysis of the algorithm}

As a tool to analyze \FPL, we consider a closely related algorithm called \emph{Be The
  Perturbed Leader} (\BPL). Imagine that when we need to choose an action at time
$t$, we already know the cost vector $v_t$, and in each round $t$ we choose
    $a_t = M(v_{0:t})$.
Note that \BPL is \emph{not} an algorithm for online learning with experts; this is because it uses $v_t$ to choose $a_t$.

The analysis proceeds in two steps. We first show that \BPL comes ``close'' to the optimal cost
\[ \OPT = \min_{a \in \mA} \cost(a) = v_{1:t} \cdot M(v_{1:t}), \]
and then we show that \FPL comes ``close'' to \BPL. Specifically, we will prove:

\begin{lemma}\label{lin:lm:FPL-analysis} For each value of parameter $\eps>0$,
\begin{OneLiners}
\item[(i)] $\cost(\BPL) \leq \OPT + \frac{d}{\eps}$
\item[(ii)]
  $\E[\cost(\FPL)] \leq \E[\cost(\BPL)] + \eps \cdot U^2 \cdot T$
\end{OneLiners}
\end{lemma}

\noindent Then choosing $\eps=\frac{\sqrt{d}}{U\sqrt{T}}$ gives Theorem~\ref{lin:thm:FPL}. Curiously, note that part (i) makes a statement about realized costs, rather than expected costs.

\subsubsection*{Step I: \BPL comes close to $\OPT$}

By definition of the oracle $M$, it holds that
\begin{align}\label{lin:eq:FPL-analysis-M}
v\cdot M(v) \leq v\cdot a
    \quad\text{for any cost vector $v$ and feasible action $a$}.
\end{align}
The main argument proceeds as follows:
\begin{align}
\cost(\BPL)+v_0\cdot M(v_0)
    &= \sum_{t=0}^T v_t \cdot M(v_{0:T})
        &\EqComment{by definition of \BPL} \nonumber\\
    &\leq v_{0:T} \cdot M(v_{0:T})
        &\EqComment{see Claim~\ref{lin:cl:FPL-analysis-BPL} below} \label{lin:eq:FPL-analysis-missing-step}\\
    &\leq v_{0:T} \cdot M(v_{1:T})
        &\EqComment{by \eqref{lin:eq:FPL-analysis-M} with $a=M(v_{1:T})$}\nonumber\\
    &= v_0\cdot M(v_{1:T}) +
    \underbrace{v_{1:T} \cdot M(v_{1:T})}_\text{\OPT}. \nonumber
\end{align}
Subtracting $v_0\cdot M(v_0)$ from both sides, we obtain Lemma~\ref{lin:lm:FPL-analysis}(i):
\begin{align*}
\cost(\BPL) -\OPT \leq
    \underbrace{v_0}_{\in [-\frac{1}{\eps}, -\frac{1}{\eps}]^d} \cdot
      \underbrace{[M(v_{1:T}) - M(v_0)]}_{\in[-1,1]^d} \leq \frac{d}{\eps}.
\end{align*}

The missing step~\eqref{lin:eq:FPL-analysis-missing-step} follows from the following claim, with $i=0$ and $j=T$.
\begin{claim}\label{lin:cl:FPL-analysis-BPL}
For all rounds $i<j$,
$\sum_{t=1}^j  v_t \cdot M(v_{i:t}) \leq v_{i:j} \cdot M(v_{i:j})$.
\end{claim}
\begin{proof}
The proof is by induction on $j-i$. The claim is trivially satisfied for the base case
$i=j$. For the inductive step:
\begin{align*}
  \sum_{t=i}^{j-1}  v_t \cdot M(v_{i:t})
    &\leq v_{i:j-1} \cdot M(v_{i:j-1})
        &\EqComment{by the inductive hypothesis}\\
    &\leq v_{i:j-1} \cdot M(v_{i:j})
        &\EqComment{by \eqref{lin:eq:FPL-analysis-M} with $a=M(v_{i:j})$}.
\end{align*}
Add $v_j \cdot M(v_{i:j})$ to both sides to complete the proof.
\end{proof}

\subsubsection*{Step II: \FPL comes close to \BPL}

We compare the expected costs of \FPL and \BPL round per round. Specifically, we prove that
\begin{align}\label{lin:eq:FPL-analysis-FPLvsBPL}
\E[\;\underbrace{v_t \cdot M(v_{0:t-1})}_{\text{$c_t(a_t)$ for \FPL}} \;] \leq
\E[\; \underbrace{v_t \cdot M(v_{0:t})}_{\text{$c_t(a_t)$ for \BPL}} \;]+ \eps U^2.
\end{align}
Summing up over all $T$ rounds gives Lemma~\ref{lin:lm:FPL-analysis}(ii).

It turns out that for proving \eqref{lin:eq:FPL-analysis-FPLvsBPL} much of the structure in our problem is irrelevant. Specifically, we can denote
    $f(u) = v_t\cdot M(u)$
and
    $v=v_{1:t-1}$,
and, essentially, prove \eqref{lin:eq:FPL-analysis-FPLvsBPL} for arbitrary $f()$ and $v$.

\begin{claim}\label{lin:cl:FPL-analysis-noise}
For any vectors $v \in \R^d$ and  $v_t \in [0, U/d]^d$, and any function
    $f : \R^d \rightarrow [0,R]$,
\begin{align*}
    \left\vert\E_{v_0\sim\mD} \left[ f(v_0 + v) - f(v_0 + v + v_t)\right]\right\vert
        \leq \eps U R.
\end{align*}
\end{claim}
In words: changing the input of function $f$ from $v_0 + v$ to $v_0 + v+v_t$ does not substantially change the output, in expectation over $v_0$. What we actually prove is the following:

\begin{claim}\label{lin:cl:FPL-analysis-noise-coupling}
Fix $v_t \in [0, U/d]^d$. There exists a random variable $v'_0\in \R^d$ such that (i) $v'_0$ and $v_0+v_t$ have the same  marginal distribution, and (ii) $\Pr[v'_0\neq v_0] \leq \eps U$.
\end{claim}

It is easy to see that Claim~\ref{lin:cl:FPL-analysis-noise} follows from Claim~\ref{lin:cl:FPL-analysis-noise-coupling}:
\begin{align*}
\left\vert\; \E \left[ f(v_0 + v) - f(v_0 + v + v_t)\right] \;\right\vert
    &= \left\vert\; \E \left[ f(v_0 + v) - f(v'_0 + v)\right] \;\right\vert \\
    &\leq \Pr[v'_0\neq v_0] \cdot R = \eps U R.
\end{align*}

It remains to prove Claim~\ref{lin:cl:FPL-analysis-noise-coupling}. First, let us prove this claim in one dimension:

\begin{claim}\label{lin:cl:FPL-analysis-noise-coupling-one}
let $X$ be a random variable uniformly distributed on the interval $[-\freps,\freps]$. We claim that for any $a\in [0,U/d]$ there exists a deterministic function $g(X,a)$ of $X$ and $a$ such that $g(X,a)$ and $X+a$ have the same marginal distribution, and $\Pr[g(X,a)\neq X] \leq \eps U/d$.
\end{claim}

\begin{proof}
Let us define
\begin{align*}
g(X,a) = \begin{cases}
    X   & \text{if $X\in [v-\freps, \freps]$}, \\
    a-X & \text{if $X\in [-\freps, v-\freps)$}.
\end{cases}
\end{align*}
It is easy to see that $g(X,a)$ is distributed uniformly on $[v-\freps, v+\freps]$. This is because $a-X$ is distributed uniformly on $[\freps, v+\freps]$ conditional on $X\in [-\freps, v-\freps)$.
Moreover,
    \[ \Pr[g(X,a)\neq X]
        \leq \Pr[X\not\in [v-\tfrac{1}{\eps}, \tfrac{1}{\eps}]]
        = \eps v/2 \leq \tfrac{\eps U}{2d}. \qedhere\]
\end{proof}

To complete the proof of Claim~\ref{lin:cl:FPL-analysis-noise-coupling},  write
    $v_t = (v_{t,1},v_{t,2} \LDOTS v_{t,d})$,
and define $v'_0\in \R^d$ by setting its $j$-th coordinate to $Y(v_{0,j}, v_{t,j})$, for each coordinate $j$. We are done!

\newpage
\sectionBibNotes
\label{lin:sec:further}

This chapter touches upon several related lines of work: online routing, combinatorial (semi-)bandits, linear bandits, and online linear optimization. We briefly survey them below, along with some extensions.

\xhdr{Online routing and combinatorial (semi-)bandits.} The study of \emph{online routing}, a.k.a. \emph{online shortest paths}, was initiated in \citet{Bobby-stoc04}, focusing on bandit feedback and achieving regret
    $\poly(d)\cdot T^{2/3}$.
Online routing with semi-bandit feedback was introduced in \citet{Gyorgy-jmlr07}, and the general problem of combinatorial bandits was initiated in \citet{Cesa-BianchiL12}. Both papers achieve regret
    $\poly(d)\cdot \sqrt{T}$,
which is the best possible. Combinatorial semi-bandits admit improved dependence on $d$ and other problem parameters, \eg for IID rewards \citep[\eg][]{Chen-icml13,Kveton-aistats15,MatroidBandits-uai14}
and when actions are ``slates" of search results \citep{Kale-slate-nips10}.

\xhdr{Linear bandits.}
A more general problem of \emph{linear bandits} allows an arbitrary action set $\mA\subset [0,1]^d$, as in Chapter~\ref{lin:sec:FPL}. Introduced in \citet{Bobby-stoc04} and \citet{McMahan-colt04}, this problem has been studied in a long line of work. In particular, one can achieve
    $\poly(d)\cdot \sqrt{T}$ regret \citep{DaniHK-nips07},
even with high probability against an adaptive adversary \citep{Bartlett-colt08},
and even via a computationally efficient algorithm if $\mA$ is convex \citep{AbernethyHR-colt08,Abernethy-colt09}. A detailed survey of this line of work can be found in \citet[Chapter 5]{Bubeck-survey12}.

In \emph{stochastic} linear bandits, $\E\sbr{c_t(a)} = v\cdot a$ for each action $a\in \mA$ and some fixed, unknown vector $v$. One way to realize the stochastic version as a special case of the adversarial version is to posit that in each round $t$, the hidden vector $v_t$ is drawn independently from some fixed (but unknown) distribution.%
\footnote{Alternatively, one could realize $c_t(a)$ as $v\cdot a$ plus independent noise $\eps_t$. This version can also be represented as a special case of adversarial linear bandits, but in a more complicated way. Essentially, one more dimension is added, such that in this dimension each arm has coefficient $1$, and each hidden vector $v_t$ has coefficient $\eps_t$.}
Stochastic linear bandits have bee introduced in \citet{Auer-focs00} and subsequently studied in
\citep{Abe-algo03,DaniHK-colt08,Paat-mor10} using the paradigm of ``optimism under uncertainty" from Chapter~\ref{ch:IID}. Generally, regret bounds admit better dependence on $d$ compared to the adversarial version.

Several notable extensions of stochastic linear bandits have been studied. In  \emph{contextual} linear bandits, the action set is provided exogenously before each round, see Sections~\ref{CB:sec:lin} for technical details and Section~\ref{CB:sec:further} for a literature review. In \emph{generalized} linear bandits \citep[starting from][]{GeneralizedLinear-nips10,Li-icml17},
    $\E\sbr{c_t(a)} = f(v\cdot a)$
for some known function $f$, building on the \emph{generalized linear model} from statistics. In \emph{sparse} linear bandits
\citep[starting from][]{Csaba-aistats12,Carpentier-aistats12},
the dimension $d$ is very large, and one takes advantage of the sparsity in the hidden vector $v$.

\xhdr{Online linear optimization.} Follow The Perturbed Leader (\FPL) has been proposed in \citet{Hannan-1957}, in the context of repeated games (as in Chapter~\ref{ch:games}). The algorithm was rediscovered in the computer science literature by \citet{FTPL-05}, with an improved analysis.

While \FPL allows an arbitrary action set $\mA\subset [0,1]^d$, a vast and beautiful theory is developed for a paradigmatic special case when $\mA$ is convex. In particular, \emph{Follow The Regularized Leader} (\FRL) is another generalization of Follow The Leader which chooses a strongly convex regularization function $\mathcal{R}_t:\mA\to \R$ at each round $t$, and minimizes the sum:
$a_t = \argmin_{a\in\mA} \mathcal{R}_t(a)+ \sum_{s=1}^{t} c_s(a)$.
The \FRL  framework allows for a unified analysis and, depending on the choice of $\mathcal{R}_t$, instantiates to many specific algorithms and their respective guarantees, see the survey  \citep{McMahan-survey-jmlr17}. In fact, this machinery extends to convex cost functions, a subject called \emph{online convex optimization}. More background on this subject can be found in books \citep{Shalev-Shwartz-survey12} and \citep{Hazan-OCO-book}. A version of \FRL achieves $\poly(d)\cdot \sqrt{T}$ regret for any convex $\mA$ in a computationally efficient manner \citep{AbernethyHR-colt08}; in fact, it is a key ingredient in the $\poly(d)\cdot \sqrt{T}$ regret algorithm for linear bandits.

\xhdr{Lower bounds and optimality.}
In linear bandits, $\poly(d)\cdot \sqrt{T}$ regret rates are inevitable in the worst case. This holds
even under full feedback \citep{DaniHK-nips07},
even for stochastic linear bandits with ``continuous" action sets \citep{DaniHK-colt08},
and even for stochastic combinatorial semi-bandits \citep[\eg][]{Audibert-MOR14,Kveton-aistats15}.
In particular, the ``price of bandit information", \ie the penalty in optimal regret bounds compared to the full-feedback version, is quite mild: the dependence on $d$ increases from one polynomial to another. This is in stark contrast with $K$-armed bandits, where the dependence on $K$ increases from logarithmic to polynomial.

A considerable literature strives to improve dependence on $d$ and/or other structural parameters. This literature is concerned with both upper and lower bounds on regret, across the vast problem space of linear bandits.  Some of the key distinctions in this problem space are as follows:
feedback model (full, semi-bandit, or bandit);
``type" of the adversary (stochastic, oblivious, or adaptive);
structure of the action space (\eg arbitrary, convex, or combinatorial);%
\footnote{One could also posit a more refined structure, \eg assume a particular family of combinatorial subsets, such as $\{$all subsets of a given size$\}$, or $\{$all paths in a graph$\}$. \citet{AbernethyHR-colt08,Abernethy-colt09} express the ``niceness" of a convex action set via the existence of an intricate object from convex optimization called ``self-concordant barrier function".}
structure of the hidden vector (\eg sparsity). 
A detailed discussion of the ``landscape" of optimal regret bounds is beyond our scope.

\xhdr{Combinatorial semi-bandits beyond additive costs.}
Several versions of combinatorial semi-bandits allow the atoms' costs to depend on the other atoms chosen in the same round. In much of this work, when a subset $S$ of atoms is chosen by the algorithm, at most one atom $a\in S$ is then selected by ``nature", and all other atoms receive reward/cost $0$. In \emph{multinomial-logit (MNL)  bandits}, atoms are chosen probabilistically according to the MNL model, a popular choice model from statistics. Essentially, each atom $a$ is associated with a fixed (but unknown) number $v_a$, and is chosen with probability  	
    $\ind{a\in S}\cdot v_a/(1+\sum_{a' \in S} v_{a'})$.
MNL bandits are typically studied as a model for \emph{dynamic assortment}, where $S$ is the assortment of products offered for sale, \eg in \citep{caro2007dynamic,saure2013optimal,rusmevichientong2010dynamic,Shipra-ec16}.
\citet{RUMbandits-colt16} consider a more general choice model from theoretical economics, called \emph{random utility model}. Here each atom $a\in S$ is assigned ``utility" equal to $v_a$ plus independent noise, and the atom with a largest utility is chosen.

A \emph{cascade feedback} model posits that actions correspond to rankings of items such as search results, a user scrolls down to the first ``relevant" item, clicks on it, and leaves. The reward is 1 if and only if some item is relevant (and therefore the user is satisfied). The study of bandits in this model has been initiated in \citep{RBA-icml08}. They allow relevance to be arbitrarily correlated across items, depending on the user's interests. In particular, the marginal importance of a given item may depend on other items in the ranking, and it may be advantageous to make the list more diverse.  \citet{Streeter08,Golovin09} study a more general version with submodular rewards, and  \cite{ZoomingRBA-icml10} consider an extension to Lipschitz bandits. In all this work, one only achieves additive regret relative to the $(1-1/e)$-th fraction of the optimal reward.
\citet{CascadingBandits-icml15,CascadingBandits-nips15,CascadingBandits-uai16} achieve much stronger guarantees, without the multiplicative approximation, assuming that relevance is independent across the items.%
\footnote{\cite{CascadingBandits-nips15} study a version of cascade feedback in which the aggregate reward for a ranked list of items is 1 if all items are ``good" and 0 otherwise, and the feedback returns the first item that is ``bad". The main motivation is network routing: an action corresponds to a path in the network, and it may be the case that only the first faulty edge is revealed to the algorithm.}

In some versions of combinatorial semi-bandits, atoms are assigned rewards in each round independently of the algorithm's choices, but the ``aggregate outcome" associated with a particular subset $S$ of atoms chosen by the algorithm can be more general compared to the standard version. For example, \citet{Chen-nips16} allows the aggregate reward of $S$ to be a function of the per-atom rewards in $S$, under some mild assumptions. \citet{Chen-icml13} allows other atoms to be ``triggered", \ie included into $S$, according to some distribution determined by $S$.

\chapter{Contextual Bandits}
\label{ch:CB}

\begin{ChAbstract}
In \emph{contextual bandits}, rewards in each round depend on a \emph{context}, which is observed by the algorithm prior to making a decision. We cover the basics of three prominent versions of contextual bandits: with a Lipschitz assumption, with a linearity assumption, and with a fixed policy class. We also touch upon offline learning from contextual bandit data. Finally, we discuss challenges that arise in large-scale applications, and a system design that address these challenges in practice.

\prereqs{Chapters~\ref{ch:IID},~\ref{ch:LB},~\ref{ch:adv} (for background/perspective only),
Chapter~\ref{ch:Lip} (for Section~\ref{CB:sec:Lip}).}
\end{ChAbstract}

We consider a generalization called \emph{contextual bandits}, defined as follows:

\begin{BoxedProblem}{Contextual bandits}
  For each round $t \in [T]$:
\begin{OneLiners}
\item[1.] algorithm observes a ``context" $x_t$,
\item[2.] algorithm picks an arm $a_t$,
\item[3.] reward $r_t\in[0,1]$ is realized.
\end{OneLiners}
\end{BoxedProblem}

\noindent The reward $r_t$ in each round $t$ depends both on the context $x_t$ and the chosen action $a_t$. We make the IID assumption: reward $r_t$ is drawn independently from some distribution parameterized by the $(x_t,a_t)$ pair, but same for all rounds $t$. The expected reward of action $a$ given context $x$ is denoted $\mu(a|x)$. This setting allows a limited amount of ``change over time", but this change is completely ``explained" by the observable
contexts. We assume contexts $x_1, x_2 , \; \ldots$ are chosen by an oblivious adversary.

Several variants of this setting have been studied in the literature. We discuss three prominent variants in this chapter: with a Lipschitz assumption, with a linearity assumption, and with a fixed policy class.

\xhdr{Motivation.} The main motivation is that a user with a known ``user profile" arrives in each round, and the context is the user profile. The algorithm can personalize the user's experience. Natural application scenarios include choosing which news articles to showcase, which ads to display, which products to recommend, or which webpage layouts to use. Rewards in these applications are often determined by user clicks, possibly in conjunction with other observable signals that correlate with revenue and/or user satisfaction. Naturally, rewards for the same action may be different for different users.

Contexts can include other things apart from (and instead of) user profiles. First, contexts can include known features of the environment, such as day of the week, time of the day, season (\eg Summer, pre-Christmas shopping season), or proximity to a major event (\eg Olympics, elections). Second, some actions may be unavailable in a given round and/or for a given user, and a context can include the set of feasible actions. Third, actions can come with features of their own, and it may be convenient to include this information into the context, esp. if these features can change over time.

\xhdr{Regret relative to best response.}
For ease of exposition, we assume a fixed and known time horizon $T$. The set of actions and the set of all contexts are $\mA$ and $\mX$, resp.;  $K=|\mA|$ is the number of actions.

The (total) reward of an algorithm $\ALG$ is $\REW(\ALG) = \sum_{t=1}^T r_t$, so that the expected reward is
    \[ \E[\REW(\ALG)] = \textstyle \sum_{t=1}^T \E\sbr{\mu(a_t|x_t)}.\]
A natural benchmark is the best-response policy,
     $\pi^*(x) = \max_{a\in \mathcal{A}}\mu(a|x)$.
Then regret is defined as
\begin{align}\label{CB:eq:CB-regret-defn-BR}
    R(T) = \REW(\pi^*)-\REW(\ALG).
\end{align}

\section{Warm-up: small number of contexts}
\label{CB:sec:CB-small}

One straightforward approach for contextual bandits  is to apply a known bandit algorithm \ALG such as \UcbOne: namely, run a separate copy of this algorithm for each context.

\LinesNotNumbered \SetAlgoLined
\begin{algorithm}
\begin{algorithmic}
\STATE {\bf Initialization:} For each context $x$, create an instance $\ALG_x$ of algorithm $\ALG$
\STATE \For{each round $t$}{
        \STATE invoke algorithm $\ALG_x$ with $x=x_t$
    \STATE ``play" action $a_t$ chosen by $\ALG_x$, return reward $r_t$ to $\ALG_x$.
}
\end{algorithmic}
\caption{Contextual bandit algorithm for a small number of contexts}
\label{CB:alg:CB-small}
\end{algorithm}

Let $n_x$ be the number of rounds in which context $x$ arrives. Regret accumulated in such rounds, denoted $R_x(T)$, satisfies
    $  \E[R_x(T)] = O(\sqrt{Kn_x\ln T})$.
The total regret (from all contexts) is
\begin{align*}
\E[R(T)]  = \textstyle \sum_{x\in \mX} \E[R_x(T)]
    = \sum_{x\in \mX} O(\sqrt{K n_x\, \ln T})
    \leq O(\sqrt{KT\,|\mX|\,\ln T}).
\end{align*}

\begin{theorem}\label{CB:thm:CB-small}
Algorithm~\ref{CB:alg:CB-small} has regret
    $\E[R(T)] = O(\sqrt{KT\,|\mX|\,\ln T})$,
provided that the bandit algorithm \ALG has regret $ \E[R_\ALG(T)] = O(\sqrt{KT\log T})$.
\end{theorem}

\begin{remark}
The square-root dependence on $|\mX|$ is slightly non-trivial, because a completely naive solution would give linear dependence. However, this regret bound is still very high if $|\mX|$ is large, \eg if contexts are feature vectors with a large number of features.  To handle contextual bandits with a large $|\mX|$, we either assume some structure (as in Sections~\ref{CB:sec:Lip} and~\ref{CB:sec:lin}, or change the objective (as in Section~\ref{CB:sec:policy-class}).
\end{remark}


\section{Lipshitz contextual bandits}
\label{CB:sec:Lip}

Let us consider contextual bandits with Lipschitz-continuity, as a simple end-to-end example of how structure allows to handle contextual bandits with a large number of contexts. We assume that contexts map into the $[0,1]$ interval (\ie $\mX\subset [0,1])$ so that the expected rewards are Lipschitz with respect to the contexts:
\begin{align}\label{CB:eq:CB-Lip}
|\mu(a|x) - \mu(a|x')| \leq L\cdot |x - x'|
 \quad\text{for any arm $a \in \mA $ and any contexts $x,x'\in \mX$},
\end{align}
where $L$ is the Lipschitz constant which is known to the algorithm.

One simple solution for this problem is given by uniform discretization of the context space. The approach is very similar to what we've seen for Lipschitz bandits and dynamic pricing; however, we need to be a little careful with some details: particularly, watch out for ``discretized best response". Let $S$ be the \emph{\eps-uniform mesh} on $[0,1]$, \ie the set of all points in $[0,1]$ that are integer multiples of $\eps$. We take $\eps=1/(d-1)$, where the integer $d$ is the number of points in $S$, to be adjusted later in the analysis.

\begin{figure}[h]
\begin{center}
\begin{tikzpicture}
\draw[-](0,0)--(5,0);
\draw[thick](0,-0.1) node[below]{0}--(0,0.1);
\draw[thick](5,-0.1) node[below]{1}--(5,0.1);
\draw[thick](1,-0.1) -- (1,0.1);
\draw[](2.5,0) node[below]{...};
\draw[](0.5,0) node[above]{$\eps$};
\draw[thick](4,-0.1) -- (4,0.1);
\end{tikzpicture}
\end{center}
\caption{Discretization of the context space}
\end{figure}
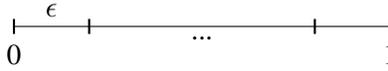

We will use the contextual bandit algorithm from Section~\ref{CB:sec:CB-small}, applied to context space $S$; denote this algorithm as $\ALG_S$. Let $f_S(x)$ be a mapping from context $x$ to the closest point in $S$:
    \[ f_S(x) = \min(\argmin_{x'\in S} |x-x'|) \]
(the $\min$ is added just to break ties). The overall algorithm proceeds as follows:
\textEqNum{CB:alg-Lip-unif}{In each round $t$, ``pre-process" the context $x_t$ by replacing it with $f_S(x_t)$, and call $\ALG_S$.}

The regret bound will have two summands: regret bound for $\ALG_S$ and (a suitable notion of) discretization error. Formally, let us define the ``discretized best response" $\pi^*_S:\mX\to \mA$:
\[ \pi^*_S(x) = \pi^*(f_S(x)) \quad \text{for each context $x\in \mX$}.
\]
Then regret of $\ALG_S$ and discretization error are defined as, resp.,
\begin{align*}
 R_S(T) &= \REW(\pi^*_S) - \REW(\ALG_S) \\
 \DE(S) &= \REW(\pi^*)-\REW(\pi^*_S).
\end{align*}
It follows that the ``overall" regret is the sum
    $R(T) = R_S(T) + \DE(S)$,
as claimed. We have
    $\E[R_S(T)] = O(\sqrt{KT|S|\ln T})$
from Lemma~\ref{CB:thm:CB-small}, so it remains to upper-bound the discretization error and adjust the discretization step $\eps$.

\begin{claim}
$\E[\DE(S)] \leq \eps LT$.
\end{claim}
\begin{proof}
For each round $t$ and the respective context $x=x_t$,
\begin{align*}
\mu(\pi^*_S(x)\mid f_S(x))
    &\geq \mu(\pi^*(x) \mid f_S(x)) & \EqComment{by optimality of $\pi^*_S$} \\
    &\geq \mu(\pi^*(x) \mid x) -\eps L   & \EqComment{by Lipschitzness} .
\end{align*}
Summing this up over all rounds $t$, we obtain
\[ \E[\REW(\pi^*_S)] \geq \REW[\pi^*]-\eps LT. \qedhere \]
\end{proof}
Thus, regret is
\begin{align*}
\E[R(T)] \leq \eps LT + O(\sqrt{\tfrac{1}{\eps}\,KT \ln T}) = O(T^{2/3}(LK\ln T)^{1/3}).
\end{align*}
where for the last inequality we optimized the choice of $\eps$.

\begin{theorem}\label{CB:thm:CB-Lip-Uniform}
Consider the Lipschitz contextual bandits problem with contexts in $[0,1]$. The uniform discretization algorithm \eqref{CB:alg-Lip-unif} yields regret
$\E[R(T)] = O(T^{2/3}(LK\ln T)^{1/3})$.
\end{theorem}

An astute reader would notice a similarity with the uniform discretization result in Theorem~\ref{Lip:thm:CAB-uniform}. In fact, these two results admit a common generalization in which the Lipschitz condition applies to both contexts and arms, and arbitrary metrics are allowed. Specifically, the Lipschitz condition is now
\begin{align}\label{CB:eq:CB-Lip-general}
   |\mu(a|x) - \mu(a'|x')| \leq D_\mX(x,x') + D_\mA(a,a')
    \quad\text{for any arms $a,a'$ and contexts $x,x'$},
\end{align}
where $D_\mX, D_\mA$ are arbitrary metrics on contexts and arms, respectively, that are known to the algorithm. This generalization is fleshed out in Exercise~\ref{CB:ex:Lip-unif}.

\section{Linear contextual bandits (no proofs)}
\label{CB:sec:lin}

We introduce the setting of \emph{linear} contextual bandits, and define an algorithm for this setting called \LinUCB.

Let us recap the setting of linear bandits from Chapter~\ref{ch:lin}, specialized to stochastic bandits. One natural formulation is that each arm $a$ is characterized by a feature vector $x_a\in [0,1]^d$, and the expected reward is linear in this vector:
    $\mu(a)=x_a \cdot \theta$,
for some fixed but unknown vector $\theta\in[0,1]^d$. The tuple
\begin{align}\label{CB:eq:CB-linear-context}
    x = \rbr{x_a\in \{0,1\}^d:a\in \mathcal{A}}
\end{align}
can be interpreted as a \emph{static context}, \ie a context that does not change from one round to another.

In linear \emph{contextual} bandits, contexts are of the form \eqref{CB:eq:CB-linear-context}, and the expected rewards are linear:
\begin{align}\label{CB:eq:CB-linear-defn}
 \mu(a|x) = x_a\cdot \theta_a
    \quad\text{for all arms $a$ and contexts $x$},
\end{align}
for some fixed but unknown vector
    $  \theta = (\theta_a\in \mathcal{R}^d:a\in \mathcal{A}) $.
Note that we also generalize the unknown vector $\theta$ so that $\theta_a$ can depend on the arm $a$. Let $\Theta$ be the set of all feasible $\theta$ vectors, known to the algorithm.

This problem can be solved by a version of the UCB technique. Instead of constructing confidence bounds on the mean rewards of each arm, we do that for the $\theta$ vector. Namely, in each round $t$ we construct a ``confidence region" $C_t\subset \Theta$ such that $\theta\in C_t$ with high probability.  Then we use $C_t$ to construct an UCB on the mean reward of each arm given context $x_t$, and play an arm with the highest UCB. This algorithm, called \LinUCB, is summarized in Algorithm~\ref{CB:alg:CB-LinUCB}.

\LinesNotNumbered \SetAlgoLined
\begin{algorithm}[ht]
\begin{algorithmic}
\STATE
\For{each round $ t =1,2, \ldots$ }{
        \STATE Form a confidence region $C_t \subset \Theta$
            \algTAB\COMMENT{\ie $\theta \in C_t$ with high probability}
        \STATE Observe context $x=x_t$ of the form~\eqref{CB:eq:CB-linear-context}
        \STATE For each arm $a$, compute
        $\UCB_t(a|x_t) = \sup\limits_{\theta\in C_t} x_a\cdot \theta_a $
        \STATE Pick an arm $a$ which maximizes $\UCB_t(a|x_t)$.
}
\end{algorithmic}
\caption{\LinUCB: UCB-based algorithm for linear contextual bandits}
\label{CB:alg:CB-LinUCB}
\end{algorithm}

Suitably specified versions of \LinUCB allow for rigorous regret bounds, and work well in experiments. The best known worst-case regret bound is
    $\E[R(T)] = \tildeO(d\sqrt{T})$,
and there is a close lower bound
    $\E[R(T)] \geq \Omega(\sqrt{dT})$.
The gap-dependent regret bound from Chapter~\ref{ch:IID} carries over, too: \LinUCB enjoys
    $(d^2/\Delta)\cdot \polylog(T)$
regret for problem instances with minimal gap at least $\Delta$. Interestingly, the algorithm is known to work well in practice even for scenarios without linearity.

To completely specify the algorithm, one needs to specify what the confidence region is, and how to compute the UCBs; this is somewhat subtle, and there are multiple ways to do this. The technicalities of specification and analysis of LinUCB are beyond our scope.

\newpage
\section{Contextual bandits with a policy class}
\label{CB:sec:policy-class}

We now consider a more general contextual bandit problem where we do not make any assumptions on the mean rewards. Instead, we make the problem tractable by making restricting the benchmark in the definition of regret (\ie the first term in  \refeq{CB:eq:CB-regret-defn-BR}). Specifically, we define a \emph{policy} as a mapping from contexts to actions, and posit a known class of policies $\Pi$. Informally, algorithms only need to compete with the best policy in $\pi$. A big benefit of this approach is that it allows to make a clear connection to the ``traditional" machine learning, and re-use some of its powerful tools.

We assume that contexts arrive as independent samples from some fixed distribution $\mD$ over contexts. The expected reward of a given policy $\pi$ is then well-defined:
\begin{equation}
\mu(\pi) = \E_{x\in \mathcal{D}}\sbr{ \mu(\pi(x))\mid x}.
\end{equation}
This quantity, also known as \emph{policy value}, gives a concrete way to compare policies to one another. The appropriate notion of regret, for an algorithm \ALG, is relative to the best policy in a given policy class $\Pi$:
\begin{equation}
R_{\Pi}(T) =  T\max_{\pi\in \Pi}\mu(\pi) - \REW(\ALG).
\end{equation}
Note that the definition \eqref{CB:eq:CB-regret-defn-BR} can be seen a special case when $\Pi$ is the class of all policies.  For ease of presentation, we assume $K<\infty$ actions throughout; the set of actions is denoted $[K]$.

\begin{remark}[Some examples]\label{CB:rem:examples}
A policy can be based on a \emph{score predictor}: given a context $x$, it assigns a numerical score $\nu(a|x)$ to each action $a$. Such score can represent an estimated expected reward of this action, or some other quality measure. The policy simply picks an action with a highest predicted reward. For a concrete example, if contexts and actions are represented as known feature vectors in $\R^d$, a \emph{linear} score predictor is
 $\nu(a|x) = a M x$,
for some fixed $d\times d$ weight matrix $M$.

A policy can also be based on a \emph{decision tree}, a flowchart in which each internal node corresponds to a ``test" on some attribute(s) of the context (\eg is the user male of female), branches correspond to the possible outcomes of that test, and each terminal node is associated with a particular action. Execution starts from the ``root node" of the flowchart and follows the branches until a terminal node is reached.
\end{remark}

\xhdr{Preliminary result.} This problem can be solved with algorithm \ExpFour from Chapter~\ref{ch:adv}, with policies $\pi \in \Pi$ as ``experts". Specializing  Theorem~\ref{adv:thm:exp4-gamma} to the setting of contextual bandits, we obtain:

\begin{theorem}\label{CB:thm:Exp4}
Consider contextual bandits with policy class $\Pi$. Algorithm \ExpFour with expert set $\Pi$ yields regret
    $\E[R_{\Pi}(T)] = O(\sqrt{KT\log |\Pi|})$.
However, the running time per round is linear in $|\Pi|$.
\end{theorem}

\noindent This is a very powerful regret bound: it works for an arbitrary policy class $\Pi$, and the logarithmic dependence on $|\Pi|$ makes it tractable even if the number of possible contexts is huge. Indeed, while there are $K^{|\mX|}$ possible policies, for many important special cases $|\Pi|=K^c$, where $c$ depends on the problem parameters but not on $|\mX|$. This regret bound is essentially the best possible. Specifically, there is a nearly  matching lower bound (similar to \eqref{adv:eq:exp4-LB}), which holds for any given triple of parameters $K,T,|\Pi|$:
\begin{align}\label{CB:eq:exp4-LB}
    \E[R(T)] \geq \min\left(T,\;\Omega\left( \sqrt{KT \log(|\Pi|)/\log(K)} \right)\right).
\end{align}
However, the running time of \ExpFour scales as $|\Pi|$ rather than $\log |\Pi|$, which makes the algorithm prohibitively slow in practice. In what follows, we achieve similar regret rates, but with faster algorithms.

\xhdr{Connection to a classification problem.}
We make a connection to a well-studied classification problem in ``traditional" machine learning. This connection leads to faster algorithms, and is probably the main motivation for the setting of contextual bandits with policy sets.

To build up the intuition, let us consider the full-feedback version of contextual bandits, in which the rewards are observed for all arms.

\begin{BoxedProblem}{Contextual bandits with full feedback}
For each round $t =1,2, \ldots $:
\begin{OneLiners}
\item[1.] algorithm observes a ``context" $x_t$,
\item[2.] algorithm picks an arm $a_t$,
\item[3.] rewards $\tilde{r}_t(a)\geq 0$ are observed for all arms $a\in \mA$.
\end{OneLiners}
\end{BoxedProblem}

\noindent In fact, let us make the problem even easier. Suppose we already have $N$ data points of the form $(x_t;\tilde{r}_t(a): a\in \mA)$. What is the ``best-in-hindsight" policy for this dataset? More precisely, what is a policy $\pi\in\Pi$ with a largest \emph{realized policy value}
\begin{align}\label{CB:eq:realized-policy-value}
    \tilde{r}(\pi) = \frac{1}{N} \sum_{t=1}^N \tilde{r}_t(\pi(x_i)).
\end{align}
This happens to be a well-studied problem called ``cost-sensitive multi-class classification":

\begin{BoxedProblem}[Problem]{Cost-sensitive multi-class classification for policy class $\Pi$}
{\bf Given:} data points $(x_t;\tilde{r}_t(a): a\in \mA)$, $t\in [N]$.\\
{\bf Find:} policy $\pi\in\Pi$ with a largest realized policy value ~\eqref{CB:eq:realized-policy-value}.
\end{BoxedProblem}

In the terminology of classification problems, each context $x_t$ is an ``example", and the arms correspond to different possible ``labels" for this example. Each label has an associated reward/cost. We obtain a standard binary classification problem if for each data point $t$, there is one ``correct label" with reward $1$, and rewards for all other labels are $0$. The practical motivation for finding the ``best-in-hindsight" policy is that it is likely to be a good policy for future context arrivals. In particular, such policy is near-optimal under the IID assumption, see Exercise~\ref{CB:ex:emp-policy-value} for a precise formulation.

An algorithm for this problem will henceforth be called a \emph{classification oracle} for policy class $\Pi$. While the exact optimization problem is NP-hard for many natural policy classes, practically efficient algorithms exist for several important policy classes such as linear classifiers, decision trees and neural nets.

A very productive approach for designing contextual bandit algorithms uses a classification oracle as a subroutine. The running time is then expressed in terms of the number of oracle calls, the implicit assumption being that each oracle call is reasonably fast. Crucially, algorithms can use any available classification oracle; then the relevant policy class $\Pi$ is simply the policy class that the oracle optimizes over.

\xhdr{A simple oracle-based algorithm.} Consider a simple explore-then-exploit algorithm that builds on a classification oracle. First, we explore uniformly for the first $N$ rounds, where $N$ is a parameter. Each round $t$ of exploration gives a data point $(x_t,\tilde{r}_t(a)\in \mA)$
for the classification oracle, where the ``fake rewards" $\tilde{r}_t(\cdot)$ are given by inverse propensity scoring:
\begin{align}\label{CB:eq:CB-IPS-costs-unif}
\tilde{r}_t(a) =
\begin{cases}
    r_t K & \text{if } a = a_t\\
    0,              & \text{otherwise}.
\end{cases}
\end{align}
We call the classification oracle and use the returned policy in the remaining rounds; see Algorithm~\ref{CB:alg:CB-simple-oracle}.

\begin{algorithm}[h]
{\bf Parameter:} exploration duration $N$, classification oracle $\mO$
\begin{enumerate}
\item Explore uniformly for the first $N$ rounds: in each round, pick an arm u.a.r.
\item Call the classification oracle with data points
    $(x_t,\tilde{r}_t(a)\in \mA)$, $t\in [N]$
as per \refeq{CB:eq:CB-IPS-costs-unif}.
\item Exploitation: in each subsequent round, use the policy $\pi_0$ returned by the oracle.
\end{enumerate}
\caption{Explore-then-exploit with a classification oracle}
\label{CB:alg:CB-simple-oracle}
\end{algorithm}

\begin{remark}
Algorithm~\ref{CB:alg:CB-simple-oracle} is modular in two ways: it can take an arbitrary classification oracle, and it can use any other unbiased estimator instead of \refeq{CB:eq:CB-IPS-costs-unif}. In particular, the proof below only uses the fact that \refeq{CB:eq:CB-IPS-costs-unif} is an unbiased estimator with $\tilde{r}_t(a)\leq K$.
\end{remark}

For a simple analysis, assume that the rewards are in $[0,1]$ and that the oracle is \emph{exact}, in the sense that it returns a policy $\pi\in\Pi$ that exactly maximizes $\mu(\pi)$.

\begin{theorem}\label{CB:thm:CB-explore-then-exploit}
Let $\mO$ be an exact classification oracle for some policy class $\Pi$. Algorithm~\ref{CB:alg:CB-simple-oracle} parameterized with oracle $\mO$ and
    $N=T^{2/3}(K \log (|\Pi| T))^{\frac{1}{3}}$
has regret
    \[ \E[R_\Pi(T)] = O(T^{2/3}) (K \log (|\Pi| T))^{\frac{1}{3}}. \]
\end{theorem}

\noindent This regret bound has two key features: logarithmic dependence on $|\Pi|$ and $\tilde{O}(T^{2/3})$ dependence on $T$.

\begin{proof}
Let us consider an arbitrary $N$ for now. For a given policy $\pi$, we estimate its expected reward $\mu(\pi)$ using the realized policy value from \eqref{CB:eq:realized-policy-value}, where the action costs $\tilde{r}_t(\cdot)$ are from \eqref{CB:eq:CB-IPS-costs-unif}. Let us prove that $\tilde{r}(\pi)$ is an unbiased estimate for $\mu(\pi)$:
\begin{align*}
\E[\tilde{r}_t(a)\mid x_t]
    &= \mu(a\mid x_t) &\EqComment{for each action $a\in\mA$}\\
\E[\tilde{r}_t(\pi(x_t))\mid x_t]
    &= \mu(\pi(x)\mid x_t)
        &\EqComment{plug in $a=\pi(x_t)$}\\
\E_{x_t \sim D}[\tilde{r}_t(\pi(x_t))]
    &=\E_{x_t \sim D}[\mu(\pi(x_t))\mid x_t]
        &\EqComment{take expectation over both $\tilde{r}_t$ and $x_t$}\\
    &= \mu(\pi),
\end{align*}
which implies $\E[\tilde{r}(\pi)]= \mu(\pi)$, as claimed. Now, let us use this estimate to set up a ``clean event":
\[ \left\{ \mid  \tilde{r}(\pi) - \mu(\pi) \mid  \leq \conf(N) \text{ for all policies $\pi\in\Pi$}  \right\}, \]
where the confidence term is
     $ \conf(N) = O( \sqrt{\frac{K \log (T|\Pi|)}{N}}). $
We can prove that the clean event does indeed happen with probability at least $1-\tfrac{1}{T}$, say, as an easy application of Chernoff Bounds. For intuition, the $K$ is present in the confidence radius is because the ``fake rewards" $\tilde{r}_t(\cdot)$ could be as large as $K$. The $|\Pi|$ is there (inside the $\log$) because we take a Union Bound across all policies. And the $T$ is there because we need the ``error probability" to be on the order of $\tfrac{1}{T}$.

Let $\pi^*=\pi^*_\Pi$ be an optimal policy. Since we have an exact classification oracle, $\tilde{r}(\pi_0)$ is maximal among all policies $\pi\in \Pi$. In particular,
    $\tilde{r}(\pi_0)\geq \tilde{r}(\pi^*)$.
If the clean event holds, then
    \[\mu(\pi^\star) - \mu(\pi_0)\leq 2\, \conf(N).\]
Thus, each round in exploitation contributes at most $\conf$ to expected regret. And each round of exploration contributes at most $1$. It follows that
    $ \E[R_\Pi(T)] \leq N+2T\,\conf(N)$.
Choosing $N$ so that $N=O(T\,\conf(N))$, we obtain
    $N=T^{2/3}(K \log(|\Pi| T))^{\frac{1}{3}}$
and $ \E[R_\Pi(T)] =O(N)$.
\end{proof}

\begin{remark}
If the oracle is only approximate -- say, it returns a policy $\pi_0\in\Pi$ which optimizes $c(\cdot)$ up to an additive factor of $\eps$ -- it is easy to see that expected regret increases by an additive factor of $\eps T$. In practice, there may be a tradeoff between the approximation guarantee $\eps$ and the running time of the oracle.
\end{remark}


\begin{remark}\label{CB:rem:minimonster}
 A near-optimal regret bound can in fact be achieved with an \emph{oracle-efficient} algorithm: one that makes only a small number of oracle calls. Specifically, one can achieve regret $O(\sqrt{KT \log(T|\Pi|)})$ with only $\tildeO(\sqrt{KT/\log |\Pi|})$ oracle calls across all $T$ rounds. This sophisticated result can be found in \citep{monster-icml14}. Its exposition is beyond the scope of this book.
 \end{remark}

\section{Learning from contextual bandit data}
\label{CB:sec:policy-learning}

Data collected by a contextual bandit algorithm can be analyzed ``offline", separately from running the algorithm. Typical tasks are estimating the value of a given policy (\emph{policy evaluation}), and learning a policy that performs best on the dataset (\emph{policy training}). While these tasks are formulated in the terminology from Section~\ref{CB:sec:policy-class}, they are meaningful for all settings discussed in this chapter.

Let us make things more precise. Assume that a contextual bandit algorithm has been running for $T$ rounds according to the following extended protocol:

\begin{BoxedProblem}{Contextual bandit data collection}
For each round $t \in [N]$:
\begin{OneLiners}
\item[1.] algorithm observes a ``context" $x_t$,
\item[2.] algorithm picks a sampling distribution $p_t$ over arms,
\item[3.] arm $a_t$ is drawn independently from distribution $p_t$,
\item[4.] rewards $\tilde{r}_t(a)$ are realized for all arms $a\in \mA$,
\item[5.] reward $r_t = \tilde{r}_t(a_t) \in[0,1]$ is recorded.
\end{OneLiners}
\end{BoxedProblem}

Thus, we have $N$ data points of the form
    $(x_t,p_t,a_t,r_t)$.
The sampling probabilities $p_t$ are essential to form the IPS estimates, as explained below. It is particularly important to record the sampling probability for the chosen action,
    $p_t(a_t)$.
Policy evaluation and training are defined as follows:

\begin{BoxedProblem}[Problem]{Policy evaluation and training}
{\bf Input:} data points $(x_t,p_t,a_t,r_t)$, $t\in [N]$.\\
{\bf Policy evaluation:} estimate policy value $\mu(\pi)$ for a given policy $\pi$. \\
{\bf Policy training:} find policy $\pi\in\Pi$ that maximizes $\mu(\pi)$ over a given policy class $\Pi$.
\end{BoxedProblem}

\xhdr{Inverse propensity scoring.}
Policy evaluation can be addressed via inverse propensity scoring (IPS), like in Algorithm~\ref{CB:alg:CB-simple-oracle}. This approach is simple and does not rely on a particular model of the rewards such as linearity or Lipschitzness. We estimate the value of each policy $\pi$ as follows:
\begin{align}
\IPS(\pi) = \sum_{t\in [N]:\; \pi(x_t) = a_t}\; \frac{r_t}{p_t(a_t)}.
\end{align}
Just as in the proof of Theorem~\ref{CB:thm:CB-explore-then-exploit}, one can show that the IPS estimator is unbiased and accurate; accuracy holds with probability as long as the sampling probabilities are large enough.

\begin{lemma}\label{CB:lm:IPS-to-mu}
The IPS estimator is unbiased:
    $\E\sbr{\IPS(\pi)} = \mu(\pi)$
for each policy $\pi$. Moreover, the IPS estimator is accurate with high probability: for each $\delta>0$, with probability at least $1-\delta$
\begin{align}\label{CB:eq:IPS-to-mu}
|\IPS(\pi) - \mu(\pi)|
    \leq O\left(\sqrt{\tfrac{1}{p_0}\;\log(\tfrac{1}{\delta})/N}\right),
    \quad \text{where $p_0 = \min_{t,a} p_t(a)$}.
\end{align}
\end{lemma}

\begin{remark}
How many data points do we need to evaluate $M$ policies simultaneously? To make this question more precise, suppose we have some fixed parameters $\eps,\delta>0$ in mind, and want to ensure that
\begin{align*}
 \Pr\left[\;
    |\IPS(\pi) - \mu(\pi)| \leq \eps \quad \text{for each policy $\pi$}
 \;\right] > 1-\delta.
\end{align*}
How large should $N$ be, as a function of $M$ and the parameters? Taking a union bound over \eqref{CB:eq:IPS-to-mu}, we see that it suffices to take
\[ N \sim \frac{\sqrt{\log (M/\delta)}}{p_0\cdot\eps^2}.\]
The logarithmic dependence on $M$ is due to the fact that each data point $t$ can be reused to evaluate many policies, namely all policies $\pi$ with $\pi(x_t)=a_t$.
\end{remark}

We can similarly compare $\IPS(\pi)$ with the \emph{realized} policy value $\tilde{r}(\pi)$, as defined in~\eqref{CB:eq:realized-policy-value}. This comparison does not rely on IID rewards: it applies whenever rewards are chosen by an oblivious adversary.

\begin{lemma}\label{CB:lm:IPS-to-emp}
Assume rewards $\tilde{r}_t(\cdot)$ are chosen by a deterministic oblivious adversary.
Then we have
    $\E[\IPS(\pi)] = \tilde{r}(\pi)$
for each policy $\pi$. Moreover, for each $\delta>0$, with probability at least $1-\delta$ we have:
\begin{align}
|\IPS(\pi) - \tilde{r}(\pi)|
    \leq O\left(\sqrt{\tfrac{1}{p_0}\;\log(\tfrac{1}{\delta})/N}\right),
    \quad \text{where $p_0 = \min_{t,a} p_t(a)$}.
\end{align}
\end{lemma}

This lemma implies Lemma~\ref{CB:lm:IPS-to-mu}. It easily follows from a concentration inequality, see Exercise~\ref{CB:ex:IPS}.

Policy training can be implemented similarly: by maximizing $\IPS(\pi)$ over a given policy class $\Pi$. A maximizer $\pi_0$ can be found by a single call to the classification oracle for $\Pi$: indeed, call the oracle with ``fake rewards" defined by estimates:
    $\rho_t(a) = \ind{a=a_t} \frac{r_t}{p_t(a_t)}$
for each arm $a$ and each round $t$. The performance guarantee follows from Lemma~\ref{CB:lm:IPS-to-emp}:

\begin{corollary}\label{CB:cor:policy-training}
Consider the setting in Lemma~\ref{CB:lm:IPS-to-emp}. Fix policy class $\Pi$ and let $\pi_0\in \argmax_{\pi\in\Pi} \IPS(\pi)$. Then for each $\delta>0$, with probability at least $\delta>0$ we have
\begin{align}
\max_{\pi\in\Pi} \tilde{r}(\pi)- \tilde{r}(\pi_0)
    \leq O\left(\sqrt{\tfrac{1}{p_0}\;\log(\tfrac{|\Pi|}{\delta})/N}\right).
\end{align}
\end{corollary}

\xhdr{Model-dependent approaches.} One could estimate the mean rewards $\mu(a|x)$ directly, \eg with linear regression, and then use the resulting estimates $\hat{\mu}(a|x)$ to estimate policy values:
    \[ \hat{\mu}(\pi) = \E_{x\sim\mD} [ \hat{\mu}(\pi(x)\mid x)].  \]
Such approaches are typically motivated by some model of rewards, \eg linear regression is motivated by the linear model. Their validity depends on how well the data satisfies the assumptions in the model.

For a concrete example, linear regression based on the model in Section~\ref{CB:sec:lin} constructs estimates $\hat{\theta}_a$ for latent vector $\theta_a$, for each arm $a$. Then rewards are estimated as
    $\hat{\mu}(a|x) = x\cdot \theta_a$.

Model-dependent reward estimates can naturally be used for policy optimization:
\[ \pi(x) = \argmax_{a\in\mA} \hat{\mu}(a|x).\]
Such policy can be good even if the underlying model is not. No matter how a policy is derived, it can be evaluated in a model-independent way using the IPS methodology described above.

\section{Contextual bandits in practice: challenges and a system design}
\label{CB:sec:system}

Implementing contextual bandit algorithms for large-scale applications runs into a number of engineering challenges and necessitates a design of a \emph{system} along with the algorithms. We present a system for contextual bandits, called the \DS, building on the machinery from the previous two sections.

The key insight is the distinction between the \emph{front-end} of the system, which directly interacts with users and needs to be extremely fast, and the \emph{back-end}, which can do more powerful data processing at slower time scales. \emph{Policy execution}, computing the action given the context, must happen in the front-end, along with data collection (\emph{logging}). Policy evaluation and training usually happen in the back-end. When a better policy is trained, it can be deployed into the front-end. This insight leads to a particular methodology, organized as a loop in Figure~\ref{CB:fig-loop}. Let us discuss the four components of this loop in more detail.

\begin{figure}[h]
\begin{center}
\includegraphics[height=4cm]{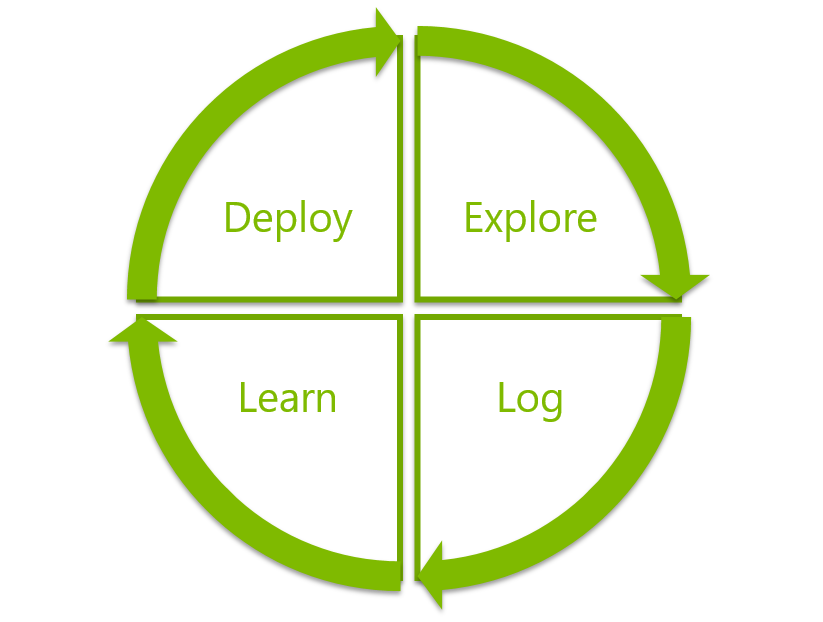}
\end{center}
\caption{The learning loop.}
\label{CB:fig-loop}
\end{figure}

\begin{description}
\item[Exploration] Actions are chosen by the \emph{exploration policy}: a fixed policy that runs in the front-end and combines exploration and exploitation. Conceptually, it takes one or more \emph{default policies} as subroutines, and adds some exploration on top. A default policy is usually known to be fairly good; it could be a policy already deployed in the system and/or trained by a machine learning algorithm. One basic exploration policy is \emph{\eps-greedy}: choose an action uniformly at random with probability $\eps$ (``exploration branch"), and execute the default policy with the remaining probability (``exploitation branch"). If several default policies are given, the exploitation branch chooses uniformly at random among them; this provides an additional layer of exploration, as the default policies may disagree with one another.

If a default policy is based on a score predictor $\nu(a|x)$, as in Example~\ref{CB:rem:examples}, an exploration policy can give preference to actions with higher scores. One such exploration policy, known as \emph{SoftMax}, assigns to each action $a$ the probability proportional to $e^{\tau\cdot \nu(a|x)}$, where $\tau$ is the exploration parameter. Note that $\tau=0$ corresponds to the uniform action selection, and increasing the $\tau$ favors actions with higher scores. Generically, exploration policy is characterized by its type (\eg \eps-greedy or SoftMax), exploration parameters (such as the $\eps$ or the $\tau$), and the default policy/policies.

\item[Logging] Logging runs in the front-end, and records the ``data points" defined in Section~\ref{CB:sec:policy-learning}. Each logged data point includes $(x_t,a_t,r_t)$ -- the context, the chosen arm, and the reward -- and also $p_t(a_t)$, the probability of the chosen action.  Additional information may be included to help with debugging and/or machine learning, \eg time stamp, current exploration policy, or full sampling distribution $p_t$.

In a typical high-throughput application such as a website, the reward (\eg a click) is observed long after the action is chosen -- much longer than a front-end server can afford to ``remember" a given user. Accordingly, the reward is logged separately. The \emph{decision tuple}, comprising the context, the chosen action, and sampling probabilities, is recorded by the front-end when the action is chosen.  The reward is logged after it is observed, probably via a very different mechanism, and joined with the decision tuple in the back-end. To enable this join, the front-end server generates a unique ``tuple ID" which is included in the decision tuple and passed along to be logged with the reward as the \emph{outcome tuple}.

Logging often goes wrong in practice. The sampling probabilities $p_t$ may be recorded incorrectly, or accidentally included as features. Features may be stored as references to a database which is updated over time (so the feature values are no longer the ones observed by the exploration policy). When optimizing an intermediate step in a complex system, the action chosen initially might be overridden by business logic, and the recorded action might incorrectly be this final action  rather than the initially chosen action. Finally, rewards may be lost or incorrectly joined to the decision tuples.

\item[Learning] Policy training, discussed in Section~\ref{CB:sec:policy-learning}, happens in the back-end, on the logged data. The main goal is to learn a better ``default policy", and perhaps also better exploration parameters. The policy training algorithm should be \emph{online}, allowing fast updates when new data points are received, so as to enable rapid iterations of the learning loop in  Figure~\ref{CB:fig-loop}.

\emph{Offline learning} uses logged data to simulate experiments on live users. Policy evaluation, for example, simulates deploying a given policy. One can also experiment with alternative exploration policies or parameterizations thereof. Further, one can try out other algorithms for policy training, such as those that need more computation, use different hyper-parameters, are based on estimators other than IPS, or lead to different policy classes. The algorithms being tested do not need to be approved for a live deployment (which may be a big hurdle), or implemented at a sufficient performance and reliability level for the said deployment (which tends to be very expensive).

\item[Policy deployment]
New default policies and/or exploration parameters are deployed into the exploration policy, thereby completing the learning loop in Figure~\ref{CB:fig-loop}. As a safeguard, one can use human oversight and/or policy evaluation to compare the new policy with others before deploying it.%
\footnote{A standard consideration in machine learning applies: a policy should not be evaluated on the training data. Instead, a separate dataset should be set aside for policy evaluation.} The deployment process should be automatic and frequent. The frequency of deployments depends on the delays in data collection and policy training, and on the need of human oversight. Also, some applications require a new default policy to improve over the old one by a statistically significant margin, which may require waiting for more data points.

\end{description}

The methodology described above leads to the following system design:

\begin{figure}[h]
\begin{center}
\includegraphics[height=4cm]{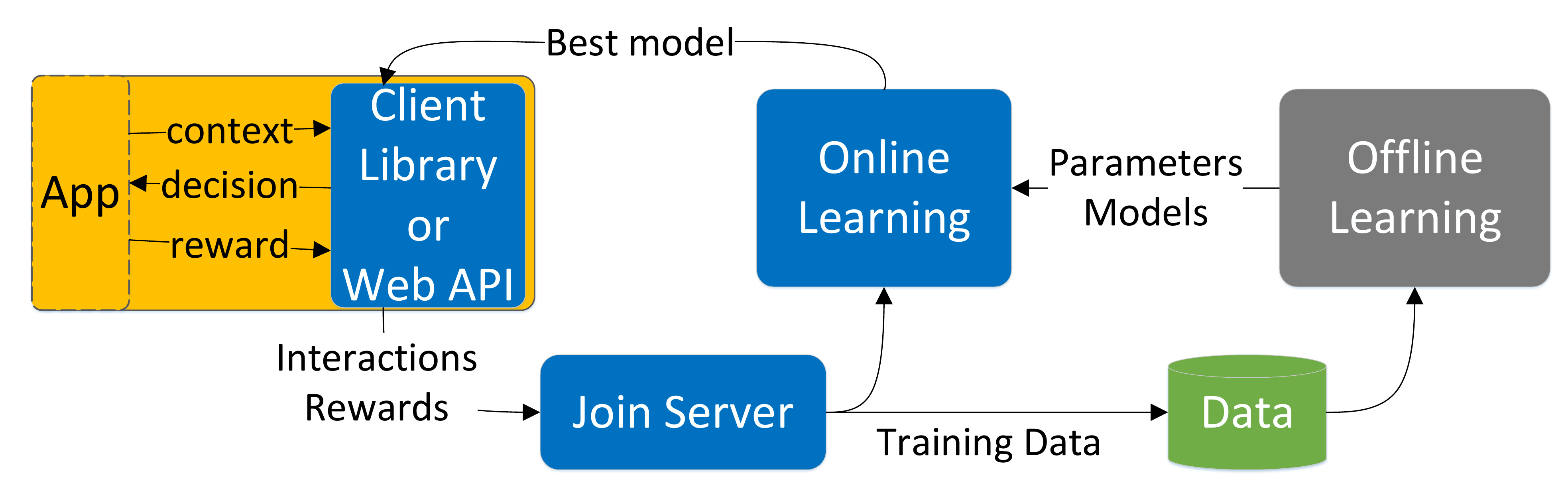}
\end{center}
\end{figure}

\noindent The front-end interacts with the application via a provided software library (``Client Library") or via an Internet-based protocol (``Web API"). The Client Library implements exploration and logging, and the Web API is an alternative interface to the said library. Decision tuples and outcome tuples are joined by the Join Server, and then fed into a policy training algorithm (the ``Online Learning" box). Offline Learning is implemented as a separate loop which can operate at much slower frequencies.

Modularity of the design is essential, so as to integrate easily with an application's existing infrastructure. In particular, the components have well-defined interfaces, and admit multiple consistent implementations. This avoids costly re-implementation of existing functionality, and allows the system to improve seamlessly as better implementations become available.

\subsection*{Essential issues and extensions}

Let discuss several additional issues and extensions that tend to be essential in applications. As a running example, consider optimization of a simple news website called \SimpleNews. While this example is based on a successful product deployment, we ignore some real-life complications for the sake of clarity. Thus, \SimpleNews displays exactly one headline to each user. When a user arrives, some information is available pertaining to the interaction with this particular user, \eg age, gender, geolocation, time of day, and possibly much more; such information is summarily called the \emph{context}. The website chooses a news topic (\eg politics, sports, tech, etc.),  possibly depending on the context, and the top headline for the chosen topic is displayed. The website's goal is to maximize the number of clicks on the headlines that it displays. For this example, we are only concerned with the choice of news topic: that is, we want to pick a news topic whose top headline is most likely to be clicked on for a given context.

\xhdr{Fragmented outcomes.} A good practice is to log the entire observable outcome, not just the reward. For example, the outcome of choosing a news topic in \SimpleNews includes the \emph{dwell time}: the amount of time the user spent reading a suggested article (if any). Multiple outcome fragments may arrive at different times. For example, the dwell time \SimpleNews is recorded long after the initial click on the article. Thus, multiple outcome tuples may be logged by the front-end. If the reward depends on multiple outcome fragments, it may be computed by the Join Server after the outcome tuples are joined, rather than in the front-end.

\xhdr{Reward metrics.} There may be several reasonable ways to define rewards, especially with fragmented outcomes. For example, in \SimpleNews the reward can include a bonus that depends on the dwell time, and there may be several reasonable choices for defining such a bonus.  Further, the reward may depend on the context, \eg in \SimpleNews some clicks may be more important than others, depending on the user's demographics. Thus, a reward metric used by the application is but a convenient proxy for the long-term objective such as cumulative revenue or long-term customer satisfaction. The reward metric may change over time when priorities change or when a better proxy for the long-term objective is found. Further, it may be desirable to consider several reward metrics at once, \eg for safeguards. Offline learning can be used to investigate the effects of switching to another reward metric in policy training.

\xhdr{Non-stationarity.}
While we've been positing a stationarity environment so far, applications exhibit only periods of near-stationarity in practice. To cope with a changing environment, we use a continuous loop in which the policies are re-trained and re-deployed quickly as new data becomes available. Therefore, the infrastructure and the policy training algorithm should process new data points sufficiently fast. Policy training should  de-emphasize older data points over time. Enough data is needed within a period of near-stationarity (so more data is needed to adapt to faster rate of change).

Non-stationarity is partially mitigated if some of it is captured by the context. For example, users of \SimpleNews may become more interested in sports during major sports events such as Olympics. If the presence of such events is included as features in the context, the response to a particular news topic given a particular context becomes more consistent across time.

In a non-stationary environment, the goal of policy evaluation is no longer to estimate the expected reward (simply because this quantity changes over time) or predict rewards in the future. Instead, the goal is \emph{counterfactual}: estimate the policy's performance if it were used when the exploration data was collected. This is a mathematically precise goal that is achievable (say) by the IPS estimator regardless of how the environment changes in the meantime. When algorithms for exploration and/or policy training are evaluated on the exploration data, the goal is counterfactual in a similar sense. This is very useful for comparing these algorithms with alternative approaches. One hopes that such comparison on the exploration data would be predictive of a similar comparison performed via a live A/B test.

\xhdr{Feature selection.} Selecting which features to include in the context is a fundamental issue in machine learning. Adding more features tends to lead to better rewards in the long run, but may slow down the learning. The latter could be particular damaging if the environment changes over time and fast learning is essential. The same features can be represented in different ways, \eg a feature with multiple possible values can be represented either as one value or as a bit vector (\emph{1-hot encoding}) with one bit for each possible value. Feature representation is essential, as general-purpose policy training algorithms tend to be oblivious to what features \emph{mean}. A good feature representation may depend on a particular policy training algorithm. Offline learning can help investigate the effects of adding/removing features or changing their representation. For this purpose, additional observable features (not included in the context) can be logged in the decision tuple and passed along to offline learning.

\xhdr{Ease of adoption.}
Developers might give up on a new system, unless it is easy to adopt. Implementation should avoid or mitigate dependencies on particular programming languages or external libraries. The trial experience should be seamless, \eg sensible defaults should be provided for all components.

\subsection*{Road to deployment}

Consider a particular application such as \SimpleNews, call it \APP. Deploying contextual bandits for this application should follow a process, even when the contextual bandit system such as the Decision Service is available. One should do some prep-work: frame the \APP as a contextual bandit problem, verify that enough data would be available, and have a realistic game plan for integrating the system with the existing infrastructure. Next step is a pilot deployment on a small fraction of traffic. The goal here is to validate the system for \APP and debug problems; deep integration and various optimizations can come later. Finally, the system is integrated into \APP, and deployed on a larger fraction of traffic.

\xhdr{Framing the problem.} Interactions between \APP and its users should be interpreted as a sequence of small interactions with individual users, possibly overlapping in time. Each small interaction should follow a simple template: observe a \emph{context}, make a \emph{decision}, choosing from the available alternatives, and observe the \emph{outcome} of this decision. The meaning of contexts, decisions and outcomes should be consistent throughout. The \emph{context}, typically represented as a vector of features, should encompass the properties of the current user and/or task to be accomplished, and must be known to \APP. The \emph{decision} must be controlled by \APP. The set of feasible actions should be known: either fixed in advance or specified by the context. It is often useful to describe actions with features of their own, a.k.a. \emph{\adfs}. The \emph{outcome} consists of one or several events, all of which must be observable by \APP not too long after the action is chosen. The outcome (perhaps jointly with the context) should define a \emph{reward}: the short-term objective to be optimized.  Thus, one should be able to fill in the table below:
%
\begin{center}
\begin{tabular}{l|l|l}
                    & \SimpleNews                           & Your \APP \\ \hline
Context             & (gender, location, time-of-day)       &\\
Decision            & a news topic to display               &\\
Feasible actions    & \{politics, sports, tech, arts\}      & \\
\Adfs               & none                                  &\\
Outcome             & click or no click within 20 seconds   & \\
Reward              & $1$ if clicked, $0$ otherwise         &
\end{tabular}
\end{center}
%
As discussed above, feature selection and reward definition can be challenging. Defining \emph{decisions} and \emph{outcomes} may be non-trivial, too. For example, \SimpleNews can be modified so that actions correspond to news articles, rather than news topics, and each \emph{decision} consists of choosing a slate of news articles.

\xhdr{Scale and feasibility.}
To estimate the scale and feasibility of the learning problem in \APP, one needs to estimate a few parameters: the number of features (\term{\#features}), the number of feasible actions (\term{\#actions}), a typical delay between making a decision and observing the corresponding outcome, and \emph{the data rate}: the number of  experimental units per a time unit. If the outcome includes a rare event whose frequency is crucial to estimate --- \eg clicks are typically rare compared to non-clicks, --- then we also need a rough estimate of this frequency.  Finally, we need the \emph{\stationarityTimescale}: the time interval during which the environment does not change too much, namely the distribution of arriving contexts (where a context includes the set of feasible actions and, if applicable, their features), and the expected rewards for each context-action pair. In the \SimpleNews example, this corresponds to the distribution of arriving user profiles and the click probability for the top headline (for a given news topic, when presented to a typical user with a given user profile).
To summarize, one should be able to fill in Table~\ref{CB:tbl:parameters}.

\begin{table}[th]
\begin{center}
\begin{tabular}{l|l|l}
                        & \SimpleNews       & Your \APP \\ \hline
\term{\#features}      & 3                 &\\
\term{\#actions}       & 4 news topics     & \\
Typical delay           & 5 sec             & \\
Data rate               & 100 users/sec     & \\
Rare event frequency    & typical click prob. 2-5\% & \\
\StationarityTimescale  & one week         &
\end{tabular}
\end{center}
\caption{The scalability parameters (using \SimpleNews as an example)}
\label{CB:tbl:parameters}
\end{table}

For policy training algorithms based on linear classifiers, a good rule-of-thumb is that the \stationarityTimescale should be much larger than the typical delay, and moreover we should have
\begin{align}\label{eq:data-req}
\mathtt{StatInterval}\times \mathtt{DataRate} \times \mathtt{RareEventFreq}
\gg \mathtt{\#actions} \times \mathtt{\#features}
\end{align}
The left-hand side is the number of rare events in the timescale, and the right-hand side characterizes the complexity of the learning problem.

\newpage
\sectionBibNotes
\label{CB:sec:further}

\xhdr{Lipschitz contextual bandits.}
Contextual bandits with a Lipschitz condition on contexts have been introduced in \citet{Hazan-colt07}, along with a solution via uniform discretization of contexts and, essentially, the upper bound in Theorem~\ref{CB:thm:CB-Lip-Uniform}. The extension to the more general Lipschitz condition \eqref{CB:eq:CB-Lip-general} has been observed, simultaneously and independently, in \citet{Pal-Bandits-aistats10} and \cite{contextualMAB-colt11}. The regret bound in Exercise~\ref{CB:ex:Lip-unif} is optimal in the worst case \citep{Pal-Bandits-aistats10,contextualMAB-colt11}.

Adaptive discretization improves over uniform discretization, much like it does in Chapter~\ref{ch:Lip}. The key is to discretize the context-arms pairs, rather than contexts and arms separately.  This approach is implemented in \citet{contextualMAB-colt11}, achieving regret bounds that are optimal in the worst case, and improve for ``nice" problem instances. For precisely, there is a contextual version of the ``raw" regret bound \eqref{Lip:eq:LipMAB-zooming-CovNum} in terms of the covering numbers, and an analog of Theorem~\ref{Lip:thm:LipMAB-zooming} in terms of a suitable version of the zooming dimension. Both regret bounds are essentially the best possible. This approach extends to an even more general Lipschitz condition when the right-hand side of \eqref{CB:eq:CB-Lip} is
an arbitrary metric on context-arm pairs.

\citet{contextualMAB-colt11} applies this machinery to handle some adversarial bandit problems. Consider the special case of Lipschitz contextual bandits when the context $x_t$ is simply the time $t$ (here, it is crucial that the contexts are \emph{not} assumed to arrive from fixed distribution). The Lipschitz condition can be written as
\begin{align}\label{CB:eq:further-slow}
 |\mu(a\mid t)-\mu(a\mid t')| \leq \mD_a(t,t'),
\end{align}
where $\mD_a$ is a metric on $[T]$ that is known to the algorithm, and possibly parameterized by the arm $a$. This condition describes a bandit problem with randomized adversarial rewards such that the expected rewards can only change slowly. The paradigmatic special cases are
    $\mD_a(t,t') = \sigma_a\cdot |t-t'|$,
bounded change in each round, and
    $\mD_a(t,t') = \sigma_a \cdot \sqrt{|t-t'|}$.
The latter case subsumes a scenario when mean rewards follow a random walk. More precisely, the mean reward of each arm $a$ evolves as a random walk with step $\pm \sigma_a$ on the $[0,1]$ interval with reflecting boundaries.  For the special cases, \citet{contextualMAB-colt11} achieves near-optimal bounds on ``dynamic regret" (see Section~\ref{adv:sec-further}), as a corollary of the "generic" machinery for adaptive discretization. In full generality, one has a metric space and a Lipschitz condition on context-arm-time triples, and the algorithm performs adaptive discretization over these triples.

\OMIT{ 
Further, \citet{contextualMAB-colt11} considers a version of Lipschitz contextual bandits with adversarial rewards, and provides a ``meta-algorithm'' which uses an off-the-shelf adversarial bandit algorithm such as \ExpThree as a subroutine and adaptively refines the space of contexts.
} 

\citet{Rakhlin-jmlr15,AlgoChaining-colt17} tackle a version of Lipschitz contextual bandits in which the comparison benchmark is the best \emph{Lipschitz policy}: a mapping $\pi$ from contexts to actions which satisfies
    $D_\mA(\pi(x),\pi(x')) \leq D_\mX(x,x')$
for any two contexts $x,x'$, where $D_\mA$ and $D_\mX$ are the metrics from \eqref{CB:eq:CB-Lip-general}. Several feedback models are considered, including bandit feedback and full feedback.

\citet{SmoothedRegret-colt19} consider a mash-up of Lipschitz bandits and contextual bandits with policy sets, in which the Lipschitz condition holds across arms for every given context, and regret is with respect to a given policy set. In addition to worst-case regret bounds that come from uniform discretization of the action space, \citet{SmoothedRegret-colt19} obtain instance-dependent regret bounds which generalize those for the zooming algorithm in Section~\ref{Lip:sec:zoom}. The algorithm is based on a different technique
    \citep[from][]{policy_elim},
and is not computationally efficient.

\xhdr{Linear contextual bandits}
have been introduced in \citet{Langford-www10}, motivated by personalized news recommendations. Algorithm \LinUCB was defined in \citet{Langford-www10}, and analyzed in \citep{Reyzin-aistats11-linear,Csaba-nips11}.%
\footnote{The original analysis in \citet{Langford-www10} suffers from a subtle bug, as observed in \citet{Csaba-nips11}. The gap-dependent regret bound is from \citet{Csaba-nips11}; \citet{DaniHK-colt08} obtain a similar result for static contexts.}
The details -- how the confidence region is defined, the computational implementation, and even the algorithm's name -- differ between from one paper to another. The name \LinUCB stems from \citet{Langford-www10}; we find it descriptive, and use as an umbrella term for the template in Algorithm~\ref{CB:alg:CB-LinUCB}.
Good empirical performance of LinUCB has been observed in \citep{semi-CB-nips16}, even when the problem is not linear. The analysis of \LinUCB in \citet{Csaba-nips11} extends to a more general formulation, where the action set can be infinite and time-dependent. Specifically, the action set in a given round $t$ is a bounded subset  $D_t\subset [0,1]^d$, so that each arm $a\in D_t$ is identified with its feature vector: $\theta_a = a$; the context in round $t$ is simply $D_t$.

The version with ``static contexts" (\ie stochastic linear bandits) has been introduced in \citet{Auer-focs00}. The non-contextual version of \LinUCB was suggested in \citet{Auer-focs00}, and analyzed in \cite{DaniHK-colt08}. \citet{Auer-focs00}, as well \citet{Abe-algo03} and \citet{Paat-mor10}, present and analyze other algorithms for this problem which are based on the same paradigm of ``optimism under uncertainty", but do not fall under \LinUCB template.

If contexts are sufficiently ``diverse", \eg if they come from a sufficiently ``diffuse" distribution, then greedy (exploitation-only) algorithm work quite well. Indeed, it achieves the $\tildeO(\sqrt{KT})$ regret rate, which  optimal in the worst case \citep{bastani2017exploiting,kannan2018smoothed}. Even stronger guarantees are possible if the unknown vector $\theta$ comes from a Bayesian prior with a sufficiently large variance, and one is interested in the Bayesian regret, \ie regret in expectation over this prior \citep{Greedy-Manish-18}. The greedy algorithm matches the best possible Bayesian regret for a given problem instance, and is at most $\tildeO_K(T^{1/3})$ in the worst case. Moreover, \LinUCB achieves Bayesian regret $\tildeO_K(T^{1/3})$ under the same assumptions.

\xhdr{Contextual bandits with policy classes.}
Theorem~\ref{CB:thm:Exp4} is from \citet{bandits-exp3}, and the complimentary lower bound \eqref{CB:eq:exp4-LB} is due to \citet{RegressorElim-aistats12}. The oracle-based approach to contextual bandits was proposed in \citet{Langford-nips07}, with an ``epsilon-greedy"-style algorithm and regret bound similar to those in Theorem~\ref{CB:thm:CB-explore-then-exploit}. \citet{policy_elim} obtained a near-optimal regret bound, $O(\sqrt{KT \log(T|\Pi|)})$, via an algorithm that is oracle-efficient ``in theory". This algorithm makes $\poly(T,K,\log |\Pi|)$ oracle calls and relies on the ellipsoid algorithm. Finally, a break-through result of \citep{monster-icml14} achieved the same regret bound via a ``truly" oracle-efficient algorithm which makes only $\tildeO(\sqrt{KT/\log |\Pi|})$ oracle calls across all $T$ rounds. \citet{semi-CB-nips16} extend this approach to combinatorial semi-bandits.

Another recent break-through extends oracle-efficient contextual bandits to adversarial rewards
\citep{Syrgkanis-AdvCB-icml16,Rakhlin-AdvCB-icml16,Syrgkanis-AdvCB-nips16}. The optimal regret rate for this problem is not yet settled: in particular, the best current upper bound has $\tildeO(T^{2/3})$ dependence on $T$, against the $\Omega(\sqrt{T})$ lower bound in \refeq{CB:eq:exp4-LB}.

\citet{Haipeng-dynamicRegret-colt18,Haipeng-dynamicRegret-19} design oracle-efficient algorithms with data-dependent bounds on dynamic regret (see the discussion of dynamic regret in Section~\ref{adv:sec-further}). The regret bounds are in terms of the number of switches $S$ and the total variation $V$, where $S$ and $V$ could be unknown, matching the regret rates (in terms of $S,V,T$) for the non-contextual case.

Classification oracles tend to be implemented via heuristics in practice \citep{monster-icml14,semi-CB-nips16}, since the corresponding classification problems are NP-hard for all/most policy classes that one would consider. However, the results on oracle-based contextual bandit algorithms (except for the $\eps$-greedy-based approach) do not immediately extend to approximate classification oracles. Moreover, the oracle's performance in practice does not necessarily carry over to its performance inside a contextual bandit algorithm, as the latter typically calls the oracle on carefully constructed artificial problem instances.

\xhdr{Contextual bandits with realizability.}
Instead of a linear model, one could posit a more abstract assumption of \emph{realizability}: that the expected rewards can be predicted perfectly by some function $\mX\times\mA \to [0,1]$, called \emph{regressor}, from a given class $\mF$ of regressors. This assumption leads to improved performance, in a strong provable sense, even though the worst-case regret bounds cannot be improved \citep{RegressorElim-aistats12}. One could further assume the existence of a \emph{regression oracle} for $\mF$: a computationally efficient algorithm which finds the best regressor in $\mF$ for a given dataset. A contextual bandit algorithm can use such oracle as a subroutine, similarly to using a classification oracle for a given policy class \citep{regressionCB-icml18,regressionCB-icml20,regressionCB-bypassing}. Compared to classification oracles, regression oracles tend to be computationally efficient without any assumptions, both in theory and in practice.

\xhdr{Offline learning.} ``Offline" learning from contextual bandit data is a well-established approach initiated in \citet{Langford-wsdm11}, and further developed in subsequent work,
\citep[\eg][]{Dudik-uai12,DR-StatScience14,Adith-jmlr15,Adith-PolicyEval-nips17}.
In particular,
\citet{DR-StatScience14} develops \emph{doubly-robust} estimators which, essentially, combine the benefits of IPS and model-based approaches;
\citet{Dudik-uai12} consider non-stationary policies (\ie algorithms);
\citet{Adith-PolicyEval-nips17} consider policies for ``combinatorial" actions, where each action is a slate of ``atoms".

\xhdr{Practical aspects.}
The material in Section~\ref{CB:sec:system} is adapted from \citep{MWT-WhitePaper-2016,DS-arxiv}. The Decision Service, a system for large-scale applications of contextual bandits, is available at GitHub.%
\footnote{{\tt https://github.com/Microsoft/mwt-ds/}.}
At the time of this writing, the system is deployed at various places inside Microsoft, and is offered externally as a Cognitive Service on Microsoft Azure.
\footnote{\url{https://docs.microsoft.com/en-us/azure/cognitive-services/personalizer/.}}
The contextual bandit algorithms and some of the core functionality is provided via \emph{Vowpal Wabbit},%
\footnote{\url{https://vowpalwabbit.org/}.}
an open-source library for machine learning.

\citet{practicalCB-arxiv18} compare the empirical performance of various contextual bandit algorithms. \citet{Lihong-www15} provide a case study of policy optimization and training in web search engine optimization.

\sectionExercises

\begin{exercise}[Lipschitz contextual bandits]\label{CB:ex:Lip-unif}
Consider the Lipschitz condition in \eqref{CB:eq:CB-Lip-general}. Design an algorithm with regret bound  $\tilde{O}(T^{(d+1)/(d+2})$, where $d$ is the covering dimension of $\mX\times \mA$.

\Hint{Extend the uniform discretization approach, using the notion of \eps-mesh from Definition~\ref{Lip:def:mesh}. Fix an \eps-mesh $S_\mX$ for $(\mX,\mD_\mX)$ and an \eps-mesh $S_\mA$ for $(\mA,\mD_\mA)$, for some $\eps>0$ to be chosen in the analysis. Fix an optimal bandit algorithm $\ALG$ such as \UcbOne, with $S_\mA$ as a set of arms. Run a separate copy $\ALG_x$ of this algorithm for each context $x\in S_\mX$.}
\end{exercise}

\begin{exercise}[Empirical policy value]\label{CB:ex:emp-policy-value}
Prove that realized policy value $\tilde{r}(\pi)$, as defined in \eqref{CB:eq:realized-policy-value}, is close to the policy value $\mu(\pi)$. Specifically, observe that
    $\E[\tilde{r}(\pi)] = \mu(\pi)$.
Next, fix $\delta>0$ and prove that
\begin{align*}
|\mu(\pi) - \tilde{r}(\pi)| \leq \conf(N)
    \quad \text{with probability at least $1-\delta/|\Pi|$},
\end{align*}
where the confidence term is
$\conf(N) = O\rbr{\sqrt{\nicefrac{1}{N}\cdot\log(|\Pi|/\delta)}}$.
Letting $\pi_0$ be the policy with a largest realized policy value, it follows that
\begin{align*}
\max_{\pi\in \Pi}\mu(\pi) - \mu(\pi_0) &\leq \conf(N)
\quad \text{with probability at least $1-\delta$}
\end{align*}

\end{exercise}

\begin{exercise}[Policy evaluation]\label{CB:ex:IPS}
Use Bernstein Inequality to prove Lemma~\ref{CB:lm:IPS-to-emp}. In fact, prove a stronger statement:
for each policy $\pi$ and each $\delta>0$, with probability at least $1-\delta$ we have:
\begin{align}
|\IPS(\pi) - \tilde{r}(\pi)|
   \leq O\left(\sqrt{V\;\log(\nicefrac{1}{\delta})/N}\right),
    \quad \text{where } V = \max_t \sum_{a\in \mA} \frac{1}{p_t(a)}.
\end{align}
\end{exercise}

\chapter{Bandits and Games}
\label{ch:games}
\begin{ChAbstract}
This chapter explores connections between bandit algorithms and game theory. We consider a bandit algorithm playing a repeated zero-sum game against an adversary (\eg another bandit algorithm).  We are interested in convergence to an equilibrium: whether, in which sense, and how fast this happens. We present a sequence of results in this direction, focusing on best-response and regret-minimizing adversaries. Our analysis also yields a self-contained proof of von Neumann's "Minimax Theorem". We also present a simple result for general games.


\prereqs{Chapters~\ref{ch:FF},~\ref{ch:adv}.}
\end{ChAbstract}


Throughout this chapter, we consider the following setup. A bandit algorithm \ALG plays a repeated game against another algorithm, called an \emph{adversary} and denoted \ADV. The game is characterized by a matrix $M$, and proceeds over $T$ rounds. In each round $t$ of the game, \ALG chooses row $i_t$ of $M$, and \ADV chooses column $j_t$ of $M$. They make their choices simultaneously, \ie without observing each other. The corresponding entry $M(i_t,j_t)$ specifies the cost for \ALG. After the round, \ALG observes the cost $M(i_t,j_t)$ and possibly some auxiliary feedback. Thus, the problem protocol is as follows:

\begin{BoxedProblem}{Repeated game between an algorithm and an adversary}
In each round $t \in 1,2,3 \LDOTS T$:
\begin{enumerate}
\item Simultaneously, \ALG chooses a row $i_t$ of $M$, and \ADV chooses a column $j_t$ of $M$.

\item \ALG incurs cost $M(i_t, j_t)$, and observes feedback
    $\mF_t = \mF(t,i_t,j_t,M)$,\\
where $\mF$ is a fixed and known \emph{feedback function}.
\end{enumerate}
\end{BoxedProblem}

%
%
%

We are mainly interested in two feedback models:
\begin{align*}
\mF_t &= M(i_t,j_t)
    &\EqComment{bandit feedback}, \\
\mF_t  &= \rbr{M(i,j_t): \text{all rows $i$}},
    &\EqComment{full feedback, from \ALG's perspective}.
\end{align*}
However, all results in this chapter hold for an arbitrary feedback function $\mF$.

\ALG's objective is to minimize its total cost,
    $\cost(\ALG) = \textstyle \sum_{t=1}^T  M(i_t,j_t)$.

We consider zero-sum games, unless noted otherwise, as the corresponding theory is well-developed and has rich applications. Formally, we posit that \ALG's cost in each round $t$ is also \ADV's reward.  In standard game-theoretic terms, each round is a \emph{zero-sum game} between \ALG and \ADV, with \emph{game matrix} $M$.%
\footnote{Equivalently, \ALG and \ADV incur costs $M(i_t,j_t)$ and $-M(i_t,j_t)$, respectively. Hence the term `zero-sum game'.}

\xhdr{Regret assumption.}
The only property of \ALG that we will rely on is its regret. Specifically, we consider a bandit problem with the same feedback model $\mF$, and for this bandit problem we consider regret $R(T)$ against an adaptive adversary, relative to the best-observed arm, as defined in Chapter~\ref{FF:sec:adv}. All results in this chapter are meaningful even if we only control \emph{expected} regret $\E[R(T)]$, more specifically if $\E[R(t)]=o(t)$. Recall that this property is satisfied by algorithm \Hedge for full feedback, and algorithm \ExpThree for bandit feedback, as per Chapters~\ref{FF:sec:hedge} and~\ref{ch:adv}).


Let us recap the definition of $R(T)$. Using the terminology from Chapter~\ref{FF:sec:adv}, the algorithm's cost for choosing row $i$ in round $t$ is
    $c_t(i) = M(i,j_t)$.
Let
    $\cost(i) = \sum_{t=1}^T c_t(i)$
be the total cost for always choosing row $i$, under the observed costs $c_t(\cdot)$. Then
    \[ \cost^* := \min_{\text{rows $i$}} \cost(i) \]
is the cost of the ``best observed arm". Thus, as per \refeq{FF:eq:experts-regret-hindsight},
\begin{align}\label{games:eq:regret-defn}
R(T) = \cost(\ALG) - \cost^*.
\end{align}
This is what we will mean by \emph{regret} throughout this chapter.


\xhdr{Motivation.} At face value, the setting in this chapter is about a repeated game between agents that use  regret-minimizing algorithms. Consider algorithms for ``actual" games such as chess, go, or poker. Such algorithm can have a ``configuration" that can be tuned over time by a regret-minimizing ``meta-algorithm" (which interprets  configurations as actions, and wins/losses as payoffs). Further, agents in online commerce, such as advertisers in an ad auction and sellers in a marketplace, engage in repeated interactions such that each interaction proceeds as a game between agents according to some fixed rules, depending on bids, prices, or other signals submitted by the agents. An agent may adjust its behavior in this environment using a regret-minimizing algorithm. (However, the resulting repeated game is not zero-sum.)

Our setting, and particularly Theorem~\ref{games:thm:AlgoAdv} in which both \ALG and \ADV minimize regret, is significant in several other ways. First, it leads to an important prediction about human behavior: namely, humans would approximately arrive at an equilibrium if they behave so as to minimize regret (which is a plausible behavioural model). Second, it provides a way to compute an approximate Nash equilibrium for a given game matrix. Third, and perhaps most importantly, the repeated game serves as a subroutine for a variety of  algorithmic problems, see Section~\ref{games:sec:further} for specific pointers. In particular, such subroutine is crucial for a bandit problem discussed in Chapter~\ref{ch:BwK}.

\section{Basics: guaranteed minimax value}
\label{games:sec:basics}

\xhdr{Game theory basics.} Imagine that the game consists of a single round, \ie $T=1$. Let $\Drows$ and $\Dcols$ denote the set of all distributions over rows and colums of $M$, respectively. Let us extend the notation $M(i,j)$ to distributions over rows/columns: for any
    $p\in \Drows, q\in \Dcols$,
we define
\begin{align*}
M(p,q)
    := \E_{i\sim p,\, j\in q} \bLR{M(i,j)}
    = p\tran Mq,
\end{align*}
where distributions are interpreted as column vectors.


Suppose \ALG chooses a row from some distribution $p\in\Drows$ known to the adversary. Then the adversary can choose a distribution $q\in\Dcols$ so as to maximize its expected reward $M(p,q)$. (Any column $j$ in the support of $q$ would maximize $M(p,\cdot)$, too.)
Accordingly, the algorithm should choose $p$ so as to minimize its maximal cost,
\begin{align}\label{games:eq:f-of-p}
f(p) := \sup_{q\in \Dcols} M(p,q)
    = \max_{\text{columns $j$}}\; M(p,j).
\end{align}
A distribution $p=p^*$ that minimizes $f(p)$ exactly is called a \emph{minimax strategy}. At least one such $p^*$ exists, as an $\argmin$ of a continuous function on a closed and bounded set. A minimax strategy achieves cost
\begin{align*}
v^* = \min_{p\in \Drows}\; \max_{q\in \Dcols}\; M(p,q),
\end{align*}
called the \emph{minimax value} of the game $M$. Note that $p^*$ guarantees cost at most $v^*$ against any adversary:
\begin{align}\label{games:eq:guaranteed}
M(p^*,j) \leq v^* \quad \forall\; \text{column $j$}.
\end{align}

\xhdr{Arbitrary adversary.} We apply the regret property of the algorithm and deduce that algorithm's expected average costs are approximately at least as good as the minimax value $v^*$. Indeed,
\begin{align*}
\cost^*
    =  \min_{\text{rows $i$}} \cost(i)
    \leq \E_{i\sim p^*}\bLR{\cost(i)}
    = \textstyle \sum_{t=1}^T M(p^*,j_t)
    \leq T\, v^*.
\end{align*}
Recall that $p^*$ is the minimax strategy, and the last inequality follows by \refeq{games:eq:guaranteed}.
By definition of regret, we have:

\begin{theorem}\label{games:thm:arbADV}
For an arbitrary adversary $\ADV$ it holds that
\[ \tfrac{1}{T}\, \cost(\ALG) \leq v^* + R(T)/T.\]
\end{theorem}

In particular, if we (only) posit sublinear expected regret,
    $\E[R(t)] = o(t)$,
then the expected average cost
    $\tfrac{1}{T}\, \E[\cost(\ALG)]$
is asymptotically upper-bounded by $v^*$.

\xhdr{Best-response adversary.} Assume a specific (and very powerful) adversary called the \emph{best-response adversary}. Such adversary operates as follows: in each round $t$, it chooses a column $j_t$ which maximizes its expected reward given the \emph{history}
\begin{align}\label{games:eq:H}
    \mH_t := \left(\; (i_1,j_1) \LDOTS (i_{t-1},j_{t-1}) \;\right),
\end{align}
which encompasses what happened in the previous rounds. In a formula,
\begin{align}\label{games:eq:bestResponseADV-defn}
 j_t = \min \left( \argmax_{\text{columns $j$}} \E\,[ M(i_t,j)\mid \mH_t \,] \right),
\end{align}
where the expectation is over the algorithm's choice of row $i_t$.%
\footnote{Note that the column $j$ in the $\argmax$ need not be unique, hence the $\min$ in front of the $\argmax$. Any other tie-breaking rule would work, too. Such column $j$ also maximizes the expected reward
    $\E\,[ M(i_t,\cdot)\mid \mH_t \,]$
over all distributions $q\in \Dcols$.}

We consider the algorithm's \emph{average play} (we give an abstract definition which we reuse later).
 
\begin{definition}\label{games:def-avePlay}
The average play of a given bandit algorithm (up to time $T$) is a distribution $D$ over arms such that for each arm $a$, the coordinate $D(a)$ is the fraction of rounds in which this arm is chosen.
\end{definition}

\noindent Thus, let $\ibar \in \Drows$ be the algorithm's average play. It is useful to represent it as a vector over the rows and write
    $\ibar := \tfrac{1}{T} \sum_{t\in [T]} \unitVrow{i_t} $,
where
    $\unitVrow{i}$
is the $i$-th unit vector over rows. We argue that the expected average play, $\E[\ibar]$, performs well against an arbitrary column $j$. This is easy to prove:
\begin{align}
\E[ M( i_t,\,j)]
    &= \E[\; \E[ M( i_t,\,j) \mid \mH_t] \;] \nonumber \\
    &\leq \E[\; \E[ M( i_t,\,j_t) \mid \mH_t] \;]
        &\EqComment{by best response}\nonumber \\
    &= \E[ M( i_t,\,j_t)] \nonumber \\
M(\, \E[\ibar],j\,)
    &= \textstyle \frac{1}{T} \sum_{t=1}^T \E[ M( i_t,\,j)]
        &\EqComment{by linearity of $M(\cdot,\cdot)$} \nonumber \\
    &\leq \textstyle \frac{1}{T}\sum_{t=1}^T \E[\;M(i_t,j_t) \;]
        \nonumber \\
    &= \tfrac{1}{T} \E[\cost(\ALG)].
    \label{games:eq:bestResponseADV}
\end{align}
Plugging this into Theorem~\ref{games:thm:arbADV}, we prove the following:

\begin{theorem}\label{games:thm:bestResponseADV}
If \ADV is the best-response adversary, as in \eqref{games:eq:bestResponseADV-defn}, then
\[ M(\, \E[\ibar],q\,)
    \leq \tfrac{1}{T}\,\E[\cost(\ALG)]
    \leq v^* + \E[R(T)]/T
    \qquad\forall  q\in\Dcols.\]
\end{theorem}

\noindent Thus, if $\E[R(t)]=o(t)$ then \ALG's expected average play $\E[\ibar]$ asymptotically achieves the minimax property of $p^*$, as expressed by \refeq{games:eq:guaranteed}.

\section{The minimax theorem}
\label{games:sec:minimax}

A fundamental fact about the minimax value is that it equals the maximin value:
\begin{align}\label{games:eq:minimax-maximin}
\min_{p\in \Drows}\; \max_{q\in \Dcols}\; M(p,q)
=\max_{q\in \Dcols}\;\min_{p\in \Drows}\;  M(p,q).
\end{align}
In other words, the $\max$ and the $\min$ can be switched. The maximin value is well-defined, in the sense that the $\max$ and the $\min$ exist, for the same reason as they do for the minimax value.

The maximin value emerges naturally in the  single-round game if one switches the roles of $\ALG$ and $\ADV$ so that the former controls the columns and the latter controls the rows (and $M$ represents algorithm's rewards rather than costs). Then a \emph{maximin strategy} -- a distribution $q=q^*\in \Dcols$ that maximizes the right-hand size of \eqref{games:eq:minimax-maximin} -- arises as the algorithm's best response to a best-responding adversary. Moreover, we have an analog of \refeq{games:eq:guaranteed}:
\begin{align}\label{games:eq:guaranteed-q}
M(p,q^*) \geq \MaxMinVal \quad \forall p\in \Drows,
\end{align}
where $\MaxMinVal$ be the right-hand side of \eqref{games:eq:minimax-maximin}. In words, maximin strategy $q^*$ guarantees reward at least $\MaxMinVal$ against any adversary. Now, since $v^* = \MaxMinVal$ by \refeq{games:eq:minimax-maximin}, we can conclude the following:

\begin{corollary}\label{games:cor:Nash}
    $M(p^*,q^*) = v^*$,
and the pair $(p^*,q^*)$ forms a \emph{Nash equilibrium}, in the sense that
\textEqNum{games:eq:Nash}{$p^*\in \argmin_{p\in \Drows} M(p,q^*)$
and $q^*\in\argmax_{q\in\Dcols} M(p^*,q)$.}
\end{corollary}

With this corollary, \refeq{games:eq:minimax-maximin} is a celebrated early result in mathematical game theory, known as the \emph{minimax theorem}. Surprisingly, it admits a simple alternative proof based on the existence of sublinear-regret algorithms and the machinery developed earlier in this chapter.

\begin{proof}[Proof of \refeq{games:eq:minimax-maximin}]
The $\geq$ direction is easy:
\begin{align*}
M(p,q)
    &\geq \min_{p'\in\Drows} M(p',q) \qquad\forall q\in\Dcols \\
\max_{q\in\Dcols} M(p,q)
    &\geq \max_{q\in\Dcols}\; \min_{p'\in\Drows} M(p',q).
\end{align*}

The $\leq$ direction is the difficult part. Let us consider a full-feedback version of the repeated game studied earlier. Let $\ALG$ be any algorithm with subliner expected regret, \eg  \Hedge algorithm from  Chapter~\ref{FF:sec:hedge}, and let $\ADV$ be the best-response adversary. Define
\begin{align}\label{games:eq:minimax-maximin-h}
    h(q):= \inf_{p\in\Drows} M(p,q)
        = \min_{\text{rows $i$}}\; M(i,q)
\end{align}
for each distribution $q\in\Dcols$, similarly to how $f(p)$ is defined in \eqref{games:eq:f-of-p}. We need to prove that
\begin{align}\label{games:eq:minimax-maximin-pf}
    \min_{p\in\Drows} f(p) \leq \max_{q\in\Dcols} h(q).
\end{align}
Let us take care of some preliminaries. Let $\jbar$ be the average play of \ADV, as per Definition~\ref{games:def-avePlay}. 
Then:
\begin{align}
\E[\cost(i)]
    &= \textstyle \sum_{t=1}^T \E[\, M(i,j_t) \,]
    = T\, \E[\, M(i,\jbar)\,]
    = T\, M(i,\, \E[\jbar])
    \qquad \text{for each row $i$}
    \nonumber \\
\tfrac{1}{T}\,\E[\cost^*]
    &\leq \tfrac{1}{T}\, \min_{\text{rows $i$}} \E[\cost(i)]
    = \min_{\text{rows $i$}} M(i, \,\E[\jbar])
    =  \min_{p\in\Drows} M(p, \,\E[\jbar])
    = h(\E[\jbar]). \label{games:eq:AlgoAdv-prelims-costs}
\end{align}

The crux of the argument is as follows:
\begin{align*}
\min_{p\in\Drows} f(p)
    &\leq f(\, \E[\ibar]\,) \\
    &\leq \tfrac{1}{T}\,\E[\cost(\ALG)]
        &\EqComment{by best response, see \eqref{games:eq:bestResponseADV}} \\
    &= \tfrac{1}{T}\,\E[\cost^*] + \tfrac{\E[R(T)]}{T}
        &\EqComment{by definition of regret} \\
    &\leq h(\E[\jbar]) + \tfrac{\E[R(T)]}{T}
        &\EqComment{using $h(\cdot)$, see \eqref{games:eq:AlgoAdv-prelims-costs}} \\
    &\leq \max_{q\in\Dcols} h(q) + \tfrac{\E[R(T)]}{T}.
\end{align*}
Now, taking $T\to\infty$ implies \refeq{games:eq:minimax-maximin-pf} because $\E[R(T)]/T \to 0$, completing the proof.
\end{proof}

\section{Regret-minimizing adversary}
\label{games:sec:Nash}

The most interesting version of our game is when the adversary itself is a regret-minimizing algorithm (possibly in a different feedback model). More formally, we posit that after each round $t$ the adversary observes its reward $M(i_t,j_t)$ and auxiliary feedback
    $\mF'_t = \mF'(t,i_t,j_t,M)$,
where $\mF'$ is some fixed feedback model. We consider the regret of \ADV in this feedback model, denote it $R'(T)$. We use the minimax theorem to prove that $(\ibar,\jbar)$, the average play of $\ALG$ and $\ADV$, approximates the Nash equilibrium $(p^*,q^*)$.

First, let us express $\cost^*$ in terms of the function
    $h(\cdot)$
from \refeq{games:eq:minimax-maximin-h}:
\begin{align}
\cost(i)
    &= \textstyle \sum_{t=1}^T M(i,j_t)
    = T\, M(i,\jbar)
    \qquad \text{for each row $i$}
    \nonumber \\
\tfrac{1}{T}\,\cost^*
    &= \min_{\text{rows $i$}} M(i, \,\jbar)
    = h(\jbar). \label{games:eq:AlgAlg-prelims-costs}
\end{align}

Let us analyze \ALG's costs:
\begin{align}
\tfrac{1}{T}\,\cost(\ALG) -\tfrac{R(T)}{T}
    &= \tfrac{1}{T}\,\cost^*  & \EqComment{regret for \ALG} \nonumber \\
    &= h(\jbar)
        & \EqComment{using $h(\cdot)$, see \eqref{games:eq:AlgAlg-prelims-costs}} \nonumber\\
    &\leq h(q^*) =\MaxMinVal    & \EqComment{by definition of maximin strategy $q^*$} \nonumber\\
    &= v^*          & \EqComment{by \refeq{games:eq:minimax-maximin}}
        \label{games:eq:AlgoAdv-costs}
\end{align}

Similarly, analyze the rewards of $\ADV$. Let
    $\rew(j) = \sum_{t=1}^T M(i_t,j)$
be the total reward collected by the adversary for always choosing column $j$, and let
    \[ \rew^* := \max_{\text{columns $j$}} \rew(j) \]
be the reward of the ``best observed column". Let's take care of some formalities:
\begin{align}
\rew(j)
    &= \textstyle \sum_{t=1}^T M(i_t,j)
    = T\, M(\ibar,\, j)
    \qquad \text{for each column $j$}
    \nonumber \\
\tfrac{1}{T}\,\rew^*
    &= \max_{\text{columns $j$}} M(\ibar,\,j)
    = \max_{q\in\Dcols} M(\ibar,\,q)
    = f(\ibar). \label{games:eq:AlgoAdv-prelims-rewards}
\end{align}

Now, let us use the regret of \ADV:
\begin{align}
\tfrac{1}{T}\,\cost(\ALG) +\tfrac{R'(T)}{T}
    &= \tfrac{1}{T}\,\rew(\ADV) +\tfrac{R'(T)}{T}
             & \EqComment{zero-sum game} \nonumber \\
    &= \tfrac{1}{T}\,\rew^*    & \EqComment{regret for \ADV} \nonumber \\
    &= f(\ibar)
        & \EqComment{using $f(\cdot)$, see \eqref{games:eq:AlgoAdv-prelims-rewards}} \nonumber\\
    &\geq f(p^*) = v^*    & \EqComment{by definition of minimax strategy $p^*$}
        \label{games:eq:AlgoAdv-rewards}
\end{align}

Putting together \eqref{games:eq:AlgoAdv-costs} and \eqref{games:eq:AlgoAdv-rewards}, we obtain
\[
v^* - \tfrac{R'(T)}{T}
\leq f(\ibar) - \tfrac{R'(T)}{T}
\leq \tfrac{1}{T}\,\cost(\ALG)
\leq h(\jbar) + \tfrac{R(T)}{T}
\leq v^* + \tfrac{R(T)}{T}.
\]
Thus, we have the following theorem:
\begin{theorem}\label{games:thm:AlgoAdv}
Let $R'(T)$ be the regret of $\ADV$. Then the average costs/rewards converge, in the sense that
\begin{align*}
 \tfrac{1}{T}\,\cost(\ALG)
    = \tfrac{1}{T}\,\rew(\ADV)
        \in \left[ v^*-\tfrac{R'(T)}{T},\, v^*+\tfrac{R(T)}{T} \right].
\end{align*}
Moreover, the average play $(\ibar,\jbar)$ forms an \emph{$\eps_T$-approximate Nash equilibrium}, with
$\eps_T:=\tfrac{R(T)+R'(T)}{T}$:
\begin{align*}
M(\ibar,\,q) \leq v^* + \eps_T &\qquad \forall q\in \Dcols,   \\
M(p,\,\jbar) \geq v^* - \eps_T &\qquad \forall p\in \Drows.
\end{align*}
\end{theorem}

The guarantees in Theorem~\ref{games:thm:AlgoAdv} are strongest if \ALG and \ADV admit high-probability regret bounds, \ie if
\begin{align}\label{games:eq:HP}
    \text{$R(T) \leq \tilde{R}(T)$ and $R'(T) \leq \tilde{R}'(T)$}
\end{align}
with probability at least $1-\gamma$, for some functions $\tilde{R}(t)$ and $\tilde{R}'(t)$. Then the guarantees in Theorem~\ref{games:thm:AlgoAdv} hold, with $R(T)$ and $R'(T)$ replaced with, resp., $\tilde{R}(T)$ and $\tilde{R}'(T)$,
whenever \eqref{games:eq:HP} holds. In particular, algorithm \Hedge and a version of algorithm \ExpThree achieve high-probability regret bounds, resp., for full feedback with
    $R(T) = O(\sqrt{T \log K})$,
for bandit feedback with
    $R(T) = O(\sqrt{T K\log K})$,
where $K$ is the number of actions \citep{FS97,bandits-exp3}.

Further, we recover similar but weaker guarantees even if we only have a bound on \emph{expected} regret:

\begin{corollary}
$\tfrac{1}{T}\,\E[\cost(\ALG)]
    = \tfrac{1}{T}\,\E[\rew(\ADV)]
        \in \left[ v^*-\tfrac{\E[R'(T)]}{T},\, v^*+\tfrac{\E[R(T)]}{T} \right].$

Moreover, the expected average play
    $(\E[\ibar],\,\E[\jbar])$
forms an $\E[\eps_T]$-approximate Nash equilibrium.
\end{corollary}

\noindent This is because
    $\E[\, M(\ibar,\,q)\,] = M(\E[\ibar],q)$
by linearity, and similarly for $\jbar$.

\begin{remark}\label{games:rem:convergence}
Convergence of $\cost(\ALG)$ to the corresponding equilibrium cost $v^*$ is characterized by the \emph{convergence rate}
    $\tfrac{|\cost(\ALG)-v^*|}{T}$,
as a function of $T$. Plugging in generic $\tilde{O}(\sqrt{T})$ regret bounds into Theorem~\ref{games:thm:AlgoAdv} yields convergence rate
    $\tilde{O}\left(\tfrac{1}{\sqrt{T}}\right)$.
\end{remark}

\OMIT{ 
\begin{remark}
The best-response adversary \eqref{games:eq:bestResponseADV-defn} can be seen as an algorithm. (More formally, this algorithm knows \ALG and the game matrix $M$, and receives auxiliary feedback $\mF'(t,i_t,j_t,M) = j_t$.) Its expected regret satisfies $\E[R'(T)]\leq 0$, see Exercise~\ref{games:ex:BR}. Thus, Theorem~\ref{games:thm:AlgoAdv} applies, with $\E[R'(T)]\leq 0$.
\end{remark}

\xhdr{Best-response adversary (revisited).}
Examining the proof of Theorem~\ref{games:thm:AlgoAdv}, we see that regret property of $\ADV$ is only used to prove \refeq{games:eq:AlgoAdv-rewards}. Meanwhile, a best-response adversary also satisfies this equation with $R'(T)=0$, by \refeq{games:eq:bestResponseADV}. Therefore:

\begin{theorem}\label{games:thm:best-response-revisited}
For the best-response adversary $\ADV$, the guarantees in Theorem~\ref{games:thm:AlgoAdv} hold with $R'(T)=0$.
\end{theorem}
} 


\newpage
\section{Beyond zero-sum games: coarse correlated equilibrium}
\label{games:sec:beyond}

What can we prove if the game is not zero-sum? While we would like to prove convergence to a Nash Equilibrium, like in Theorem~\ref{games:thm:AlgoAdv}, this does not hold in general. However, one can prove a weaker notion of convergence, as we explain below. Consider distributions over row-column pairs, let $\Dpairs$ be the set of all such distributions. We are interested in the average distribution defined as follows:
\begin{align}\label{games:eq:beyond-sigma}
    \bar{\sigma} := (\sigma_1 +\ldots + \sigma_T)/T \in \Dpairs,
    \qquad\text{where } \sigma_t := p_t\times q_t \in \Dpairs
\end{align}
We argue that $\bar{\sigma}$ is, in some sense, an approximate equilibrium.

Imagine there is a ``coordinator" who takes some distribution $\sigma\in \Dpairs$, draws a pair $(i,j)$ from this distribution, and recommends row $i$ to \ALG and column $j$ to \ADV. Suppose each  player has only two choices: either ``commit" to following the recommendation before it is revealed, or ``deviate" and not look at the recommendation. The equilibrium notion that we are interested here is that each player wants to ``commit" given that the other does.

Formally, assume that \ADV ``commits". The expected costs for \ALG are
\begin{align*}
U_\sigma := \E_{(i,j)\sim \sigma} M(i,j)
    &\qquad \text{if \ALG ``commits"}, \\
U_{\sigma}(i_0) := \E_{(i,j)\sim\sigma} M(i_0,j)
    &\qquad\text{if \ALG ``deviates" and chooses row $i_0$ instead}.
\end{align*}
Distribution $\sigma\in\Dpairs$ is a \emph{coarse correlated equilibrium} (CCE) if
    $U_\sigma \geq U_\sigma(i_0)$ for each row $i_0$,
and a similar property holds for \ADV.

We are interested in the approximate version of this property: $\sigma\in\Dpairs$ is an $\eps$-approximate CCE if
\begin{align}\label{games:eq:beyond-CCE}
    U_\sigma \geq U_\sigma(i_0) -\eps \qquad\text{for each row $i_0$}
\end{align}
and similarly for \ADV. It is easy to see that distribution $\bar{\sigma}$ achieves this with $\eps=\tfrac{\E[R(T)]}{T}$. Indeed,
\begin{align*}
U_{\bar{\sigma}}
    &:= \E_{(i,j)\sim \bar{\sigma}} \bLR{ M(i,j) }
    =\tfrac{1}{T}\, \sum_{t=1}^T \E_{(i,j)\sim \sigma_t} \bLR{ M(i,j) }
    = \tfrac{1}{T}\, \E[\cost(\ALG)] \\
U_{\bar{\sigma}}(i_0)
    &:= \E_{(i,j)\sim\bar{\sigma}} \bLR{ M(i_0,j) }
    = \tfrac{1}{T}\, \sum_{t=1}^T \E_{j\sim q_t} \bLR{ M(i_0,j) }
    = \tfrac{1}{T}\, \E[\cost(i_0)].
\end{align*}
Hence, $\bar{\sigma}$ satisfies \refeq{games:eq:beyond-CCE} with $\eps=\tfrac{\E[R(T)]}{T}$ by definition of regret. Thus:

\begin{theorem}\label{games:thm:beyond}
Distribution $\bar{\sigma}$ defined in \eqref{games:eq:beyond-sigma} forms an $\eps_T$-approximate CCE, where $\eps_T = \tfrac{\E[R(T)]}{T}$.
\end{theorem}

\newpage
\sectionBibNotes
\label{games:sec:further}

The results in this chapter stem from extensive literature on learning in repeated games, an important subject in theoretical economics. A deeper discussion of this subject from the online machine learning perspective can be found in
\citet[Chapter 7]{CesaBL-book} and the bibliographic notes therein.
Empirical evidence that regret-minimization is a plausible model of self-interested behavior can be found in recent studies \citep{Nekipelov-ec15,Noti-www17}.

\subsection{Zero-sum games}

Some of the early work concerns \emph{fictitious play}: two-player zero-sum games where each player best-responds to the empirical play of the opponent (\ie the historical frequency distribution over the arms). Introduced in \citet{Brown-49}, fictitious play was proved to converge at the rate $O(t^{-1/n})$ for $n\times n$ game matrices \citep{Robinson-51}. This convergence rate is the best possible \citep{Daskalakis-focs14}.

The repeated game between two regret-minimizing algorithms serves as a subroutine for a variety of  algorithmic problems. This approach can be traced to \cite{freund1996game,freund1999adaptive}. It has been used as a unifying algorithmic framework for several  problems: boosting~\cite{freund1996game}, linear programs~\cite{AroraHK}, maximum flow~\cite{Christiano11MaxFlow},
and  convex optimization~\cite{abernethy2017frank,wang2018acceleration}. In conjunction with a specific way to define the game matrix, this approach can solve a variety of constrained optimization problems, with application domains ranging from differential privacy to algorithmic fairness to learning from revealed preferences \citep{rogers2015inducing,hsu2016jointly,roth2016watch,pmlr-v80-kearns18a,agarwal2018reductions,roth2017multidimensional}.
In Chapter~\ref{ch:BwK}, this approach is used to solve bandits with global constraints.

\xhdr{Refinements.}
Our analysis can be refined for algorithm \Hedge and/or the best-response adversary: 
\begin{itemize}
\item The best-response adversary can be seen as a regret-minimizing algorithm which satisfies Theorem~\ref{games:thm:AlgoAdv} with $\E[R'(T)]\leq 0$; see Exercise~\ref{games:ex:BR}.


\item The \Hedge vs. \Hedge game satisfies the approximate Nash property from Theorem~\ref{games:thm:AlgoAdv} \emph{with probability $1$}, albeit in a modified feedback model; see Exercise~\ref{games:ex:Hedge}(b).

\item The \Hedge vs. Best Response game also satisfies the approximate Nash property with probability $1$; see Exercise~\ref{games:ex:Hedge}(c).

\end{itemize}

\noindent The last two results can be used for computing an approximate Nash equilibrium for a known matrix $M$.

\OMIT{ 
Consider the setting in Section~\ref{games:sec:Nash} under different feedback model called \emph{conditional-expected feedback}: instead of the actual cost $M(\cdot ,j_t)$, \ALG is given its conditional expected cost
        $\E[M(\cdot,j_t) \mid \mH_t]$,
where $\mH_t$ is the history, as in \eqref{games:eq:H}. Likewise, \ADV is given feedback
    $\E[M(i_t,\cdot) \mid \mH_t]$.
} 

The results in this chapter admit multiple extensions, some of which are discussed below.

\begin{itemize}

\item Theorem~\ref{games:thm:AlgoAdv} can be extended to \emph{stochastic games}, where in each round $t$, the game matrix $M_t$ is drawn independently from some fixed distribution over game matrices (see Exercise~\ref{games:ex:IID}).

\item One can use algorithms with better regret bounds if the game matrix $M$ allows it. For example, \ALG can be an algorithm for Lipschitz bandits if all functions $M(\cdot,j)$ satisfy a Lipschitz condition.%
\footnote{That is, if
    $ |M(i,j) - M(i',j)| \leq D(i,i')$
        for all rows $i,i'$ and all columns $j$,
for some known metric $D$ on rows.}

\item Theorem~\ref{games:thm:AlgoAdv} and the minimax theorem can be extended to infinite action sets, as long as \ALG and \ADV admit $o(T)$ expected regret for the repeated game with a given game matrix; see Exercise~\ref{games:ex:inf} for a precise formulation.

\OMIT{ 
In fact, all theorems in this chapter extend to game matrices with infinitely many rows, as long as \ALG's regret on the restricted problem satisfies
    $\E[R(t)] = o(t)$, see Exercise~\ref{games:ex:restricted} for details. In particular, we obtain the corresponding extension of the minimax theorem. We obtain a similar extension for infinitely many columns, if the analogous property holds for \ADV.
} 

\end{itemize}

\xhdr{Faster convergence.} \citet{Rakhlin-nips13,Daskalakis-GEB15} obtain $\tilde{O}(\frac{1}{t})$ convergence rates for repeated zero-sum games with full feedback. (This is a big improvement over the $\tilde{O}(t^{-1/2})$ convergence results in this chapter, see Remark~\ref{games:rem:convergence}.) \citet{Haipeng-colt18} obtains $\tilde{O}(t^{-3/4})$ convergence rate for repeated zero-sum games with bandit feedback. These results are for specific classes of algorithms, and rely on improved analyses of the repeated game.

\xhdr{Last-iterate convergence.}
While all convergence results in this chapter are only in terms of the \emph{average} play,
it is very tempting to ask whether the actual play $(i_t,j_t)$ converges, too. In the literature, such results are called \emph{last-iterate convergence} and \emph{topological convergence}.

With \Hedge and some other standard algorithms, the results are mixed. On the one hand, strong negative results are known, with a detailed investigation of the non-converging behavior \citep{Piliouras-ec18,Piliouras-soda18,Piliouras-colt19}. On the other hand, \Hedge admits a strong positive result for \emph{congestion games}, a well-studied family of games where a player's utility for a given action depends only on the number of other players which choose an ``overlapping" action.
In fact, there is a family of regret-minimizing algorithms which generalize \Hedge such that we have last-iterate convergence in repeated congestion games if each player uses some algorithm from this family \citep{Piliouras-stoc09}.

A recent flurry of activity, starting from \citet{Daskalakis-iclr18}, derives general results on last-iterate convergence \citep[see][and references therein]{Daskalakis-itcs19-lastIterate,Daskalakis-neurips20-lastIterate,Haipeng-iclr21-lastIterate}. These results apply to the main setting in this chapter: arbitrary repeated zero-sum games with two players and finitely many actions, and extend beyond that under various assumptions. All these results require full feedback, and hinge upon two specific, non-standard regret-minimizing algorithms.

\subsection{Beyond zero-sum games}

\xhdr{Correlated equilibria.}
Coarse correlated equilibrium, introduced in \citep{Moulin-78}, is a classic notion in theoretical economics. The simple argument in Section~\ref{games:sec:beyond} appears to be folklore.

A closely related notion called \emph{correlated equilibrium} \citep{Aumann-74}, posits that
\[
\E_{(i,j)\sim\sigma} \left[\;
   M(i,j) - M(i_0,j) \mid i
\;\right] \geq 0 \qquad \text{for each row $i_0$},
\]
and similarly for \ADV. This is a stronger notion, in the sense that any correlated equilibrium is a coarse correlated equilibrium, but not vice versa. One obtains an approximate correlated equilibria, in a sense similar to Theorem~\ref{games:thm:beyond}, if \ALG and \ADV have sublinear \emph{internal regret}
\citep[][also see Section~\ref{adv:sec-further} in this book]{HartMasCollel-econometrica00}.
Both results easily extend to games with multiple players. We specialized the discussion to the case of two players for ease of presentation only.

\xhdr{Smooth games}
is a wide class of multi-player games which admit strong guarantees about the self-interested behavior being ``not too bad"
\citep{Roughgarden-PoA-stoc09,Syrgkanis-stoc13,Lykouris-soda16}.%
\footnote{More formally, the \emph{price of anarchy} -- the ratio between the social welfare (agents' total reward) in the best centrally coordinated solution and the worst Nash equilibrium  -- can be usefully upper-bounded in terms of the game parameters. More background can be found in the textbook \citet{Tim-20lectures-2016}, as well as in many recent classes on algorithmic game theory.}
Repeated smooth games admit strong convergence guarantees: social welfare of average play converges over time, for arbitrary regret-minimizing algorithms \citep{TotalAnarchy-stoc08,Roughgarden-PoA-stoc09}. Moreover, one can achieve faster convergence, at the rate $\tilde{O}(\tfrac{1}{t})$, under some assumptions on the algorithms' structure
\citep{Syrgkanis-nips15,Foster-nips16}.

\xhdr{Cognitive radios.}
An application to \emph{cognitive radios} has generated much interest, starting from
\citet{MultiPlayerMAB-Poor08,MultiPlayerMAB-Liu10,MultiPlayerMAB-Anima11},
and continuing to, \eg \citet{MultiPlayerMAB-Mannor14,MultiPlayerMAB-Shamir-icml16,MultiPlayerMAB-Perchet-18,MultiPlayerMAB-Sellke-19}.
In this application, multiple radios transmit simultaneously in a shared medium. Each radio can switch among different available channels. Whenever two radios transmit on the same channel, a \emph{collision} occurs, and the transmission does not get through. Each radio chooses channels over time using a multi-armed bandit algorithm. The whole system can be modeled as a repeated game between bandit algorithms.

This line of work has focused on designing algorithms which work well in the repeated game, rather than studying the repeated game between arbitrary algorithms. Various assumptions are made regarding whether and to which extent communication among algorithms is allowed, the algorithms can be synchronized with one another, and collisions are detected. It is typically assumed that each radio transmits continuously.

\sectionExercises

\begin{exercise}[game-theory basics]~
\label{games:ex:p}
\begin{itemize}
\item[(a)] Prove that a distribution $p\in\Drows$ satisfies \refeq{games:eq:guaranteed} if and only if it is a minimax strategy.

\Hint{$M(p^*,q) \leq f(p^*)$ if $p^*$ is a minimax strategy; $f(p)\leq v^*$ if $p$  satisfies \eqref{games:eq:guaranteed}.}

\item[(b)] Prove that $p\in\Drows$ and $q\in\Dcols$ form a Nash equilibrium if and only if $p$ is a minimax strategy and $q$ is a maximin strategy.
\end{itemize}
\end{exercise}

\begin{exercise}[best-response adversary]\label{games:ex:BR}
The best-response adversary, as defined in \refeq{games:eq:bestResponseADV-defn}, can be seen as an algorithm in the setting of Chapter~\ref{games:sec:Nash}. (More formally, this algorithm knows \ALG and the game matrix $M$, and receives auxiliary feedback $\mF_t = j_t$.)
Prove that its expected regret satisfies $\E[R'(T)]\leq 0$.

\TakeAway{Thus, Theorem~\ref{games:thm:AlgoAdv} applies to the best-response adversary, with $\E[R'(T)]\leq 0$.}
\end{exercise}

\begin{exercise}[Hedge vs. Hedge]\label{games:ex:Hedge}
Suppose both \ALG and \ADV are implemented by algorithm \Hedge. Assume that $M$ is a $K\times K$ matrix with entries in $[0,1]$, and the parameter in \Hedge is
    $\eps=\sqrt{(\ln K) / (2T)}$.

\begin{itemize}
\item[(a)] Consider the full-feedback model. Prove that the average play $(\ibar,\jbar)$ forms a $\delta_T$-approximate Nash equilibrium with high probability (\eg with probability at least $1-T^{-2}$),
where
    $\delta_T = O\left( \sqrt{\tfrac{\ln (KT)}{T}} \right)$.

\item[(b)] Consider a modified feedback model: in each round $t$, \ALG is given costs $M(\cdot, q_t)$, and \ADV is given rewards $M(p_t,\cdot)$, where $p_t\in \Drows$ and $q_t\in \Dcols$ are the distributions chosen by \ALG and \ADV, respectively. Prove that the average play $(\ibar,\jbar)$ forms a $\delta_T$-approximate Nash equilibrium.

\item[(c)] Suppose \ADV is the best-response adversary, as per \refeq{games:eq:bestResponseADV-defn}, and \ALG is \Hedge with full feedback. Prove that $(\ibar,\jbar)$ forms a $\delta_T$-approximate Nash equilibrium.

\end{itemize}

\Hint{Follow the steps in the proof of Theorem~\ref{games:thm:AlgoAdv}, but use the probability-1 performance guarantee for \Hedge (\refeq{FF:eq:Hedge-prob1} on page \pageref{FF:eq:Hedge-prob1}) instead of the standard definition of regret   \eqref{games:eq:regret-defn}.

For part (a), define
    $\widetilde{\cost}(\ALG) := \sum_{t\in [T]} M(p_t,j_t)$
and
    $\widetilde{\rew}(\ADV) := \sum_{t\in [T]} M(i_t,q_t)$,
and use \refeq{FF:eq:Hedge-prob1} to conclude that they are close to
    $\cost^*$ and $\rew^*$, respectively. Apply Azuma-Hoeffding inequality to
prove that both
    $\widetilde{\cost}(\ALG)$ and $\widetilde{\rew}(\ADV)$ are close to
    $\sum_{t\in [T]} M(p_t,q_t)$.

For part (b), define
    $\widetilde{\cost}(\ALG)
        := \widetilde{\rew}(\ADV) := \sum_{t\in [T]} M(p_t,q_t)$,
and use \refeq{FF:eq:Hedge-prob1} to conclude that it is close to both $\cost^*$ and $\rew^*$.

For part (c), define $\widetilde{\cost}(\ALG) = \widetilde{\rew}(\ADV)$ and handle \Hedge as in part (b). Modify \eqref{games:eq:AlgoAdv-rewards}, starting from $\widetilde{\rew}(\ADV)$, to handle the best-response adversary.} 
\end{exercise}

\begin{exercise}[stochastic games]\label{games:ex:IID}
Consider an extension to stochastic games: in each round $t$, the game matrix $M_t$ is drawn independently from some fixed distribution over game matrices. Assume all matrices have the same dimensions (number of rows and columns). Suppose both \ALG and \ADV are regret-minimizing algorithms, as in Section~\ref{games:sec:Nash}, and let $R(T)$ and $R'(T)$ be their respective regret. Prove that the average play $(\ibar,\jbar)$ forms a $\delta_T$-approximate Nash equilibrium for the expected game matrix $M = \E[M_t]$, with
\[ \delta_T = \tfrac{R(T)+R'(T)}{T} + \Err,
\text{ where }
\textstyle
    \Err = \left| \sum_{t\in [T]} M_t(i_t,j_t) - M(i_t,j_i) \right|.
\]

\Hint{The total cost/reward is now
    $\sum_{t\in [T]} M_t(i_t,j_t)$.
Transition from this to
    $\sum_{t\in [T]} M(i_t,j_t)$
using the error term $\Err$, and follow the steps in the proof of Theorem~\ref{games:thm:AlgoAdv}.}

\Note{If all matrix entries are in the $[a,b]$ interval, then
$\Err\in [0,b-a]$, and by Azuma-Hoeffing inequality
    $\Err< O(b-a)\sqrt{\log(T)/T} $
with probability at most $1-T^{-2}$.}
\end{exercise}

\begin{exercise}[infinite action sets]\label{games:ex:inf}
Consider an extension of the repeated game to infinite action sets. Formally, \ALG and \ADV have action sets $I$, $J$, resp., and the game matrix $M$ is a function $I\times J \to [0,1]$. $I$ and $J$ can be arbitrary sets with well-defined probability measures such that each singleton set is measurable. Assume there exists a function $\tilde{R}(T)=o(T)$ such that expected regret of \ALG is at most $\tilde{R}(T)$ for any \ADV, and likewise expected regret of \ADV is at most $\tilde{R}(T)$ for any \ALG.

\begin{itemize}
\item[(a)] Prove an appropriate versions of the minimax theorem:
\begin{align*}
\inf_{p\in \Drows}\; \sup_{j\in J}\; \int M(p,j) \dif p
=\sup_{q\in \Dcols}\;\inf_{i\in I}\; \int M(i,q) \dif q,
\end{align*}
where $\Drows$ and $\Dcols$ are now the sets of all probability measures on $I$ and $J$.

\item[(b)] Formulate and prove an appropriate version of Theorem~\ref{games:thm:AlgoAdv}.
\end{itemize}

\Hint{Follow the steps in Sections~\ref{games:sec:minimax} and~\ref{games:sec:Nash} with minor modifications. Distributions over rows and columns are replaced with probability measures over $I$ and $J$. Maxima and minima over $I$ and $J$ are replaced with $\sup$ and $\inf$. Best response returns a distribution over $Y$ (rather than a particular column).}
\end{exercise}

\chapter{Bandits with Knapsacks}
\label{ch:BwK}
\begin{ChAbstract}
\emph{Bandits with Knapsacks} (\BwK) is a general framework for bandit problems with global constraints such as supply constraints in dynamic pricing. We define and motivate the framework, and solve it using the machinery from Chapter~\ref{ch:games}. We also describe two other algorithms for \BwK, based on ``successive elimination" and ``optimism under uncertainty" paradigms from Chapter~\ref{ch:IID}.

\prereqs{Chapters~\ref{ch:IID},~\ref{ch:FF},~\ref{ch:adv},~\ref{ch:games}.}
\end{ChAbstract}

\section{Definitions, examples, and discussion}
\label{BwK:sec:intro}

\xhdr{Motivating example.} We start with a motivating example: \emph{dynamic pricing with limited supply}. The basic version of this problem is as follows. The algorithm is a seller, with a limited inventory of $B$ identical items for sale. There are $T$ rounds. In each round $t$, the algorithm chooses a price $p_t\in [0,1]$ and offers one item for sale at this price. A new customer arrives, having in mind some value $v_t$ for this item (known as \emph{private value}). We posit that $v_t$ is drawn independently from some fixed but unknown distribution. The customer buys the item if and only if $v_t\geq p_t$, and leaves. The algorithm stops after $T$ rounds or after there are no more items to sell, whichever comes first. The algorithm's goal is to maximize revenue from the sales; there is no premium or rebate for the left-over items. Recall that the special case $B=T$ (\ie unlimited supply of items) falls under ``stochastic bandits", where arms corresponds to prices. However, with $B<T$ we have a ``global" constraint: a constraint that binds across all rounds and all actions.

More generally, the algorithm may have $n>1$ products in the inventory, with a limited supply of each. In each round $t$, the algorithm chooses a price $p_{t,i}$ for each product $i$, and offers one copy of each product for sale. A new customer arrives, with a vector of private values
    $v_t = (v_{t,1} \LDOTS v_{t,n})$,
and buys each product $i$ such that $v_{t,i}\geq p_{t,i}$. The vector $v_t$ is drawn independently from some distribution.%
\footnote{If the values are independent across products, \ie $v_{t,1} \LDOTS v_{t,n}$ are mutually independent random variables, then the problem can be decoupled into $n$ separate per-product problems. However, in general the values may be correlated.}
We interpret this scenario as a stochastic bandits problem, where actions correspond to price vectors
    $p_t = (p_{t,1} \LDOTS p_{t,n})$,
and we have a separate ``global constraint" on each product.

\xhdr{General framework.}
We introduce a general framework for bandit problems with global constraints, called ``bandits with knapsacks", which subsumes dynamic pricing and many other examples. In this framework, there are several constrained \emph{resources} being consumed by the algorithm, such as the inventory of products in the dynamic pricing problem. One of these resources is time: each arm consumes one unit of the ``time resource" in each round, and its budget is the time horizon $T$. The algorithm stops when the total consumption of some resource $i$ exceeds its respective budget $B_i$.

\newpage
\begin{BoxedProblem}{Bandits with Knapsacks (\BwK)}
Parameters: $K$ arms, $d$ resources with respective budgets
    $B_1 \LDOTS B_d \in [0,T]$.\\
In each round $t = 1,2,3 \, \ldots $:
\begin{OneLiners}
  \item[1.] Algorithm chooses an arm $a_t \in [K]$.
  \item[2.] Outcome vector
    $\vec{o}_t = (r_t; c_{t,1} \LDOTS c_{t,d})\in [0,1]^{d+1}$
    is observed, \\where $r_t$ is the algorithm's reward, and $c_{t,i}$ is consumption of each resource $i$.
\end{OneLiners}
Algorithm stops when the total consumption of some resource $i$ exceeds its budget $B_i$.
\end{BoxedProblem}

In each round, an algorithm chooses an arm, receives a reward, and also consumes some amount of each resource. Thus, the outcome of choosing an arm is now a $(d+1)$-dimensional vector rather than a scalar. As a technical assumption, the reward and consumption of each resource in each round lie in $[0,1]$. We posit the ``IID assumption", which now states that for each arm $a$ the outcome vector is sampled independently from a fixed distribution  over outcome vectors. Formally, an instance of \BwK is specified by parameters $T,K,d$, budgets $B_1 \LDOTS B_d$, and a mapping from arms to distributions over outcome vectors. The algorithm's goal is to maximize its \emph{adjusted total reward}: the total reward over all rounds but the very last one.

The name ``bandits with knapsacks" comes from an analogy with the well-known \emph{knapsack problem} in algorithms. In that problem, one has a knapsack of limited size, and multiple items each of which has a value and takes a space in the knapsack. The goal is to assemble the knapsack: choose a subset of items that fits in the knapsacks so as to maximize the total value of these items. Similarly, in dynamic pricing each action $p_t$ has ``value" (the revenue from this action) and ``size in the knapsack" (namely, the number of items sold). However, in \BwK the ``value" and "size" of a given action are not known in advance.

\begin{remark}
The special case $B_1 = \ldots = B_d = T$ is just ``stochastic bandits", as in Chapter~\ref{ch:IID}.
\end{remark}

\begin{remark}
An algorithm can continue while there are sufficient resources to do so, even if it almost runs out of some resources. Then the algorithm should only choose ``safe" arms if at all possible, where an arm is called ``safe" if playing this arm in the current round  cannot possibly cause the algorithm to stop.
\end{remark}

\xhdr{Discussion.}
Compared to stochastic bandits, \BwK is more challenging in several ways.
First, resource consumption during exploration may limit the algorithm's ability to exploit in the future rounds. A stark consequence is that  Explore-first algorithm fails if the budgets are too small, see Exercise~\ref{BwK:ex:explore-first}(a).
Second, per-round expected reward is no longer the right objective. An arm with high per-round expected reward may be undesirable because of high resource consumption. Instead, one needs to think about the \emph{total} expected reward over the entire time horizon.
Finally, learning the best arm is no longer the right objective! Instead, one is interested in the best fixed \emph{distribution} over arms. This is because a fixed distribution over arms can perform much better than the best fixed arm.
All three challenges arise even when $d=2$ (one resource other than time), $K=2$ arms, and $B>\Omega(T)$.

To illustrate the distinction between the best fixed distribution and the best fixed arm, consider the following example. There are two arms $a\in \{1,2\}$ and two resources $i\in \{1,2\}$ other than time. In each round, each arm $a$ yields reward $1$, and consumes $\ind{a=i}$ units of each resource $i$. For intuition, plot resource consumption on a plane, so that there is a ``horizontal" am which consumes only the ``horizontal" resource, and a ``vertical" arm which consumes only the ``vertical" resource, see Figure~\ref{BwK:fig:example}. Both resources have budget $B$, and time horizon is $T=2B$. Then always playing the same arm gives the total reward of $B$, whereas alternating the two arms gives the total reward of $2B$. Choosing an arm uniformly (and independently) in each round yields the same expected total reward of $2B$, up to a low-order error term.

\begin{figure}[h]
 \begin{center}
 \includegraphics[width = .5\textwidth]{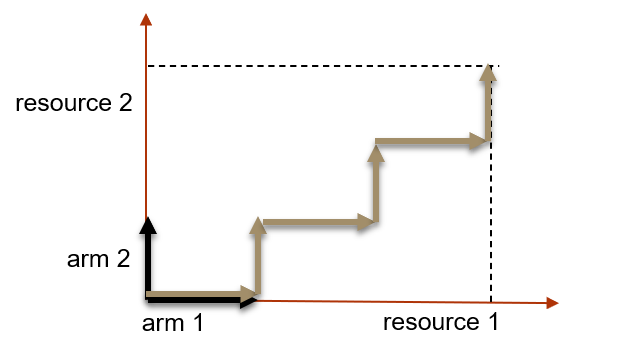}
 \end{center}
 \vspace{-7mm}
  \caption{Example: alternating arm is twice as good as a fixed arm}
  \label{BwK:fig:example}
\end{figure}

\xhdr{Benchmarks.}
We compare a \BwK algorithm against an \emph{all-knowing benchmark}: informally, the best thing one could do if one knew the problem instance. Without resources, ``the best thing one could do" is always play the best arm. For \BwK, there are three reasonable benchmarks one could consider: best arm, best distribution over arms, and best algorithm. These benchmarks are defined in a uniform way:
\begin{align}\label{BwK:eq:benchmarks-defn}
\OPT_\mA(\mI) =
    \sup_{\text{algorithms}\; \ALG \in \mA}
        \E\sbr{\REW(\ALG\mid \mI)},
\end{align}
where $\mI$ is a problem instance, $\REW(\ALG\mid \mI)$ is the adjusted total reward of algorithm $\ALG$ on this problem instance, and $\mA$ is a class of algorithms. Thus:
\begin{itemize}
\item if $\mA$ is the class of all \BwK algorithms, we obtain the \emph{best algorithm benchmark}, denoted $\OPT$; this is the main benchmark in this chapter.

\item if algorithms in $\mA$ are fixed distributions over arms, \ie they draw an arm independently from the same distribution in each round, we obtain the \emph{fixed-distribution benchmark}, denoted $\OPTFD$.

\item if algorithms in $\mA$ are fixed arms, \ie if they choose the same arm in each round, we obtain the \emph{fixed-arm benchmark}, denoted $\OPTFA$.

\end{itemize}
Obviously,
    $\OPTFA\leq \OPTFD \leq \OPT$.
Generalizing the example above, $\OPTFD$ could be up to $d$ times larger than $\OPTFA$, see Exercise~\ref{BwK:ex:benchmarks}. In fact, similar examples exist for the main motivating applications of \BwK \citep{BwK-focs13-conf,BwK-focs13}. The difference between $\OPT$ and $\OPTFD$ is not essential in this chapter's technical presentation (but it is for some of the other results spelled out in the literature review).

\section{Examples}
\label{BwK:sec:examples}

We illustrate the generality of \BwK with several examples. In all these examples, "time resource" is the last component of the outcome vector.

\begin{itemize}

\item \emph{Dynamic pricing.}
Dynamic pricing with a single product is a special case of \BwK with two resources: time (\ie the number of customers) and supply of the product. Actions correspond to chosen prices $p$. If the price is accepted, reward is $p$ and resource consumption is $1$. Thus, the outcome vector is
\begin{align}\label{BwK:eq:outcomeV-DP}
   \vec{o}_t =
   \begin{cases}
                (p, 1,1) &\text{if price $p$ is accepted}  \\
                (0, 0,1) &\text{otherwise}.
        \end{cases}
\end{align}

\item \emph{Dynamic pricing for hiring}, a.k.a. \emph{dynamic procurement}.
A contractor on a crowdsourcing market has a large number of similar tasks, and a fixed amount of money, and wants to hire some workers to perform these tasks. In each round $t$, a worker shows up, the algorithm chooses a price $p_t$, and offers a contract for one task at this price. The worker has a value $v_t$ in mind, and accepts the offer (and performs the task) if and only if $p_t\geq v_t$. The goal is to maximize the number of completed tasks.

This problem as a special case of \BwK with two resources: time (\ie the number of workers) and contractor's budget. Actions correspond to prices $p$; if the offer is accepted, the reward is $1$ and the resource consumption is $p$. So, the outcome vector is
\begin{align}\label{BwK:eq:outcomeV-DP-hiring}
    \vec{o}_t =
    \begin{cases}
                (1, p, 1) &\text{if price $p$ is accepted}  \\
                (0, 0, 1) &\text{otherwise}.
        \end{cases}
\end{align}

\item \emph{Pay-per-click ad allocation.}
There is an advertising platform with pay-per-click ads (advertisers pay only when their ad is clicked). For any ad $a$ there is a known per-click reward $r_a$: the amount an advertiser would pay to the platform for each click on this ad. If shown, each ad $a$ is clicked with some fixed but unknown probability $q_a$. Each advertiser has a limited budget of money that he is allowed to spend on her ads. In each round, a user shows up, and the algorithm chooses an ad. The algorithm's goal is to maximize the total reward.

This problem is a special case of \BwK with one resource for each advertiser (her budget) and the ``time" resource (\ie the number of users). Actions correspond to ads. Each ad $a$ generates reward $r_a$ if clicked, in which case the corresponding advertiser spends $r_a$ from her budget. In particular, for the special case of one advertiser the outcome vector is:
\begin{align}\label{BwK:eq:outcomeV-ads}
    \vec{o}_t =
    \begin{cases}
                (r_a, r_a, 1) &\text{if ad $a$ is clicked}  \\
                (0, 0, 1) &\text{otherwise}.
        \end{cases}
\end{align}

\OMIT{ 
Customers may have valuations over \emph{subsets} of products that are not necessarily additive: \eg a pair of shoes is usually more valuable than two copies of the left shoe, even though each shoe alone may have the same value.

 a new customer arrives, and the algorithm posts a \emph{menu}: \eg it may offer at most one copy of each product at a given price. Further generalizations allow the menu to offer multiple copies of the same product, and/or bundles of products, and use discounts and/or surcharges. The customer chooses according to the menu, and leaves.
} 

\item \emph{Repeated auctions.}
An auction platform such as eBay runs many instances of the same auction to sell $B$ copies of the same product. At each round, a new set of bidders arrives, and the platform runs a new auction to sell an item.  The auction s parameterized by some parameter $\theta$: \eg the second price auction with the reserve price $\theta$. In each round $t$, the algorithm chooses a value $\theta=\theta_t$ for this parameter, and announces it to the bidders. Each bidder is characterized by the value for the item being sold; in each round, the tuple of bidders' values is drawn from some fixed but unknown distribution over such tuples. Algorithm's goal is to maximize the total profit from sales.

This is a special case of \BwK with two resources: time (\ie the number of auctions) and the limited supply of the product. Arms correspond to feasible values of parameter $\theta$. The outcome vector is:
\begin{align*}
    \vec{o}_t =
    \begin{cases}
                (p_t, 1, 1) &\text{if an item is sold at price $p_t$}  \\
                (0, 0, 1) &\text{otherwise}.
        \end{cases}
\end{align*}
The price $p_t$ is determined by the parameter $\theta$ and the bids in this round.

\item \emph{Dynamic bidding on a budget.}
Let's look at a repeated auction from a bidder's perspective. It may be a complicated auction that the bidder does not fully understand. In particular, the bidder often not know the best bidding strategy, but may hope to learn it over time. Accordingly, we consider the following setting. In each round $t$, one item is offered for sale.  An algorithm chooses a bid $b_t$ and observes whether it receives an item and at which price. The outcome (whether we win an item and at which price) is drawn from a fixed but unknown distribution. The algorithm has a limited budget and aims to maximize the number of items bought.

This is a special case of \BwK with two resources: time (\ie the number of auctions) and the bidder's budget.
The outcome vector is
\begin{align*}
    \vec{o}_t =
    \begin{cases}
                (1, p_t, 1) &\text{if the bidder wins the item and pays $p_t$}  \\
                (0, 0, 1) &\text{otherwise}.
        \end{cases}
\end{align*}
\noindent The payment $p_t$ is determined by the chosen bid $b_t$, other bids, and the rules of the auction.
\end{itemize}

\section{LagrangeBwK: a game-theoretic algorithm for \BwK}
\label{BwK:sec:Lagrange}

We present an algorithm for \BwK, called \LagrangeBwK, which builds on the zero-sum games framework developed in Chapter~\ref{ch:games}. On a high level, our approach consists of four steps:

\begin{description}

\item[Linear relaxation]
We consider a relaxation of \BwK in which a fixed distribution $D$ over arms is used in all rounds, and outcomes are equal to their expected values. This relaxation can be expressed as a linear program for optimizing $D$, whose per-round value is denoted $\OPTLP$ We prove that
        $\OPTLP\geq \OPT/T$.

\item[Lagrange game]
We consider the Lagrange function $\Lag$ associated with this linear program. We focus on the \emph{Lagrange game}: a zero-sum game, where one player chooses an arm $a$, the other player chooses a resource $i$, and the payoff is $\Lag(a,i)$. We prove that the value of this game is $\OPTLP$.

\item[Repeated Lagrange game]
We consider a repeated version of this game. In each round $t$, the payoffs are given by Lagrange function $\Lag_t$, which is defined by this round's outcomes in a similar manner as $\Lag$ is defined by the expected outcomes. Each player is controlled by a regret-minimizing algorithm. The analysis from Chapter~\ref{ch:games} connects the average play in this game with its Nash equilibrium.

\item[Reward at the stopping time]
The final step argues that the reward at the stopping time is large compared to the ``relevant" value of the Lagrange function (which in turn is large enough because of the Nash property). Interestingly, this step only relies on the definition of $\Lag$, and holds for any algorithm.

\end{description}

\noindent Conceptually, these steps connect \BwK to the linear program, to the Lagrange game, to the repeated game, and back to \BwK, see Figure~\ref{BwK:fig:LagrangeBwK}. We flesh out the details in what follows.

\begin{figure}[h]
 \begin{center}
 \includegraphics[width = 1.0\textwidth]{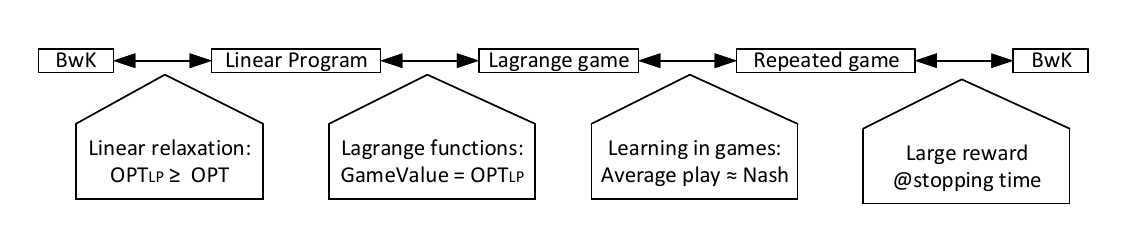}
 \end{center}
 \vspace{-7mm}
  \caption{The approach in \LagrangeBwK.}
  \label{BwK:fig:LagrangeBwK}
\end{figure}

\subsection*{Preliminaries}

We will use the following notation. Let $r(a)$ be the expected per-round reward of arm $a$, and $c_i(a)$ be the expected per-round consumption of a given resource $i$. The sets of rounds, arms and resources are denoted $[T]$, $[K]$ and $[d]$, respectively. The distributions over arms and resources are denoted $\Delta_K$ and $\Delta_d$. The adjusted total reward of algorithm \ALG is denoted $\REW(\ALG)$.

Let $B = \min_{i\in[d]} B_i$ be the smallest budget. Without loss of generality, we rescale the problem so that all budgets are $B$. For this, divide the per-round consumption of each resource $i$ by $B_i/B$. In particular, the per-round consumption of the time resource is now $B/T$.

We posit that one of the arms, called the \emph{null arm}, brings no reward and consumes no resource except the ``time resource". Playing this arm is tantamount to skipping a round.  The presence of such arm is essential for several key steps in the analysis. In dynamic pricing the largest possible price is usually assumed to result in no sale, and therefore can be identified with the null arm.

\begin{remark}
Even when skipping rounds is not allowed, existence of the null arm comes without loss of generality. This is because, whenever the null arm is chosen, an algorithm can proceed to the next round \emph{internally}, without actually outputting an action. After $T$ rounds have passed from the algorithm's perspective, the algorithm can choose arms arbitrarily, which can only increase the total reward.
\end{remark}

\begin{remark}\label{BwK:rem:matrix}
Without loss of generality, the outcome vectors are chosen as follows. In each round $t$, the \emph{outcome matrix}
    $\vM_t \in [0,1]^{K\times (d+1)}$
is drawn from some fixed distribution. Rows of $\vM_t$ correspond to arms: for each arm $a\in [K]$, the $a$-th row of $\vM_t$ is
\begin{align*}
    \vM_t(a) = (r_t(a); \, c_{t,1}(a) \LDOTS c_{t,d}(a)),
\end{align*}
so that $r_t(a)$ is the reward and $c_{t,i}(a)$ is the consumption of each resource $i$. The round-$t$ outcome vector is defined as the $a_t$-th row of this matrix:
    $\vec{o}_t= \vM_t(a_t)$.
Only this row is revealed to the algorithm.
\end{remark}

\subsection*{Step 1: Linear relaxation}
\label{BwK:sec:LP}

Let us define a relaxation of \BwK as follows. Fix some distribution $D$ over arms. Suppose this distribution is used in a given round, \ie the arm is drawn independently from $D$. Let $r(D)$ be the expected reward, and $c_i(D)$ be the expected per-round consumption of each resource $i$:
\[
\textstyle r(D)= \sum_{a\in [K]}\; D(a)\; r(a)
    \quad\text{and}\quad
c_i(D)= \sum_{a\in [K]}\; D(a)\; c_i(a).
\]
In the relaxation distribution $D$ is used in each round, and the reward and resource-$i$ consumption are deterministically equal to $r(D)$ and $c_i(D)$, respectively. We are only interested in distributions $D$ such that the algorithm does not run out of resources until round $T$. The problem of choosing $D$ so as to maximize the per-round reward in the relaxation can be formulated as a linear program:

\begin{equation}
\label{BwK:eq:LP}
\begin{array}{ll}
\text{maximize}
    &r(D) \\
\text{subject to} \\
    &D \in \Delta_K  \\
    & T\cdot c_i(D) \leq B \qquad \forall i \in [d].
\end{array}
\end{equation}

\noindent The value of this linear program is denoted $\OPTLP$. We claim that the corresponding total reward, $T\cdot \OPTLP$, is an upper bound on $\OPT$, the best expected reward achievable in \BwK.

\begin{claim}\label{BwK:cl:OPTLP}
$T\cdot \OPTLP \geq \OPT$.
\end{claim}

\begin{proof}
Fix some algorithm \ALG for \BwK. Let $\tau$ be the round after which this algorithm stops. Without loss of generality, if $\tau<T$ then \ALG continues to play the null arm in all rounds $\tau,\tau+1 \LDOTS T$.

Let
    $X_t(a) = \ind{a_t=a}$
be the indicator variable of the event that \ALG chooses arm $a$ in round $t$. Let $D\in \Delta_K$ be the algorithm's expected average play, as per Definition~\ref{games:def-avePlay}: \ie for each arm $a$, $D(a)$ is the expected fraction of rounds in which this arm is chosen.

First, we claim that the expected total adjusted reward is
    $\E[\REW(\ALG)] = r(D)$.
Indeed,
\begin{align*}
\E[r_t]
&= \textstyle \sum_{a\in [K]}\; \Pr[a_t=a] \cdot \E[r_t \mid a_t=a] \\
&= \textstyle \sum_{a\in [K]}\; \E[X_t(a)]\cdot r(a). \\
\E[\REW]
&=\sum_{t\in [T]} \E[r_t]
= \sum_{a\in [K]} r(a)\cdot \sum_{t\in [T]} \E[ X_t(a) ]
= \sum_{a\in [K]} r(a) \cdot T\cdot D(a)
= T\cdot r(D).
\end{align*}

Similarly, the expected total consumption of each resource $i$ is
    $\sum_{t\in [T]} \E[c_{t,i}] = T\cdot c_i(D)$.

Since the (modified) algorithm does not stop until time $T$, we have
$ \sum_{t\in [T]} c_{i,t}  \leq B $, and consequently
    $c_i(D) \leq B/T$.
Therefore, $D$ is a feasible solution for the linear program \eqref{BwK:eq:LP}. It follows that
    $\E[\REW(\ALG)] = r(D) \leq \OPTLP $.
\end{proof}

\subsection*{Step 2: Lagrange functions}
\label{BwK:sec:Lag}

Consider the Lagrange function associated with the linear program \eqref{BwK:eq:LP}. For our purposes, this function inputs a distribution $D$ over arms and a distribution $\lambda$ over resources,
\begin{align}\label{BwK:eq:Lag}
\Lag(D, \lambda)
:=  r(D) +
    \sum_{i\in [d]} \lambda_i \left[ 1-\tfrac{T}{B}\;  c_i(D) \right].
\end{align}
Define the \emph{Lagrange game}: a zero-sum game, where the \emph{primal player} chooses an arm $a$, the \emph{dual player} chooses a resource $i$, and the payoff is given by the Lagrange function:
\begin{align}\label{BwK:eq:Lag-game}
\Lag(a,i) = r(a) + 1 - \tfrac{T}{B}\;  c_i(a).
\end{align}
The primal player receives this number as a reward, and the dual player receives it as cost. The two players move simultaneously, without observing one another.

\begin{remark}
The terms \emph{primal player} and \emph{dual player} are inspired by the duality in linear programming. For each linear program (LP), a.k.a. \emph{primal} LP, there is an associated \emph{dual} LP. Variables in \eqref{BwK:eq:LP} correspond to arms, and variables in its dual LP correspond to resources. Thus, the primal player chooses among the variables in the primal LP, and the dual player chooses among the variables in the dual LP.
\end{remark}

The Lagrangian game is related to the linear program, as expressed by the following lemma.

\begin{lemma}\label{BwK:lm:Lag}
Suppose $(D^*, \lambda^*)$ is a mixed Nash equilibrium for the Lagrangian game. Then
\begin{itemize}
\item[(a)]$1-\tfrac{T}{B}\, c_i(D^*) \geq 0$ for each resource $i$,
with equality if $\lambda_i^* > 0$.
\item[(b)] $D^*$ is an optimal solution for the linear program \eqref{BwK:eq:LP}.
\item[(c)] The minimax value of the Lagrangian game equals the LP value:
    $\mL(D^*, \lambda^*) = \OPTLP$.
\end{itemize}
\end{lemma}

\begin{remark}
Lagrange function is a standard notion in mathematical optimization. For an arbitrary linear program (with at least one solution and a finite LP value), the function satisfies a max-min property:
\begin{align}\label{BwK:eq:Lag-minimax}
\min_{\lambda\in \R^d,\;\lambda\geq 0} \; \max_{D \in \Delta_K}
    \mL(D, \lambda)
= \max_{D \in \Delta_K} \; \min_{\lambda\in \R^d,\;\lambda\geq 0}
    \mL(D, \lambda)
= \OPTLP.
\end{align}
Because of the special structure of \eqref{BwK:eq:LP}, we obtain the same property with $\lambda\in \Delta_d$.
\end{remark}

\begin{remark}
In what follows, we only use part (c) of Lemma~\ref{BwK:lm:Lag}. Parts (ab) serve to prove (c), and are stated for intuition only. The property in part (a) is known as complementary slackness.  The proof of Lemma~\ref{BwK:lm:Lag} is a standard linear programming argument, not something about multi-armed bandits.
\end{remark}

\begin{proof}[Proof of Lemma~\ref{BwK:lm:Lag}]
By the definition of the mixed Nash equilibrium,
\begin{equation} \label{BwK:eq:Nash-Saddle}
                 \Lag(D^*, \lambda) \geq \Lag(D^*, \lambda^*) \geq  \Lag(D, \lambda^*) \qquad \forall D \in \Delta_K, \lambda \in \Delta_d.
\end{equation}

\noindent {\bf Part (a).}
First, we claim that
    $Y_i := 1-\tfrac{T}{B}\, c_i(D^*) \geq 0$
for each resource $i$ with $\lambda^*_i = 1$.

To prove this claim, assume $i$ is not the time resource (otherwise $c_i(D^*)=B/T$, and we are done). Fix any arm $a$, and consider the distribution $D$ over arms which assigns probability $0$ to arm $a$, puts probability $D^*(\Null)+D^*(a)$ on the null arm, and coincides with $D^*$ on all other arms. Using \refeq{BwK:eq:Nash-Saddle},
\begin{align*}
 0  & \leq \Lag(D^*, \lambda^*) - \Lag(D, \lambda^*) \\
    &= \left[ r(D^*)-r(D) \right]
        - \tfrac{T}{B}\,\left[ c_i(D^*)-c_i(D) \right] \\
    &= D^*(a) \left[  r(a)-\tfrac{T}{B}\,c_i(a) \right] \\
    &\leq D^*(a) \left[  1-\tfrac{T}{B}\,c_i(a) \right],
    \quad \forall \in [K].
\end{align*}
Summing over all arms, we obtain $Y_i\geq 0$, claim proved.

Second, we claim that $Y_i \geq 0$ all resources $i$. Suppose this is not the case. Focus on resource $i$ with the smallest $Y_i<0$; note that $\lambda^*_i<1$ by the first claim. Consider putting all probability on this resource: we have
    $\Lag(D^*, i) < 0=\Lag(D^*,\term{time}) \leq \Lag(D^*, \lambda^*)$,
contradicting \eqref{BwK:eq:Nash-Saddle}.

Third, assume that $\lambda_i^* > 0$ and $Y_i > 0$ for some resource $i$.  Then
    $\Lag(D^*, \lambda^*) > r(D^*)$.
Now, consider distribution $\lambda$ which puts probability $1$ on the dummy resource. Then
    $\Lag(D^*, \lambda) = r(D^*) < \Lag(D^*, \lambda^*)$,
contradicting \eqref{BwK:eq:Nash-Saddle}. Thus, $\lambda_i^* = 0$ implies $Y_i > 0$.

{\bf Part (bc).}
By part (a), $D^*$ is a feasible solution to \eqref{BwK:eq:LP}, and
    $\Lag(D^*, \lambda^*) = r(D^*)$.
Let $D$ be some other feasible solution for \eqref{BwK:eq:LP}. Plugging in the feasibility constraints for $D$, we have
    $\Lag(D, \lambda^*) \geq r(D)$.
Then
        \[
         r(D^*) = \Lag(D^*, \lambda^*) \geq \Lag(D, \lambda^*) \geq r(D).
                \]
So, $D^*$ is an optimal solution to the $\LP$. In particular,
    $\OPTLP = r(D^*) = \Lag(D^*, \lambda^*) $.
\end{proof}

\subsection*{Step 3: Repeated Lagrange game}
\label{BwK:sec:repeated}

The round-$t$ outcome matrix $\vM_t$, as defined in Remark~\ref{BwK:rem:matrix}, defines the respective Lagrange function $\Lag_t$:
\begin{align}\label{BwK:eq:Lag-t}
\Lag_t(a,i) = r_t(a) + 1 - \tfrac{T}{B}\;  c_{t,i}(a),
    \quad a\in [K],\, i\in [d].
\end{align}
Note that $\E[\Lag_t(a,i)] = \Lag(a,i)$, so we will refer to $\Lag$ as the \emph{expected} Lagrange function.

\begin{remark}
The function defined in \eqref{BwK:eq:Lag-t} is a Lagrange function for the appropriate ``round-$t$ version"  of the linear program \eqref{BwK:eq:LP}. Indeed, consider the expected outcome matrix $\vMexp := \E[\vM_t]$, which captures the expected per-round rewards and consumptions,%
\footnote{The $a$-th row of $\vMexp$ is
    $(r(a); \, c_1(a) \LDOTS c_d(a))$,
for each arm $a\in [K]$.}
and therefore defines the LP. Let us instead plug in an \emph{arbitrary} outcome matrix
    $\vM\in [0,1]^{K\times (d+1)}$
instead of $\vMexp$. Formally, let $\LP_{\vM}$ be a version of \eqref{BwK:eq:LP} when $\vMexp = \vM$, and let $\Lag_{\vM}$ be the corresponding Lagrange function. Then $\Lag_t = \Lag_{\vM_t}$  for each round $t$, \ie $\Lag_t$ is the Lagrange function for the version of the LP induced by the round-$t$ outcome matrix $\vM_t$.
\end{remark}

The \emph{repeated Lagrange game} is a game between two algorithms, the \emph{primal algorithm} $\ALG_1$ and the \emph{dual algorithm} $\ALG_2$, which proceeds over $T$ rounds. In each round $t$, the primal algorithm chooses arm $a_t$, the dual algorithm chooses a resource $i_t$, and the payoff  --- primal player's reward and dual player's cost --- equals $\Lag_t(a_t,i_t)$. The two algorithms make their choices simultaneously, without observing one another.

Our algorithm, called \LagrangeBwK, is very simple: it is a repeated Lagrangian game in which the primal algorithm receives bandit feedback, and the dual algorithm receives full feedback. The pseudocode, summarized in Algorithm~\ref{BwK:alg:LagrangianBwK}, is self-contained: it specifies the algorithm even without defining repeated games and Lagrangian functions. The algorithm is \emph{implementable}, in the sense that the outcome vector $\vec{o}_t$ revealed in each round $t$ of the \BwK problem suffices to generate the feedback for both $\ALG_1$ and $\ALG_2$.

\LinesNotNumbered \SetAlgoLined
\begin{algorithm}[!h]
\caption{Algorithm \LagrangeBwK}
\label{BwK:alg:LagrangianBwK}
\DontPrintSemicolon
\SetKwInOut{Input}{Input}
\SetKwInOut{Given}{Given}
\Given{time horizon $T$, budget $B$, number of arms $K$, number of resources $d$.\\
Bandit algorithm $\ALG_1$: action set $[K]$, maximizes rewards, bandit feedback.\\
Bandit algorithm $\ALG_2$: action set $[d]$, minimizes costs, full feedback.}

\For{round $t = 1, 2,\, \ldots$ (until stopping)}{
$\ALG_1$ returns arm $a_t\in [K]$, algorithm $\ALG_2$ returns resource $i_t\in [d]$.\;

Arm $a_t$ is chosen, outcome vector
        $\vec{o}_t = (r_t(a_t); c_{t,1}(a_t) \LDOTS c_{t,d}(a_t)) \in [0,1]^{d+1}$
    is observed.\;

The payoff $\mL_t(a_t,i_t)$ from \eqref{BwK:eq:Lag-t} is reported to $\ALG_1$ as reward, and to $\ALG_2$ as cost.\;

The payoff $\mL_t(a_t,i)$ is reported to $\ALG_2$ for each resource $i\in [d]$.
} 
\end{algorithm}

Let us apply the machinery from Section~\ref{ch:games} to the repeated Lagrangian game. For each algorithm $\ALG_j$, $j\in \{1,2\}$, we are interested in its regret $R_j(T)$ relative to the best-observed action, as defined in Chapter~\ref{FF:sec:adv}. We will call it \emph{adversarial regret} throughout this chapter, to distinguish from the regret in \BwK. For each round $\tau\in [T]$, let
    $\abar\in \Delta_K$ and $\ibar\in \Delta_d$
be the average play of $\ALG_1$ and $\ALG_2$, resp., up to round $\tau$, as per Definition~\ref{games:def-avePlay}.
Now, Exercise~\ref{games:ex:IID}, applied for the time horizon $\tau\in [T]$, implies the following:

\begin{lemma}\label{BwK:lm:learning-in-games}
For each round $\tau\in [T]$, the average play $(\abar_\tau,\ibar_\tau)$ forms a $\delta_\tau$-approximate Nash equilibrium for the expected Lagrange game defined by $\Lag$, where
\[ \tau\cdot \delta_\tau = R_1(\tau)+R_2(\tau) + \Err_{\tau},
\text{ with error term }
\textstyle
    \Err_\tau := \left| \sum_{t\in [\tau]} \Lag_t(i_t,j_t) - \Lag(i_t,j_t) \right|.
\]
\end{lemma}
\begin{corollary}\label{BwK:cor:learning-in-games}
    $\Lag(\abar_\tau,i) \geq \OPTLP -\delta_\tau$
~for each resource $i$.
\end{corollary}

\subsection*{Step 4: Reward at the stopping time}
\label{BwK:sec:}


We focus on the \emph{stopping time} $\tau$, the first round when the total consumption of some resource $i$ exceeds its budget; call $i$ the \emph{stopping resource}. We argue that $\REW$ is large compared to $\Lag(\abar_\tau,i)$ (which plugs nicely into Corollary~\ref{BwK:cor:learning-in-games}). In fact, this step holds for any \BwK algorithm.

\begin{lemma}
For an arbitrary \BwK algorithm \ALG,
\[ \REW(\ALG)  \geq \tau\cdot \Lag(\abar_\tau,i)
    + (T-\tau)\cdot \OPTLP -\Err^*_{\tau,i},\]
where
$\Err^*_{\tau,i} :=
    \left| \tau\cdot r(\abar_\tau) - \sum_{t\in [\tau]} r_t \right| + \tfrac{T}{B}\,
    \left| \tau\cdot c_i(\abar_\tau) - \sum_{t\in [\tau]} c_{i,t} \right|
$
is the error term.
\end{lemma}

\begin{proof}
Note that $\sum_{t\in [\tau]} c_{t,i} > B$ because of the stopping. Then:
\begin{align*}
\tau\cdot \Lag(\abar_\tau,i)
    &= \tau\cdot \left( r(\abar) + 1 - \tfrac{T}{B} c_i(\abar) \right) \\
    &\leq \sum_{t\in [\tau]} r_t +\tau -
        \tfrac{T}{B} \sum_{t\in [\tau]} c_{i,t}
        + \Err^*_{\tau,i} \\
    &\leq \REW + \tau-T + \Err^*_{\tau,i}.
\end{align*}
The Lemma follows since $\OPTLP\leq 1$.
\end{proof}

Plugging in Corollary~\ref{BwK:cor:learning-in-games}, the analysis is summarized as follows:
\begin{align}\label{BwK:eq:main-analysis}
\REW(\LagrangeBwK)
    \geq T\cdot \OPTLP -\left[\;
        R_1(\tau) + R_2(\tau) + \Err_\tau + \Err^*_{\tau,i}
    \;\right],
\end{align}
where $\tau$ is the stopping time and $i$ is the stopping resource.

\subsection*{Wrapping up}

It remains to bound the error/regret terms in
 \eqref{BwK:eq:main-analysis}.
Since the payoffs in the Lagrange game lie in the range
    $[a,b] := [1-\tfrac{T}{B},2]$,
all error/regret terms are scaled up by the factor
    $b-a = 1+\tfrac{T}{B}$.

Fix an arbitrary failure probability $\delta>0$. An easy application of Azuma-Hoeffding Bounds implies that
\[ \textstyle
\Err_{\tau} +\Err^*_{\tau,i}
    \leq O(b-a)\cdot \sqrt{T K  \log \left(\frac{dT}{\delta}\right)}
\qquad \forall \tau\in [T],\,i\in[d],
\]
with probability at least $1-\delta$.
(Apply Azuma-Hoeffding to each $\Err_{\tau}$ and $\Err^*_{\tau,i}$ separately. )

Let use use algorithms $\ALG_1$ and $\ALG_2$ that admit high-probability upper bounds on adversarial regret. For $\ALG_1$, we use algorithm \term{EXP3.P.1} from \citet{bandits-exp3}, and for $\ALG_2$ we use a version of algorithm \Hedge from Chapter~\ref{FF:sec:hedge}. With these algorithms, with probability at least $1-\delta$ it holds that, for each $\tau\in [T]$,
\begin{align*}
 R_1(\tau)
    &\textstyle \leq O(b-a)\cdot
     \sqrt{TK \log \left(\frac{T}{\delta} \right)}, \\
R_2(\tau)
    &\textstyle \leq O(b-a)\cdot
    \sqrt{T \log \left(\frac{dT}{\delta} \right)}.
\end{align*}

Plugging this into \eqref{BwK:eq:main-analysis}, we obtain the main result for \LagrangeBwK.

\begin{theorem}\label{BwK:thm:LagrangeBwK}
Suppose algorithm \LagrangeBwK is used with \term{EXP3.P.1} as $\ALG_1$, and \Hedge as $\ALG_2$. Then the following regret bound is achieved, with probability at least $1-\delta$:
\begin{align*}
\textstyle \OPT - \REW \leq
    O\rbr{\nicefrac{T}{B}}\cdot
        \sqrt{T K  \ln \rbr{\frac{dT}{\delta}}}.
\end{align*}
\end{theorem}

This regret bound is optimal in the worst case, up to logarithmic factors, in the regime when $B=\Omega(T)$. This is because of the $\Omega(\sqrt{KT})$ lower bound for stochastic bandits. However, as we will see next, this regret bound is not optimal when $\min(B,\OPT)\ll T$.

\LagrangeBwK achieves optimal $\tildeO(KT)$ regret on problem instances with zero resource consumption, \ie the $\nicefrac{T}{B}$ factor in the regret bound vanishes, see Exercise~\ref{BwK:ex:adv-zero}.

\newpage
\section{Optimal algorithms and regret bounds (no proofs)}
\label{BwK:sec:optimal}

The optimal regret bound is in terms of (unknown) optimum $\OPT$ and budget $B$ rather than time horizon $T$:
\begin{align}\label{BwK:eq:optimal-UB}
 \OPT - \E[\REW]
    \leq \tildeO\left( \sqrt{K\cdot \OPT} + \OPT\sqrt{K/B} \right).
\end{align}
This regret bound is essentially optimal for any given triple $(K,B,T)$: no algorithm can achieve better regret, up to logarithmic factors, over all problem instances with these $(K,B,T)$.%
\footnote{More precisely, no algorithm can achieve regret
    $\Omega(\min(\OPT,\reg))$,
where
    $\reg = \sqrt{K\cdot \OPT} + \OPT\sqrt{K/B}$.}
The first summand is essentially regret from stochastic bandits, and the second summand is due to the global constraints. The dependence on $d$ is only logarithmic.

We obtain regret
    $\tildeO(\sqrt{KT})$
when $B>\Omega(T)$, like in Theorem~\ref{BwK:thm:LagrangeBwK}. We have an improvement when $\OPT/\sqrt{B} \ll \sqrt{T}$: \eg $\OPT\leq B$ in the dynamic pricing example described above, so we obtain regret
    $\tildeO(\sqrt{KB})$
if there are only $K$ feasible prices. The bad case is when budget $B$ is small, but $\OPT$ is large.


Below we outline two algorithms that achieve the optimal regret bound in \eqref{BwK:eq:optimal-UB}.%
\footnote{More precisely, Algorithm~\ref{BwK:alg:BwK-balance} achieves the regret bound \eqref{BwK:eq:optimal-UB} with an extra multiplicative term of $\sqrt{d}$.}
These algorithms build on techniques from IID bandits: resp., Successive Elimination and Optimism under Uncertainty (see Chapter~\ref{ch:IID}). We omit their analyses, which are very detailed and not as lucid as the one for \LagrangeBwK.

\subsection*{Prepwork}

We consider outcome matrices: formally, these are matrices in
    $[0,1]^{K\times (d+1)}$.
The round-$t$ outcome matrix $\vM_t$ is defined as per Remark~\ref{BwK:rem:matrix}, and the expected outcome matrix is $\vMexp = \E[\vM_t]$.

Both algorithms maintain a \emph{confidence region} on $\vMexp$: a set of outcome matrices which contains $\vMexp$ with probability at least $1-T^{-2}$. In each round $t$, the confidence region $\ConfRegion_t$ is recomputed based on the data available in this round. More specifically, a confidence interval is computed separately for each ``entry" of $\vMexp$, and $\ConfRegion_t$ is defined as a product set of these confidence intervals.

Given a distribution $D$ over arms, consider its value in the linear program \eqref{BwK:eq:LP}:
\[  \LP(D \mid B,\vMexp)
= \begin{cases}
     r(D) & \text{if $c_i(D)\leq B/T$ for each resource $i$}, \\
     0  & \text{otherwise}.
\end{cases}
\]
We use a flexible notation which allows to plug in arbitrary outcome matrix $\vMexp$ and budget $B$.

\subsection*{Algorithm I: Successive Elimination with Knapsacks}

We look for optimal \emph{distributions} over arms. A distribution $D$ is called \emph{potentially optimal} if it optimizes $\LP(D\mid B,\vM)$ for some $\vM \in \ConfRegion_t$. In each round, we choose a potentially optimal distribution, which suffices for exploitation. But \emph{which} potentially optimal distribution to choose so as to ensure sufficient exploration? Intuitively, we would like to explore each arm as much as possible, given the constraint that we can only use potentially optimal distributions. We settle for something almost as good: we choose an arm $a$ uniformly at random, and then explore it as much as possible, in the sense that we choose a potentially optimal distribution $D$ that maximizes $D(a)$. See Algorithm~\ref{BwK:alg:BwK-balance} for the pseudocode.

\LinesNotNumbered \SetAlgoLined \DontPrintSemicolon
\begin{algorithm}[!h]
\caption{Successive Elimination with Knapsacks}\label{BwK:alg:BwK-balance}
\For{round $t = 1, 2,\, \ldots$ (until stopping)}{
$S_t \gets$ the set of all potentially optimal distributions over arms.\;

Pick arm $b_t$ uniformly at random.\;

Pick a distribution $D=D_t$ so as to maximize $D(b_t)$ over all $D\in S_t$.\;

Pick arm $a_t \sim D_t$.\;
} 
\end{algorithm}

\begin{remark}
The step of maximizing $D(b_t)$ does not have a computationally efficient implementation (more precisely, such implementation is not known for the general case of \BwK).
\end{remark}

This algorithm can be seen as an extension of Successive Elimination. Recall that in Successive Elimination, we start with all arms being ``active" and permanently de-activate a given arm $a$ once we have high-confidence evidence that some other arm is better. The idea is that each arm that is currently active can potentially be an optimal arm given the evidence collected so far. In each round we choose among arms that are still ``potentially optimal", which suffices for the purpose of exploitation. And choosing \emph{uniformly} (or round-robin) among the potentially optimal arms suffices for the purpose of exploration.

\subsection*{Algorithm II: Optimism under Uncertainty}

For each round $t$ and each distribution $D$ over arms, define the Upper Confidence Bound as
\begin{align}\label{BwK:eq:BwK-UCB}
     \UCB_t(D \mid B) = \sup_{\vM \in \ConfRegion_t} \LP(D\mid B,\vM).
\end{align}
The algorithm is very simple: in each round, the algorithm picks distribution $D$ which maximizes the UCB. An additional trick is to pretend that all budgets are scaled down by the same factor $1-\eps$, for an appropriately chosen parameter $\eps$. This trick ensures that the algorithm does not run out of resources too soon due to randomness in the outcomes or to the fact that the distributions $D_t$ do not quite achieve the optimal value for $\LP(D)$. The algorithm, called \UcbBwK, is as follows:

\LinesNotNumbered \SetAlgoLined \DontPrintSemicolon
\begin{algorithm}
\caption{\UcbBwK: Optimism under Uncertainty with Knapsacks.}
\label{BwK:alg:BwK-UCB}
Rescale the budget: $B' \leftarrow B(1-\eps) $,
where $\eps = \tilde{\Theta}(\sqrt{K/B}) $\;
Initialization: pull each arm once.\;
\For{all subsequent rounds $t$}{
In each round $t$, pick distribution $D=D_t$ with highest
    $\UCB_t(\cdot \mid B')$.\;
Pick arm $a_t \sim D_t$.
} 
\end{algorithm}

\begin{remark}
For a given distribution $D$, the supremum in \eqref{BwK:eq:BwK-UCB} is obtained when upper confidence bounds are used for rewards, and lower confidence bounds are used for resource consumption:
\begin{align*}
\UCB_t(D \mid B')
= \begin{cases}
     r^{\UCB}(D) & \text{if $c^{\LCB}_i(D)\leq B'/T$ for each resource $i$}, \\
     0  & \text{otherwise}.
\end{cases}
\end{align*}
Accordingly, choosing a distribution with maximal $\UCB$ can be implemented by a linear program:
\begin{equation}
\label{BwK:eq:LP-UCB}
\begin{array}{ll}
\text{maximize}
    & r^\UCB(D) \\
\text{subject to}\\
    & D \in \Delta_K\\
    & c_i^\LCB(D) \leq B'/T.
\end{array}
\end{equation}
\end{remark}

\newpage
\sectionBibNotes
\label{BwK:sec:further}

The general setting of \BwK is introduced in \citet{BwK-focs13-conf,BwK-focs13}, along with an optimal solution \eqref{BwK:eq:optimal-UB}, the matching lower bound, and a detailed discussion of various motivational examples. (The lower bound is also implicit in earlier work of \citet{DevanurJSW-jacm19}.) \LagrangeBwK, the algorithm this chapter focuses on, is from a subsequent paper \citep{AdvBwK-focs19}.%
\footnote{\citet{rivera2018online}, in a simultaneous and independent work, put forward a similar algorithm, and analyze it under full feedback for the setting of ``bandit convex optimization with knapsacks" (see Section~\ref{BwK:sec:lit-extensions}).}
This algorithm is in fact a ``reduction" from \BwK to bandits (see Section~\ref{BwK:sec:lit-reduction}), and  also ``works" for the adversarial version (see Section~\ref{BwK:sec:lit-adv})

Various motivating applications of \BwK have been studied separately:
dynamic pricing \citep{BZ09,DynPricing-ec12,BesbesZeevi-or12,Wang-OR14},
dynamic procurement \citep{DynProcurement-ec12,Krause-www13}, and pay-per-click ad allocation \citep{AdsWithBudgets-arxiv13,combes2015bandits}. Much of this work preceded and inspired \citet{BwK-focs13-conf}.

The optimal regret bound in \eqref{BwK:eq:optimal-UB} has been achieved by three different algorithms: the two in Section~\ref{BwK:sec:optimal} and one other algorithm from \citet{BwK-focs13-conf,BwK-focs13}. The successive elimination-based algorithm
(Algorithm~\ref{BwK:alg:BwK-balance}) is from \citep{BwK-focs13-conf,BwK-focs13}, and \UcbBwK (Algorithm~\ref{BwK:alg:BwK-UCB}) is from a follow-up paper of \citet{AgrawalDevanur-ec14,AgrawalDevanur-ec14-OpRe}. The third algorithm, also from \citep{BwK-focs13-conf,BwK-focs13}, is a ``primal-dual" algorithm superficially similar to \LagrangeBwK. Namely, it decouples into two online learning algorithms: a ``primal" algorithm which chooses among arms, and a ``dual" algorithm similar to $\ALG_2$, which chooses among resources. However, the primal and dual algorithms are not playing a repeated game in any meaningful sense. Moreover, the primal algorithm is very problem-specific: it interprets the dual distribution as a vector of costs over resources, and chooses arms with largest reward-to-cost ratios, estimated using ``optimism under uncertainty".

Some of the key ideas in \BwK trace back to earlier work.%
\footnote{In addition to the zero-sum games machinery from Chapter~\ref{ch:games} and stochastic bandit techniques from Chapter~\ref{ch:IID}.}
First, focusing on total expected rewards rather than per-round expected rewards, approximating total expected rewards with a linear program, and using ``optimistic" estimates of the LP values in a UCB-style algorithm goes back to \citet{DynPricing-ec12}. They studied the special case of dynamic pricing with limited supply, and applied these ideas to fixed arms (not distributions over arms).
Second, repeated Lagrange games, in conjunction with regret minimization in zero-sum games, have been used as an algorithmic tool to solve various convex optimization problems (different from \BwK), with application domains ranging from differential privacy to algorithmic fairness to learning from revealed preferences \citep{rogers2015inducing,hsu2016jointly,roth2016watch,pmlr-v80-kearns18a,agarwal2018reductions,roth2017multidimensional}.
All these papers deal with deterministic games (\ie same game matrix in all rounds). Most related are \citep{roth2016watch,roth2017multidimensional}, where a repeated Lagrangian game is used as a subroutine (the ``inner loop") in an online algorithm; the other papers solve an offline problem.
Third, estimating an optimal ``dual" vector from samples and using this vector to guide subsequent ``primal" decisions is a running theme in the work on \emph{stochastic packing} problems
\citep{DevanurH-ec09,AgrawalWY-OR14,DevanurJSW-jacm19,FeldmanHKMS-esa10,MolinaroR-icalp12}.
These are full information problems in which the costs and
rewards of decisions in the past and present are fully known, and the only uncertainty is about the future. Particularly relevant is the algorithm of \citet{DevanurJSW-ec11}, in which the dual vector
is adjusted using multiplicative updates, as in \LagrangeBwK and the primal-dual algorithm from \citet{BwK-focs13-conf,BwK-focs13}.

\BwK with only one constrained resource and unlimited number of rounds tends to be an easier problem, avoiding much of the complexity of the general case. In particular, $\OPTFD=\OPTFA$, \ie the best distribution over arms is the same as the best fixed arm. \citet{Gyorgy-ijcai07} and subsequently \cite{TranThanh-aaai10,TranThanh-aaai12,Qin-aaai13} obtain instance-dependent $\text{polylog}(T)$ regret bounds under various assumptions.

\subsection{Reductions from \BwK to bandits}
\label{BwK:sec:lit-reduction}

Taking a step back from bandits with knapsacks, recall that global constrains is just one of several important ``dimensions" in the problem space of multi-armed bandits. It is desirable to unify results on \BwK with those on other ``problem dimensions". Ideally, results on bandits would seamlessly translate into \BwK. In other words, solving some extension of \BwK should reduce to solving a similar extension of bandits. Such results are called \emph{reductions} from one problem to another. \LagrangeBwK and \UcbBwK algorithms give rise to two different reductions from \BwK to bandits, discussed below.

\LagrangeBwK takes an arbitrary ``primal" algorithm $\ALG_1$, and turns it into an algorithm for \BwK. To obtain an extension to a particular extension of \BwK, algorithm $\ALG_1$ should work for this extension in bandits, and achieve a high-probability bound on its ``adversarial regret" $R_1(\cdot)$ (provided that per-round rewars/costs lie in $[0,1]$). This regret bound $R_1(\cdot)$ then plugs into \refeq{BwK:eq:main-analysis} and propagates through the remaining steps of the analysis. Then with probability at least $1-\delta$ one has
\begin{align}
\textstyle \OPT - \REW \leq
    O\rbr{\nicefrac{T}{B}}\cdot
        \rbr{ R_1(T) + \sqrt{T K  \ln (dT/\delta)} }.
\end{align}
\citet{AdvBwK-focs19} makes this observation and uses it to obtain several extensions: to contextual bandits, combinatorial semi-bandits, bandit convex optimization, and full feedback. (These and other extensions are discussed in detail in Section~\ref{BwK:sec:lit-extensions}.)

The reduction via \UcbBwK is different: it inputs a \emph{lemma} about stochastic bandits, not an algorithm. This lemma works with an abstract \emph{confidence radius} $\rad_t(\cdot)$, which generalizes that from Chapter~\ref{ch:IID}. For each arm $a$ and round $t$, $\rad_t(a)$ is a function of algorithm's history which is an upper confidence bound for
    $|r(a) - \hat{r}_t(a)|$ and $|c_i(a) - \hat{c}_{t,i}(a)|$,
where $\hat{r}_t(a)$ and $\hat{c}_{t,i}(a)$ are some estimates computable from the data.
The lemma asserts that for a particular version of confidence radius \emph{and any bandit algorithm} it holds that
\begin{align}\label{BwK:eq:confSum}
    \textstyle \sum_{t\in S} \rad_t(a_t) \leq \sqrt{\beta\, |S|}
    \quad \text{for all subsets of rounds $S\subset [T]$},
\end{align}
where $\beta\ll K$ is some application-specific parameter. (The left-hand side is called the \emph{confidence sum}.) \refeq{BwK:eq:confSum} holds with $\beta=K$ for stochastic bandits; this follows from the analysis in Section~\ref{IID:sec:succ}, and stems from the original analysis of \UcbOne algorithm in \citet{bandits-ucb1}. Similar results are known for several extensions of stochastic bandits: this is typically a key step in the analysis of any extension of \UcbOne algorithm. Whenever \refeq{BwK:eq:confSum} holds for some extension of stochastic bandits, one can define a version of \UcbBwK algorithm which uses
    $\hat{r}(a)+\rad_t(a)$ and $\hat{c}_{t,i}-\rad_t(a)$
as, resp., UCB on $r(a)$ and LCB on $c_i(a)$ in \eqref{BwK:eq:LP-UCB}. Plugging \eqref{BwK:eq:confSum} into the original analysis of \UcbBwK in \citet{AgrawalDevanur-ec14,AgrawalDevanur-ec14-OpRe} yields
\begin{align}\label{BwK:eq:lit-UcbBwK}
   \OPT -\E[\REW] \leq O(\sqrt{\beta T})(1+\OPT/B).
\end{align}
\citet{Karthik-BwK-2020} make this observation explicit, and apply it to derive extensions to combinatorial semi-bandits, linear contextual bandits, and multinomial-logit bandits.

Extensions via either approach take little or no extra work (when there is a result to reduce from), whereas many papers, detailed in Section~\ref{BwK:sec:lit-extensions}, target one extension each. The resulting regret bounds are usually optimal in the regime when
    $\min(B,\OPT) > \Omega(T)$.
Moreover, the \LagrangeBwK reduction also ``works" for the adversarial version (see Section~\ref{BwK:sec:lit-adv}).

However, these results come with several important caveats. First, the regret bounds can be suboptimal if $\min(B,\OPT) \ll \Omega(T)$, and may be improved upon via application-specific results. Second, algorithm $\ALG_1$ in \LagrangeBwK needs a regret bound against an adaptive adversary, even though for \BwK we only need an regret bound against stochastic outcomes. Third, this regret bound needs to hold for the rewards specified by the Lagrange functions, rather than the rewards in the underlying \BwK problem (and the latter may have some useful properties that do not carry over to the former). Fourth, most results obtained via the \UcbBwK reduction do not come with a computationally efficient implementation.

\subsection{Extensions of \BwK}
\label{BwK:sec:lit-extensions}

\noindent\textbf{\BwK with generalized resources.}
\citet{AgrawalDevanur-ec14,AgrawalDevanur-ec14-OpRe} consider a version of \BwK with a more abstract version of resources. First, they remove the distinction between rewards and resource consumption. Instead, in each round $t$  there is an outcome vector $o_t \in [0,1]^d$, and the ``final outcome" is the average
    $\bar{o}_T = \tfrac{1}{T} \sum_{t\in [T]} o_t$.
The total reward is determined by $\bar{o}_T$, in a very flexible way: it is $T\cdot f(\bar{o}_T)$, for an arbitrary Lipschitz concave function
    $f: [0,1]^d \mapsto [0,1]$
known to the algorithm. Second, the resource constraints can be expressed by an arbitrary convex set $S\subset [0,1]^d$  which $\bar{o}_T$ must belong to. Third, they allow \emph{soft constraints}: rather than requiring $\bar{o}_t \in S$ for all rounds $t$ (\emph{hard constraints}), they upper-bound the distance between $\bar{o}_t$ and $S$. They obtain regret bounds that scale as $\sqrt{T}$, with distance between $\bar{o}_t$ and $S$ scaling as $1/\sqrt{T}$. Their results extend to the hard-constraint version by rescaling the budgets, as in Algorithm~\ref{BwK:alg:BwK-UCB}, as long as the constraint set $S$ is downward closed.


\xhdr{Contextual bandits with knapsacks.}
Contextual bandits with knapsacks (\emph{cBwK}) is a common generalization of \BwK and contextual bandits (background on the latter can be found in Chapter~\ref{ch:CB}). In each round $t$, an algorithm observes context $x_t$ before it chooses an arm. The pair $(x_t,\vM_t)$, where $\vM_t\in [0,1]^{K\times (d+1)}$ is the round-$t$ outcome matrix, is chosen independently from some fixed distribution (which is included in the problem instance). The algorithm is given a set $\Pi$ of policies (mappings from arms to actions), as in Section~\ref{CB:sec:policy-class}. The benchmark is the best all-knowing algorithm restricted to policies in $\Pi$: $\OPT_{\Pi}$ is the expected total reward of such algorithm, similarly to \eqref{BwK:eq:benchmarks-defn}. The following regret bound can be achieved:
\begin{align}\label{BwK:eq:cBwK-regret}
 \OPT_\Pi - \E[\REW] \leq \tildeO(1+ \OPT_{\Pi}/B)\, \sqrt{KT \log |\Pi|}.
\end{align}
The
    $\sqrt{KT \log(|\Pi|)}$
dependence on $K$, $T$ and $\Pi$ is optimal for contextual bandits, whereas the $(1+ \OPT_{\Pi}/B)$ term is due \BwK. In particular, this regret bound is optimal in the regime $B>\Omega(\OPT_{\Pi})$.

\citet{cBwK-colt14} achieve \eqref{BwK:eq:cBwK-regret}, with an extra factor of $\sqrt{d}$, unifying Algorithm~\ref{BwK:alg:BwK-balance} and the \emph{policy elimination} algorithm for contextual bandits \citep{policy_elim}. Like policy elimination, the algorithm in \citet{cBwK-colt14} does not come with a computationally efficient implementation. Subsequently, \citet{CBwK-colt16} obtained \eqref{BwK:eq:cBwK-regret} using an ``oracle-efficient" algorithm: it uses a classification oracle for policy class $\Pi$ as a subroutine, and calls this oracle only
    $\tildeO(d\sqrt{KT\log |\Pi|})$
times. Their algorithm builds on the contextual bandit algorithm from \citet{monster-icml14},
see Remark~\ref{CB:rem:minimonster}.

\LagrangeBwK reduction can be applied, achieving regret bound
\begin{align*}
 \OPT_\Pi - \E[\REW] \leq \tildeO(T/B)\, \sqrt{KT \log |\Pi|}.
\end{align*}
This regret rate matches \eqref{BwK:eq:cBwK-regret} when $\OPT_{\Pi}>\Omega(T)$, and is optimal in the regime $B>\Omega(T)$. The ``primal" algorithm $\ALG_1$ is \term{Exp4.P} from \citet{exp4p}, the high-probability version of algorithm \ExpFour from Section~\ref{sec:BwK-Exp4}. Like \ExpFour, this algorithm is not computationally efficient.

In a stark contrast with \BwK, non-trivial regret bounds are not necessarily achievable when $B$ and $\OPT_{\Pi}$ are small. Indeed, $o(\OPT)$ worst-case regret is impossible in the regime
    $\OPT_{\Pi} \leq B \leq \sqrt{KT}/2$
\citep{cBwK-colt14}. Whereas  $o(\OPT)$ regret bound holds for \BwK whenever $B=\omega(1)$, as per~\eqref{BwK:eq:optimal-UB}.

\citet{LagCBwK-2023} and \citet{SquareCBwK-competition} pursue an alternative approach whereby one posits \emph{realizability} and applies a regression oracle (see Chapter~\ref{CB:sec:further}). They combine \LagrangeBwK and \texttt{SquareCB}, the regression-based algorithm for contextual bandits from \citet{regressionCB-icml20}.

\xhdr{Linear contextual \BwK.}
In the extension to \emph{linear} contextual bandits (see Section~\ref{CB:sec:lin}), the expected outcome matrix is linear in the context $x_t$, and all policies are allowed. Formally, each context is a matrix:
    $x_t \in [0,1]^{K\times m}$,
where rows correspond to arms, and each arm's context has dimension $m$. The linear assumption states that
    $\E[\vM_t \mid x_t] = x_t\, \mathbf{W}$,
for some matrix $\mathbf{W}\in [0,1]^{(d+1)\times m}$ that is fixed over time, but not known to the algorithm. \citet{CBwK-nips16} achieve regret bound
\begin{align}\label{eq:BwK:cBwK-linear-regret}
    \OPT - \E[\REW] \leq \tildeO(m\sqrt{T})(1+\OPT/B)
    \qquad \text{in the regime $B>mT^{3/4}$}.
\end{align}
 The $\tildeO(m\sqrt{T})$ dependence is the best possible for linear bandits \citep{DaniHK-colt08}, whereas the $(1+\OPT/B)$ term and the restriction to $B>mT^{3/4}$ is due to \BwK. In particular, this regret bound is optimal, up to logarithmic factors, in the regime
    $B>\Omega(\max(\OPT,\,m\,T^{3/4}))$.

\refeq{eq:BwK:cBwK-linear-regret} is immediately obtained via \UcbBwK reduction, albeit without a computationally efficient implementation \citep{Karthik-BwK-2020}.

\xhdr{Combinatorial semi-bandits with knapsacks.}
In the extension to combinatorial semi-bandits (see Section~\ref{lin:sec:semi}), there is a finite set $S$ of \emph{atoms}, and a collection $\mF$ of feasible subsets of $S$. Arms correspond to the subsets in $\mF$. When an arm $a=a_t\in \mF$ is chosen in some round $t$, each atom $i\in a$ collects a reward and consumes resources, with its outcome vector $v_{t,i}\in [0,1]^{d+1}$ chosen independently from some fixed but unknown distribution. The ``total" outcome vector $\vec{o}_t$ is a sum over atoms:
    $\vec{o}_t = \sum_{i\in a} v_{t,i}$.
We have \emph{semi-bandit feedback}: the outcome vector $v_{t,i}$ is revealed to the algorithm, for each atom $i\in a$. The central theme is that, while the number of arms, $K=|\mF|$, may be exponential in the number of atoms, $m = |S|$, one can achieve regret bounds that are polynomial in $m$.

Combinatorial semi-bandits with knapsacks is a special case of linear cBwK, as defined above, where the context $x_t$ is the same in all rounds, and defines the collection $\mF$. (For each arm $a\in \mF$, the $a$-th row of $x_t$ is a binary vector that represents $a$.) Thus, the regret bound \eqref{eq:BwK:cBwK-linear-regret} for linear cBwK applies. \citet{Karthik-aistats18} achieve an improved regret bound when the set system $\mF$ is a matroid. They combine Algorithm~\ref{BwK:alg:BwK-UCB} with the \emph{randomized rounding} techniques from approximation algorithms. In the regime when $\min(B,\OPT)>\Omega(T)$, these two regret bounds become, respectively,
    $\tildeO(m\sqrt{T})$ and $\tildeO(\sqrt{mT})$.
The $\sqrt{mT}$ regret is optimal, up to constant factors, even without resources  \citep{MatroidBandits-uai14}.

Both reductions from Section~\ref{BwK:sec:lit-reduction} apply.
\LagrangeBwK reduction achieves regret
    $\tildeO(T/B) \sqrt{mT}$
\citep{AdvBwK-focs19}.
\UcbBwK reduction achieves regret
    $\tildeO(m \sqrt{T})(1+\OPT/B)$
via a computationally inefficient algorithm \citep{Karthik-BwK-2020}.

\xhdr{Multinomial-logit Bandits with Knapsacks.}
The setup starts like in combinatorial semi-BwK. There is a ground set of $N$ \emph{atoms}, and a fixed family $\mF \subset 2^{[N]}$ of feasible actions. In each round, each atom $a$ has an outcome $\vo_t(a)\in [0,1]^{d+1}$, and the outcome matrix $\rbr{\vo_t(a): a \in [N]}$ is drawn independently from some fixed but unknown distribution. The aggregate outcome is formed in a different way: when a given subset $A_t\in\mF$ is chosen by the algorithm in a given round $t$, at most one atom $a_t\in A_t$ is chosen stochastically by ``nature", and the aggregate outcome is then $\vo_t(A_t) := \vo_t(a)$; otherwise, the algorithm skips this round. A common interpretation is \emph{dynamic assortment}: that the atoms correspond to products, the chosen action $A_t\in \mF$ is the bundle of products offered to the customer; then at most one product from this bundle is actually purchased. As usual, the algorithm continues until some resource (incl. time) is exhausted.


The selection probabilities are defined via the multinomial-logit model. For each atom $a$ there is a hidden number $v_a\in [0,1]$, interpreted as the customers' valuation of the respective product, and
\[\Pr\sbr{ \text{atom $a$ is chosen} \mid A_t} =
\begin{cases}
	\tfrac{v_a}{1+\sum_{a' \in A_t} v_{a'}} & \text{if $a \in A_t$} \\
			0 & \text{otherwise}.
\end{cases}
\]
The set of possible bundles is
    $ \mF = \{\; A\subset [N]:\; \mathbf{W} \cdot \term{bin}(A) \leq \vec{b} \;\}$,
for some (known) totally unimodular matrix $\mathbf{W}\in \R^{N\times N} $ and a vector $\vec{b} \in \R^N$, where $\term{bin}(A)\in \{0,1\}^N$ is a binary-vector representation.
	
\emph{Multinomial-logit (MNL) bandits}, the special case without resources, was studied in connection with dynamic assortment,
\eg in \citet{caro2007dynamic,saure2013optimal,rusmevichientong2010dynamic,Shipra-ec16}.

MNL-BwK was introduced in \citet{Cheung-MNLBwK-arxiv17} and solved via a computationally inefficient algorithm. \citet{Karthik-BwK-2020} solve this problem via the \UcbBwK reduction (also computationally inefficiently). \citet{MnlBwK-ec21} obtains a computationally efficient version of \UcbBwK. These algorithms achieve regret rates similar to \eqref{BwK:eq:lit-UcbBwK}, with varying dependence on problem parameters.

\xhdr{Bandit convex optimization (BCO) with knapsacks.}
BCO is a version of multi-armed bandits where the action set is a convex set $\mX \subset \R^K$, and in each round $t$, there is a concave function
        $f_t: \mathcal{X}\to [0,1]$
such that the reward for choosing action $\vec{x}\in \mX$ in this round is $f_t(\vec{x})$. BCO is a prominent topic in bandits, starting from \citep{Bobby-nips04,FlaxmanKM-soda05}.

We are interested in a common generalization of BCO and \BwK, called \emph{BCO with knapsacks}. For each round $t$ and each resource $i$, one has convex functions
    $g_{t, i}: \mathcal{X} \rightarrow [0, 1]$,
such that the consumption of this resource for choosing action $\vec{x}\in \mX$ in this round is $g_{t,i}(\vec{x})$.
The tuple of functions
    $(f_t; g_{t, 1} \LDOTS g_{t,d})$
is sampled independently from some fixed distribution (which is not known to the algorithm). \LagrangeBwK reduction in \citet{AdvBwK-focs19} yields a regret  bound of the form
    $\frac{T}{B}\sqrt{T} \cdot \text{poly}(K\log T)$,
building on the recent breakthrough in BCO in \citet{bubeck2017kernel}.%
\footnote{The regret bound is simply $O(\nicefrac{T}{B})$ times the state-of-art regret bound from \citet{bubeck2017kernel}. \citet{rivera2018online} obtain a similar regret bound for the full feedback version, in a simultaneous and independent work relative to \citet{AdvBwK-focs19}.}

\xhdr{Experts with knapsacks.}
In the full-feedback version of \BwK, the entire outcome matrix $\vM_t$ is revealed after each round $t$.
Essentially, one can achieve regret bound \eqref{BwK:eq:optimal-UB} with $K$ replaced by $\log K$ using \UcbBwK algorithm (Algorithm~\ref{BwK:alg:BwK-UCB}) with a slightly modified analysis. \LagrangeBwK reduction achieves regret
    $O(T/B)\cdot \sqrt{T \log(dKT)}$
if the ``primal" algorithm $\ALG_1$ is \Hedge from Section~\ref{FF:sec:hedge}.


\subsection{Beyond the worst case}
\label{BwK:sec:lit-beyond}

\noindent\textbf{Characterization for $O(\log T)$ regret.}
Going beyond worst-case regret bounds, it is natural to ask about smaller, instance-dependent regret rates, akin to $O(\log T)$ regret rates for stochastic bandits from Chapter~\ref{ch:IID}. \citet{Karthik-BwK-2020} find that $O(\log T)$ regret rates are possible for \BwK if and only if two conditions hold: there is only one resource other than time (\ie $d=2$), and the best distribution over arms reduces to the best fixed arm (\emph{best-arm-optimality}). If either condition fails, any algorithm is doomed to $\Omega(\sqrt{T})$ regret in a wide range of problem instances.%
\footnote{The precise formulation of this lower bound is somewhat subtle. It starts with any problem instance $\mI_0$ with three arms, under mild assumptions, such that either $d>2$ or best-arm optimality fails with some margin. Then it constructs two problem instances $\mI,\mI'$ which are $\eps$-perturbations of $\mI_0$,
    $\eps = O(\nicefrac{1}{\sqrt{T}})$,
in the sense that expected rewards and resource consumptions of each arm differ by at most $\eps$. The guarantee is that any algorithm suffers regret $\Omega(\sqrt(T))$ on either $\mI$ or $\mI'$.}
Here both upper and lower bounds are against the fixed-distribution benchmark ($\OPTFD$).


Assuming $d=2$ and best-arm optimality, $O(\log T)$ regret against $\OPTFD$ is achievable with \UcbBwK algorithm \citep{Karthik-BwK-2020}. In particular, the algorithm does not know in advance whether best-arm-optimality holds, and attains the optimal worst-case regret bound for all instances, best-arm-optimal or not. The instance-dependent parameter in this regret bound generalizes the \emph{gap} from stochastic bandits (see Remark~\ref{IID:rem:conventions}), call it \emph{reward-gap} in the context of \BwK.  The definition uses Lagrange functions from \refeq{BwK:eq:Lag}:
\begin{equation}\label{eq:Glag}
\LagGap(a) :=  \OPTLP - \mL(a, \lambda^*)
\qquad\text{\emph{(Lagrangian gap of arm $a$)}},
\end{equation}
where $\lambda^*$ is a minimizer dual vector in \refeq{BwK:eq:Lag-minimax},
and
    $\LagGap := \textstyle \min_{\myArms} \LagGap(a)$.
The regret bound scales as
    $O(K\LagGap^{-1}\,\log T)$,
which is optimal in $\LagGap$,
under a mild additional assumption that the expected consumption of the best arm is not very close to $\nicefrac{B}{T}$. Otherwise, regret is
$\mO(K\LagGap^{-2}\,\log T)$.

While the $O(\log T)$ regret result is most meaningful as a part of the characterization, it also stands on its own,  even though it requires $d=2$, best-arm-optimality, and a reasonably small number of arms $K$. Indeed, \BwK problems with $d=2$ and small $K$ capture the three challenges of \BwK from the discussion in Section~\ref{BwK:sec:intro} and, as spelled out in \citet[][Appendix A]{Karthik-BwK-2020}, arise in many motivating applications. Moreover, best-arm-optimality is a typical, non-degenerate case, in some rigorous sense.

\xhdr{Other results on $O(\log T)$ regret.}
Several other results achieve $O(\log T)$ regret in \BwK,
with various assumptions and caveats, cutting across the characterization discussed above.

\citet{Wu-BwK-nips15}
assume deterministic resource consumption, whereas all motivating examples of \BwK require consumption to be stochastic, and in fact correlated with rewards (\eg dynamic pricing consumes supply only if a sale happens). They posit $d=2$ and no other assumptions, whereas ``best-arm optimality" is necessary for stochastic resource consumption.

\citet{flajolet2015logarithmic} assume ``best-arm-optimality"  (it is  implicit in the definition of their generalization of reward-gap). Their algorithm inputs an instance-dependent parameter which is normally not revealed to the algorithm (namely, an exact value of some continuous-valued function of mean rewards and consumptions). For $d=2$, regret bounds for their algorithm scale with $\cmin$,  minimal expected consumption among arms: as $\cmin^{-4}$ for the $O(\log T)$ bound, and as $\cmin^{-2}$ for the worst-case bound. Their analysis extents to $d>2$, with regret
    $\cmin^{-4}\, K^K /\term{gap}^6$
and without a worst-case regret bound.

\citet{vera2019online} study a contextual version of \BwK with two arms, one of which does nothing. This formulation is well-motivated for contextual \BwK, but meaningless when specialized to \BwK.

\citet{Li2021} do not make any assumptions, but use additional instance-dependent parameters (\ie other than their generalization of reward-gap). These parameters spike up and yield $\Omega(\sqrt{T})$ regret whenever the $\Omega(\sqrt{T})$ lower bounds from \citet{Karthik-BwK-2020} apply. Their algorithm does not appear to achieve any non-trivial regret bound in the worst case.

Finally, as mentioned before, \cite{Gyorgy-ijcai07,TranThanh-aaai10,TranThanh-aaai12,Qin-aaai13,rangi2018unifying} posit only one constrained resource and $T=\infty$.

\xhdr{Simple regret,} which tracks algorithm's performance in a given round, can be small in all but a few rounds. Like in stochastic bandits, simple regret can be at least $\eps$ in at most $\tilde{\mO}(K/\eps^2)$ rounds, for all $\eps>0$ simultaneously.%
\footnote{For stochastic bandits, this is implicit in the analysis in Chapter~\ref{IID:sec:succ} and the original analysis of \UcbOne algorithm in \cite{bandits-ucb1}. This result is achieved along with the worst-case and logarithmic regret bounds.}
In fact, this is achieved by \UcbBwK algorithm \citep{Karthik-BwK-2020}.

Simple regret for \BwK is defined, for a given round $t$, as
    $\OPT/T- r(X_t)$,
where $X_t$ is the distribution over arms chosen by the algorithm in this round. The benchmark $\OPT/T$ generalizes the best-arm benchmark from stochastic bandits. If each round corresponds to a user and the reward is this user's utility, then $\OPT/T$ is the ``fair share" of the total reward. Thus, with \UcbBwK, all but a few users receive close to their fair share. This holds if $B>\Omega(T) \gg K$, without any other assumptions.

\subsection{Adversarial bandits with knapsacks}
\label{BwK:sec:lit-adv}

In the adversarial version of \BwK, each outcome matrices $\vM_t$ is chosen by an adversary. Let us focus on the oblivious adversary, so that the entire sequence $\vM_1, \vM_2 \LDOTS \vM_T$ is fixed before round $1$. All results below are from \citet{AdvBwK-focs19}, unless mentioned otherwise. The version with IID outcomes that we've considered before will be called \emph{Stochastic BwK}.

\xhdr{Hardness of the problem.} Adversarial BwK is a much harder problem compared to the stochastic version. The new challenge is that the algorithm needs to decide how much budget to save for the future, without being able to predict it. An algorithm must compete, during any given time segment $[1,\tau]$, with a distribution $D_\tau$ over arms that maximizes the total reward on this time segment. However, these distributions may be very different for different $\tau$. For example, one distribution $D_{\tau}$ may exhaust some resources by time $\tau$, whereas another distribution $D_{\tau'}$, $\tau'> \tau$ may save some resources for later.

Due to this hardness, one can only approximate optimal rewards up to a multiplicative factor, whereas sublinear regret is no longer possible. To state this point more concretely, consider the ratio
    $\OPTFD/\E[\REW]$,
called \emph{competitive ratio}, where $\OPTFD$ is the fixed-distribution benchmark (as defined in Section~\ref{BwK:sec:intro}).
A very strong lower bound holds: no algorithm can achieve competitive ratio better than $O(\log T)$ on all problem instances. The lower-bounding construction involves only two arms and only one resource, and forces the algorithm to make a huge commitment without knowing the future.

It is instructive to consider a simple example in which the competitive ratio is at least $\tfrac54-o(1)$ for any algorithm.
There are two arms and one resource with budget $\nicefrac{T}{2}$. Arm $1$ has zero rewards and zero consumption. Arm $2$  has consumption $1$ in each round, and offers reward $\nicefrac12$ in each round of the first half-time ($\nicefrac{T}{2}$ rounds). In the second half-time, it offers either reward $1$ in all rounds, or reward $0$ in all rounds. Thus, there are two problem instances that coincide for the first half-time and differ in the second half-time. The algorithm needs to choose how much budget to invest in the first half-time, without knowing what comes in the second. Any choice leads to competitive ratio at least $\nicefrac54$ on one of the two instances.

A more elaborate version of this example, with $\Omega(\nicefrac{T}{B})$ phases rather than two, proves that competitive ratio can be no better than
    $\OPTFD/\E[\REW] \geq \Omega(\log \nicefrac{T}{B})$
in the worst case.%
\footnote{More precisely, for any $B\leq T$ and any algorithm, there is a problem instance with budget $B$ and time horizon $T$ such that
$\OPTFD/\E[\REW] \geq \tfrac12 \ln \Cel{\nicefrac{T}{B}} + \zeta - \tildeO(1/\sqrt{B})$,
where $\zeta = 0.577...$ is the Euler-Masceroni constant.}

The best distribution is arguably a good benchmark for this problem. The best-algorithm benchmark is arguably \emph{too harsh}: essentially, the ratio $\OPT/\E[\REW]$ cannot be better than $\nicefrac{T}{B}$ in the worst case.%
\footnote{\label{BwK:fn:lit-adv-LB}
\citet{BalseiroGur19} construct this lower bound for dynamic bidding in second-price auctions (see Section~\ref{BwK:sec:examples}). Their guarantee is as follows: for any time horizon $T$, any constants $0<\gamma<\rho<1$, and any algorithm there is a problem instance with budget $B=\rho T$ such that
    $\OPT/\E[\REW]\geq \gamma-o(T)$.
The construction uses only $K=T/B$ distinct bids. \newline \indent
\citet{AdvBwK-focs19} provide a simpler but weaker guarantee with only $K=2$ arms: for any time horizon $T$, any budget $B<\sqrt{T}$ and any algorithm, there is a problem instance with
    $\OPT/\E[\REW] \geq T/B^2$.}
The fixed-arm benchmark  $\OPTFA$ can be arbitrarily worse compared to $\OPTFD$ (see Exercise~\ref{BwK:ex:adv-OPTFA}). Moreover it is \emph{uninteresting}:
    $\OPTFA/\E[\REW] \geq \Omega(K)$
for some problem instances, matched by a trivial algorithm that chooses one arm uniformly at random and plays it forever.

\xhdr{Algorithmic results.}
One can achieve a near-optimal competitive ratio,
\begin{align}\label{BwK:eq:adv}
(\OPTFD-\reg)/\E[\REW] \leq O_d(\log T),
\end{align}
up to a  regret term
    $\reg = O(1+ \tfrac{\OPTFD}{d B}) \sqrt{T K \log (Td)}$.
This is achieved using a version of \LagrangeBwK algorithm, with two important differences: the time resource is not included in the outcome matrices, and the $\nicefrac{T}{B}$ ratio in the Lagrange function \eqref{BwK:eq:Lag-t} is replaced by a parameter $\gamma\in (0,\nicefrac{T}{B}]$, sampled at random from some exponential scale. When $\gamma$ is sampled close to $\OPTFD/B$, the algorithm obtains a constant competitive ratio.%
\footnote{On instances with zero resource consumption, this algorithm achieves $\tildeO(\sqrt{KT})$ regret, for any choice of parameter $\gamma$.}
A completely new analysis of \LagrangeBwK is needed, as the zero-games framework from Chapter~\ref{ch:games} is no longer applicable. The initial analysis in \citet{AdvBwK-focs19} obtains \eqref{BwK:eq:adv} with competitive ratio
    $\tfrac{d+1}{2} \ln T$.
\citet{Singla-colt20} refine this analysis to \eqref{BwK:eq:adv} with competitive ratio
    $O\rbr{ \log(d)\, \log (T)}$.
They also prove that such competitive ratio is optimal up to constant factors.

One can also achieve an $O(d\log T)$ competitive ratio with high probability:
\begin{align}\label{BwK:eq:adv-HP}
\Pr\sbr{(\OPTFD-\reg)/\REW \leq O(d\log T)} \geq 1-T^{-2},
\end{align}
with a somewhat larger regret term $\reg$. This result uses \LagrangeBwK as a subroutine, and its analysis as a key lemma. The algorithm is considerably more involved: instead of guessing the parameter $\gamma$ upfront, the guess is iteratively refined over time.

It remains open, at the time of this writing, whether constant-factor competitive ratio is possible when $B>\Omega(T)$, whether against $\OPTFD$ or against $\OPT$. \citet{BalseiroGur19} achieve competitive ratio $\nicefrac{T}{B}$ for dynamic bidding in second-price auctions, and derive a matching lower bound (see Footnote~\ref{BwK:fn:lit-adv-LB}). Their positive result makes some convexity assumptions, and holds (even) against $\OPT$.

\LagrangeBwK reduction discussed in Section~\ref{BwK:sec:lit-reduction} applies to adversarial \BwK as well. We immediately obtain extensions to the same settings as before: contextual \BwK, combinatorial semi-\BwK, BCO with knapsacks, and \BwK with full feedback. For each setting, one obtains the competitive ratios in \eqref{BwK:eq:adv} and \eqref{BwK:eq:adv-HP}, with some problem-specific regret terms.

\citet{Singla-colt20} consider a more general version of the stopping rule: the algorithm stops at time $t$ if
    $\| (C_{t,1} \LDOTS C_{t,d} \|_p > B$,
where $p\geq 1$ and $C_{t,i}$ is the total consumption of resource $i$ at time $t$. The case $p=\infty$ corresponds to \BwK. They obtain the same competitive ratio,
    $O\rbr{ \log(d)\, \log (T)}$,
using a version of \LagrangeBwK with a different ``dual" algorithm and a different analysis.

\xhdr{Problem variants with sublinear regret.}
Several results achieve sublinear regret,  \ie a regret bound that is sublinear in $T$, via various simplifications.

\citet{rangi2018unifying} consider the special case when there is only one constrained resource, including time. They assume a known lower bound $c_{\min}>0$ on realized per-round consumption of each resource, and their regret bound scales as $1/c_{\min}$. They also achieve $\text{polylog}(T)$ instance-dependent regret for the stochastic version using the same algorithm (matching results from prior work on the stochastic version).

Several papers posit a relaxed benchmark: they only compare to distributions over actions which satisfy the time-averaged resource constraint \emph{in every round}. \citet{Sun-icml17} handles \BwK with $d=2$ resources; their results extend to contextual \BwK with policy sets, via a computationally inefficient algorithm. \emph{Online convex optimization with constraints} \citep{mahdavi2012trading,mahdavi2013stochastic,chen2017online,neely2017online,chen2018bandit}
assumes that the action set is a convex subset of $\R^m$, $m\in\N$, and in each round rewards are concave and consumption of each resource is convex (as functions from actions). In all this work, resource constraints only apply at the last round, and more-than-bandit feedback is observed.%
\footnote{Full feedback is observed for the resource consumption, and (except in \citet{chen2018bandit}) the algorithm also observes either full feedback on rewards or the rewards gradient around the chosen action.}


\subsection{Paradigmaric application: Dynamic pricing with limited supply}
\label{BwK:sec:further-DP}


In the basic version, the algorithm is a seller with $B$ identical copies of some product, and there are $T$ rounds. In each round $t\in [T]$, the algorithm chooses some price $p_t\in [0,1]$ and offers one copy for sale at this price. The outcome is either a sale or no sale; the corresponding outcome vectors are shown in \refeq{BwK:eq:outcomeV-DP}. The customer response is summarized by the probability of making a sale at a given price $p$, denoted $S(p)$, which is assumed to be the same in all rounds, and non-increasing in $p$.%
\footnote{In particular, suppose one customer arrives at time $t$, with private value $v_t$, and buys if and only if $v_t \geq p_t$. If $v_t$ is drawn independently from some distribution $D$, then $S(p) = \Pr_{v\sim D}[v\geq p]$. Considering the sales probability directly is more general, \eg it allows for multiple customers to be present at a given round.}
The function $S(\cdot)$ is the \emph{demand curve}, a well-known notion in Economics. The problem was introduced in \citet{BZ09}. The unconstrained case ($B=T$) is discussed in Section~\ref{Lip:sec:lit-DP}.

The problem can be solved via \emph{discretization}: choose a finite subset $P\subset [0,1]$ of prices, and run a generic \BwK algorithm with action set $P$. The generic guarantees for \BwK provide a regret bound against $\OPT(P)$, the expected total reward of the best all-knowing algorithm restricted to prices in $P$. One needs to choose $P$ so as to balance the \BwK regret bound (which scales with $\sqrt{|P|}$) and the \emph{discretization error} $\OPT-\OPT(P)$ (or a similar difference in the LP-values, whichever is more convenient). Bounding the discretization error is a new, non-trivial step, separate from the analysis of the \BwK algorithm. With this approach, \citet{BwK-focs13-conf,BwK-focs13} achieve regret bound $\tildeO(B^{2/3})$ against $\OPT$. Note that, up to logarithmic factors, this regret bound is driven by $B$ rather than $T$. This regret rate is optimal for dynamic pricing, for any given $B,T$: a complimentary $\Omega(B^{2/3})$ lower bound has been proved in \citet{DynPricing-ec12}, even against the best fixed price. Interestingly, the optimal regret rate is attained using a generic \BwK algorithm: other than the choice of discretization $P$, the algorithm is not adapted to dynamic pricing.

Earlier work focused on competing with the best fixed price, \ie $\OPTFA$. \citet{DynPricing-ec12} achieved $\tildeO(B^{2/3})$ regret against $\OPTFA$, along with the lower bound mentioned above. Their algorithm is a simplified version of \UcbBwK: in each round, it chooses a price with a largest UCB on the expected total reward; this algorithm is run on a pre-selected subset of prices. The initial result from \citet{BZ09} assumes $B>\Omega(T)$ and achieves regret
    $\tildeO(T^{3/4})$,
using the explore-first technique (see Chapter~\ref{ch:IID}).

Consider the same problem under \emph{regular demands}, a standard assumption in theoretical economics which states that $R(p) = p\cdot S(p)$, the expected revenue at price $p$, is a concave function of $p$. Then the best fixed price is close to the best algorithm: $\OPTFA\geq \OPT-\tildeO(\sqrt{B})$ \citep{Yan11}. In particular, the $\tildeO(B^{2/3})$ regret bound from \citet{DynPricing-ec12} carries over to $\OPT$. Further, \citet{DynPricing-ec12} achieve $\tildeO(c_S\cdot \sqrt{B})$ regret provided that $B/T\leq c'_S$, where $c_S$ and $c'_S>0$ are some positive constants determined by the demand curve $S$. They also provide a matching $\Omega(c_S\cdot \sqrt{T})$ lower bound, even if an $S$-dependent constant $c_S$ is allowed in $\Omega(\cdot)$. A simultaneous and independent work, \citet{Wang-OR14} attains a similar upper bound via a different algorithm, under additional assumptions that $B>\Omega(T)$ and the demand curve $S$ is Lipschitz. The initial result in \citet{BZ09} achieved $\tildeO(T^{2/3})$ regret under the same assumptions. \citet{Wang-OR14} also proves a $\Omega(\sqrt{T})$ lower bound, but without an $S$-dependent constant.

Both lower bounds mentioned above, \ie $\Omega(B^{2/3})$ against $\OPTFA$ and $\Omega(c_S\cdot \sqrt{B})$ for regular demands, are proved in  \citet{DynPricing-ec12} via a reduction to  the respective lower bounds from \citet{KleinbergL03} for the unconstrained case ($B=T$). The latter lower bounds encapsulate most of the ``heavy lifting" in the analysis.



Dynamic pricing with $n\geq 2$ products in the inventory is less understood. As in the single-product case, one can use discretization and run an optimal \BwK algorithm on a pre-selected finite subset $P$ of price vectors. If the demand curve is Lipschitz, one can bound the discretization error, and achieve regret rate on the order of $T^{(n+1)/(n+2)}$, see Exercise~\ref{BwK:ex:DP}(a). However, it is unclear how to bound the discretization error without Lipschitz assumptions. The only known result in this direction is when there are multiple feasible bundles of goods, and in each round the algorithm chooses a bundle and a price. Then the technique from the single-product case applies, and one obtains a regret bound of
    $\tildeO\rbr{n\, B^{2/3}\,(N \ell)^{1/3}}$,
where $N$ is the number of bundles, each bundle consists of at most $\ell$ items, and prices are in the range $[0,\ell]$ \citep{BwK-focs13-conf,BwK-focs13}. Dynamic pricing with $n\geq 2$ products was first studied in \citep{BesbesZeevi-or12}. They provide several algorithms with non-adaptive exploration for the regime when all budgets are $\Omega(T)$. In particular, they attain regret $\tildeO(T^{1-1/(n+3)})$ when demands are Lipschitz (as a function of prices) and expected revenue is concave (as a function of demands).

While this discussion focuses on regret-minimizing formulations of dynamic pricing, Bayesian and parametric formulations versions have a rich literature in Operations Research and Economics \citep{Boer-survey15}.

\subsection{Rewards vs. costs}

\BwK as defined in this chapter does not readily extend to a version with costs rather than rewards. Indeed, it does not make sense to stop a cost-minimizing algorithm once it runs out of resources --- because sich algorithm would seek high-consumption arms in order to stop early! So, a cost-minimizing version of \BwK must require the algorithm to continue till the time horizon $T$. Likewise, an algorithm cannot be allowed to skip rounds (otherwise it would just skip \emph{all} rounds). Consequently, the null arm --- which is now an arm with maximal cost and no resource consumption --- is not guaranteed to exist. In fact, this is the version of \BwK studied in \citet{Sun-icml17} and the papers on online convex optimization (discussed in Section~\ref{BwK:sec:lit-adv}).


\newpage
\sectionExercises

(Assume Stochastic BwK unless specified otherwise.)

\begin{exercise}[Explore-first algorithm for \BwK]\label{BwK:ex:explore-first}
Fix time horizon $T$ and budget $B$. Consider an algorithm $\ALG$ which explores uniformly at random for the first $N$ steps, where $N$ is fixed in advance, then chooses some distribution $D$ over arms and draws independently from this distribution in each subsequent rounds.

\begin{itemize}

\item[(a)]
Assume $B<\sqrt{T}$. Prove that there exists a problem instance on which $\ALG$ suffers linear regret:
    \[ \OPT-\E[\REW] > \Omega(T) .\]

\Hint{Posit one resource other than time, and three arms:
\begin{OneLiners}
\item the \emph{bad arm}, with deterministic reward $0$ and consumption $1$;
\item the \emph{good arm}, with deterministic reward $1$ and expected consumption $\tfrac{B}{T}$;
\item the \emph{decoy arm}, with deterministic reward $1$ and expected consumption
    $2\tfrac{B}{T}$.
\end{OneLiners}
Use the following fact: given two coins with expectations $\tfrac{B}{T}$ and $\tfrac{B}{T} + c/\sqrt{N}$, for a sufficiently low absolute constant $c$, after only $N$ tosses of each coin, for any algorithm it is a constant-probability event that this algorithm cannot tell one coin from another.}

\item[(b)]
Assume $B>\Omega(T)$. Choose $N$ and $D$ so that $\ALG$ achieves regret
    $\OPT-\E[\REW] < \tildeO(T^{2/3})$.

\Hint{Choose $D$ as a solution to the ``optimistic" linear program \eqref{BwK:eq:LP-UCB}, with rescaled budget
    $B' = B(1-\sqrt{K/T})$.
Compare $D$ to the value of the original linear program \eqref{BwK:eq:LP} with budget $B'$, and the latter to the value of \eqref{BwK:eq:LP} with budget $B$.}

\end{itemize}
 \end{exercise}

\begin{exercise}[Best distribution vs. best fixed arm]\label{BwK:ex:benchmarks}
Recall that $\OPTFA$ and $\OPTFD$ are, resp., the fixed-arm benchmark and the fixed-distribution benchmark, as defined in Section~\ref{BwK:sec:intro}. Let $d$ be the number of resources, including time.

\begin{itemize}
\item[(a)] Construct an example such that
    $ \OPTFD\geq d\cdot \OPTFA - o(\OPTFA)$.

\Hint{Extend the $d=2$ example from Section~\ref{BwK:sec:intro}.}

\item[(b)]
Prove that
    $ \OPT \leq d\cdot \OPTFA.$

\Hint{\eqref{BwK:eq:LP} has an optimal solution with support size at most $d$. Use the pigeonhole principle!}
\end{itemize}
\end{exercise}

\begin{exercise}[Discretization in dynamic pricing]\label{BwK:ex:DP}
Consider dynamic pricing with limited supply of $d$ products: actions are price vectors $p\in [0,1]^d$, and
    $c_i(p)\in [0,1]$
is the expected per-round amount of product $i$ sold  at price vector $p$. Let
    $P_\eps := [0,1]^d \cap \eps\,\N^d$
be a uniform mesh of prices with step $\eps\in(0,1)$. Let $\OPT(P_\eps)$ and $\OPTFA(P_\eps)$ be the resp. benchmark restricted to the price vectors in $P_\eps$.

\begin{itemize}
\item[(a)] Assume that $c_i(p)$ is Lipschitz in $p$, for each product $i$:
\[ |c_i(p) - c_i(p')| \leq L\cdot \|p-p\|_1,\quad \forall p,p'\in[0,1]^d. \]
 Prove that the discretization error is
        $\OPT - \OPT(P_\eps) \leq O(\eps d L)$.
Using an optimal \BwK algorithm with appropriately action set $P_\eps$, obtain regret rate
    $\OPT - \E[\REW] < \tildeO(T^{(d+1)/(d+2)})$.

\Hint{To bound the discretization error, use the approach from Exercise~\ref{BwK:ex:explore-first}; now the deviations in rewards/consumptions are due to the change in $p$.}

\item[(b)] For the single-product case ($d=1$), consider the fixed-arm benchmark, and prove that the resp. discretization error is
    $\OPTFA-\OPTFA(P_\eps) \leq O(\eps\,\sqrt{B}) $.

\Hint{Consider ``approximate total reward" at price $p$ as
    $V(p) = p\cdot \min(B,\,T\cdot S(p))$.
Prove that the expected total reward for always using price $p$ lies between
    $V(p) -\tildeO(\sqrt{B})$ and $V(p)$.}
\end{itemize}
\end{exercise}

\begin{exercise}[\LagrangeBwK with zero resource consumption]\label{BwK:ex:adv-zero}
Consider a special case of \BwK with $d\geq 2$ resources and zero resource consumption. Prove that \LagrangeBwK achieves regret
    $\tildeO(\sqrt{KT})$.

\Hint{Instead of the machinery from Chapter~\ref{ch:games}, use the regret bound for the primal algorithm, and the fact that there exists an optimal solution for \eqref{BwK:eq:LP} with support size $1$.}
\end{exercise}

\begin{exercise}[\OPTFA for adversarial \BwK]\label{BwK:ex:adv-OPTFA}
Consider Adversarial \BwK. Prove that $\OPTFA$ can be arbitrarily worse than $\OPTFD$. Specifically, fix arbitrary time horizon $T$, budget $B<T/2$, and number of arms $K$, and construct a problem instance with $\OPTFA=0$ and $\OPTFD>\Omega(T)$.

\Hint{Make all arms have reward $0$ and consumption $1$ in the first $B$ rounds.}
\end{exercise}

\chapter{Bandits and Agents}
\label{ch:BIC}
\begin{ChAbstract}
In many scenarios, multi-armed bandit algorithms interact with self-interested parties, a.k.a. \emph{agents}. The algorithm can affects agents' incentives, and agents' decisions in response to these incentives can influence the algorithm's objectives. We focus this chapter on a particular scenario, \emph{incentivized exploration}, motivated by exploration in recommendation systems, and we survey some other scenarios in the literature review.

\prereqs{Chapter~\ref{ch:IID}; Chapter~\ref{ch:LB} (results only, just for perspective).}
\end{ChAbstract}


Consider a population of self-interested agents that make decisions under uncertainty. They \emph{explore} to acquire new information and \emph{exploit} this information to make good decisions. Collectively they need to balance these two objectives, but their incentives are skewed toward exploitation. This is because exploration is costly, but its benefits are spread over many agents in the future. Thus, we ask, \textsl{How to incentivize self-interested agents to explore when they prefer to exploit?}

Our motivation comes from recommendation systems. Users therein consume information from the previous users, and produce information for the future. For example, a decision to dine in a particular restaurant may be based on the existing reviews, and may lead to some new subjective observations about this restaurant. This new information can be consumed either directly (via a review, photo, tweet, etc.) or indirectly through aggregations, summarizations or recommendations, and can help others make similar choices in similar circumstances in a more informed way. This phenomenon applies very broadly, to the choice of a product or experience, be it a movie, hotel, book, home appliance, or virtually any other consumer's choice. Similar issues, albeit with higher stakes, arise in health and lifestyle decisions such as adjusting exercise routines or selecting a doctor or a hospital. Collecting, aggregating and presenting users' observations is a crucial value proposition of numerous businesses in the modern economy.


When self-interested individuals (\emph{agents}) engage in the information-revealing decisions discussed above, individual and collective incentives are misaligned. If a social planner were to direct the agents, she would trade off exploration and exploitation so as to maximize the social welfare. However, when the decisions are made by the agents rather than enforced by the planner, each agent's incentives are typically skewed in favor of exploitation, as (s)he would prefer to benefit from exploration done by others. Therefore, the society as a whole may suffer from insufficient amount of exploration. In particular, if a given alternative appears suboptimal given the information available so far, however sparse and incomplete, then this alternative may remain unexplored forever (even though it may be the best).

Let us consider a simple example in which the agents fail to explore. Suppose there are two actions $a\in\{1,2\}$ with deterministic rewards $\mu_1, \mu_2$ that are initially unknown. Each $\mu_a$ is drawn independently from a known Bayesian prior such that $\E[\mu_1]> \E[\mu_2]$. Agents arrive sequentially: each agent chooses an action, observes its reward and reveals it to all subsequent agents. Then the first agent chooses action $1$ and reveals $\mu_1$. If $\mu_1>\E[\mu_2]$, then all future agents also choose arm $1$. So, action $2$ never gets chosen. This is very wasteful if the prior assigns a large probability to the event $\{\mu_2 \gg \mu_1>\E[\mu_2]\}$.




Our problem, called \textbf{incentivized exploration}, asks how to incentivize the agents to explore. We consider a \emph{principal} who cannot control the agents, but can communicate with them, \eg recommend an action and observe the outcome later on.
Such a principal would typically be implemented via a website, either one dedicated to recommendations and feedback collection (\eg Yelp, Waze), or one that actually provides the product or experience being recommended (\eg Netflix, Amazon). While the principal would often be a for-profit company, its goal for our purposes would typically be well-aligned with the social welfare.

We posit that the principal creates incentives \emph{only} via communication, rather monetary incentives such as rebates or discounts. Incentives arise due \emph{information asymmetry}: the fact that the principal collects observations from the past agents and therefore has more information than any one agent. Accordingly, agents realize  that (s)he may benefit from following the principal's recommendations, even these recommendations sometimes include exploration.

Incentivizing exploration is a non-trivial task even in the simple example described above, and even if there are only two agents. This is \emph{Bayesian persuasion}, a well-studied problem in theoretical economics. When rewards are noisy and incentives are not an issue, the problem reduces to stochastic bandits, as studied in Chapter~\ref{ch:IID} . Essentially, incentivized exploration needs to solve both problems simultaneously. We design algorithms which create the desired incentives, and (with some caveats) match the performance of the optimal bandit algorithms.

\section{Problem formulation: incentivized exploration}
\label{BIC:sec:stt}

We capture the problem discussed above as a bandit problem with auxiliary constraints that arise due to incentives. The problem formulation has two distinct parts: the ``algorithmic" part, which is essentially about Bayesian bandits, and the ``economic" part, which is about agents' knowledge and incentives. Such ``two-part" structure is very common in the area of \emph{algorithmic economics}. Each part individually is very standard, according to the literature on bandits and theoretical economics, respectively. It is their combination that leads to a novel and interesting problem.

An algorithm (henceforth, the \emph{principal}) interacts with self-interested decision-makers (henceforth, \emph{agents}) over time. There are $T$ rounds and $K$ possible actions, a.k.a. arms; we use $[T]$ and $[K]$ to denote, resp., the set of all rounds and the set of all actions. In each round $t\in [T]$, the principal recommends an arm $\rec_t\in [K]$. Then an agent arrives, observes the recommendation $\rec_t$, chooses an arm $a_t$, receives a reward $r_t\in [0,1]$ for this arm, and leaves forever.
Rewards come from a known parameterized family
    $(\mD_x:\;x\in [0,1])$
of reward distributions such that $\E[\mD_x]=x$. Specifically, each time a given arm $a$ is chosen, the reward is realized as an independent draw from $\mD_x$ with mean reward $x=\mu_a$. The mean reward vector $\mu\in[0,1]^K$ is drawn from a Bayesian prior $\mP$. Prior $\mP$ are known, whereas $\mu$ is not.

\begin{BoxedProblem}{Incentivized exploration}
Parameters: $K$ arms, $T$ rounds, common prior $\mP$,
    reward distributions $(\mD_x:\;x\in [0,1])$.\vspace{2mm}

\noindent Initialization:  the mean rewards vector $\mu \in [0,1]^K$ is drawn from the prior $\mP$.
\vspace{2mm}

\noindent In each round $t = 1,2,3 \LDOTS T $:
\begin{OneLiners}
  \item[1.] Algorithm chooses its recommended arm $\rec_t \in [K]$.
  \item[2.] Agent $t$ arrives, receives recommendation $\rec_t$, and chooses arm $a_t\in [K]$.
  \item[3.] (Agent's) reward $r_t\in [0,1]$ is realized as an independent draw from $D_x$, where $x = \mu_{a_t}$.
  \item[4.] Action $a_t$ and reward $r_t$ are observed by the algorithm.
\end{OneLiners}
\end{BoxedProblem}

\begin{remark}
We allow \emph{correlated priors}, \ie random variables $\mu_a:\, a\in [K]$ can be correlated. An important special case is \emph{independent priors}, when these random variables are mutually independent.
\end{remark}

\begin{remark}
If all agents are guaranteed to \emph{comply}, \ie follow the algorithm's recommendations, then the problem protocol coincides with Bayesian bandits, as defined in Chapter~\ref{ch:TS}.
\end{remark}

What does agent $t$ know before (s)he chooses an action? Like the algorithm, (s)he knows the parameterized reward distribution $(\mD_x)$ and the prior $\mP$, but not the mean reward vector $\mu$. Moreover, (s)he knows the principal's recommendation algorithm, the recommendation $\rec_t$, and the round $t$.  However, (s)he does not observe what happened with the previous agents.

\begin{remark}
All agents share the same beliefs about the mean rewards (as expressed by the prior $\mP$), and these beliefs are \emph{correct}, in the sense that $\mu$ is actually drawn from $\mP$. While  idealized, these two assumptions are very common in theoretical economics.
\end{remark}

For each agent $t$, we put forward a constraint that compliance is in this agent's best interest. We condition on the event that a particular arm $a$ is being recommended, and the event that all previous agents have complied. The latter event, denoted
    $ \mE_{t-1} = \cbr{ a_s=\rec_s:\; s\in [t-1] }$,
ensures that an agent has well-defined beliefs about the behavior of the previous agents.

\begin{definition}\label{BIC:def:BIC}
An algorithm is called \emph{Bayesian incentive-compatible} (\emph{BIC}) if for all rounds $t$ we have
\begin{align}\label{BIC:eqn:bic-constraint}
\E\sbr{\mu_a-\mu_{a'} \mid \rec_t=a,\, \mE_{t-1}}
    \geq 0,
\end{align}
where $a,a'$ are any two distinct arms such that
    $\Pr\sbr{\rec_t=a,\,\mE_{t-1}}>0$.
\end{definition}

We are (only) interested in BIC algorithms. We posit that all agents comply with such algorithm's recommendations. Accordingly, a BIC algorithm is simply a bandit algorithm with an auxiliary BIC constraint.

\begin{remark}\label{BIC:rem:BIC-paradigm}
The definition of BIC follows one of the standard paradigms in theoretical economics: identify the desirable behavior (in our case, following algorithm's recommendations), and require that this behavior maximizes each agent's expected reward, according to her beliefs. Further, to define the agents' beliefs, posit that all uncertainty is realized as a random draw from a Bayesian prior, that the prior and the principal's algorithm are known to the agents, and that all previous agents follow the desired behavior.
\end{remark}

Algorithm's objective is to maximize the total reward over all rounds. A standard performance measure is \emph{Bayesian regret}, \ie pseudo-regret in expectation over the Bayesian prior. We are also interested in comparing BIC bandit algorithms with \emph{optimal} bandit algorithms.

\xhdr{Preliminaries.}
We focus the technical developments on the special case of $K=2$ arms (which captures much of the complexity of the general case). Let
    $\mu_a^0 = \E[\mu_a]$
denote the prior mean reward for arm $a$. W.l.o.g., $\mu_1^0 \geq \mu_2^0$, \ie arm $1$ is (weakly) preferred according to the prior. For a more elementary exposition, let us assume that the realized rewards of each arm can only take finitely many possible values.

Let $\samples{n}$ denote an ordered tuple of $n$ independent samples from arm $1$. (Equivalently, $\samples{n}$ comprises the first $n$ samples from arm $1$.) Let $\AvgR{n}$ be the average reward in these $n$ samples.

Throughout this chapter, we use a more advanced notion of conditional expectation given a random variable. Letting $X$ be a real-valued random variable, and let $Y$ be another random variable with an arbitrary (not necessarily real-valued) but finite support $\mathcal{Y}$. The conditional expectation of $X$ given $Y$ is a itself a random variable,
   $\E[X\mid Y] := F(Y)$,
where $F(y) = \E[X\mid Y=y]$ for all $y\in \mathcal{Y}$.
The conditional expectation given an event $E$ can be expressed as
    $\E[X | E] = \E[X | \indE{E}]$.
We are particularly interested in
    $\E[\,\cdot \mid \samples{n}]$,
the posterior mean reward after $n$ samples from arm $1$.

We will repeatedly use the following fact, a version of the \emph{law of iterated expectation}.

\begin{fact}\label{BIC:fact:iterated-expectations}
Suppose random variable $Z$ is determined by $Y$ and some other random variable $Z_0$ such that $X$ and $Z_0$ are independent (think of $Z_0$ as algorithm's random seed). Then
$ \E[\; \E[X|Y] \mid Z ] = \E[X|Z]$.
\end{fact}




\section{How much information to reveal?}

How much information should the principal reveal to the agents? Consider two extremes: recommending an arm without providing any supporting information, and revealing the entire history.  We argue that the former suffices, whereas the latter does not work.


\xhdr{Recommendations suffice.}
Let us consider a more general model: in each round $t$, an algorithm $\ALG$ sends the $t$-th agent an arbitrary message $\sigma_t$, which includes a recommended arm $\rec_t\in [K]$. The message lies in some fixed, but otherwise arbitrary, space of possible messages; to keep exposition elementary, assume that this space is finite.
A suitable BIC constraint states that recommendation $\rec_t$ is optimal given message $\sigma_t$ and compliance of the previous agents. In a formula,
\[ \rec_t \in \argmax_{a\in[K]} \E\sbr{ \mu_a \mid \sigma_t,\; \mE_{t-1}},
        \qquad\forall t\in[T]. \]

Given an algorithm $\ALG$ as above, consider another algorithm $\ALG'$ which only reveals recommendation $\rec_t$ in each round $t$. It is easy to see that $\ALG'$ is BIC, as per \eqref{BIC:eqn:bic-constraint}. Indeed, fix round $t$ and arm $a$ such that
    $\Pr\sbr{\rec_t=a,\, \mE_{t-1}}>0$.
Then
\[ \E\sbr{\mu_a-\mu_{a'} \mid \sigma_t,\;\rec_t=a,\, \mE_{t-1}}
    \geq 0
    \qquad \forall a'\in [K].
\]
We obtain \eqref{BIC:eqn:bic-constraint} by integrating out the message $\sigma_t$, \ie by taking conditional expectation of both sides given $\{\rec_t = a,\, \mE_{t-1} \}$; formally, \eqref{BIC:eqn:bic-constraint} follows by
Fact~\ref{BIC:fact:iterated-expectations}.

Thus, it suffices to issue recommendations, without any supporting information. This conclusion, along with the simple argument presented above, is a version of a well-known technique from theoretical economics called Myerson's \emph{direct revelation principle}. While surprisingly strong, it relies on several subtle assumptions implicit in our model. We discuss these issues more in Section~\ref{BIC:sec:further}.

\xhdr{Full revelation does not work.} Even though recommendations suffice, does the principal need to bother designing and deploying a bandit algorithm? In theoretical terms, does the principal need to \emph{explore}? An appealing alternative is to reveal  the full history, perhaps along with some statistics, and let the agents choose for themselves. Being myopic, the agents would follow the \emph{Bayesian-greedy} algorithm, a Bayesian version of the ``greedy" bandit algorithm which always ``exploits" and never ``explores".

Formally, suppose in each round $t$, the algorithm reveals a message $\sigma_t$ which includes the history,
    $H_t = \{ (a_s,r_s):\, s\in [t-1] \}$.
Posterior mean rewards are determined by $H_t$:
    \[  \E[ \mu_a \mid \sigma_t ] = \E[ \mu_a \mid H_t ]
        \quad\text{for all arms $a$}, \]
because the rest of the message can only be a function of $H_t$, the algorithm's random seed, and possibly other inputs that are irrelevant. Consequently, agent $t$ chooses an arm
\begin{align}\label{BIC:eq:BG}
     a_t \in \argmax_{a\in [K]} \E\sbr{\mu_a \mid H_t}.
\end{align}
Up to tie-breaking, this defines an algorithm, which we call \greedy. We make no assumptions on how the ties are broken, and what else is included in algorithm's messages. In contrast with \eqref{BIC:eqn:bic-constraint}, the expectation in \eqref{BIC:eq:BG} is well-defined without $\mE_{t-1}$ or any other assumption about the choices of the previous agents, because these choices are already included in the history $H_t$.


\greedy performs terribly on a variety of problem instances, suffering Bayesian regret $\Omega(T)$. (Recall that bandit algorithms can achieve regret $\tildeO(\sqrt{T})$ on all problem instances, as per Chapter~\ref{ch:IID}.) The root cause of this inefficiency is that \greedy may never try arm $2$. For the special case of deterministic rewards, this happens with probability $\Pr[\mu_1\leq \mu_2^0]$, since $\mu_1$ is revealed in round $1$ and arm $2$ is never chosen if $\mu_1\leq \mu_2^0$. With a different probability, this result carries over to the general case.

\begin{theorem}\label{BIC:thm:BG}
With probability at least $\mu_1^0-\mu_2^0$, \greedy never chooses arm $2$.
\end{theorem}

\begin{proof}
In each round $t$, the key quantity is
    $Z_t = \E[ \mu_1-\mu_2 \mid H_t ]$.
Indeed, arm $2$ is chosen if and only if $Z_t<0$. Let $\tau$ be the first round when \greedy chooses arm $2$, or $T+1$ if this never happens. We use martingale techniques to prove that
\begin{align}\label{BIC:eq:thm:BG-Bayes-OST}
\E[Z_\tau] = \mu_1^0-\mu_2^0.
\end{align}

We obtain \refeq{BIC:eq:thm:BG-Bayes-OST} via a standard application of the optional stopping theorem; it can be skipped by readers who are not familiar with martingales. We observe that $\tau$ is a \emph{stopping time} relative to the sequence
    $\mH = \rbr{H_t:\, t\in [T+1]}$,
and $\rbr{ Z_t: t\in [T+1]}$ is a martingale relative to $\mH$.
\footnote{The latter follows from a general fact that sequence
    $\E[X\mid H_t]$, $t\in [T+1]$
is a martingale w.r.t. $\mH$ for any random variable $X$ with $\E\sbr{|X|}\infty$. It is known as \emph{Doob martingale} for $X$.}
The optional stopping theorem asserts that $\E[Z_\tau] = \E[Z_1]$ for any martingale $Z_t$ and any bounded stopping time $\tau$. \refeq{BIC:eq:thm:BG-Bayes-OST} follows because
    $\E[Z_1] = \mu_1^0-\mu_2^0$.

On the other hand, by Bayes' theorem it holds that
\begin{align}\label{BIC:eq:thm:BG-Bayes}
\E[Z_\tau]
    = \Pr[ \tau\leq T ]\,\E[ Z_\tau \mid \tau\leq T ]
        + \Pr[ \tau>T ]\,\E[ Z_\tau \mid \tau>T ]
\end{align}
Recall that $\tau\leq T$ implies that $\greedy$ chooses arm $2$ in round $\tau$, which in turn implies that $Z_\tau \leq 0$ by definition of \greedy. It follows that
    $\E[ Z_\tau \mid \tau\leq T ]\leq 0$.
Plugging this into \refeq{BIC:eq:thm:BG-Bayes}, we find that
\[ \mu_1^0-\mu_2^0 = \E[Z_\tau] \leq \Pr[\tau>T].  \]
And $\{\tau>T\}$ is precisely the event that \greedy never tries arm 2.
\end{proof}

This is a very general result: it holds for arbitrary priors.
Under some mild assumptions, the algorithm never tries arm $2$ \emph{when it is in fact the best arm}, leading to $\Omega(T)$ Bayesian regret.

\begin{corollary}\label{BIC:cor:BG}
Consider independent priors such that $\Pr[\mu_1=1]<(\mu_1^0-\mu_2^0)/2$. Pick any $\alpha>0$ such that
    $\Pr[\mu_1\geq 1-2\,\alpha] \leq (\mu_1^0-\mu_2^0)/2$.
Then \greedy suffers Bayesian regret
\[ \E[R(T)] \geq
    T\cdot \rbr{ \tfrac{\alpha}{2}\;(\mu_1^0-\mu_2^0) \; \Pr[\mu_2>1-\alpha] }.
\]
\end{corollary}

\begin{proof}
Let $\mE_1$ be the event that $\mu_1<1-2\alpha$ and \greedy never chooses arm $2$. By Theorem~\ref{BIC:thm:BG} and the definition of $\alpha$, we have
    $\Pr[\mE_1]\geq (\mu_1^0-\mu_2^0)/2$.

Let $\mE_2$ be the event that $\mu_2>1-\alpha$. Under event $\mE_1\cap \mE_2$, each round contributes
    $\mu_2-\mu_1\geq \alpha$
to regret, so
    $ \E\sbr{ R(T) \mid \mE_1\cap \mE_2} \geq \alpha\,T$.
Since event $\mE_1$ is determined by the draw of $\mu_1$ and the realized rewards of arm $1$, it is independent from $\mE_2$. It follows that
\begin{align*}
\E[R(T)]
    &\geq \E[R(T) \mid \mE_1\cap \mE_2] \cdot \Pr[\mE_1\cap \mE_2] \\
    &\geq \alpha T\cdot (\mu_1^0-\mu_2^0)/2\cdot\Pr[\mE_2]. \qedhere
\end{align*}
\end{proof}

Here's a less quantitative but perhaps cleaner implication:

\begin{corollary}\label{BIC:cor:BG-basic}
Consider independent priors. Assume that each arm's prior has a positive density, \ie for each arm $a$, the prior on $\mu_a\in[0,1]$ has probability density function that is strictly positive on $[0,1]$. Then \greedy suffers Bayesian regret at least $c_\mP\cdot T$, where the constant $c_\mP> 0$ depends only on the prior $\mP$.
\end{corollary}

\OMIT{ 
\xhdr{Minimal revelation suffices.}
The algorithm's message can, without loss of generality, be restricted to the recommended action, \ie $\sigma_t = \rec_t$ for all rounds $t$; call such algorithms \emph{minimal-revelation}. First, including the recommended arm into the message $\sigma_t$ is without loss of generality, because one can define
    $ \rec_t = \argmax_{a\in[K]} \E[\mu_a \mid \sigma_t]$,
breaking ties arbitrarily. Second, suppose the algorithm
satisfies the BIC property \eqref{BIC:eqn:bic-constraint}. Integrating out the message $\sigma_t$,
\footnote{In other words, taking the conditional expectation of both sides in \eqref{BIC:eqn:bic-constraint}, given $\{\rec_t = a,\, \mE_{t-1} \}$.}
we obtain:
\begin{align}\label{BIC:eq:BIC-rec}
\E[\, \mu_a-\mu_{a'} \mid \rec_t = a,\, \mE_{t-1}\,] \geq 0
    \qquad \forall t\in [T],\,\forall a,a'\in [K].
\end{align}
Thus, the BIC condition holds even if only $\rec_t$ is revealed.

While the above argument is an application of a well-known technique from theoretical economics (Myerson's \emph{direct revelation} principle), its conclusion is quite surprising: it suffices to issue recommendations, without any supporting information! However, this conclusion relies on several subtle assumptions implicit in our model. We discuss these issues more in Section~\ref{BIC:sec:further}. In the meantime, we focus on minimal-revelation algorithms, \ie multi-armed bandit algorithms subject to the BIC constraint \eqref{BIC:eq:BIC-rec}.
} 

\section{Basic technique: hidden exploration}


The basic technique to ensure incentive-compatibility is to \emph{hide a little exploration in a lot of exploitation}. Focus on a single round of a bandit algorithm. Suppose we observe a realization of some random variable $\iSig\in \sigSpace$, called the \emph{signal}.%
\footnote{Think of $\iSig$ as the algorithm's history, but it is instructive to keep presentation abstract. For elementary exposition, we assume that the universe $\sigSpace$ is finite. Otherwise we require a more advanced notion of conditional expectation.}
With a given probability $\eps>0$, we recommend an arm what we actually want to choose, as described by the (possibly randomized) \emph{target function}
    $\ExploreFn: \sigSpace \to \{1,2\}$.
The basic case is just choosing arm $\ExploreFn=2$. With the remaining probability we \emph{exploit}, \ie choose an arm that
maximizes  $\E\sbr{\mu_a \mid \iSig}$. Thus, the technique, called \HiddenExploration, is as follows:



\LinesNotNumbered \SetAlgoLined
\begin{algorithm}[h]
\caption{\HiddenExploration with signal $\iSig$.}
\label{BIC:alg:basic}
\DontPrintSemicolon
{\bf Parameters:} probability $\eps>0$,
    function $\ExploreFn: \sigSpace \to \{1,2\}$.\;
{\bf Input:} signal realization $S \in \sigSpace$.\;
{\bf Output:} recommended arm $\rec$.\;\vspace{2mm}
With probability $\eps>0$,
    \tcp*{exploration branch}
\algTAB $\rec \leftarrow \ExploreFn(S)$ \;
else \tcp*{exploitation branch}
\algTAB
    $\rec \leftarrow
        \min\rbr{ \arg\max_{a\in \{1,2\}} \E\sbr{\mu_a \mid \iSig = S}}$
    \tcp*{tie $\Rightarrow$ choose arm $1$}
\end{algorithm}

We are interested in the (single-round) BIC property: for any two distinct arms $a,a'$
\begin{align}\label{BIC:eqn:bic-basic}
\Pr[\rec=a] >0 \;\;\Rightarrow\;\;
\E\left[\;
    \mu_a-\mu_{a'} \mid \rec=a
\;\right]
    \geq 0
\end{align}
We prove that \HiddenExploration satisfies this property when the exploration probability $\eps$ is sufficiently small, so that the exploration branch is offset by exploitation.  A key quantity here is a random variable which summarizes the meaning of signal $\iSig$:
\[ G := \E[\mu_2-\mu_1 \mid \iSig]
\qquad\qquad\EqComment{posterior gap}.
\]


\begin{lemma}\label{BIC:lm:basic}
Algorithm~\ref{BIC:alg:basic} is BIC, for any target function $\ExploreFn$, as long as
    $\eps \leq \tfrac13\,\E\left[ G\cdot \ind{G>0} \right]$.
\end{lemma}

\begin{remark}\label{BIC:rem:HP-works}
A suitable $\eps>0$ exists if and only if $\Pr[G>0]>0$.
Indeed, if $\Pr[G>0]>0$ then $\Pr[G>\delta]=\delta'>0$ for some $\delta>0$, so
\begin{align*}
\E\left[ G\cdot \ind{G>0} \right]
    \geq \E\left[ G\cdot \ind{G>\delta} \right]
    = \Pr[G>\delta]\cdot \E[G\mid G>\delta]
    \geq  \delta\cdot\delta'>0.
\end{align*}
\end{remark}

The rest of this section proves Lemma~\ref{BIC:lm:basic}.
We start with an easy observation: for any algorithm, it suffices to guarantee the BIC property when arm $2$ is recommended.

\begin{claim}
Assume
    \eqref{BIC:eqn:bic-basic}
holds for arm $\rec=2$. Then it also holds for $\rec=1$.
\end{claim}

\begin{proof}
If arm $2$ is never recommended, then the claim holds trivially since
    $\mu_1^0 \geq \mu_2^0$.
Now, suppose both arms are recommended with some positive probability. Then
\begin{align*}
0   \geq \E[\mu_2 - \mu_1]
    = \textstyle \sum_{a\in\{1,2\}}\; \E[\mu_2-\mu_1 \mid \rec=a]\,\Pr[\rec=a].
\end{align*}
Since
    $\E[\mu_2-\mu_1 \mid \rec=2] > 0$
by the BIC assumption,
    $\E[\mu_2-\mu_1 \mid \rec=1] < 0$.
\end{proof}

Thus, we need to prove \eqref{BIC:eqn:bic-basic} for $\rec=2$, \ie that
\begin{align}\label{BIC:eqn:bic-basic-2}
    \E\sbr{ \mu_2-\mu_1 \mid \rec=2 } > 0.
\end{align}
(We note that $\Pr[\mu_2-\mu_1]>0$, \eg because $\Pr[G>0]>0$, as per Remark~\ref{BIC:rem:HP-works}).

Denote the event $\{\rec=2\}$ with $\mE_2$.
By Fact~\ref{BIC:fact:iterated-expectations},
    $ \E\sbr{ \mu_2-\mu_1 \mid \mE_2 }  = \E[G \mid \mE_2]$.
\footnote{This is the only step in the analysis where it is essential that both the exploration and exploitation branches (and therefore event $\mE_2$) are determined by the signal $\iSig$.}

We focus on the posterior gap $G$ from here on. More specifically, we work with expressions of the form
    $ F(\mE) := \E\sbr{G\cdot \indE{\mE}}$,
where $\mE$ is some event. Proving \refeq{BIC:eqn:bic-basic-2} is equivalent to proving that
    $ F(\mE_2) > 0$;
we prove the latter in what follows.

We will use the following fact:
\begin{align}\label{BIC:eq:F-disjoint}
    F(\mE \cup \mE') = F(\mE) + F(\mE')
\quad \text{for any disjoint events $\mE,\mE'$}.
\end{align}
Letting $\ExploreE$ (resp., $\ExploitE$) be the event that the algorithm chooses exploration branch (resp., exploitation branch), we can write
\begin{align}\label{BIC:eq:F-U-disjoint}
     F(\mE_2) = F(\ExploreE \eqAND \mE_2) + F(\ExploitE \eqAND \mE_2).
\end{align}

We prove that this expression is non-negative by analyzing the exploration and exploitation branches separately. For the exploitation branch, the events
    $\{\ExploitE \eqAND \mE_2\}$
and
    $\{\ExploitE \eqAND G>0\}$
are the same by algorithm's specification. Therefore,
\begin{align*}
F(\ExploitE \eqAND \mE_2)
    &=  F(\ExploitE \eqAND G>0)\\
    &= \E[G \mid \ExploitE \eqAND G>0] \cdot \Pr[\ExploitE \eqAND G>0]
        &\EqComment{by definition of $F$} \\
    &= \E[G\mid G>0] \cdot \Pr[G>0]\cdot (1-\eps)
        &\EqComment{by independence} \\
    &= (1-\eps)\cdot F(G>0)
        &\EqComment{by definition of $F$}.
\end{align*}

For the exploration branch, recall that $F(\mE)$ is non-negative for any event $\mE$ with $G\geq 0$, and non-positive for any event $\mE$ with $G\leq 0$. Therefore,
\begin{align*}
F(\ExploreE \eqAND \mE_2)
    &= F(\ExploreE \eqAND \mE_2 \eqAND G<0)
        + F(\ExploreE \eqAND \mE_2 \eqAND G\geq 0)
       &\EqComment{by \eqref{BIC:eq:F-disjoint}} \\
    &\geq F(\ExploreE \eqAND \mE_2 \eqAND G<0) \\
    &= F(\ExploreE \eqAND G<0) - F(\ExploreE \eqAND \neg \mE_2 \eqAND G<0)
        &\EqComment{by \eqref{BIC:eq:F-disjoint}}\\
    &\geq F(\ExploreE \eqAND G<0)\\
    &= \E[G\mid \ExploreE \eqAND G<0]\cdot \Pr[\ExploreE \eqAND G<0]
        &\EqComment{by defn. of $F$}\\
    &= \E[G\mid G<0]\cdot \Pr[G<0]\cdot \eps
        &\EqComment{by independence}\\
    &= \eps\cdot F(G<0)
    &\EqComment{by defn. of $F$}.
\end{align*}

Putting this together and plugging into \eqref{BIC:eq:F-U-disjoint}, we have
\begin{align}\label{BIC:eq:F-U}
F(\mE_2) \geq \eps\cdot F(G<0) + (1-\eps)\cdot F(G>0).
\end{align}

Now, applying \eqref{BIC:eq:F-disjoint} yet again we see that
   $ F(G<0) + F(G>0) = \E[\mu_2-\mu_1]$.
Plugging this back into \eqref{BIC:eq:F-U} and rearranging, it follows that $F(\mE_2)> 0$ whenever
\[ F(G>0) > \eps \rbr{ 2F(G>0)+\E[\mu_1-\mu_2] }. \]
In particular, $\eps< \tfrac13\cdot F(G>0)$ suffices. This completes the proof of Lemma~\ref{BIC:lm:basic}.

\section{Repeated hidden exploration}

Let us develop the hidden exploration technique into an algorithm for incentivized exploration. We take an arbitrary bandit algorithm \ALG, and consider a repeated version of \HiddenExploration (called \RepeatedHE), where the exploration branch executes one call to \ALG. We interpret calls to \ALG as exploration. To get started, we include $N_0$ rounds of ``initial exploration", where arm $1$ is chosen. The exploitation branch conditions on the history of all previous exploration rounds:
\begin{align}\label{BIC:eq:repeatedHP-history}
\mS_t = \rbr{ (s,a_s,r_s): \text{all exploration rounds $s<t$}}.
\end{align}

\LinesNotNumbered \SetAlgoLined
\begin{algorithm}[!h]
\caption{\RepeatedHE with bandit algorithm \ALG.}
\label{BIC:alg:reduction}
\DontPrintSemicolon
{\bf Parameters:} $N_0\in\N$, exploration probability $\eps>0$\;
In the first $N_0$ rounds, recommend arm $1$. \tcp*{initial exploration}
In each subsequent round $t$,\;
\algTAB With probability $\eps$ \tcp*{explore}
\algTAB \algTAB call \ALG, let $\rec_t$ be the chosen arm,
feed reward $r_t$ back to \ALG. \;
\algTAB else \tcp*{exploit}
\algTAB \algTAB
    $\rec_t \leftarrow
        \min\rbr{\arg\max_{a\in \{1,2\}} \E[\mu_a \mid \mS_t] }$.
        \tcp*{$\mS_t$ from \eqref{BIC:eq:repeatedHP-history}}
\end{algorithm}
\vspace{-4mm}


\begin{remark}
\RepeatedHE can be seen as a reduction from bandit algorithms to BIC bandit algorithms. The simplest version always chooses arm $2$ in exploration rounds, and (only) provides non-adaptive exploration. For better regret bounds, \ALG needs to perform adaptive exploration, as per Chapter~\ref{ch:IID}.
\end{remark}

Each round $t>N_0$ can be interpreted as \HiddenExploration with signal $\mS_t$,
where the ``target function" executes one round of algorithm \ALG. Note that $\rec_t$ is determined by $\mS_t$ and the random seed of \ALG, as required by the specification of \HiddenExploration. Thus, Lemma~\ref{BIC:lm:basic} applies, and yields the following corollary in terms of
    $G_t = \E[\mu_2-\mu_1 \mid \mS_t]$,
the posterior gap given signal $\mS_t$.

\begin{corollary}\label{BIC:cor:basic}
\RepeatedHE is BIC if
    $\eps < \tfrac13\,\E\left[ G_t\cdot \ind{G_t>0} \right]$
for each time $t>N_0$.
\end{corollary}

For the final BIC guarantee, we show that it suffices to focus on $t=N_0+1$.

\begin{theorem}\label{BIC:thm:reduction-BIC}
\RepeatedHE with exploration probability $\eps>0$ and $N_0$ initial samples of arm $1$ is BIC as long as
    $\eps < \tfrac13\,\E\left[ G\cdot \ind{G>0} \right]$,
where $G = G_{N_0+1}$.
\end{theorem}

\begin{proof}
The only remaining piece is the claim that the quantity
    $\E\left[ G_t\cdot \ind{G_t>0} \right]$
does not decrease over time. This claim holds for any sequence of signals
    $(\mS_1,\mS_2 \LDOTS S_T)$
such that each signal $S_t$ is determined by the next signal $S_{t+1}$.

Fix round $t$. Applying Fact~\ref{BIC:fact:iterated-expectations} twice, we obtain
\begin{align*}
\E[G_t \mid G_t >0]
    = \E[\mu_2-\mu_1 \mid G_t>0]
    = \E[G_{t+1} \mid G_t >0].
\end{align*}
(The last equality uses the fact that $S_{t+1}$ determines $S_t$.) Then,
\begin{align*}
\E\left[ G_t\cdot \ind{G_t>0} \right]
    &= \E[G_t \mid G_t >0]\cdot \Pr[G_t>0] \\
    &= \E[G_{t+1} \mid G_t >0]\cdot \Pr[G_t>0] \\
    &= \E\left[ G_{t+1}\cdot \ind{G_t>0} \right] \\
    &\leq \E\left[ G_{t+1}\cdot \ind{G_{t+1}>0} \right].
\end{align*}
The last inequality holds because
    $ x\cdot \ind{\cdot} \leq x\cdot \ind{x>0}$
for any $x\in R$.
\end{proof}


\begin{remark}\label{BIC:rem:repeatedHP-condition}
The theorem focuses on the posterior gap $G$ given $N_0$ initial samples from arm $1$.
The theorem requires parameters $\eps>0$ and $N_0$ to satisfy some condition that depends only on the prior. Such parameters exist if and only if $\Pr[G>0]>0$ for some $N_0$ (for precisely the same reason as in Remark~\ref{BIC:rem:HP-works}). The latter condition is in fact necessary, as we will see in Section~\ref{BIC:sec:assn}.
\end{remark}

Performance guarantees for \RepeatedHE completely separated from the BIC guarantee, in terms of results as well as proofs. Essentially, \RepeatedHE learns at least as fast as an appropriately slowed-down version of \ALG. There are several natural ways to formalize this, in line with the standard performance measures for multi-armed bandits. For notation,
let $\REW^{\ALG}(n)$ be the total reward of \ALG in the first $n$ rounds of its execution, and let $\BReg^{\ALG}(n)$ be the corresponding Bayesian regret.

\begin{theorem}\label{BIC:thm:reduction-perf}
Consider \RepeatedHE with exploration probability $\eps>0$ and $N_0$ initial samples.
 Let $N$ be the number of exploration rounds $t>N_0$. %
\footnote{Note that $\E[N] = \eps(T-N_0)$, and
    $|N-\E[N]|\leq O(\sqrt{T\,\log T})$
with high probability.}
Then:

\begin{itemize}
\item[(a)] If $\ALG$ always chooses arm $2$,  \RepeatedHE chooses arm $2$ at least N times.

\item[(b)] The expected reward of  \RepeatedHE is at least
        $\tfrac{1}{\eps}\, \E\sbr{ \REW^{\ALG}(N) }$.

\item[(c)] Bayesian regret of  \RepeatedHE is
        $ \BReg(T) \leq N_0+\tfrac{1}{\eps}\, \E\sbr{ \BReg^{\ALG}(N) }$.

\end{itemize}
\end{theorem}

\begin{proof}[Proof Sketch]
Part (a) is obvious. Part (c) trivially follows from part (b). The proof of part (b) invokes Wald's identify and the fact that the expected reward in ``exploitation" is at least as large as in ``exploration" for the same round.
\end{proof}



\begin{remark}\label{BIC:rem:price-of-BIC}
We match the Bayesian regret of \ALG up to factors $N_0,\,\tfrac{1}{\eps}$, which depend only on the prior $\mP$ (and not on the time horizon or the realization of the mean rewards). In particular, we can achieve $\tildeO(\sqrt{T})$ regret for all problem instances, \eg using algorithm \UcbOne from Chapter~\ref{ch:IID}. If a smaller regret rate $f(T) = o(T)$ is achievable for a given prior using some other algorithm \ALG, we can match it, too. However, the prior-dependent factors can be arbitrarily large, depending on the prior.
\end{remark}

\section{A necessary and sufficient assumption on the prior}
\label{BIC:sec:assn}

We need to restrict the prior $\mP$ so as to give the algorithm a fighting chance to convince some agents to try arm $2$. (Recall that $\mu_1^0\geq \mu_2^0$.) Otherwise the problem is just hopeless.
For example, if $\mu_1$ and $\mu_1-\mu_2$ are independent, then samples from arm $1$ have no bearing on the conditional expectation of $\mu_1-\mu_2$, and therefore cannot possibly incentivize any agent to try arm $2$.

We posit that arm $2$ \emph{can} appear better after seeing sufficiently many samples of arm $1$. Formally, we consider the posterior gap given $n$ samples from arm $1$:
\begin{align}\label{BIC:eq:gap1}
    G_{1,n} := \E[\, \mu_2-\mu_1 \mid \samples{n} ],
\end{align}
where $\samples{n}$ denotes an ordered tuple of $n$ independent samples from arm $1$.
We focus on the property that this random variable can be positive:
\begin{align}\label{BIC:eq:prop}
\Pr\left[ G_{1,n} > 0 \right]>0 \quad
\text{for some prior-dependent constant $n=n_\mP<\infty$}.
\end{align}
For independent priors, this property can be simplified to
    $\Pr[\mu_2^0>\mu_1]>0$.
Essentially, this is because
    $G_{1,n} = \mu_2^0 - \E[\mu_1 \mid \samples{n}]$
converges to $\mu_2^0 - \mu_1$ as $n\to\infty$.

Recall that Property~\eqref{BIC:eq:prop} is sufficient for \RepeatedHE, as per Remark~\ref{BIC:rem:repeatedHP-condition}. We prove that it is necessary for BIC bandit algorithms.


\begin{theorem}\label{BIC:thm:LB}
Suppose ties in Definition~\ref{BIC:def:BIC} are always resolved in favor of arm $1$ (\ie imply $\rec_t = 1$). Absent \eqref{BIC:eq:prop}, any BIC algorithm never plays arm $2$.
\end{theorem}

\begin{proof}
Suppose \propref{BIC:eq:prop} does not hold. Let $\ALG$ be a strongly BIC algorithm.
We prove by induction on $t$ that $\ALG$ cannot recommend arm $2$ to agent $t$.

This is trivially true for $t=1$. Suppose the induction hypothesis is true for some $t$. Then the decision whether to recommend arm $2$ in round $t+1$ (\ie whether $a_{t+1}=2$) is determined by the first $t$ outcomes of arm $1$ and the algorithm's random seed. Letting $U=\{a_{t+1}=2\}$, we have
\begin{align*}
\E[\mu_2-\mu_1 \mid U]
    &= \E\sbr{\; \E[\mu_2-\mu_1 \mid \samples{t}]\;\; \mid U }
        &\EqComment{by Fact~\ref{BIC:fact:iterated-expectations}}\\
    &= \E[G_{1,t}| U]
        &\EqComment{by definition of $G_{1,t}$}\\
    &\leq 0
        &\EqComment{since \eqref{BIC:eq:prop} does not hold}.
\end{align*}
The last inequality holds because the negation of \eqref{BIC:eq:prop} implies
    $\Pr[G_{1,t}\leq 0]=1$.
This contradicts \ALG being BIC, and completes the induction proof.
\end{proof}

\sectionBibNotes[: incentivized exploration]
\label{BIC:sec:further}


The study of incentivized exploration has been initiated in \citet*{Kremer-JPE14} and \citet*{Che-13}. The model in this chapter was introduced in \citet{Kremer-JPE14},
 and studied under several names, \eg ``BIC bandit exploration" in \citet{ICexploration-ec15} and ``Bayesian Exploration" in \citet{ICexplorationGames-ec16}. All results in this chapter are from \citet*{ICexploration-ec15}, specialized to $K=2$, with slightly simplified algorithms and a substantially simplified presentation. While \citet{ICexploration-ec15} used a version of \HiddenExploration as a common technique in several results, we identify it as an explicit ``building block"  with standalone guarantees, and use it as a subroutine in the algorithm and a lemma in the overall analysis. The version in \citet{ICexploration-ec15} runs in phases of fixed duration, which consist of exploitation rounds with a few exploration rounds inserted uniformly at random.

Incentivized exploration is connected to theoretical economics in three different ways. First, it adopts the BIC paradigm, as per Remark~\ref{BIC:rem:BIC-paradigm}. Second, the game between the principal and a single agent in our model has been studied, under the name \emph{Bayesian Persuasion}, in a long line of work starting from
\citet{Kamenica-aer11}, see \citet{Kamenica-survey19} for a survey.
This is an idealized model for many real-life scenarios in which a more informed ``principal" wishes to persuade the ``agent" to take an action which benefits the principal. A broader theme here is the design of ``information structures": signals received by players in a game \citep{BergemannMorris-survey19}. A survey \citep{IncentivizedExploration-chapter} elucidates the connection between this work and incentivized exploration. Third, the field of \emph{social learning} studies self-interested agents that jointly learn over time in a shared environment \citep{Golub-survey16}. In particular, \emph{strategic experimentation}
studies models similar to incentivized exploration, but without a coordinator \citep{Horner-survey16}.

The basic model defined in this chapter was studied, and largely resolved, in \citep{Kremer-JPE14,ICexploration-ec15,ICexplorationGames-ec16,Cohen-Mansour-ec19,Selke-PoIE-ec21}.
While very idealized, this model is very rich, leading to a variety of results and algorithms. Results extend to $K>2$, and come in several ``flavors" other than Bayesian regret: to wit,
optimal policies for deterministic rewards,
regret bounds for all realizations of the prior,
and (sample complexity of) exploring all arms that can be explored.
The basic model has been extended in various ways
\citep{ICexploration-ec15,ICexplorationGames-ec16,Bahar-ec16,Bahar-ec19,Jieming-multitypes-www19,Jieming-unbiased18,IncentivizedRL}.
Generally, the model can be made more realistic in three broad directions:
generalize the \emph{exploration} problem (in all ways that one can generalize multi-armed bandits),
generalize the \emph{persuasion} problem (in all ways that one can generalize Bayesian persuasion),
and relax the standard (yet strong) assumptions about agents' behavior.


Several papers start with a similar motivation, but adopt substantially different technical models:
time-discounted rewards \citep{Bimpikis-exploration-ms17};
continuous information flow and a continuum of agents \citep{Che-13};
incentivizing exploration using money \citep{Frazier-ec14,Kempe-colt18};
incentivizing the agents to ``participate" even if they knew as much as the algorithm
\citep{FiduciaryBandits-arxiv19};
not expecting the agents to comply with recommendations, and instead treating recommendations as ``instrumental variables" in statistics \citep{Kallus-alt18,Ngo-icml21}.

Similar issues, albeit with much higher stakes, arise in medical decisions: selecting a doctor or a hospital, choosing a drug or a treatment, or deciding whether to participate in a medical trial. An individual can consult information from similar individuals in the past, to the extent that such information is available, and later he can contribute his experience as a review or as an outcome in a medical trial. A detailed discussion of the connection to medical trials can be found in \citet{ICexploration-ec15}.


In what follows, we spell out the results on the basic model, and briefly survey the extensions.


\xhdr{Diverse results in the basic model.}
The special case of deterministic rewards and two arms has been optimally solved in the original paper of \citet{Kremer-JPE14}. That is, they design a BIC algorithm which exactly optimizes the expected reward, for a given Bayesian prior, among all BIC algorithms. This result has been extended to $K>2$ arms in \citet{Cohen-Mansour-ec19}, under additional  assumptions.

\RepeatedHE comes with no guarantees on pseudo-regret for each realization of the prior. \citet{Kremer-JPE14} and \citet{ICexploration-ec15} provide such guarantees, via different algorithms: \citet{Kremer-JPE14} achieve $\tildeO(T^{2/3})$ regret, and \citet{ICexploration-ec15} achieve regret bounds with a near-optimal dependence on $T$, both in the $\tildeO(\sqrt{T})$ worst-case sense and in the $O(\log T)$ instance-dependent sense. Both algorithms suffer from prior-dependent factors similar to those in Remark~\ref{BIC:rem:price-of-BIC}.

\RepeatedHE and the optimal pseudo-regret algorithm from \citet{ICexploration-ec15} extend to $K>2$ arms under independent priors. \RepeatedHE also works for correlated priors, under a version of assumption~\eqref{BIC:eq:prop}; however, it is unclear whether this assumption is necessary. Both algorithms require a \emph{warm-start}: some pre-determined number of samples from each arm. Regret bounds for both algorithms suffer exponential dependence on $K$ in the worst case. Very recently, \citet{Selke-PoIE-ec21} improved this dependence to $\poly(K)$ for Bayesian regret and  independent priors (more on this below).

While all these algorithms are heavily tailored to incentivized exploration, \citet{Selke-PoIE-ec21} revisit Thompson Sampling, the Bayesian bandit algorithm from Chapter~\ref{ch:TS}. They prove that Thompson Sampling is BIC for independent priors (and any $K$), when initialized with prior $\mP$ and a sufficient warm-start. If  prior mean rewards are the same for all arms, then Thompson Sampling is BIC even without the warm-start. Recall that Thompson Sampling achieves $\tildeO(\sqrt{KT})$ Bayesian regret for any prior \citep{Russo-MathOR-14}. It is unclear whether other ``organic" bandit algorithms such as \UcbOne can be proved BIC under similar assumptions.

Call an arm \emph{explorable} if it can be explored with some positive probability by some BIC algorithm. In general, not all arms are explorable. \citet{ICexplorationGames-ec16} design an algorithm which explores all explorable arms, and achieves regret $O(\log T)$ relative to the best explorable arm (albeit with a very large instance-dependent constant). Interestingly, the set of all explorable arms is not determined in advance: instead, observing a particular realization of one arm may ``unlock" the possibility of exploring another arm. In contrast, for independent priors explorability is completely determined by the arms' priors, and admits a simple characterization \citep{Selke-PoIE-ec21}: each arm $a$ is explorable if and only if
\begin{align}\label{BIC:eq:pwise-cond}
\Pr\sbr{ \mu_{a'}< \mu_a^0 }>0
\quad \text{for all arms $a'\neq a$}.
\end{align}
All explorable arms can be explored, \eg via the $K$-arms extension of \RepeatedHE mentioned above.

\newpage
\xhdr{Sample complexity.}
How many rounds are needed to sample each explorable arm even once? This is arguably the most basic objective in incentivized exploration, call it \emph{sample complexity}. While \citet{ICexploration-ec15,ICexplorationGames-ec16} give rather crude upper bounds  for correlated priors, \citet{Selke-PoIE-ec21} obtain tighter results for independent priors. Without loss of generality, one can assume that all arms are explorable, \ie focus on the arms which satisfy \eqref{BIC:eq:pwise-cond}. If all per-arm priors belong to some collection $\mC$, one can cleanly decouple the dependence on the number of arms $K$ and the dependence on $\mC$. We are interested in the \emph{$\mC$-optimal} sample complexity: optimal sample complexity in the worst case over the choice of per-arm priors from $\mC$. The dependence on $\mC$ is driven by the smallest variance
    $\sigmin^2(\mC) = \inf_{\mP\in\mC} \text{Var}(\mP)$.
The key issue is whether the dependence on $K$ and $\sigmin(\mC)$ is polynomial or exponential; \eg the sample complexity obtained via an extension of the \RepeatedHE can be exponential in both.

\citet{Selke-PoIE-ec21} provide a new algorithm for sampling each arm. Compared to \RepeatedHE, it inserts a third ``branch" which combines exploration and exploitation, and allows the exploration probability to increase over time. This algorithm is \emph{polynomially optimal} in the following sense: there is an upper bound $U$ on its sample complexity and a lower bound $L$ on the sample complexity of any BIC algorithm such that
    $U<\poly\rbr{L/\sigmin(\mC)}$.
This result achieves polynomial dependence on $K$ and $\sigmin(\mC)$ whenever such dependence is possible, and allows for several refinements detailed below.

The dependence on $K$ admits a very sharp separation: essentially, it is either linear or at least exponential, depending on the collection $\mC$ of feasible per-arm priors. In particular, if $\mC$ is finite then one compares
\begin{align}\label{BIC:eq:lit-IE-mC}
 \minsupp(\mC) := \min_{\mP\in\mC} \sup(\text{support}(\mP))
\quad\text{and}\quad
    \Phi_\mC := \max_{\mP\in\mC} \E[\mP].
\end{align}
The $\mC$-optimal sample complexity is $O_\mC(K)$ if
    $\minsupp(\mC)>\Phi_\mC$,
and
    $\exp\rbr{\Omega_\mC(K)}$
if
    $\minsupp(\mC)<\Phi_\mC$.
The former regime is arguably quite typical, \eg it holds in the realistic scenario when all per-arm priors have full support $[0,1]$, so that $\minsupp(\mC)=1>\Phi_\mC$.

The $\mC$-optimal sample complexity is exponential in $\sigmin(\mC)$, for two canonical special cases: all per-arm priors are, resp., Beta distributions and truncated Gaussians. For the latter case, all per-arm priors are assumed to be Gaussian with the same variance $\sigma^2\leq 1$, conditioned to lie in $[0,1]$. For Beta priors, different arms may have different variances. Given a problem instance, the optimal sample complexity is exponential in the \emph{second}-smallest variance, but only polynomial in the smallest variance. This is important when one arm represents a well-known, default alternative, whereas all other arms are new to the agents.




\xhdr{The price of incentives.}
What is the penalty in performance incurred for the sake of the BIC property? We broadly refer to such penalties as the \emph{price of incentives} (\PoI). The precise definition is tricky to pin down, as the \PoI can be multiplicative or additive, can be expressed via different performance measures, and may depend on the comparator benchmark. Here's one version for concreteness: given a BIC algorithm $\mA$ and a bandit algorithm $\mA^*$ that we wish to compare against, suppose
    $\BReg_{\mA}(T)  = c_{\term{mult}}\cdot \BReg_{\mA^*}(T)  + c_{\term{add}}$,
where $\BReg_{\mA}(T)$ is Bayesian regret of algorithm $\mA$. Then $c_{\term{mult}}$, $c_{\term{add}}$ are, resp., multiplicative and additive \PoI.

Let us elucidate the \PoI for independent priors, using the results in \citet{Selke-PoIE-ec21}. Since Thompson Sampling is BIC with a warm-start (and, arguably, a reasonable benchmark to compare against), the \PoI is only additive, arising  due to collecting data for the warm-start. The \PoI is upper-bounded by the sample complexity of collecting this data. The sufficient number of data points per arm, denote it $N$, is $N=O(K)$ under very mild assumptions, and even $N=O(\log K)$ for Beta priors with bounded variance. We retain all polynomial sample complexity results described above, in terms of $K$ and $\minsupp(\mC)$. In particular, the \PoI is $O_{\mC}(K)$ if
    $\minsupp(\mC)>\Phi_\mC$,
in the notation from \eqref{BIC:eq:lit-IE-mC}.

Alternatively, the initial data points may be collected exogenously, \eg purchased at a fixed price per data point (then the \PoI is simply the total payment). The $N=O(\log K)$ scaling is particularly appealing if each arm represents a self-interested party, \eg a restaurant, which wishes to be advertised on the platform. Then each arm can be asked to pay a small, $O(\log K)$-sized entry fee to subsidise the initial samples.


\xhdr{Extensions of the basic model.}
Several extensions generalize the exploration problem, \ie the problem faced by an algorithm (even) without the BIC constraint. \RepeatedHE allows the algorithm to receive auxiliary feedback after  each round, \eg as in combinatorial semi-bandits. This auxiliary feedback is then included in the signal $\mS_t$. Moreover, \citet{ICexploration-ec15} extend \RepeatedHE to contextual bandits, under a suitable assumption on the prior which makes all context-arm pairs explorable. \citet{Jieming-multitypes-www19} study an extension to contextual bandits without any assumptions, and explore all context-arm pairs that are explorable. \citet{IncentivizedRL} consider incentivized exploration in reinforcement learning.

Other extensions generalize the \emph{persuasion} problem in incentivized exploration.

\begin{itemize}

\item \emph{(Not) knowing the prior:}
While \RepeatedHE requires the full knowledge of the prior in order to perform the Bayesian update, the principal is not likely to have such knowledge in practice. To mitigate this issue, one of the algorithms in \citet{ICexploration-ec15} does not input the prior, and instead only requires its parameters (which are similar to $\eps,N_0$ in \RepeatedHE) to be consistent with it. In fact, agents can  have different beliefs, as long as they are consistent with the algorithm's parameters.

\item \emph{Misaligned incentives:}
The principal's incentives can be misaligned with the agents': \eg a vendor may favor more expensive products, and a hospital running a free medical trial may prefer less expensive treatments. Formally, the principal may receive its own, separate rewards for the chosen actions. \RepeatedHE is oblivious to the principal's incentives, so \ALG can be a bandits-with-predictions algorithm (see Section~\ref{sec:IID-further}) that learns the best action for the principal. The algorithm in \citet{ICexplorationGames-ec16} (which explores all explorable actions) can also optimize for the principal.

\item \emph{Heterogenous agents:}
Each agent has a \emph{type} which determines her reward distributions and her prior. Extensions to contextual bandits, as discussed above, correspond  to \emph{public types} (\ie observed by the principal). \citet{Jieming-multitypes-www19} also investigate \emph{private types}, the other standard variant when the types are \emph{not} observed. Their algorithm offers \emph{menus} which map types to arms, incentivizes each agent to follow the offered menu, and explores all ``explorable" menus.

\item \emph{Multiple agents in each round:}
Multiple agents may arrive simultaneously and directly affect one another \citep{ICexplorationGames-ec16}. \Eg drivers that choose to follow a particular route at the same time may create congestion, which affects everyone. In each round, the principal chooses a distribution $D$ over joint actions, samples a joint action from $D$, and recommends it to the agents. The BIC constraint requires $D$ to be the \emph{Bayes Correlated Equilibrium} \citep{BS13}.

\item \emph{Beyond ``minimal revelation":}
What if the agents observe some aspects of the history, even if the principal does not wish them to? In \citet{Bahar-ec16}, the agents observe recommendations to their friends on a social networks (but not the corresponding rewards). In \citet{Bahar-ec19}, each agent observes the action and the reward of the previous agent. Such additional information skews the agents further towards exploitation, and makes the problem much more challenging. Both papers focus on the case of two arms, and assume, resp., deterministic rewards or one known arm.
\end{itemize}

All results in \citet{ICexplorationGames-ec16,Jieming-multitypes-www19,IncentivizedRL} follow the perspective of exploring all explorable ``pieces". The ``pieces" being explored range from actions to joint actions to context-arm pairs to menus to policies,
depending on a particular extension.

\xhdr{Behavioral assumptions.}
All results discussed above rely heavily on standard but very idealized assumptions about agents' behavior. First, the principal announces his algorithm, the agents know and understand it, and trust the principal to faithfully implement it as announced.
\footnote{In theoretical economics, these assumptions are summarily called the (principal's) \emph{power to commit}.}
 Second, the agents either trust the BIC property of the algorithm, or can verify it independently. Third, the agents act rationally, \ie choose arms that maximize their expected utility (\eg they don't favor less risky arms, and don't occasionally choose less preferable actions). Fourth, the agents find it acceptable that they are given recommendations without any supporting information, and that they may be singled out for low-probability exploration.

One way forward is to define a particular class of algorithms and a model of agents' behavior that is (more) plausible for this class. To this end, \citet{Jieming-unbiased18} consider algorithms which reveal some of the history to each agent, and allow a flexible model of greedy-like response. To justify such response, the sub-history revealed to each agent $t$ consists of all rounds that precede $t$ in some fixed and pre-determined partial order. Consequently, each agent observes the history of all agents that could possibly affect her. Put differently, each agent only interacts with a full-revelation algorithm, and the behavioral model does not need to specify how the agents interact with any algorithms that actually explore. \citet{Jieming-unbiased18} design an algorithm in this framework which matches the state-of-art regret bounds.

\xhdr{The greedy algorithm.}
Theorem~\ref{BIC:thm:BG} on the Bayesian-greedy algorithm and its corollaries are from
\citet{BSL-myopic23}.%
\footnote{\citet{BSL-myopic23} attribute Theorem~\ref{BIC:thm:BG} to \citet{GreedyFails19}.}
A similar result holds for $K>2$ arms, albeit with a somewhat more complex formulation. While it has been understood for decades that exploitation-only bandit algorithms fail badly for some special cases, Theorem~\ref{BIC:thm:BG} is the first non-trivial general result for stochastic rewards that we are aware of. Characterizing the learning performance of Bayesian-greedy more precisely is an open question, even for $K=2$ arms and independent priors, and especially if $\E[\mu_1]$ is close to $\E[\mu_2]$. This concerns both the probability of never choosing arm $2$ and Bayesian regret. The latter could be a more complex issue, because Bayesian regret can be accumulated even when arm $2$ is chosen.

The frequentist version of the greedy algorithm replaces posterior mean with empirical mean: in each round, it chooses an arm with a largest empirical mean reward. Initially, each arm is tried $N_0$ times, for some small constant $N_0$ (\emph{warm-up}). Focusing on $K=2$ arms, we have a similar learning failure like in Theorem~\ref{BIC:thm:BG}: the good arm is never chosen again. A trivial argument yields failure probability  $e^{-\Omega(N_0)}$: consider the event when the good arm receives $0$ reward in all warm-up rounds, and the bad arm receives a non-zero reward in some warm-up round. 
However, this trivial guarantee is rather weak because of the exponential dependence on $N_0$. A similar but exponentially stronger guarantee is proved in \citet{BSL-myopic23}, with failure probability on the order of $1/\sqrt{N_0}$. This result is extended to a broader class of agent behaviors (including, \eg mild optimism and pessimism), and to $K>2$ arms.

Nevertheless, \citet{Bayati-nips20,Jedor-Greedy21} prove non-trivial (but suboptimal) regret bounds for the frequentist-greedy algorithm on problem instances with a very large number of near-optimal arms. Particularly, 
for Bayesian bandits with $K\geq\sqrt{T}$ arms, where the arms' mean rewards are sampled independently and uniformly.

Several papers find that the greedy algorithm (equivalently, incentivized exploration with full disclosure) performs well in theory under substantial assumptions on heterogeneity of the agents and structure of the rewards.  \citet{bastani2017exploiting,kannan2018smoothed,Greedy-Manish-18} consider linear contextual bandits, where the contexts come from a sufficiently ``diffuse" distribution (see Section~\ref{CB:sec:further} for a more detailed discussion). \citet{Sven-aistats18} assume additive agent-specific shift in expected reward of each action, and posit that each agent knows her shift and removes it from the reported reward.



\OMIT{ 
\begin{itemize}
\item \emph{Agents control the arms:} Algorithm is a recommendation system that recommends, say, restaurants to users and learns from their feedback. Agents are the users who decide which restaurants to go to.

\item \emph{Agents control auxiliary inputs:} consider an \emph{ad auction}: an auction for allocating ads across available slots on webpages. Agents are advertisers who place bids in the auction. Algorithm adjusts the ad allocation over time (either directly or by adjusting some parameters), using bids as auxiliary inputs.

\item \emph{Agents control the outcomes}:  Algorithm is a seller who adjusts the offered prices over time; agents are the customers who choose which items to buy at these prices. Likewise, algorithm is a contractor who hires workers on a crowdsourcing market, and adjusts the the offered contract over time. Agents are the workers who choose whether to accept the contract and how much effort to put in.
\end{itemize}
} 

\newpage
\sectionBibNotes[: other work on bandits and agents]
\label{BIC:sec:further-other}

Bandit algorithms interact with self-interested agents in a number of applications. The technical models vary widely, depending on how the agents come into the picture.  We partition this literature based on what the agents actually choose: which arm to pull (in incentivized exploration), which bid to report (in an auction), how to respond to an offer (in contract design), or which bandit algorithm to interact with. While this chapter focuses on incentivized exploration, let us now survey the other three lines of work.

\subsection{Repeated auctions: agents choose bids}

Consider an idealized but fairly generic repeated auction, where in each round the auctioneer allocates one item to one of the auction  participants (\emph{agents}):

\begin{BoxedProblem}{Repeated auction}
In each round $t = 1,2,3 \LDOTS T $:
\begin{OneLiners}
  \item[1.] Each agent submits a message (\emph{bid}).
  \item[2.] Auctioneer's ``allocation algorithm" chooses an agent and allocates one item to this agent.
  \item[3.] The agent's reward is realized and observed by the algorithm and/or the agent.
  \item[4.] Auctioneer assigns payments.
\end{OneLiners}

\end{BoxedProblem}

The auction should incentivize each agent to submit ``truthful bids" representing his current knowledge/beliefs about the rewards.%
\footnote{The technical definition of truthful bidding differs from one model to another. These details tend to be very lengthy, \eg compared to a typical setup in multi-armed bandits, and often require background in theoretical economics to appreciate. We keep our exposition at a more basic level.}
Agents' rewards may may not be directly observable by the auctioneer, but they can be derived (in some formal sense) from the auctioneer's observations and the agents' truthful bids. The auctioneer has one of the two standard objectives: \emph{social welfare} and \emph{revenue} from the agents' payments. Social welfare is the total ``utility" of the agents and the auctioneer, \ie since the payments cancel out, the total agents' reward. Thus, the allocation algorithm can be implemented as a bandit algorithm, where agents correspond to ``arms", and algorithm's reward is either the agent's reward or the auctioneer's revenue, depending on the auctioneer's objective.

A typical motivation is \emph{ad auctions}: auctions which allocate advertisement opportunities on the web among the competing advertisers. Hence, one ad opportunity is allocated in each round of the above model, and agents correspond to the advertisers. The auctioneer is a website or an \emph{ad platform}: an intermediary which connects advertisers and websites with ad opportunities.

The model of \emph{dynamic auctions}~\citep{DynPivot-econometrica10,AtheySegal-econometrica13,Segal-dynamic11,Kakade-pivot-or13,DynAuctions-survey11}
posits that the agents do not initially know much about their future rewards. Instead, each agent learns over time by observing his realized rewards when and if he is selected by the algorithm. Further, the auctioneer does not observe the rewards, and instead needs to rely on agents' bids. The auctioneer needs to create the right incentives for the agents to stay in the auction and bid their posterior mean rewards. This line of work has been an important development in theoretical economics. It is probably the first appearance of the ``bandits with incentives" theme in the literature (going by the working papers). \citet{NSV08} consider a similar but technically different model in which the agents are incentivized to report their realized rewards. They create incentives in the ``asymptotically approximate" sense: essentially, truthful bidding is at least as good as any alternative, minus a regret term.

A line of work from algorithmic economics literature \citep{MechMAB-ec09,DevanurK09,Transform-ec10-conf,BKS2-ec13,Transform-ec10-jacm,SingleCall-ec12,Gatti-ec12}
 considers a simpler model, specifically designed to showcase the interaction of bandits and auctions aspects. They consider \emph{pay-per-click} ad auctions, where advertisers derive value only when users click on their ads, and are charged per click. Click probabilities are largely unknown, which gives rise to a bandit problem. Bids correspond to agents' per-click values, which are assumed to be fixed over time. Only the initial bids are considered, all payments are assigned after the last round, and the allocation proceeds over time as a bandit algorithm. Combination of bandit feedback and truthfulness brings about an interesting issue: while it is well-known what the payments must be to ensure truthfulness, the algorithm might not have enough information to compute them. For this reason, Explore-First is essentially the only possible deterministic algorithm \cite{MechMAB-ec09,DevanurK09}. Yet, the required payments can be achieved in expectation by a randomized algorithm. Furthermore, a simple randomized transformation can turn any bandit algorithm into a truthful algorithm, with only a small loss in rewards, as long as the original algorithm satisfies a well-known necessary condition
\cite{Transform-ec10-conf,BKS2-ec13,Transform-ec10-jacm}.

Another line of work \citep[\eg][]{RepeatedAuctions-soda13,OnlineLearningAuctions-focs17}
concerns tuning the auction over time, \ie adjusting its parameters such as a reserve price. A fresh batch of agents is assumed to arrive in each round (and when and if the same agent arrives more than once, she only cares about the current round). This can be modeled as a contextual problem where ``contexts" are bids, ``arms" are the different choices for allocations, and ``policies" (mappings from arms to actions) correspond to the different parameter choices. Alternatively, this can be modeled as a non-contextual problem, where the ``arms" are the parameter choices.

\subsection{Contract design: agents (only) affect rewards}

In a variety of settings, the algorithm specifies ``contracts" for the agents, \ie rules which map agents' performance to outcomes. Agents choose their responses to the offered contracts, and the said responses affect algorithm's rewards. The nature of the contracts and the responses depends on a particular application.

Most work in this direction posits that the contracts are adjusted over time. In each round, a new agent arrives, the algorithm chooses a ``contract", the agent responds, and the reward is revealed. Agents' incentives typically impose some structure on the rewards that is useful for the algorithm.
In \emph{dynamic pricing}
\citep[][a survey; see also Section~\ref{BwK:sec:further-DP}]{Boer-survey15}
and \emph{dynamic assortment} \citep[\eg][]{Zeevi-assortment-13,Shipra-ec16} the algorithm offers some items for sale, a contract specifies, resp., the price(s) and the offered assortment of items, and the agents decide whether and which products to buy. (This is a vast and active research area; a more detailed survey is beyond our scope.) In \emph{dynamic procurement} \citep[\eg][]{DynProcurement-ec12,BwK-focs13,Krause-www13}
the algorithm is a buyer and the agents are sellers; alternatively, the algorithm is a contractor on a crowdsourcing market and the agents are workers. The contracts specify the payment(s) for the completed tasks, and the agents decide whether and which tasks to complete. \citet{RepeatedPA-ec14} study a more general model in which the agents are workers who choose their effort level (which is not directly observed by the algorithm), and the contracts specify payment for each quality level of the completed work. One round of this model is a well-known \emph{principal-agent model} from contract theory \citep{LM02}.

In some other papers, the entire algorithm is a contract, from the agents' perspective. In \citet{Ghosh-itcs13}, the agents choose how much effort to put into writing a review, and then a bandit algorithm chooses among relevant reviews, based on user feedback such as ``likes". The effort level affects the ``rewards" received for the corresponding review. In \citet{Jieming-colt19} the agents collect rewards directly, unobserved by the algorithm, and pass some of these rewards to the algorithm. The algorithm chooses among agents based on the observed ``kick-backs", and therefore incentivizes the agents.

A growing line of work studies bandit algorithms for dynamic pricing which can interact with the same agent multiple times. Agents' self-interested behavior is typically restricted. One typical assumption is that they are more myopic compared to the algorithm, placing less value on the rewards in the far future
\citep[\eg][]{Amin-auctions-nips13,Amin-auctions-nips14}.
Alternatively, the agents also learn over time, using a low-regret online learning algorithm
\citep[\eg][]{Hoda-icml16,Jieming-ec18}.

\subsection{Agents choose between bandit algorithms}

Businesses that can deploy bandit algorithms (\eg web search engines, recommendation systems, or online retailers) often compete with one another. Users can choose which of the competitors to go to, and hence which of the bandit algorithms to interact with. Thus, we have bandit algorithms that compete for users. The said users bring not only utility (such as revenue and/or market share), but also new data to learn from. This leads to a three-way tradeoff between exploration, exploitation, and competition.

\citet{CompetingBandits-itcs18,CompetingBandits-ec19,CompetingBandits-merged} consider bandit algorithms that optimize a product over time and compete on product quality. They investigate whether competition incentivizes the adoption of better bandit algorithms, and how these incentives depend on the intensity of the competition. In particular, exploration may hurt algorithm's performance and reputation in the near term, with adverse competitive effects. An algorithm may even enter a ``death spiral", when the short-term reputation cost decreases the number of users for the algorithm to learn from, which degrades the system's performance relative to competition and further decreases the market share. These issues are related to  the relationship between competition and innovation, a well-studied topic in economics \citep{Schumpeter-42,Aghion-QJE05}.

\citet{bergemann1997market,bergemann2000experimentation,keller2003price} target a very different scenario when the competing firms experiment with \emph{prices} rather than design alternatives. All three papers consider strategies that respond to competition and analyze Markov-perfect equilibria in the resulting game.

\sectionExercises
\label{BIC:sec:ex}

%
%
%

\begin{exercise}[Bayesian-Greedy fails]\label{BIC:ex:BG}
Prove Corollary~\ref{BIC:cor:BG}.

\Hint{Consider the event
    $\{ \mE \eqAND \mu_1<1-2\,\alpha  \}$;
this event is independent with $\mu_2$.}
\end{exercise}

\OMIT{ 
\begin{exercise}[Frequentist-Greedy fails]\label{BIC:ex:FG}
Consider an instance of multi-armed bandits with two arms and mean rewards $\mu_1,\mu_2$. Consider the frequentist-greedy algorithm: in each round, and arm with the largest mean reward is chosen (breaking ties arbitrarily). There is an initial exploration phase, where each arm is tried $N_0$ times, for some fixed $N_0$.  Assume that
    $\tfrac14 +\Delta \leq \mu_1<\mu_2 \leq \tfrac34$,
where $\Delta = \mu_2-\mu_1$, and that
    $N_0 < \Delta^{-2}$.
Prove that the algorithm fails, in the sense that arm $2$ is never tried after the initial exploration, with probability at least $\Omega(\Delta)$.

\TakeAway{This guarantee is a big improvement over the trivial fact that the algorithm fails with probability at least $e^{-\Omega(N_0)}$, simply because arm $2$ can collect $0$ reward in each round of initial exploration. }

\Hint{Prove that after the initial exploration, it holds with absolute-constant probability that the empirical mean reward of arm $2$ is strictly smaller than $\mu_1-\Delta$. Apply Doob's martingale inequality to deduce that, with probability at least $\Omega(\Delta)$, the empirical mean reward of arm $1$ is always larger than $\mu_1-\Delta$.

Doob's martingale inequality asserts that  for any martingale
    $(F_n\geq 0:\;n\in\N)$
and any $C>0$ we have
\begin{align}\label{BIC:eq:ex:Doob}
 \Pr[F_n <C\quad\forall n] >1-\E[F_0]/C.
\end{align}
}

\end{exercise}
} 

\begin{exercise}[performance of \RepeatedHE]\label{BIC:ex:reduction-perf}
Prove Theorem~\ref{BIC:thm:reduction-perf}(b).

\Hint{The proof relies on Wald's identity. Let $t_i$ be the $i$-th round $t>N_0$ in which the ``exploration branch" has been chosen. Let $u_i$ be the expected reward of \ALG in the $i$-th round of its execution. Let $X_i$ be the expected reward in the time interval $[t_i, t_{i+1})$. Use Wald's identity to prove that $\E[X_i] = \tfrac{1}{\eps}\, u_i$. Use Wald's identity again (in a version for non-IID random variables) to prove that
    $\E[\sum_{i=1}^N X_i] = \tfrac{1}{\eps}\, \E[\sum_{i=1}^N u_i]  $.
Observe that the right-hand side is simply $\tfrac{1}{\eps}\, \E[U(N)]$.}
\end{exercise}

\OMIT{ 
The proof template is as follows:
\begin{itemize}
\item[(i)] A particular ``bad" realization of the mean rewards $(\mu_1,\mu_2)$ happens with constant probability. Formally, for some constants $\muR,\alpha,\alpha'$, the prior assigns a positive probability to the event
\begin{align}\label{BIC:eq:BG-template-1}
\mu_2^0
    <\muR
    \leq \mu_1
    \leq \alpha
    < \alpha'
    \leq \mu_2.
\end{align}

\item[(ii)] Arm $1$ brings high average rewards: with constant probability, this happens for any prefix of $n$ samples. Formally, for any constants $\beta<\muR\leq 1$ the average reward $\AvgR{n}$ satisfies
\begin{align}\label{BIC:eq:BG-template-2}
\Pr\left[\; \AvgR{n} \geq \beta \;\;\forall n\in [T]
    \mid \mu_1 = \muR
\;\right]\geq c>0,
\end{align}
where $c$ depends only on the difference $\muR-\beta$, but not on $T$.

\item[(iii)] Connect average rewards to posterior means: if the average reward of arm $1$ is sufficiently large, then the posterior mean reward exceeds $\mu_2^0$. Formally,
\begin{align}\label{BIC:eq:BG-template-3}
\forall n\in [T]:\qquad
    \AvgR{n}\geq \beta \text{~~~implies~~~} \E[ \mu_1 \mid \samples{n} ] > \mu_2^0,
\end{align}
where $\beta$ is an appropriately chosen constant (same for all $n$).

\end{itemize}

Putting this together, with constant probability the following ``bad event" happens: arm $1$ is worse by a constant margin, yet in each round its average reward exceeds some appropriately chosen parameter $\beta$, and therefore its posterior mean reward exceeds $\mu_2^0$. Consequently, under the bad event, the algorithm always
chooses arm $1$, incurring constant regret in each round.

We flesh out this proof template in Exercise~\ref{BIC:ex:BG}, focusing on the special case of Beta-Bernoulli priors. Step (i) holds broadly, \eg for any independent priors such that as each interval is assigned a positive probability. Step (ii) is independent of the prior; it holds whenever the realized rewards belong to a given bounded range $[a,b]$. We only use Beta-Bernoulli priors for step (iii), where we take advantage of the explicit expression for posterior means.%
\footnote{Step (iii) is quite subtle in general. First, the conditioning in \eqref{BIC:eq:BG-template-3} is on the entire tuple of samples $\samples{n}$, not just on the event $\AvgR{n}\geq \beta$. This leads to a rather strange consequence: \eqref{BIC:eq:BG-template-3} does not hold for arbitrary reward distributions. Indeed, if realized rewards can take a particular value $x_0$ only if the mean reward $\mu_1$ is very low, then observing reward $x_0$ leads to a low posterior mean reward $\E[ \mu_1 \mid \samples{n} ]$ even if the average reward $\AvgR{n}$ is large. Further, reward $x_0$ may be possible for high $\mu_1$, but unlikely. Second, we need \eqref{BIC:eq:BG-template-3} to hold for all $n\in [T]$ (not just for sufficiently large $n$), and with $\beta$ that does not depend on $T$. This is challenging even if rewards can only take two values.}

\begin{exercise}[Bayesian Greedy]\label{BIC:ex:BG}
Assume $\E[\mu_2] \leq \alpha_1<\alpha_2 \leq \mu(1)$, for some absolute constants $\alpha_1,\alpha_2$.  Prove that \eqref{BIC:eq:BG-event-posterior} holds with positive constant probability.

\Hint{Use the following fact. Let $X_1, X_2, X_3 \ldots$ be IID random variables supported on some finite interval $[a,b]$, with $\nu:=\E[X_i]>0$. Consider the stopping time
    $N = \inf\{n:\; X_1 +\ldots + X_n \leq 0\} $.
Then
\begin{align}\label{BIC:eq:ex-BG:N-crude}
\Pr[N>T] > c>0 \qquad \forall T>n_0,
\end{align}
where $c$ and $n_0$ are some numbers that depend only on $\nu$ and $b-a$.%
\footnote{\refeq{BIC:eq:ex-BG:N-crude} suffices for our purposes, and can be derived using elementary tools. By Azuma-Hoeffding, there is $n_0$, which depends only on $\nu$ and $b-a$, such that
    $S_n := X_1 +\ldots + X_n > 0$
with high probability for each $n\geq n_0$, say with probability at least $1-n^{-3}$. So, we only need to worry about $S_n:\, n<n_0$.

A much more precise statement about $N$ can be derived using more advanced tools:
\begin{align}\label{BIC:eq:ex-BG:N-precise}
    \Pr\left[ N = \infty \right] \geq \nu/b.
\end{align}
One way to derive it via a related stopping time,
    $M = \inf\{n:\, S_n>0\}$.
From
\citep[][Exercise 4.1.10]{Durrett-probability-ed4}, we have
\[  \E[M] \,\Pr[N=\infty] = 1. \]
By Wald's inequality,
    $\E[M]  = \E[S_M]/\E[X_1] \leq b/\nu$
(because $S_M\leq b$ by definition of $M$), so \eqref{BIC:eq:ex-BG:N-precise} follows.}
}
\end{exercise}

} 

\appendix

\chapterstyle{article}

\chapter{Concentration inequalities}
\label{app:concentration}

%

\begin{ChAbstract}
This appendix provides background on concentration inequalities, sufficient for this book. We use somewhat non-standard formulations that are most convenient for our applications. More background can be found in \citep{McDiarmid-concentration} and
\citep{DubhashiPanconesi-book09}.
\end{ChAbstract}

Consider random variables $X_1, X_2, \ldots $. Assume they are mutually independent, but not necessarily identically distributed. Let
    $\barX_n = \tfrac{X_1 + \ldots + X_n}{n}$
be the average of the first $n$ random variables, and let $\mu_n = \E[\barX_n]$ be its expectation. According to the Strong Law of Large Numbers,
\[ \Pr\sbr{ \barX_n \to \mu } =1. \]

We want to show that $\barX_n$ is \emph{concentrated} around $\mu_n$ when $n$ is sufficiently large, in the sense that $|\barX_n - \mu_n|$ is small with high probability. Thus, we are interested in statements of the following form:
\[ \Pr\sbr{ |\barX_n - \mu_n| \leq \text{``small"}   } \geq 1- \text{``small"}. \]
Such statements are called ``concentration inequalities".

Fix $n$, and focus on the following high-probability event:
\begin{align}\label{app:eq:HP-event}
 \mE_{\alpha,\beta} := \cbr{
    |\barX_n - \mu_n| \leq \sqrt{\alpha\beta \log(T) \,/\,n}
}, \qquad\alpha > 0.
\end{align}
The following statement holds, under various assumptions:
\begin{align} \label{app:eq:HP-stt}
\Pr\sbr{ \mE_{\alpha,\beta} } \geq 1-2\cdot T^{-2\alpha}, \qquad \forall \alpha>0.
\end{align}

Here $T$ is a fixed parameter; think of it as the time horizon in multi-armed bandits. The $\alpha$ controls the failure probability; taking $\alpha=2$ suffices for most applications in this book. The additional parameter $\beta$ depends on the assumptions. The
     $r_n=\sqrt{\frac{\alpha \log T}{n}}$
term in \refeq{app:eq:HP-event} is called the \emph{confidence radius}. The interval
    $[\mu_n-r_n,\, \mu_n+r_n]$
is called the \emph{confidence interval}.

\begin{theorem}[Hoeffding Inequality]\label{app:thm:Hoeffding}
\refeq{app:eq:HP-stt} holds, with $\beta=1$, if
    $X_1 \LDOTS X_n \in [0,1]$.
\end{theorem}

This is the basic result. The special case of Theorem~\ref{app:thm:Hoeffding} with $X_i\in \{0,1\}$ is known as \emph{Chernoff Bounds}.

\begin{theorem}[Extensions]\label{app:thm:Hoeffding-extensions}
\refeq{app:eq:HP-stt} holds, for appropriate $\beta$, in the following cases:
\begin{itemize}
\item[(a)] \emph{Bounded intervals:} $X_i \in [a_i,b_i]$ for all $i\in [n]$, and
    $\beta = \tfrac{1}{n} \sum_{i\in [n]} (b_i-a_i)^2$.

\item[(b)] \emph{Bounded variance:}
    $X_i \in [0,1]$ and $\mathtt{Variance}(X_i) \leq \beta/8$ for all $i\in [n]$.

\item[(c)] \emph{Gaussians:} Each $X_i$, $i\in [n]$ is Gaussian with variance at most $\beta/4$.

\end{itemize}
\end{theorem}

\begin{theorem}[Beyond independence]\label{app:thm:Azuma}
Consider random variables $X_1, X_2, \ldots \in [0,1]$ that are not necessarily independent or identically distributed. For each $i\in [n]$, posit a number $\mu_i\in [0,1]$ such that
\begin{align}\label{app:eq:martingale-assn}
\E\left[ X_i \mid X_1 \in J_1 \LDOTS X_{i-1} \in J_{i-1} \right] = \mu_i,
\end{align}
for any intervals $J_1 \LDOTS J_{i-1} \subset [0,1]$.
Then \refeq{app:eq:HP-stt} holds, with $\beta=4$.
\end{theorem}

This is a corollary from a well-known \emph{Azuma-Hoeffding Inequality}. \refeq{app:eq:martingale-assn} is essentially the \emph{martingale assumption}, often used in the literature to extend results on independent random variables.

\chapter{Properties of KL-divergence}
\label{app:KL}
Let us prove the properties of KL-divergence stated in Theorem~\ref{LB:thm:KL-props}. To recap the main definition, consider a finite sample space $\Omega$, and let $p, q$ be two probability distributions on $\Omega$. \emph{KL-divergence} is defined as:
\begin{align*}
\KL(p, q) = \sum_{x \in \Omega} p(x) \ln \frac{p(x)}{q(x)}
= \mathbb{E}_p \left[ \ln \frac{p(x)}{q(x)}  \right].
\end{align*}

\begin{lemma}[Gibbs' Inequality, Theorem~\ref{LB:thm:KL-props}(a)]~\newline
$\KL(p, q) \geq 0$ for any two distributions $p,q$, with equality if and only if $p = q$.
\end{lemma}

\begin{proof}
Let us define: $f(y) = y \ln (y)$. $f$ is a convex function under the domain $y > 0$. Now, from the definition of the KL divergence we get:
\begin{align*}
\KL(p, q) &= \sum_{x \in \Omega} q(x)\, \frac{p(x)}{q(x)}
 \ln \frac{p(x)}{q(x)} \\
&= \sum_{x \in \Omega} q(x) f \left( \frac{p(x)}{q(x)} \right) \\
&\ge f \left( \sum_{x \in \Omega} q(x) \frac{p(x)}{q(x)}  \right)&
    \EqComment{by Jensen's inequality} \\
&= f \left(   \sum_{x \in \Omega} p(x) \right) = f(1) = 0,
\end{align*}
In the above application of Jensen's inequality, $f$ is not a linear function, so the equality holds (\ie $\KL(p, q) = 0$)  if and only if $p=q$.
\end{proof}

\begin{lemma}[Chain rule for product distributions, Theorem~\ref{LB:thm:KL-props}(b)]~\newline
Let the sample space be a product $\Omega = \Omega_1 \times \Omega_1 \times \dots \times \Omega_n$. Let $p$ and $q$ be two distributions on $\Omega$ such that
        $p = p_1 \times p_2 \times \dots \times p_n$
    and
        $q = q_1 \times q_2 \times \dots \times q_n$,
    where $p_j,q_j$ are distributions on $\Omega_j$, for each $j\in[n]$. Then $\KL(p, q) = \sum_{j = 1}^n \KL(p_j, q_j)$.
\end{lemma}

\begin{proof}
Let $x = (x_1, x_2, \dots, x_n) \in \Omega$ such that $x_i \in \Omega_i$ for all $i = 1 \LDOTS  n$. Let
$h_i (x_i) = \ln \frac{p_i (x_i)}{q_i (x_i)}.$
Then:
\begin{align*}
\KL(p, q) &= \sum_{x\in \Omega} p(x) \ln \frac{p(x)}{q(x)} \\
&= \sum_{i=1}^n \sum_{x\in \Omega} p(x) h_i(x_i) &\left[\text{since } \ln \frac{p(x)}{q(x)} =  \sum_{i=1}^n h_i(x_i) \right]\\
&=  \sum_{i=1}^n \sum_{x^\star_i \in \Omega_i}  h_i(x_i^\star) \sum_{\substack{x \in \Omega,\\ x_i = x_i^\star}} p(x) \\
&=  \sum_{i=1}^n \sum_{x_i \in \Omega_i} p_i (x_i) h_i (x_i) &\left[\text{since } \sum_{x \in \Omega,\ x_i = x_i^\star} p(x) = p_i (x_i^\star)  \right]  \\
&=  \sum_{i=1}^n \KL(p_i, q_i). & \qedhere
\end{align*}
\end{proof}

\begin{lemma}[Pinsker's inequality, Theorem~\ref{LB:thm:KL-props}(c)]
Fix event $A \subset \Omega$. Then
\[ 2 \left( p(A) - q(A) \right)^2 \le \KL(p, q).\]
\end{lemma}

\begin{proof}
First, we claim that
\begin{align}\label{LB:pinsker_claim}
 \sum_{x \in B} p(x) \ln \frac{p(x)}{q(x)} \ge p(B) \ln \frac{p(B)}{q(B)} \qquad \text{for each event $B \subset \Omega$}.
\end{align}
For each $x\in B$, define
    $p_B(x) = p(x)/p(B)$ and
    $q_B(x) = q(x)/q(B)$.
Then
\begin{align*}
\sum_{x \in B} p(x) \ln \frac{p(x)}{q(x)}
    &= p(B) \sum_{x \in B} p_B(x) \ln \frac{p(B)\cdot p_B(x)}{q(B) \cdot q_B(x)} \\
    &= p(B) \left( \sum_{x \in B} p_B (x) \ln \frac{p_B(x)}{q_B(x)} \right)
        + p(B) \ln \frac{p(B)}{q(B)} \sum_{x\in B} p_B(x) \\
    &\ge p(B) \ln \frac{p(B)}{q(B)} \qquad \qquad \left[\text{since } \sum_{x \in B} p_B (x) \ln \frac{p_B(x)}{q_B(x)} = \KL(p_B, q_B) \ge 0 \right].
\end{align*}

Now that we've proved \eqref{LB:pinsker_claim}, let us use it twice: for $B=A$ and for $B=\bar{A}$, the complement of $A$:
\begin{align*}
\sum_{x \in A} p(x) \ln \frac{p(x)}{q(x)} \ge p(A) \ln \frac{p(A)}{q(A)}, \\
\sum_{x \notin A} p(x) \ln \frac{p(x)}{q(x)} \ge p(\bar A) \ln \frac{p(\bar A)}{q(\bar A)}.
\end{align*}
Now, let $a = p(A)$ and $b = q(A)$, and w.l.o.g. assume that $a < b$. Then:
\begin{align*}
\KL(p, q) &\geq a \ln \frac{a}{b} + (1-a) \ln \frac{1-a}{1-b} \\
&= \int_a^b \left( - \frac{a}{x} + \frac{1-a}{1-x}  \right) dx
= \int_a^b \frac{x-a}{x(1-x)} dx \\
&\ge \int_a^b 4(x-a) dx = 2(b-a)^2.
    &\EqComment{since $x(1-x) \le \tfrac14$} \qquad\qquad\qedhere
\end{align*}
\end{proof}

\begin{lemma}[Random coins, Theorem~\ref{LB:thm:KL-props}(d)]
Fix $\eps\in(0,\tfrac12)$.
Let $\RC_\eps$ denote a random coin with bias $\tfrac{\eps}{2}$, \ie a distribution over $\{0,1\}$ with expectation $(1+\eps)/2$.
Then
        $\KL(\RC_\eps, \RC_0) \le 2\eps^2$
and
        $\KL(\RC_0, \RC_\eps) \le \eps^2$.
\end{lemma}

\begin{proof}
\begin{align*}
\KL(\RC_0, \RC_\eps)
    &= \tfrac12 \, \ln(\tfrac{1}{1+\eps}) + \tfrac12\, \ln(\tfrac{1}{1-\eps})
    = -\tfrac{1}{2}\, \ln(1-\eps^2) \\
    &\leq -\tfrac{1}{2}\,(-2\eps^2)
        &\EqComment{as $\log(1-\eps^2)\geq -2\eps^2$ whenever $\eps^2\leq \tfrac12$} \\
    &=\eps^2. \\
\KL(\RC_\eps, \RC_0)
    &= \tfrac{1+\eps}{2} \ln(1+\eps) + \tfrac{1-\eps}{2} \ln (1-\eps) \\
&= \tfrac{1}{2} \left( \ln(1+\eps) + \ln(1-\eps)  \right) + \tfrac{\eps}{2} \left( \ln(1+\eps) - \ln (1-\eps)  \right) \\
&= \tfrac{1}{2} \ln (1-\eps^2) + \tfrac{\eps}{2} \ln \tfrac{1+\eps}{1- \eps}.
\end{align*}
Now, $\ln (1-\eps^2) < 0$ and we can write $\ln \tfrac{1+\eps}{1- \eps} = \ln \left(1 + \tfrac{2\eps}{1- \eps} \right) \le \tfrac{2\eps}{1-\eps}$. Thus, we get:
\begin{align*}
\KL(\RC_\eps, \RC_0) < \tfrac{\eps}{2} \cdot \tfrac{2\eps}{1-\eps} = \tfrac{\eps^2}{1 - \eps} \le 2\eps^2. &\qedhere
\end{align*}

\end{proof}



\begin{small}
\bibliographystyle{plainnat}
\bibliography{bib-abbrv,bib-slivkins,bib-bandits,bib-AGT,bib-competition,bib-ML,bib-random,bib-misc,bib-embedding,bib-nodeLabeling,bib-math,bib-medical,bib-dynamicPA}
\end{small}

\end{document}